\documentclass[10pt, letterpaper]{article}
\usepackage{arxiv}

\title{\Large\bf Robust Assortment Optimization from Observational Data}
\author{
Miao Lu\thanks{Department of Management Science \& Engineering, Stanford University. Email: \texttt{\{miaolu, jose.blanchet\}@stanford.edu}} \quad Yuxuan Han\thanks{Stern School of Business, New York University. Email: \texttt{\{yh6061, zzhou\}@stern.nyu.edu}} \quad Han Zhong\thanks{Center for Data Science, Peking University. Email: 
 \texttt{hanzhong@stu.pku.edu.cn}} \quad Zhengyuan Zhou\footnotemark[2] \quad Jose Blanchet\footnotemark[1]
}

\date{\today}

\begin{document}


\maketitle

\begin{abstract}
    Assortment optimization is a fundamental challenge in modern retail and recommendation systems, where the goal is to select a subset of products that maximizes expected revenue under complex customer choice behaviors. 
    While recent advances in data-driven methods have leveraged historical data to learn and optimize assortments, these approaches typically rely on strong assumptions -- namely, the stability of customer preferences and the correctness of the underlying choice models. 
    However, such assumptions frequently break in real-world scenarios due to preference shifts and model misspecification, leading to poor generalization and revenue loss. 
    Motivated by this limitation, we propose a robust framework for data-driven assortment optimization that accounts for potential distributional shifts in customer choice behavior. 
    Our approach models potential preference shift from a nominal choice model that generates data and seeks to maximize worst-case expected revenue. 
    We first establish the computational tractability of robust assortment planning when the nominal model is known, then advance to the data-driven setting, where we design statistically optimal algorithms that minimize the data requirements while maintaining robustness. 
    Our theoretical analysis provides both upper bounds and matching lower bounds on the sample complexity, offering theoretical guarantees for robust generalization. 
    Notably, we uncover and identify the notion of ``robust item-wise coverage'' as the minimal data requirement to enable sample-efficient robust assortment learning.
    Our work bridges the gap between robustness and statistical efficiency in assortment learning, contributing new insights and tools for reliable assortment optimization under uncertainty.
\end{abstract}


\newpage
\tableofcontents


\newpage

\section{Introduction}

In modern retail operations, e-commerce, and recommendation systems, the seller needs to optimize a limited subset of all products presented to the customers in order to maximize the expected revenues. 
Mathematically, this is known as assortment optimization, which is a challenging problem due to complex patterns of customer preference, product substitution, and revenue implications \citep{luce1959individual, mcfadden1978modeling, mcfadden2000mixed, asuncion2007uci, train2009discrete, daganzo2014multinomial, davis2014assortment, alptekinouglu2016exponomial,blanchet2016markov,berbeglia2022comparative}.
Over the past years, following the emergence of data-driven decision making paradigms, researches on assortment optimization have also witnessed a flourish of learning-based approaches, which propose to learn the choice patterns of the customers from historical data, either collected interactively \citep{saure2013optimal,agrawal2017thompson,agrawal2019mnl,saha2024stop} or purely observational \citep{dong2023pasta, han2025learning}, and then optimize the assortment based on the learned choice pattern. 

The rationale of such a learning-based assortment optimization pipeline heavily relies on the assumption that the historical data should truthfully reflect the customers' choice patterns in the future when the learned assortment is deployed. 
However, this can easily break in practice.
In the real-world situations, the customers' preference patterns might shift from time to time due to hidden factors unmeasured by the seller, causing a mismatch between the choice pattern encoded in the data and the choice patterns from future costumers.
Consequently, existing data-driven algorithms could easily suffer from overfitting to the historical data and fail to generalize to the changing choice patterns, resulting in poor expected revenues of the learned assortments.
We illustrate such degenerated performance of standard data-driven approaches in Figure~\ref{fig: robust comparison}.
Motivated by the insufficiency of existing algorithms in terms of robust generalization, in this work, we study how to design efficient algorithms for data-driven  assortment optimization robust to choice pattern shifts.

Studying this problem necessitates a new theoretical framework and new algorithm design for data-driven assortment optimization. 
Towards such a goal, we approach the assortment optimization problem from a (distributionally) robust optimization perspective to achieve robust generalization from historical data to potential varying customer choice patterns. 
Specially, we propose a unified framework tailored for data-driven robust assortment optimization from pre-collected offline data of customer choices.
These customer choice data can be either collected actively or just observational.
In this paper, we refer to the choice model by which the offline dataset is generated as the nominal choice model.
The distributionally robust approach we take is to maximize the expected revenue of the assortment when the customer choice distribution adversarially shifts around the nominal choice model.
By maximizing the worst case expected return of the assortment learned from fixed and finite data, our approach guarantees its performance even when the future customer choice pattern is different from the past pattern encoded in the data. We illustrate this in Figure~\ref{fig: robust comparison}.

As our playground of studying data-driven robust assortment learning, we first design polynomial-time computational approaches of the  proposed framework when the nominal choice model is known as a prior, also dubbed as the planning stage.
This shows the computational tractability of the formulation.
Then we aim to study the following fundamental question underlying data-driven robust assortment optimization: for the data-driven scenario where the nominal choice pattern needs to be learned from pre-collected offline data, 
\begin{center}
    \emph{How to design statistically optimal algorithms to learn the optimal robust assortment \\  with minimal data requirement and in a computationally tractable fashion?}
\end{center}
Answering the above question is of particular importance since it informs the practitioners the minimal data requirement such that robustly generalizable assortment learning is statistically possible. 
During the past few years, understanding the minimal data requirement for provably sample-efficient policy learning and decision making has been gathering increasing attention. 
This line of works features the study of the minimal data requirement and corresponding optimal algorithm design for reinforcement learning in Markov decision processes (MDP) \citep{jin2021pessimism, rashidinejad2021bridging, xie2021bellman, uehara2021pessimistic}, policy learning in causal inference \citep{jin2022policy, lu2023pessimism}, and also data-driven assortment optimization from offline data \citep{dong2023pasta, han2025learning}.
It turns out that for those decision-making models, the minimal requirement on the data is to contain enough information regarding the optimal decision policy, which hugely releases the typically adopted uniform coverage assumption or overlapping assumption in related literature.
Recently, \cite{han2025learning} demonstrates that for (non-robust) assortment optimization problem with multinomial logit (MNL) choice model \citep{mcfadden1978modeling}, it suffices for the data to contain enough information for each item in the optimal assortment separately. This greatly improves the previous result that needs the observation of the optimal assortment as a whole \citep{dong2023pasta}, a still impractical condition due to the combinatorial nature of the assortments.

Moving from standard data-driven assortment optimization setup \citep{dong2023pasta, han2025learning} to the robust framework we propose, it is then natural to ask and study the following fundamental problems: (i) \emph{what is the minimal data requirement to learn an optimal robust assortment that could robustly generalize}, and (ii) \emph{how to design sample-optimal algorithms that can leverage the data with minimal coverage assumption}.

Solving these problems face coupled challenges. 
Firstly, for robust assortment learning, the learner faces a doubled source of uncertainty: the uncertainty inherent to the customer choice preference and the statistical uncertainty of the data.
This casts new challenges different to the problem setups mentioned above to ensure minimal data usage, where therein the learner only faces the statistical uncertainty of the data.
Secondly, due to the combinatorial nature of the robust assortment optimization problem, the algorithm design to handle the coupled sources of uncertainty needs to take essential consideration to keep its computational tractability.

In this paper, we give thorough investigation to the above problems. We answer the question affirmatively with computationally trackable algorithm design together with statistical lower bounds.
The lower bounds demonstrate the minimal data requirement achieved by our algorithms. 
We summarize in more detail the main contributions of this work in the following section.

\begin{figure}[t]
        \hspace{-5mm}
       \begin{minipage}[b]{0.49\textwidth}
        \centering
        \includegraphics[width=1.12\linewidth]{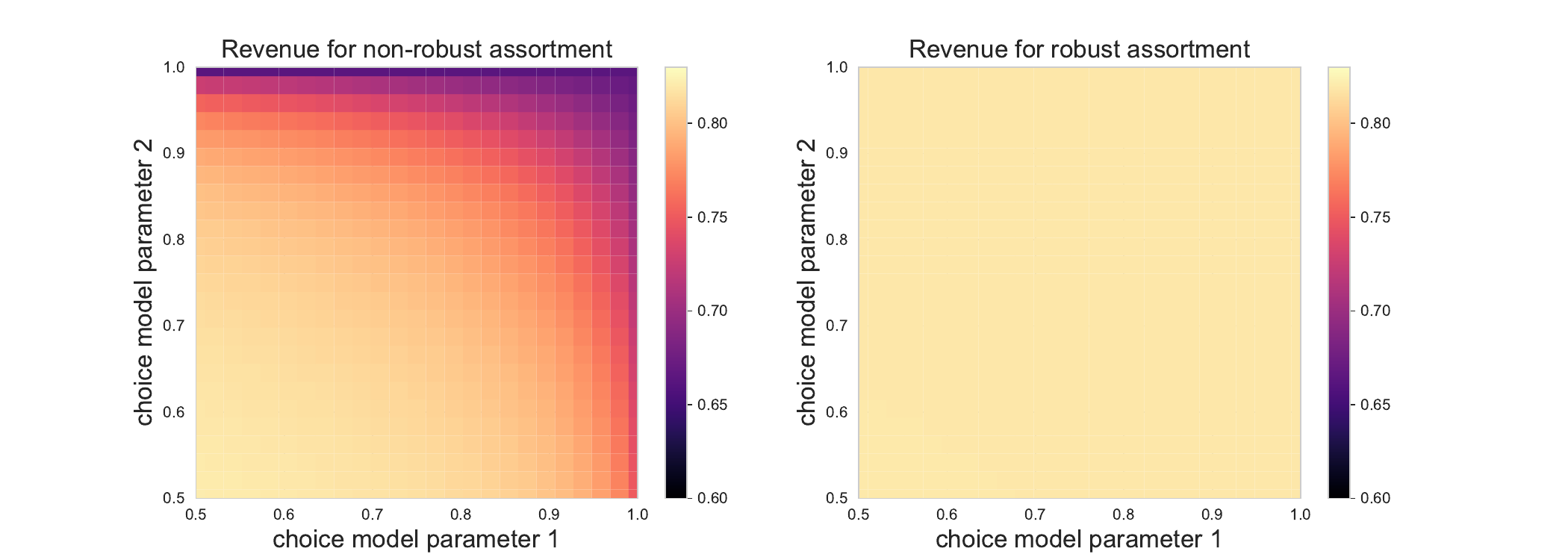}
    \end{minipage}
    \begin{minipage}[b]{0.24\textwidth}
        \centering
        \includegraphics[width=1.05\textwidth]{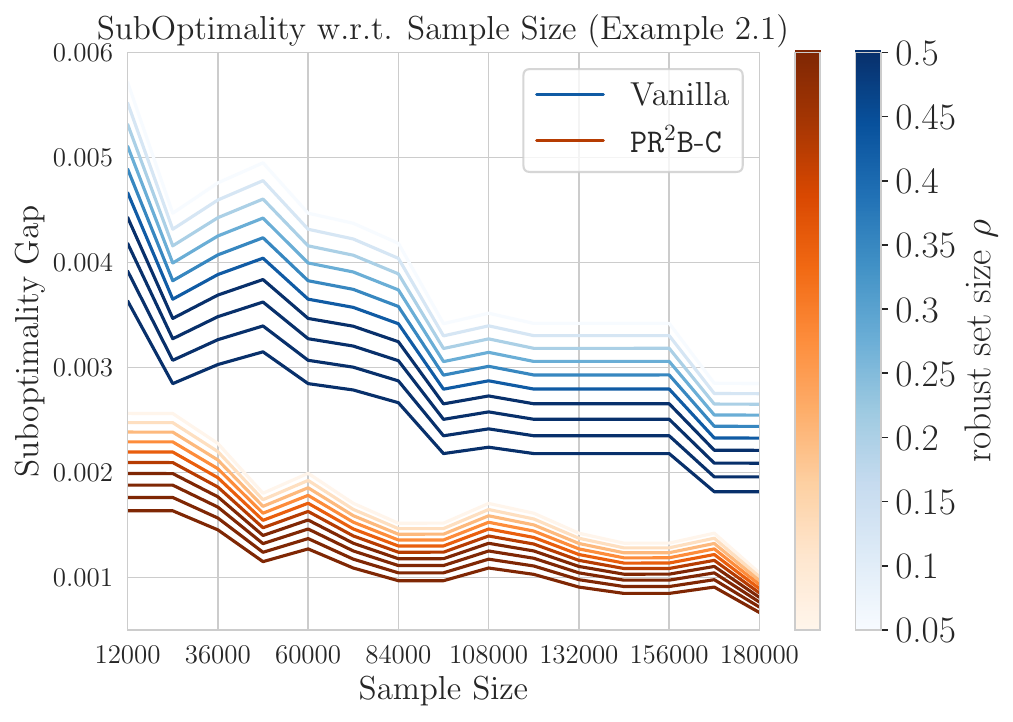}   
    \end{minipage}
    \hspace{1mm}
    \begin{minipage}[b]{0.24\textwidth}
        \centering
        \includegraphics[width=1.05\textwidth]{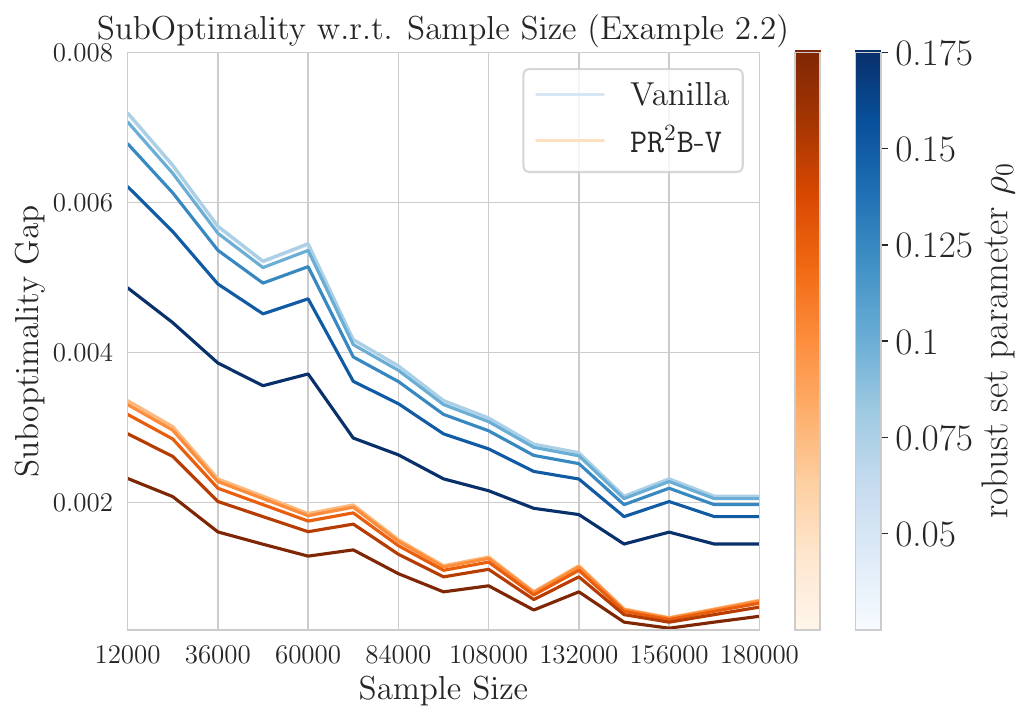}   
    \end{minipage}
    \caption{
    \small {\bf Left two figures:} Illustration of degeneration of expected revenue value  under customer choice pattern shifts. We consider a simple MNL model with $6$ items and a randomly generated nominal choice distribution $\mathbb{P}$.
    The non-robust assortment (left) is solved for the MNL model under $\mathbb{P}$, while the robust assortment (right) is solved for the robust assortment framework we consider in this paper (by Example~\ref{exp: new}). 
    The two axis's represent two parameters $\alpha_1$, $\alpha_2\in[0,1]$ that control the shift of the choice distribution $\mathbb{P}_{\alpha_1,\alpha_2}$ from the nominal model $\mathbb{P}$ to two adversarial choice models. 
    The KL divergence between the shifted model and the nominal model is within $0.1$. 
    {\bf Right two figures:} The suboptimality gap in terms of robust expected revenue of different algorithms under two concrete examples of the framework we consider the paper. 
    See Section~\ref{sec: experiments}. Our algorithms P$\mathrm{R}^2$B-(C/V) enjoy superior sample efficiency.
    }
    \label{fig: robust comparison}
\end{figure}

\subsection{Contributions}

\paragraph{Problem formulations.}
Firstly, in order to handle the fail of robust generalization facing existing data-driven assortment optimization approaches, we propose a generic framework of \emph{data-driven robust assortment optimization}, which seeks to learn from historical data an optimal robust assortment that can ensure optimized expected revenue under adversarial customer choice pattern shift. 
Mathematically, we propose to optimize  
\begin{tcolorbox}[ams align, colback=gray!2, boxrule=0.10mm]
     S^{\star}=\argsup_{\substack{S\subseteq[N],\,|S|\leq K}}\inf_{\substack{\mathbb{Q}_{S_+}\in\cP(S_+)\\D_{\mathrm{KL}}(\mathbb{Q}_{S_+}\|\mathbb{P}(\cdot|S))\leq \rho(S;\mathbb{P}(\cdot|S))}} \Big\{R(S;\QQ_{S+})\Big\}.\label{eq: robust revenue intro}
\end{tcolorbox}
\noindent
We provide a comprehensive introduction of the theoretical setup of assortment optimization and our robust learning framework based on \eqref{eq: robust revenue intro} in Section~\ref{sec: robust assortment optimization}. 
Here we provide a high level explanation.
Here $R(S;\QQ_{S_+})$ denotes the expected revenue of an assortment $S$ under a given choice pattern $\QQ_{S_+}$, a probability distribution over $S_+$ ($S$ accompanied with the non-purchasing behavior, denoted as $S_+$).
Then \eqref{eq: robust revenue intro} proposes to find an optimal robust assortment $S^{\star}$ in the sense that the worst case of its expected revenue is maximized when the customer choice pattern varies. 
The collection of the alternatives of the preference pattern is depicted by a Kullback–Leibler (KL) divergence ball centered at a nominal choice probability distribution $\PP(\cdot|S)$ that gives the customer choice pattern in the historical data. 

Notably, the general framework \eqref{eq: robust revenue intro} allows for adaptive size of the distributional ball across different assortments, characterized by the function $\rho(\cdot;\cdot)$. 
This allows for a customized design of the robust set size based upon the pattern of the nominal choice model. 
To concentrate the scope of this work, we focus on two specific instances of the general framework \eqref{eq: robust revenue intro}, both with MNL choice model as the nominal choice model, but with different choices of the robust set size $\rho(\cdot;\cdot)$. 
The first case (Example~\ref{exp: jin}) adopts a constant robust set size, recovering the model proposed by \cite{jin2022distributionally}, while the second case (Example~\ref{exp: new}) adopts a varying robust set size that depends on the MNL parameters associated with specific assortments, featuring higher risk aversion against no purchasing behaviors. 
When the nominal choice model is given as a prior, we provide computationally tractable algorithms to solve the optimal robust assortment for both Examples with $\widetilde{\cO}(N^2)$ running time, with $N$ the number of all items.
See Section~\ref{subsec: computation} for details.

Based on the framework \eqref{eq: robust revenue intro}, we propose the data-driven robust assortment optimization problem, which seeks to find the optimal robust assortment $S^{\star}$ purely from historical data generated by the nominal customer choice pattern $\PP$. 
Each data point is in the form of assortment-choice pair $(S, i)$, where the choice $i$ is sampled from the nominal model $\PP(\cdot|S)$.
The goal is to design algorithms to learn from data an approximately optimal assortment such that the gap between its robust expected revenue and that of $S^{\star}$ is minimized.

\begin{table}[t]
    \centering
    \begin{tabular}{|c|c|c|c|}
        \hline
         & Setting & Upper Bound & Lower Bound \\
        \hline 
        \cite{dong2023pasta} &\makecell[c]{Non-uniform reward, \\non-robust}& $\widetilde{\cO}\left(\sqrt{\frac{d}{n_{S^\star_{\mathrm{non}}} \min_{i\in [N],\lvert S \rvert = K} p_i(S)}}\right)$ & N/A \\
        \hline
         
 &  \makecell[c]{Non-uniform reward, \\non-robust} & $\widetilde{\cO}\left(\frac{K}{\sqrt{\min_{i \in S^\star_{\mathrm{non}}} n_i}}\right)$ &$\Omega\left(\frac{K}{\sqrt{\min_{i \in S^\star_{\mathrm{non}}} n_i}}\right)$ \\
        \cline{2-4}
         \multirow{-2.6}{*}{\cite{han2025learning}} &\makecell[c]{Uniform reward, \\non-robust}&$\widetilde{\cO}\left(\sqrt{\frac{K}{\min_{i\in S^\star_{\mathrm{non}}} n_i}}\right)$  &$\Omega\left(\sqrt{\frac{K}{\min_{i\in S^\star_{\mathrm{non}}} n_i}}\right)$ \\
        \hline
        \hline
        \cellcolor{tianqin!5}&  \cellcolor{tianqin!5}Non-uniform reward, robust & \cellcolor{tianqin!5}$\widetilde{\cO}_{\rho}\left(\frac{K}{\sqrt{\min_{i \in S^\star} n_i}}\right)$ &\cellcolor{tianqin!5} $\Omega_{\rho}\left(\frac{K}{\sqrt{\min_{i \in S^\star} n_i}}\right)$ \\
        \cline{2-4}
         \multirow{-2.6}{*}{\cellcolor{tianqin!5} This work} &\cellcolor{tianqin!5} Uniform reward, robust &\cellcolor{tianqin!5} $\widetilde{\cO}_{\rho}\left(\sqrt{\frac{K}{\min_{i\in S^\star} n_i}}\right)$  &\cellcolor{tianqin!5} $\Omega_{\rho}\left(\sqrt{\frac{K}{\min_{i\in S^\star} n_i}}\right)$ \\
         \hline
    \end{tabular}
  \caption{\small Comparison of offline data-driven assortment optimization sample complexities.
    Here $S^{\star}_{\mathrm{non}}$ denotes the non-robust optimal assortment under MNL model, $S^{\star}$ denotes the optimal robust assortment under our robust framework \eqref{eq: robust revenue intro} with MNL nominal choice model, $n_i$ denotes the number of times that item $i$ in the optimal assortment appears in the offline data, $n_{S}$ denotes the number of times that the whole assortment $S$ appears in the offline data, $K$ denotes the size constraint of the assortment, and $p_i(S)$ denotes the nominal choice probability of item $i$ given assortment $S$.
    The algorithm proposed by \citet{dong2023pasta} is applicable to the  $d$-dimensional linear MNL setting, where $d = N$ when it is reduced to the discrete MNL choice model. For the results of our work, we use notations $\widetilde{\cO}_{\rho}(\cdot)$ and $\Omega_{\rho}(\cdot)$ to hide multiplicative factors depending on the robust set size quantities $\rho$.
    Details are referred to Section~\ref{sec: theory}.}
    \label{tab:comparison}
\end{table}

\paragraph{Algorithm design, analysis, and technical novelties.} 
For both Example~\ref{exp: jin} and Example~\ref{exp: new}, we propose computationally efficient algorithms to solve the data-driven robust assortment optimization problem in a statistically optimal fashion. 
See Section~\ref{sec:_algorithms} for the detail algorithm design.
Our algorithm and its analysis have three key features.
Firstly, it adopts the core principle of ``double pessimism'' \citep{blanchet2023double, shi2024distributionally} to handle the key challenge of coupled source of uncertainty in robust offline learning problems. 
Secondly, in spite of the complicated optimization procedures of the original ``double pessimism'' objective proposed in \cite{blanchet2023double}, we explore the structure of the MNL nominal choice model and customize a tractable objective that reduces the computational task to the planning problem we have studied that only has quadratic running complexity.
This is made possible because of a novel ``monotonicity'' argument developed specific to the models we consider. See Section~\ref{sec: sketch} for discussions.
Finally, both the algorithms for Example~\ref{exp: jin} and Example~\ref{exp: new} feature a nearly-minimax optimal sample complexity (Section~\ref{sec: theory}).
Especially, the corresponding sample complexity requires minimal data coverage, a condition we propose as the \emph{robust item-wise coverage condition}. 
As is shown by our proposed lower bounds, the robust assortment optimization problem requires and only requires that each single item of the optimal robust assortment $S^{\star}$ is observed for sufficiently many times.
Comparing to existing results, this extends the item-wise coverage condition proposed by \cite{han2025learning} for non-robust offline assortment learning to its robust counterpart.
Also, this result relates to the robust partial/single-policy coverage condition considered by \cite{shi2024distributionally} and  \cite{blanchet2023double} for offline reinforcement learning in robust MDPs. 
Different from the conditions therein, here we further emphasize on the item-wise property, which decomposes the need to cover the optimal robust assortment into the requirement of observing each single item in the optimal robust assortment separately.

Another interesting statistical property is that there exists a gap of order $\cO(\sqrt{K})$ between the minimax-optimal sample complexities of the general case and the uniform revenue case. 
Here $K$ refers to the cardinality constraint of the assortment to be optimized, and the uniform revenue case means that each item has the same revenue value. 
This special case corresponds to user engagement or click through rate in recommendation systems for digital content platforms.
A similar gap has been previously observed in the non-robust assortment optimization learning problems for MNLs \citep{han2025learning}, and we show that such a gap still exists for the robust assortment optimization problems with MNL nominal choice model.
To help present the state of the art of relevant results under different setups, we refer to Table~\ref{tab:comparison}.

\paragraph{Empirical studies.}
To demonstrate the sound theoretical properties our proposed data-driven algorithms enjoy, and to showcase the effectiveness of the framework in terms of robust generalization in the face of choice distribution shifts, we conduct extensive simulation studies.
In specific, we first verify the superior sample efficiency of our proposed algorithms compared to naive baselines (see Section~\ref{sec: experiments}) in terms of learning an optimal robust assortment. 
See comparison of results in Figure~\ref{fig: robust comparison} (the two right figures).
Then, we simulate customer choice pattern shifts and show that our robust assortment learning methods can learn assortments that perform reasonably well even under these shifts for most of the time, where the assortments learned by standard non-robust data-driven methods degenerate. 
Through randomly generated customer choice pattern, we plot empirical distributions of the improved revenue when using robust methods.
Finally, we investigate the influence of the cardinality constraint on the suboptimality of the learned assortment, verifying our theory on the existence of statistical gap between the general case and the special case with uniform revenue.

\subsection{Related Works}

\paragraph{Robust assortment optimization.} Robust assortment optimization has been previously studied under several different contexts. 
For the MNL choice model, \cite{rusmevichientong2012robust} formulate a maximin optimization problem over an uncertainty set of the MNL parameters, solving for the assortment that maximizes worst‐case revenue. 
They prove that for the unconstrained robust assortment optimization setup, the optimal robust assortment is still a revenue‐ordered one (the products with revenue higher than a threshold).
We notice that our framework (with MNL choice model as the nominal choice model) is more general than that of \cite{rusmevichientong2012robust} by allowing the robust set to depend on specific assortment $S$.
See our formulation in Section~\ref{sec: robust assortment optimization} and see Example~\ref{exp: new}.
Subsequent works extend this paradigm to richer models. 
In nonparametric (ranking‐based) choice, \cite{farias2013nonparametric} and \cite{sturt2025value} consider the worst‐case revenue over all the ranking‐based models consistent with historical data.
They demonstrate that these robust problems reduce to tractable LPs or MIPs and that the optimal assortments again have a simple revenue‐ordered structure. 
For Markov chain choice model \citep{blanchet2016markov}, \cite{desir2024robust} propose the corresponding robust optimization framework and prove the maximin duality of the optimization problem.
Since the Markov chain choice model strictly generalizes the MNL choice model, their results extend that of \cite{rusmevichientong2012robust} to a broader class of choice behaviors. 
Besides, \cite{wang2024randomized} prove that the randomized assortment policies—mixing over assortments according to a designed distribution—can strictly improve worst‐case revenue under the MNL choice model, Markov chain choice model, or preference ranking models. 
\cite{jin2022distributionally} is the most related work to our paper.
They consider a specific formulation of distributionally robust assortment optimization.
They present a thorough investigation of the properties of the corresponding optimal assortment, the most adversarial choice pattern, and the computational algorithms to solve the optimal robust assortment. See Sections~\ref{sec: robust assortment optimization} and \ref{subsec: computation} for discussions of their formulations and our framework.
Our general framework subsumes their formulation as a specific instance.  

However, despite this line of previous research on robust assortment optimization, the design of statistically efficient algorithms and the corresponding sample complexity as well as the minimal data requirements remain untouched. 
Our work makes the first step towards understanding the statistical foundations of this problem.

\paragraph{Data-driven assortment optimization.} The data-driven assortment optimization problem under the MNL choice model has been extensively studied in the revenue management literature. 
However, most prior work focuses on the online setting, where the seller can interact with customers sequentially over time~\citep{caro2007dynamic, rusmevichientong2010dynamic, saure2013optimal, agrawal2017thompson, agrawal2019mnl, chen2021optimal, saha2024stop}. 
In contrast, our work on data-driven assortment optimization focuses on the offline counterpart, where online exploration is not available and the seller can only access a pre-collected dataset. 
In this setting, a robust formulation naturally arises due to uncertainty in the data collection process and potential distributional shifts in the target environment. 
This scenario remains relatively underexplored, with only a few recent works in the non-robust setting~\citep{dong2023pasta,  dong2025pasta, han2025learning}. 
Specifically, \citet{dong2025pasta} propose the PASTA framework for learning the optimal assortments from historical data. 
While the earlier version focuses on the MNL model~\citep{dong2023pasta}, the latest version extends PASTA to a general framework applicable to a broad class of choice models while maintaining the assumption that historical observations consist of optimal assortments. In contrast, \citet{han2025learning} focus on reducing the required assumptions for the MNL model, showing that it suffices to observe the individual items contained in the optimal assortments. 
Our work extends the item-wise partial coverage condition identified by \cite{han2025learning} to the robust assortment learning setup.

\paragraph{Offline policy learning and the pessimism principle.} 
Our work is also related to the offline policy learning problem, where the learner aims to find the optimal policy from purely offline data without interaction with environments. 
A recent line of works \citep{jin2021pessimism,rashidinejad2021bridging,xie2021bellman,uehara2021pessimistic,zhong2022pessimistic,liu2022welfare,xiong2023nearly,lu2023pessimism,liu2024provably, han2025learning} designs pessimism principle based algorithms in the context of bandits, Markov decision processes (MDPs), and assortment optimization problems. 
They prove that these algorithms only require minimal data coverage assumptions. 
Building on this line of works, \citet{blanchet2023double} and \citet{shi2024distributionally} propose the ``double pessimism'' learning principle for distributionally robust offline Markov decision processes. 
Their algorithms can identify near-optimal policies under a minimal robust partial coverage data assumption, significantly improving upon previous works \citep{zhou2021finite,si2023distributionally} that require full coverage data.
Our work leverages the double pessimism principle developed by \citet{blanchet2023double} and \citet{shi2024distributionally}, but applies it to the assortment optimization domain. 
Due to fundamental differences in problem structure between robust bandit/MDP \citep{iyengar2005robust,el2005robust} and robust assortment optimization, our work --- the first to address offline learning for robust assortment optimization --- tackles the unique challenges of the assortment domain and introduces novel models and algorithms requiring minimal data assumptions. 
In other words, while our approach shares the same offline learning spirit with previous work on robust bandit and robust MDP \citep{zhou2021finite,yang2022toward,panaganti2022sample,kuang2022learning,blanchet2023double,wang2023finite,shi2024distributionally,si2023distributionally,liu2024distributionally,liu2024minimax,wang2024sample,lu2024distributionally}, our work represents a parallel contribution as we address domain-specific challenges distinct from those tackled in prior research.

\subsection{Notations}

For an integer $N\in\NN$, we denote $[N]$ as the set $\{1,\cdots, N\}$, and denote $[N]_+$ as the extended set $\{0,1,\cdots,N\}$.
For any set $A$, we denote $2^A$ as its power set that contains all the subsets of $A$.
For any two sets $A$ and $B$, we denote $\Delta(A,B):=|(A\cup B)\setminus (A\cap B)|$. 
For any measurable space $\cX$, we define $\cP(\cX)$ as the space of probability measures supported on $\cX$. 
For any two distributions $\PP$ and $\QQ$ over $[N]_+$, we define the Kullback–Leibler (KL)-divergence between $\PP$ and $\QQ$ as $D_{\mathrm{KL}}(\PP \| \QQ) = \sum_{i \in [N]_+} \PP(i) \cdot \log (\PP(i)/\QQ(i))$.

\section{Preliminaries of Robust Assortment Optimization}\label{sec: robust assortment optimization}

We begin with the problem setup. 
In Section~\ref{subsec: setup}, we give mathematical definitions of assortments, choice models, and expected revenue. 
Then we propose our definition of the \emph{distributionally robust expected revenue}, a general framework to model worst case revenue under distributional shifts of choice patterns.
In Section~\ref{subsec: example}, we provide two examples of this framework, which we concentrate on in the sequel of this paper.
In Section~\ref{subsec: data driven}, we formulate our main goal, that is, to design computationally and statistically efficient algorithms to learn an optimal robust assortment from observational data of customer choices.

\subsection{Problem Setups}\label{subsec: setup}

We study the assortment optimization problem which models the interaction between a seller and a customer.
The seller has $N$ different products or items to sell, denoted by $[N]:=\{1,2,\cdots,N\}$.
An assortment $S\subset[N]$ is a subset of the whole item set that the seller offers to the customer. 
Upon seeing the assortment set $S$, the customer chooses an item $i\in S$ to purchase or leaves without buying anything, denoted as $i=0$.
Each single item $i\in[N]$ is associated with a revenue $r_i\in[0,r_{\max}]$ when it is purchased by the customer, and the non-purchasing situation is associated with a vanishing revenue $r_0=0$. 
We let $S_+ := S \cup \{0\}$ represent the complete set of customer choices, including the items in assortment $S$ and the option of making no purchase.

An important special case is when the revenues $\{r_i\}_{i\in[N]}$ are uniform, i.e., $r_i=r_j$ for any item $i,j\in[N]$. 
In such a uniform revenue case, revenue maximization (introduced later) can be corresponded to maximizing the user engagement or click through rate in recommendation systems for digital content platforms.

\paragraph{Choice model and expected revenue.}
The behavior of the customer when deciding which product to purchase is characterized by a \emph{choice model}.
Specifically, a choice model is a function $\PP(\cdot|\cdot) : [N]_+ \times 2^{[N]} \rightarrow [0,1]$, such that $\PP(\cdot | S)$ is a probability mass function supported on $S_+$ for any set $S \in 2^{[N]}$.
That is, $\PP(\cdot|S)$ specifies the probability of purchasing different items from $S$ (and non-purchasing) conditioned on seeing the assortment.

A simple yet powerful example of choice model is the multinomial logit (MNL) model \citep{mcfadden1978modeling}, which is defined as 
\begin{align}
    \PP(i|S) := \frac{v_i}{\sum_{j\in S_+}v_i},\quad \forall i\in S_+,\quad \forall S\subset[N].\label{eq:_mnl}
\end{align}
Here $\bv := (v_0,v_1,\ldots,v_N)$ is the vector of parameters that fully specifies the MNL model.
Each $v_i$ represents the attraction parameter for purchasing item $i$ (or the non-purchasing choice denoted by $i = 0$). 
Without loss of generality, for MNL models, we let the attraction parameter for non-purchasing behavior $v_0=1$, and denote $v_{\max}:=\max_{j\in[N]}v_j$ as the maximum attraction parameter value over all the items.

With the choice model that specifies the behavior of the customer, the seller is interested in maximizing their \emph{expected revenue}. 
Specifically, given any assortment set $S\subseteq [N]$ and a probability mass function $\QQ_{S_+}(\cdot)$ supported on the set $S_+$, the expected revenue of $S$ under choice probability $\QQ_{S+}$ is defined as
\begin{align}
    R(S;\QQ_{S_+}):= \mathbb{E}_{j\sim \mathbb{Q}_{S_+}(\cdot)}[r_j] = \sum_{j\in S_+}\QQ_{S_+}(j)\cdot r_j.\label{eq: revenue}
\end{align}
Then in non-robust assortment optimization problems, the goal is to find the optimal assortment such that the expected revenue of that assortment under a specific choice model $\PP(\cdot|\cdot)$ is maximized, i.e.,  
\begin{align}
    S^{\star}_{\mathrm{non}\text{-}\mathrm{robust}}:= \argmax_{S\subseteq [N], |S|\leq K} \,R\big(S;\PP(\cdot|S)\big),\label{eq: non robust target}
\end{align}
where the revenue $R(S;\PP(\cdot|S))$ is defined by \eqref{eq: revenue}, and $K\in\NN$ represents the size constraint of the assortment. When $K = N$, it corresponds to the unconstrained setting.
In a standard data-driven assortment optimization pipeline, the choice model $\PP$ is typically taken to be the one that generates the observational choice data.

\paragraph{Robust expected revenue.}

We aim to further robustify the assortment in the face of customer choice shifts from a \emph{nominal choice model} $\PP$ corresponding to historical choice patterns.
Mathematically, we introduce the following (distributionally) robust expected revenue of an assortment $S$ associated with $\PP$,
\begin{align}
    R_{\rho(\cdot;\cdot)}(S;\PP) := \inf_{\mathbb{Q}_{S_+}\in\cP(S_+),\,\,D_{\mathrm{KL}}(\mathbb{Q}_{S_+}\|\mathbb{P}(\cdot|S))\leq \rho(S;\mathbb{P}(\cdot|S))} \Big\{R(S;\QQ_{S_+})\Big\},\quad \forall S\subset[N],\label{eq: robust revenue}
\end{align}
where $R(S;\QQ_{S_+})$ is the standard expected revenue of assortment $S$ under choice probability $\QQ_{S_+}$ (see \eqref{eq: revenue}), and $\rho(\cdot;\cdot):2^{[N]}\times\mathcal{P}([N]_+)\mapsto\mathbb{R}_{\geq 0}$ maps an assortment $S$ and a choice probability over $S$ to a non-negative real number that specifies the radius of the distributional set associated with this assortment.

Intuitively, the robust expected revenue of an assortment $S$ characterizes the worst case expected revenue of $S$ when the choice distribution deviates from the nominal choice probability $\PP(\cdot|S)$ given assortment $S$.
We highlight two properties of the robust expected revenue defined in \eqref{eq: robust revenue}.
Firstly, it allows different choices of the nominal choice model.
Secondly, the robust level $\rho(\cdot;\cdot)$ can depend on the assortment and the nominal choice probability, which allows for flexible design of the distributional set that utilizes the structure of the nominal choice model.
Both two features demonstrate great generality of the definition \eqref{eq: robust revenue}. 

Based on \eqref{eq: robust revenue}, we dub the process of finding the optimal assortment that maximizes the robust expected revenue \eqref{eq: robust revenue} as \emph{robust assortment optimization}.
Specifically, given a nominal choice model $\PP$, a robust set size function $\rho(\cdot;\cdot)$, and a size constraint $K$, we define the optimal robust assortment as
\begin{align}
    S^{\star}:=\argmax_{S\subseteq[N],|S|\leq K} R_{\rho(\cdot,\cdot)}(S;\PP), \label{eq: optimal robust assortment}
\end{align} 
When the function $\rho(\cdot;\cdot)=0$, the robust assortment optimization problem is reduced back to the standard assortment optimization problem \eqref{eq: non robust target}.

\paragraph{Dual formulation.}
A bless of the distributional formulation of the choice pattern shift is the following proposition: 
the robust expected revenue \eqref{eq: robust revenue} can be equivalently formulated by its dual problem that can be solved based on the knowledge of the nominal choice model.
This is useful for solving the optimal robust assortment given nominal choice model (Section~\ref{subsec: computation}) as well as for designing data-driven algorithms (Section~\ref{sec:_algorithms}).

\begin{proposition}[Dual representation of robust expected revenue]\label{prop:_dual}
Given any assortment $S\subseteq[N]$, the robust expected revenue \eqref{eq: robust revenue} can be reformulated as
\begin{align}
    R_{\rho(\cdot;\cdot)}(S;\PP) = \sup_{0\leq \lambda\leq B(S;\PP(\cdot|S))} \Big\{-\lambda\cdot\log\left(\mathbb{E}_{i\sim \PP(\cdot|S)}\left[\exp(-r_i/\lambda)\right]\right) - \lambda\cdot \rho\big(S;\PP(\cdot|S)\big)\Big\},
\end{align}    
where the upper bound of the dual variable $B(S,\PP(\cdot|S)):=r_{\max}/\rho(S;\PP(\cdot|S))$.
\end{proposition}

\begin{proof}[Proof of Proposition~\ref{prop:_dual}]
    The dual formulation follows from the dual representation for KL-divergence based distributionally robust optimization problem, see e.g., \cite{hu2013kullback}. 
    The upper bound on the dual variable follows from Lemma H.7 of \cite{blanchet2023double}.
\end{proof}

\subsection{Examples of Robust Assortment Optimization Framework}\label{subsec: example}

To concentrate the scope of this work, in the sequel of this paper, we consider the nominal choice model $\PP(\cdot|\cdot)$ to be the MNL model \eqref{eq:_mnl}. 
When it is clear from the context, we use the notation of $\bv$ and the choice model $\PP$ interchangeably since $\bv$ fully specifies an MNL model.
Furthermore, in this paper, we study the following two concrete and representative examples of the robust set size function $\rho(\cdot;\cdot)$.

The first example is when the robust set size function is a constant.

\begin{theorembox}

\begin{example}[Constant robust set size \citep{jin2022distributionally}]\label{exp: jin}
    The first example is when robust set size function is a constant, i.e., $\rho(\cdot;\cdot)=\rho$ for some constant $\rho\geq 0$, regardless of specific assortment $S$.
    In this example, we simply refer to the constant $\rho$ as the robust set size.
    This formulation was firstly considered by 
    \cite{jin2022distributionally} but under a different perspective.
    It can be easily shown that the robust expected revenue \eqref{eq: robust revenue} with $\rho(\cdot;\cdot)=\rho$ can be equivalently formulated as
    \begin{align}
    R_{\rho}(S) := \inf_{\substack{\bp\in\cP(\RR^{N+1}),\,\,D_{\mathrm{KL}}(\bp\|\bp_0)\leq \rho}} \mathbb{E}_{\bu\sim \bp}\big[f(S, \bu)\big],\quad f(S,\bu):= \min\Big\{r_{i^{\star}}\Big|i^{\star}\in\argmax_{i\in S_+}u_i\Big\},\label{eq: jin}
\end{align}
where the minimization operation in function $f$ is to break a tie, and the nominal distribution of the customer utility $\bp_0$ is defined as a product Gumbel distribution
\begin{align}
    \bp_0:= \cG\big(\sqrt{6}u_0/\pi,1\big)\otimes\cG\big(\sqrt{6}u_1/\pi,1\big)\otimes\cdots\otimes\cG\big(\sqrt{6}u_N/\pi,1\big).
\end{align}
\cite{jin2022distributionally} starts from the perspective of \eqref{eq: jin} and studies the computation of the corresponding optimal robust assortment given the nominal choice model.
In contrast, our main target is to study the algorithm design for robust assortment optimization in a data-driven setup.
\end{example}

\end{theorembox}

In the following, we propose another instance of \eqref{eq: robust revenue}, where the robust set size function varies depending on different assortments following a specific pattern.

\begin{theorembox}
    
\begin{example}[Varying robust set size]\label{exp: new}
    The second example is a non-constant robust set size function,
    \begin{align}
        \rho(S;\bv) := \rho_0 - \log\left(e^{\rho_0} - \frac{(e^{\rho_0} - 1)\cdot\sum_{j\in[N]_+}v_j}{\sum_{j\in S_+}v_j}\right),\quad \forall S\subseteq[N], \label{eq: rho new}
    \end{align}
    for a constant $0\leq \rho_0 < \log((\sum_{j\in[N]_+}v_j) / (\sum_{j\in [N]}v_j))$\footnote{We remark that this range of $\rho_0$ ensures that the logarithm term in \eqref{eq: rho new} is finite for any assortment $S$.
    That is, for $\rho_0$ in this interval, 
        $e^{\rho_0} - ((e^{\rho_0} - 1)\cdot\sum_{j\in[N]_+}v_j)/(\sum_{j\in S_+}v_j) >0$ for any $S\subseteq[N]$.
    In the equivalent formulation \eqref{eq: conditioned robust revenue formulation}, such a constraint on $\rho_0$ actually ensures that the robust expected revenue does not degenerate to $0$ for any non-empty assortment $S\subset [N]$.}. 
    We refer to $\rho_0$ as robust set size parameter.
    Intuitively, under \eqref{eq: rho new}, the robust set size would be larger for the assortments $S$ with smaller total attraction $\sum_{j\in S_+}v_j$.
    It is then more unlikely for the assortments that achieve high expected revenue under the nominal model through small total attraction to become the optimal robust assortment. 
    This is because, for example, assortments achieving high expected revenue under the nominal model through a few items with super high revenue but relatively low attraction could easily induce non-purchasing behavior under preference distributional shift, resulting in a low expected revenue.
    
    Here we further provide a second way to interpret the definition of $\rho(\cdot,\cdot)$ in \eqref{eq: rho new}. 
    Under \eqref{eq: rho new}, the distributionally robust expected revenue \eqref{eq: robust revenue} can be equivalently understood as following. 
    Firstly, one assigns a distribution $\bp_0 = (p_{0,0},p_{0,1},\cdots,p_{0,N})\in\cP([N]_+)$ over the whole item set by 
    \begin{align}
        p_{0,i} = \frac{v_i}{\sum_{j\in [N]_+}v_j},\quad \forall i \in [N]_+.
    \end{align}
    Then the robust expected revenue \eqref{eq: robust revenue} with $\rho(\cdot,\cdot)$ in \eqref{eq: rho new} is equivalently given by
    \begin{align}
        \widetilde{R}_{\rho_0}(S;\bp_0):= \inf_{\substack{\bp\in \cP([N]_+),\,\,
        D_{\mathrm{KL}}(\bp\|\bp_0)\leq \rho_0}}\left\{\mathbb{E}_{j\sim \bp|j\in S_+}[r_j]=\frac{\sum_{j\in S}p_jr_j}{\sum_{j'\in S_+}p_{j'}}\right\}.\label{eq: conditioned robust revenue formulation}
    \end{align}
    That is, we treat the MNL choice model under an assortment $S$ as a posterior probability distribution given $S$, where each item has a prior distribution  $p_{0,i}$ induced by the attraction parameters.
    The robust counterpart is by robustifying the prior distribution of the items (and the non-purchasing choice).
    
    Finally, we remark that one particular feature of this example is that such a problem formulation necessitates the learning of the robust set size,
    We discuss this from a technical perspective in Section~\ref{sec: sketch}.
\end{example}
\end{theorembox}

\begin{wrapfigure}{r}{0.4\textwidth}
  \centering
\vspace{-6mm}\hspace{-2mm}\includegraphics[width=0.38\textwidth]{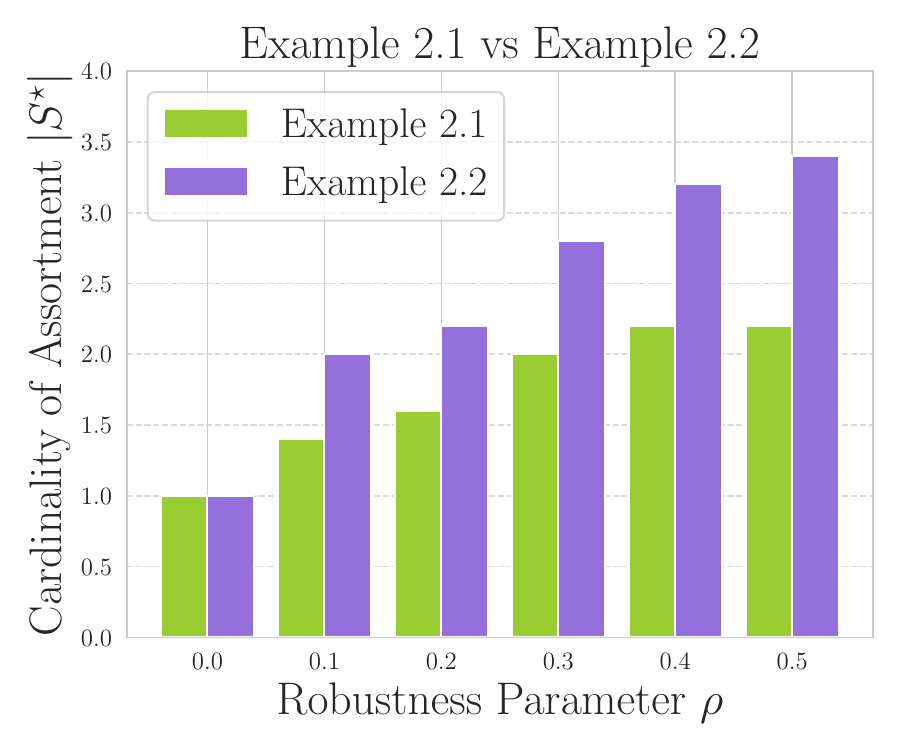}
  \caption{\small Comparing $|S^{\star}|$, avged over $5$ instances. Here $\rho$ refers to $\rho_0$ in Example~\ref{exp: new}. 
}
  \label{fig: card comparison}
\end{wrapfigure}

\paragraph{Comparison between Example~\ref{exp: jin} and Example~\ref{exp: new}.} 

Examples~\ref{exp: jin} and~\ref{exp: new} are both instances of the unified robust assortment optimization objective, but they encode two different philosophies of model uncertainty.
The key conceptual distinction is the \emph{order of conditioning and robustification}.
Example~\ref{exp: jin} robustifies the assortment-conditional choice behavior $\PP(i|S)$ directly:
for each offered assortment $S$, the true conditional choice distribution is allowed to deviate from the nominal one
within a KL ball of a \emph{constant} radius $\rho$.
This viewpoint is actually ``local'' in $S$: it does not require the perturbed conditionals across different assortments to arise from a single globally coherent latent preference model.
The induced perturbed conditional models need not remain within the MNL family, which makes Example~\ref{exp: jin}
flexible for guarding against broad forms of misspecification or estimation error at the conditional level.

In contrast, Example~\ref{exp: new} starts from a ``global'' uncertainty viewpoint: Example~\ref{exp: new} can be interpreted as treating the MNL choice model under an assortment $S$ as a posterior
distribution given $S$, and robustifying the corresponding prior $p_0$. 
It robustifies a single global object
(the prior distribution $p_0$ over $[N]_+$) within a KL ball of radius $\rho_0$,
and then obtains choice probabilities by conditioning on the assortment set $S$.
This construction couples uncertainty across assortments, and preserves structural validity of MNL choice model.

\paragraph{Algebraic connection: Example~\ref{exp: new} as an induced $S$-dependent radius.}
These two viewpoints are connected algebraically: Example~\ref{exp: new} can be viewed as a structured special case of Example~\ref{exp: jin}
with an $S$-dependent effective radius $\rho(S;v)$.
Specifically, Example~\ref{exp: new} defines $\rho_{\mathrm{vary}}(S;v)$ in \eqref{eq: rho new} as
\begin{align}
\label{eq:rho_vary_original_repeat}
\rho_{\mathrm{vary}}(S;v)
&= \rho_0 - \log\!\Big(
e^{\rho_0}-(e^{\rho_0}-1)\cdot \frac{\sum_{j\in [N]^+} v_j}{\sum_{j\in S^+} v_j}
\Big).
\end{align}
Let $v_{\mathrm{all}} := \sum_{j\in [N]_+} v_j$, $
v_S := \sum_{j\in S^+} v_j$, $c(S) := v_{\mathrm{all}}/v_S\ge 1 $.
Then \eqref{eq:rho_vary_original_repeat} can be rewritten in an equivalent and more interpretable form:
\begin{align}
\label{eq:rho_vary_simplified}
\rho_{\mathrm{vary}}(S;v)
&= -\log\!\Big(1-(1-e^{-\rho_0})\cdot c(S)\Big).
\end{align}
Equation~\eqref{eq:rho_vary_simplified} makes explicit how conditioning amplifies global uncertainty: since $c(S)\ge 1$,
we have
\begin{align}
\label{eq:rho_vary_properties}
\rho_{\mathrm{vary}}([N];v) = \rho_0, \quad \rho_{\mathrm{vary}}(S;v) \ge \rho_0, \quad \forall\, S\subsetneq [N],
\end{align}
with the inequality becoming strict whenever $c(S)>1$ (i.e., whenever $v_S<v_{\mathrm{all}}$).
Consequently, if we ``match'' the robustness parameters by setting $\rho=\rho_0$, the constant-radius model in
Example~\ref{exp: jin} is always weakly less conservative than the induced radius from Example~\ref{exp: new}:
\begin{align}
\label{eq:rho_compare_const_vary}
\rho_{\mathrm{const}}(S;v) = \rho_0 = \rho_{\mathrm{vary}}([N];v), \quad \rho_{\mathrm{const}}(S;v) \le \rho_{\mathrm{vary}}(S;v), \quad \forall\, S\neq [N].
\end{align}
This amplification has an immediate behavioral implication for the optimizer: because $\rho_{\mathrm{vary}}(S;v)$ is
larger when $\sum_{j\in S^+} v_j$ is small, the robust objective in Example~\ref{exp: new} tends to favor assortments with larger
total attraction, which often pushes the robust optimal solution toward selecting more items to ``stabilize'' purchase likelihood under shifts.
This mechanism is illustrated empirically in Figure~2, where we compare the optimal robust assortment cardinality
$|S^\star|$ under matched parameters ($\rho$ in Example~\ref{exp: jin} versus $\rho_0$ in Example~\ref{exp: new}): across a range of
robustness levels, the solutions under Example~\ref{exp: new} tend to have larger average cardinality, consistent with
\eqref{eq:rho_compare_const_vary}.
Finally, the above discussion also clarifies how $\rho$ and $\rho_0$ should be compared:
Example~\ref{exp: jin} treats $\rho$ as a uniform ``budget'' of conditional perturbation for every assortment,
while Example~\ref{exp: new} treats $\rho_0$ as a global ``budget'' on the prior, which induces an assortment-dependent
conditional budget $\rho_{\mathrm{vary}}(S;v)$.
Thus, calibrating $\rho$ to $\rho_0$ can be done either at $S=[N]$ (where they coincide), or at a reference
assortment $S_{\mathrm{ref}}$ (e.g., the nominal/robust optimizer) via the induced effective radius
$\rho_{\mathrm{vary}}(S_{\mathrm{ref}};v)$, depending on which operating point is most relevant to the application.

\subsection{Data Driven Learning of the Optimal Robust Assortment} \label{subsec: data driven}
We aim to study how to find an \emph{optimal robust assortment} with cardinality constraint $K\in\NN$ from observational data generated from the nominal MNL choice model $\PP_{\bv}$.
The observational offline dataset $\mathbb{D}$ is in the form of 
$$\mathbb{D}:=\big\{(S_k,i_k)\big\}_{k=1}^n,$$
where $S_k\subseteq[N]$ is an observed assortment, and $i_k\sim \PP_{\bv}(\cdot|S_k)$ is the observed customer choice sampled from the \emph{nominal} choice model.
We assume the following conditional independence structure. 
For any $k\in[n]$, when conditioned on $S_k$, $i_k$ is independent of all previous choices $\{i_{k'}\}_{k'<k}$.
Such a structure allows for the observational data to be collected from an online decision-making process from the nominal choice model \citep{caro2007dynamic, rusmevichientong2010dynamic, saure2013optimal, agrawal2017thompson, agrawal2019mnl, chen2021optimal, saha2024stop}.

We remark that when $K= N$, the above target reduces to the unconstrained case that permits assortments of any size. 
Meanwhile, when the robust set size function $\rho(\cdot;\cdot)=0$, the above target recovers the standard non-robust data-driven assortment optimization problem \citep{dong2023pasta, han2025learning}.

Finally, the metric we are interested in is the \emph{sub-optimality gap} of the learned assortment $\widehat{S}$ in terms of the robust expected revenue \eqref{eq: robust revenue}, defined as,
\begin{align}
    \mathrm{SubOpt}_{\rho(\cdot,\cdot)}(\widehat{S};\bv):=R_{\rho(\cdot,\cdot)}(S^{\star};\bv) - R_{\rho(\cdot,\cdot)}(\widehat{S};\bv),
\end{align}
where $\bv$ denotes the nominal choice model and $S^{\star}$ is the optimal robust assortment defined in \eqref{eq: optimal robust assortment}.
The goal is to design algorithms that utilize the observational data $\mathbb{D}$ to find robust assortments such that the suboptimality gap is small, matching the statistical lower bound of the gap.

\section{Solving Optimal Robust Assortment with Known Nominal Model}\label{subsec: computation}

Before we dive into the learning algorithm design, we first present how to solve the optimal robust assortments given a fixed nominal choice model $\PP$, which is a purely computational problem. 
We focus on the two examples (Examples~\ref{exp: jin} and \ref{exp: new}) for studying this question.  
It turns out that the answer to this question depends on two dimensions of the problem setups: (i) constrainedness of the assortments and (ii) the uniformness of the individual item revenues.
If the problem instance is either unconstrained or of uniform revenue, the solution can be greatly simplified with linear time algorithm. 
For the general case, we also present algorithms to find the optimal robust assortment with polynomial computational complexity.

\paragraph{Unconstrained assortment optimization.}
For the case of unconstrained assortment optimization, i.e., constraint size $K=N$, the following result shows that the optimal robust assortment is a complete set of the whole item set $[N]$. 
That is, if one assumes (without loss of generality) that $r_1>r_2>\cdots>r_N$, then the optimal robust assortment set is in the form of $S^{\star} = \{1,2,\cdots,i^{\star}\}$ for some $i^{\star}\in[N]$.
This allows us to find the optimal robust assortment with compute time linear in the item number $N$.

\begin{proposition}[Optimal robust assortment set without size constraints]\label{prop:_optimal_set_unconstrained}
    Under either Example~\ref{exp: jin} or Example~\ref{exp: new}, if $K=N$, then an optimal robust assortment set $S^{\star}$ is a complete set $S^{\star} = \{1,2,\cdots,i^{\star}\}$.
\end{proposition}

\begin{proof}[Proof of Proposition~\ref{prop:_optimal_set_unconstrained}]
    Please refer to Appendix~\ref{subsec:_proof_optimal_set_unconstrained} for a detailed proof.
\end{proof}

\paragraph{Assortment optimization with uniform revenue.}

For the case of uniform individual revenues, i.e., $r_i=r_j$ for any $i\neq j\in [N]$, the optimal assortment set $S^{\star}$ contains the $K$ items that have the top $K$ nominal attraction parameters $v_i$, which is computationally tractable within linear time. 
Meanwhile, one can see that in this case, the optimal robust assortment set actually coincides with the non-robust optimal assortment set.

\begin{proposition}[Optimal robust assortment set under uniform reward]\label{prop:_optimal_set_uniform}
    Under either Example~\ref{exp: jin} or Example~\ref{exp: new}, if $r_i=r_j$ for any $i\neq j\in [N]$, the optimal robust assortment $S^{\star}$ contains the $K$ items that have the top $K$ nominal attraction parameters $v_i$.
\end{proposition}

\begin{proof}[Proof of Proposition~\ref{prop:_optimal_set_uniform}]
    Please refer to Appendix~\ref{subsec:_proof_optimal_set_uniform} for a detailed proof.
\end{proof}

\paragraph{General setups.}
Finally, for the general case where the revenues are non-uniform and the assortment set size is constrained to an arbitrary $K<N$, there exist algorithms that can find the optimal robust assortment given the nominal choice model with computational complexity of order $\widetilde{\cO}(N^2)$, where we recall that $N$ is the number of whole items.
For Example~\ref{exp: jin}, the algorithm is proposed by \cite{jin2022distributionally}. 
For Example~\ref{exp: new}, we provide a detailed algorithm design (Algorithm~\ref{alg: computation new 1}) and computational time analysis in Appendix~\ref{subsubsec: algorithm design}.

\begin{proposition}[Solving optimal robust assortment: general case]\label{prop:_optimal_set_general}
    Under either Example~\ref{exp: jin} or Example~\ref{exp: new}, there exists an algorithm to find an $\epsilon$-optimal optimal robust assortment $S^{\star}_{\epsilon}$ using $\widetilde{\cO}(N^2)$ times of evaluation of a known function, where $\widetilde{\cO}(\cdot)$ hides poly-logarithm factors in $1/\epsilon$.
\end{proposition}

\begin{proof}[Proof of Proposition~\ref{prop:_optimal_set_general}]
    Please refer to Appendix~\ref{subsec:_proof_optimal_set_general} for a detailed proof. 
\end{proof}

Equipped with the propositions in this section, we are able to claim an oracle that can return an optimal robust assortment given any nominal choice model for both Example~\ref{exp: jin} and Example~\ref{exp: new}, which would be utilized as a subroutine in the coming section for data-driven algorithm design.

\section{Learning an Optimal Robust Assortment from Data}\label{sec:_algorithms}

In this section, we present the algorithms to solve the data-driven robust assortment optimization problem (Section~\ref{subsec: data driven}). 
We first propose a unified algorithm framework for this problem in Section~\ref{subsec: unified-framework}.
In Sections~\ref{subsec: algorithm jin} and~\ref{subsec: algorithm new}, we specify the unified algorithm to solve the two instances of the robust assortment optimization framework introduced in Section~\ref{subsec: setup}, namely the examples of constant robust set size (Example~\ref{exp: jin}) and varying robust set size (Example~\ref{exp: new}).
Later in Section~\ref{sec: theory} we establish the finite sample suboptimality analysis.

\subsection{A Unified Algorithm Design: Pessimistic Robust Rank-Breaking}\label{subsec: unified-framework}

To address the data-driven assortment optimization problem (Section~\ref{subsec: data driven}), we propose Pessimistic Robust Rank-Breaking (PR$^2$B).
It (i) uses rank-breaking technique \citep{saha2019active} to estimate the nominal choice model, and (ii) solves the optimal robust assortment based on the idea of ``double pessimism'' \citep{blanchet2023double}.
Specifically, ``double pessimism'' represents the principle of performing nested pessimistic estimation in the face of two sources of uncertainty, namely, the uncertainty stems from finite data and the uncertainty in the choice model.
We detail these two algorithmic components in the sequel.

\paragraph{Rank-breaking estimation of the nominal choice model.} 
The first step of the algorithm is to estimate the nominal choice model, i.e., the attraction parameters $\{v_j\}_{j\in[N]}$ of the nominal model that generates the data. 
We apply the technique of rank-breaking, which is based upon the idea of converting ranking data into independent pairwise comparisons. 
In the context of the MNL model, this utilizes the key insights that the conditional probability $p_j:=\PP(i_k=j|i_k\in\{0,j\}) = v_j/(1+v_j)$ can be estimated from observational data, which can in turn induces estimates of $v_j$. 
Mathematically, we define 
\begin{align}
    \widehat p_j = \frac{\tau_j}{\tau_{j,0}}, \quad \tau_j = \sum_{k=1}^{n}\mathbf{1}\{i_k = j\},\quad \tau_{j,0} = \sum_{k=1}^{n}\mathbf{1}\big\{i_k \in\{0,j\},j\in S_k\big\}.\label{eq: tau}
\end{align}
One key observation is that such an estimation framework treats each unknown $p_j$ (and thus $v_j$) individually. 
Effective estimation of each $p_j$ only requires that the dataset contains sufficiently many assortments involving this specific item $j$, regardless of other items.
Our theory in Section~\ref{sec: theory} shows that it enables PR$^2$B to utilize the data in the most efficient way: PR$^2$B requires a minimal coverage assumption on the data. 
The minimal assumption means without it sample-efficient learning of the optimal robust assortment is proven impossible.

\paragraph{Pessimistic robust assortment optimization: a double pessimism approach.} 
With the estimate of the nominal choice model, we solve the desired assortment via a double pessimism approach \citep{blanchet2023double}. 
Namely, we behave pessimistically in the face of two sources of uncertainty: (i) the statistical uncertainty in estimation the nominal model using finite data and (ii) the epistemic uncertainty about the choice probability.
This is demonstrated effective and efficient in the context of offline policy learning with direct reward feedback \citep{blanchet2023double, shi2024distributionally}. 
In the context of assortment optimization, following the objective of \cite{blanchet2023double}, we first propose the following abstract target: 
\begin{align}
    \widetilde{S}:=\argmax_{S\subset [N], |S|\leq K} \,\,\inf_{\widetilde{\mathbb{P}}(\cdot|S)\in \cP_{n}(\PP(\cdot|S))}\,\,\inf_{D_{\mathrm{KL}}(\mathbb{Q}_{S_+}\|\widetilde{\mathbb{P}}(\cdot|S))\leq \rho(S;\tilde\PP(\cdot|S))} \Big\{R(S;\QQ_{S+})\Big\},\label{eq: double pessimism}
\end{align}
where $\cP_n(\PP(\cdot|S))$ denotes any valid confidence region\footnote{By valid we mean that it should contain the true nominal choice probability.} of the nominal choice probability given the assortment $S$ constructed from finite data.
Intuitively, the first infimum in \eqref{eq: double pessimism} captures the statistical uncertainty in estimating the nominal choice model from finite data, while the second infimum captures the epistemic uncertainty in the choice model itself.

However, despite the potential effectiveness of \eqref{eq: double pessimism} in terms of statistical efficiency \citep{blanchet2023double}, it remains unknown how to computationally efficiently solve such a target.
To address such a challenge, we propose the following practical version of the target \eqref{eq: double pessimism},
\begin{align}
    \widehat S:=\argmax_{S\subseteq[N],|S|\leq K} \inf_{D_{\mathrm{KL}}(\mathbb{Q}_{S_+}\|\PP_{\bv^{\mathrm{LCB}}}(\cdot|S))\leq \rho(S;\tilde{\PP}(\cdot|S))} \Big\{R(S;\QQ_{S+})\Big\},\label{eq: object}
\end{align}
where $\PP_{\bv^{\mathrm{LCB}}}$ is the MNL model induced by a pessimistic attraction parameter estimate $\bv^{\mathrm{LCB}}$, defined as
\begin{align}
    v_0:=1,\quad v^{\mathrm{LCB}}_j:= \frac{p_j^{\mathrm{LCB}}}{1-p_j^{\mathrm{LCB}}},\quad p_j^{\mathrm{LCB}}:= \max\left\{0, \widehat p_j - \sqrt{\frac{2\widehat p_j(1-\widehat p_j)\log(1/\delta)}{\tau_{j,0}}} - \frac{\log(1/\delta)}{\tau_{j,0}}\right\},\quad \forall j\in[N].
\end{align}
Here $\{(p_j,\tau_j,\tau_{j,0})\}_{j\in[N]}$ is obtained from the first step of the algorithm.
In other words, the new objective \eqref{eq: object} treats $\PP_{\bv^{\mathrm{LCB}}}$ as the true nominal model and solves the optimal robust assortment associated with it, which is computationally tractable (see Section~\ref{subsec: computation}). 

We now explain the insights why \eqref{eq: object} is a good surrogate of \eqref{eq: double pessimism}. 
Firstly, with concentration arguments, one can conclude that $\bv^{\mathrm{LCB}}$ does form a pessimistic estimation of the unknown nominal model parameter $\bv$ in the sense that with high probability $v^{\mathrm{LCB}}_j\leq v_j$ for each $j\in[N]_+$ (Lemma~\ref{lem:_concentration}).
Then most importantly, we are able to prove that for the two examples we consider, the robust expected revenue with the pessimistically estimated attraction parameter also becomes a pessimistic estimate of the robust revenue with the true nominal attraction parameter in a proper sense (Lemmas~\ref{lem:_monotonicity_1} and \ref{lem:_monotonicity_2}).
This is non-trivial since the dependence of the robust revenue w.r.t. the attraction parameters is highly non-linear.
But thanks to these new results, by treating the pessimistic parameter estimation as the true parameter, we can effectively approximate the first infimum in the double pessimism objective \eqref{eq: double pessimism}.

\subsection{Algorithm Design: Constant Robust Set Size (Example~\ref{exp: jin})}\label{subsec: algorithm jin}

\begin{algorithm}[h]
\caption{Pessimistic Robust Rank-Breaking: Constant Robust Set Size (PR$^2$B-C)} \label{alg: jin}
\begin{algorithmic}[1]
\STATE \textbf{Inputs:} Item set $[N]$, cardinality constraint $K$, offline dataset $\{(i_k, S_k)\}_{k = 1}^n$, failure probability $\delta$, constant robust set size $\rho\geq 0$.
 \STATE \textcolor{blue!55}{\texttt{// Construct pessimistic estimation of the nominal choice model parameters.}}
\FOR{item $j =1,\cdots,N$}
    \STATE Set $\widehat{p}_j \leftarrow \tau_j / \tau_{j0}$, where $\tau_{j,0} \leftarrow \sum_{k=1}^n \mathbf{1}\{i_k \in \{0, j\}, j \in S_k\}$, $\tau_j \leftarrow \sum_{k=1}^n \mathbf{1}\{i_k = j\}$
    \STATE Set $p_j^{\text{LCB}} \leftarrow (\widehat{p}_j - \sqrt{2\widehat{p}_j(1-\widehat{p}_j)\log(1/\delta)/\tau_{j0}} - \log(1/\delta)/\tau_{j0})_+$\\ 
    \STATE Set $v_j^{\text{LCB}} \leftarrow p_j^{\text{LCB}} / (1 - p_j^{\text{LCB}})$
\ENDFOR
 \STATE \textcolor{blue!55}{\texttt{// Pessimistic robust assortment optimization.}}
\STATE Solve $\widehat{S}:= \arg\max_{S\subset[N], \lvert S \rvert \leq K} R_{\rho}(S;{\bm v}^{\mathrm{LCB}})$
\STATE \textbf{Output:} assortment $\widehat{S}$
\end{algorithmic}
\end{algorithm}

To address the data-driven assortment optimization problem with constant robust set size (Example~\ref{exp: jin}), we adopt the general framework in Section~\ref{subsec: unified-framework} to  propose Pessimistic Robust Rank-Breaking (PR$^2$B-C), with details in Algorithm~\ref{alg: jin}. Under this formulation, the pessimistic robust assortment optimization step (line~5 of Algorithm~\ref{alg: jin}), can be solved in polynomial time, as shown in Proposition~\ref{prop:_optimal_set_general}.

\subsection{Algorithm Design: Varying Robust Set Size (Example~\ref{exp: new})}\label{subsec: algorithm new}

To address the data-driven assortment optimization problem with varying robust set size in Example~\ref{exp: new}, we adopt the double pessimism framework in Section~\ref{subsec: unified-framework} to propose another variant of the Pessimistic Robust Rank-Breaking algorithm named PR$^2$B-V, with the details in Algorithm~\ref{alg: new}.
The PR$^2$B-V algorithm follows a similar design to PR$^2$B-C in model estimation, but incorporates customized designs to handle the robust assortment optimization problem with varying robust set size.

\paragraph{Known total attraction.} Before we introduce the details, we first make the following assumption, which assumes that the summation of the attraction parameters of the nominal MNL model is known to the learner.

\begin{assumption}[Known total attraction]\label{ass: known total attraction}
    We assume that for the nominal choice model $\bv$ that generates the observational data, the total attraction of all items are known, i.e., $v_{\mathrm{tot}} := \sum_{i\in[N]}v_i$ is given.
\end{assumption}

We explain the known of total attraction as follows. 
Such an assumption plays a key role in enabling a ``partial coverage'' style minimal data requirement we hope for under the varying robust set setup (Example~\ref{exp: new}).
Intuitively, as we explain in Example~\ref{exp: new}, our formulation in Example~\ref{exp: new} can be understood as robustifying the prior choice distribution of the items, see \eqref{eq: conditioned robust revenue formulation}.
Since the prior distribution on the items depend on all the attraction parameters, essentially only observing the items associated the optimal robust assortment is thus not enough to estimate the probabilities for those items, causing an insufficiency of ``partial coverage'' data.
However, ideally one would still like the minimal data requirement could only depend on the items within the optimal robust assortment. 
To bypass such an impossibility, we impose such an assumption on the knowledge of the learner. 
In practice this is also a reasonable assumption since the learner could easily calibrate the relative attraction of the whole item set compared to the attraction of the non-purchasing behavior $v_0=1$.

Under Assumption~\ref{ass: known total attraction}, we consider the parameter $\rho_0$ in robust set size function $\rho(\cdot;\cdot)$ satisfying 
\begin{align}
0\leq \rho_0 < \log\left(1+\frac{1}{v_{\mathrm{tot}}}\right),\quad \text{$v_{\mathrm{tot}}$ defined in Assumption~\ref{ass: known total attraction}}.
\end{align}
This range of $r_0$ is obtained from the discussions in Example~\ref{exp: new} which ensures that the robust assortment optimization problem does not become trivial.
Also, under Assumption~\ref{ass: known total attraction}, the robust expected revenue associated with the nominal choice model $\bv$ can be equivalently represented using the following dual form. 
Specifically, defining the function $\widetilde{H}(S,\lambda,\bv): [N]\times \mathbb{R}_{\geq 0}\times\RR^{(N+1)}\mapsto\RR$ as 
\begin{align}
    \widetilde{H}(S,\lambda;\bv):=-\lambda\cdot\log\left(\frac{\sum_{i\in S_+}v_i\cdot\exp(-r_i/\lambda)}{\sum_{i\in S_+}v_i}\right) - \lambda\cdot \rho_0 +\lambda\cdot  \log\left(e^{\rho_0} - \frac{(e^{\rho_0} - 1)\cdot (1+v_{\mathrm{tot}})}{\sum_{i\in S_+}v_i}\right),\label{eq: h tilde}
\end{align}
then the robust expected revenue under the nominal choice model $\bv$ equals to 
\begin{align}
    R_{\rho(\cdot,\cdot)}(S;\bv) =\inf_{\substack{\mathbb{Q}_{S_+}\in\cP(S_+),\\D_{\mathrm{KL}}(\mathbb{Q}_{S_+}\|\mathbb{P}(\cdot|S))\leq \rho(S;\mathbb{P}(\cdot|S))}} \Big\{R(S;\QQ_{S+})\Big\}= \sup_{\lambda \geq 0} \left\{\widetilde{H}(S,\lambda;\bv)\right\},\quad \forall S\subseteq[N].\label{eq: dual new}
\end{align}
Notice that under the robust parameter constraint $0\leq \rho_0 < \log(1+1/v_{\mathrm{tot}})$, the function $\widetilde{H}$ is always finite for any input $(S,\lambda,\bv)\in [N]\times \mathbb{R}_{\geq 0}\times\RR^{(N+1)}$.

\paragraph{Algorithm details.} 
Now we are ready to introduce the algorithm. Similar to Algorithm~\ref{alg: jin}, Algorithm~\ref{alg: new} first establishes a pessimistic estimation $\bv^{\mathrm{LCB}}$ of the attraction parameters $\bv$ of the nominal choice model via the technique of rank-breaking.
We omit the details here and refer to Section~\ref{subsec: unified-framework}.

Then with the pessimistic estimation $\bv^{\mathrm{LCB}}$, we solve the following pessimistic robust assortment optimization problem to obtain $\widehat{S}$,
\begin{align}
    \widehat S:=\argmax_{S\subseteq[N],|S|\leq K} \,\,\sup_{\lambda\geq 0} \left\{\widetilde{H}(S,\lambda;\bv^{\mathrm{LCB}})\right\},\quad \text{with $\widetilde{H}$ defined in \eqref{eq: h tilde}.}\label{eq: object new}
\end{align}
To understand the above objective, we note that it also follows the idea of double pessimism in Section~\ref{subsec: unified-framework}. 
More specifically, as we prove in Lemma~\ref{lem:_monotonicity_2}, for the pessimistic attraction estimate $\bv^{\mathrm{LCB}}$, the objective we optimize, i.e.,  $\sup_{\lambda\geq 0}\widetilde{H}(\widehat{S},\lambda;\bv^{\mathrm{LCB}})$, also forms a valid lower bound of the ground truth $\sup_{\lambda\geq 0}\widetilde{H}(\widehat{S},\lambda;\bv) = R_{\rho(\cdot;\cdot)}(\widehat{S};\bv)$, despite the high non-linearity in the dependence on $\bv^{\mathrm{LCB}}$.
This then makes the objective \eqref{eq: object new} approximately solve the doubly pessimistic objective \eqref{eq: double pessimism} where the inner infimum is now performed with a varying robust set size $\rho(\cdot;\cdot)$.

Finally, we conclude by discussing the computation of \eqref{eq: object new}. 
One can notice that the optimization problem in \eqref{eq: object new} is not the same as the original robust assortment optimization problem given a nominal choice model. 
But as we show in Remark~\ref{rmk: computation}, our proposed computation algorithm for Example~\ref{exp: new} can be directly used to solve \eqref{eq: object new}. 
Therefore, the objective \eqref{eq: object new} can still be solved with complexity $\widetilde{\cO}(N^2)$ in the worst case.

\begin{algorithm}[t]
\caption{Pessimistic Robust Rank-Breaking: Varying Robust Set Size (PR$^2$B-V)} \label{alg: new}
\begin{algorithmic}[1]
\STATE \textbf{Inputs:} Item set $[N]$, cardinality constraint $K$, offline dataset $\{(i_k, S_k)\}_{k = 1}^n$, failure probability $\delta$, robust set size function $\rho(\cdot,\cdot)$ defined in Example~\ref{exp: new}, total nominal attraction $v_{\mathrm{tot}}$.
 \STATE \textcolor{blue!55}{\texttt{// Construct pessimistic estimation of the nominal choice model parameters.}}
\FOR{item $j =1,\cdots,N$}
    \STATE Set $\widehat{p}_j \leftarrow \tau_j / \tau_{j0}$, where $\tau_{j,0} \leftarrow \sum_{k=1}^n \mathbf{1}\{i_k \in \{0, j\}, j \in S_k\}$, $\tau_j \leftarrow \sum_{k=1}^n \mathbf{1}\{i_k = j\}$
    \STATE Set $p_j^{\text{LCB}} \leftarrow (\widehat{p}_j - \sqrt{2\widehat{p}_j(1-\widehat{p}_j)\log(1/\delta)/\tau_{j0}} - \log(1/\delta)/\tau_{j0})_+$\\ 
    \STATE Set $v_j^{\text{LCB}} \leftarrow p_j^{\text{LCB}} / (1 - p_j^{\text{LCB}})$
\ENDFOR
 \STATE \textcolor{blue!55}{\texttt{// Pessimistic robust assortment optimization.}}
\STATE Solve $\widehat{S}:= \arg\max_{S\subset[N], \lvert S \rvert \leq K} \sup_{\lambda\geq 0}\widetilde{H}(S,\lambda;{\bm v}^{\mathrm{LCB}})$ with $\widetilde{H}$ defined in \eqref{eq: h tilde}.
\STATE \textbf{Output:} assortment $\widehat{S}$
\end{algorithmic}
\end{algorithm}

\section{Theoretical Analysis of Pessimistic Robust Rank-Breaking}\label{sec: theory}

In this section, we establish the theoretical guarantees for the algorithms proposed in Section~\ref{sec:_algorithms}.
The main results are suboptimality upper bounds.
We also establish minimax lower bounds to demonstrate the tightness of our results.
We analyze Algorithm~\ref{alg: jin} in Section~\ref{subsec: theory jin} and Algorithm~\ref{alg: new} in Section~\ref{subsec: theory new}.

\subsection{Theoretical Results for Algorithm~\ref{alg: jin}}\label{subsec: theory jin}

\subsubsection{Suboptimality Upper Bounds}
Our first result is the suboptimality upper bound.
It first gives the upper bound for the general non-uniform revenue case, and then provides a tighter bound for the special case of uniform revenue.

\begin{theorem}[Suboptimality of PR$^2$B-C (Algorithm~\ref{alg: jin})]\label{thm:_algorithm_design_1}
    Conditioning on the observed assortments $\{S_k\}_{k=1}^n$, suppose that 
    \begin{align}
        n_i:= \sum_{k=1}^n\mathbf{1}\{i \in S_k\}\geq 138\cdot\max\left\{1,\frac{256}{138v_i}\right\}\cdot (1+v_{\max})(1+Kv_{\max})\log(3N/\delta),\quad \forall i\in S^{\star},\label{eq: condition_n_i_main}
    \end{align}
    where $S^{\star}$ is the optimal robust assortment \eqref{eq: optimal robust assortment}.
    Then with probability at least $1-2\delta$, the suboptimality of the robust expected revenue of the assortment $\widehat S$ output by Algorithm~\ref{alg: jin} satisfies
    \begin{align}
        \mathrm{SubOpt}_{\rho}(\widehat{S};\bv)\leq  r_{\max} \cdot (1+v_{\max})\cdot K\cdot \left( 32  \sqrt{\frac{\log(N/\delta)}{\min_{i\in S^{\star}} n_i}} + \frac{200 (1+v_{\max})K\log(N/\delta)}{\min_{i\in S^{\star}} n_i}\right).
    \end{align}
    If the problem instance further satisfies $r_j=r_{\max}$ for each $j\in[N]$, then the following tighter bound holds,
    \begin{align}
        \mathrm{SubOpt}_{\rho}(\widehat{S};\bv)\leq  r_{\max} \cdot (1+v_{\max})\cdot\sqrt{K}\cdot \left( 32  \sqrt{\frac{\log(N/\delta)}{\min_{i\in S^{\star}} n_i}} + \frac{200 \sqrt{K}\log(N/\delta)}{\min_{i\in S^{\star}} n_i}\right).
    \end{align}
\end{theorem}

\begin{proof}[Proof of Theorem~\ref{thm:_algorithm_design_1}]
    Both of the two bounds are corollaries of a more general result in Theorem~\ref{thm:_algorithm_design_1_general} (see Appendix~\ref{subsec:_general_upper_bound}), based on which we can obtain Theorem~\ref{thm:_algorithm_design_1} through additional proofs in Appendix~\ref{subsec:_proof_algorithm_design_1}.
    We sketch the proof and highlight the technique novelties in Section~\ref{sec: sketch}.
\end{proof}

We make following four key remarks for understanding and interpreting our Theorem~\ref{thm:_algorithm_design_1}.

\emph{Firstly,} and most importantly, the leading term of the suboptimality gap scales with 
\begin{align}
    \mathrm{SubOpt}_{\rho}(\widehat{S};\bv)\lesssim \widetilde{\cO}\left(\sqrt{\frac{1}{\min_{i\in S^{\star}}n_i}}\right),\quad\text{where}\,\,n_i = \sum_{k=1}^n\mathbf{1}\{i\in S_k\},
\end{align}
in terms of the number of observational data, where $S^{\star}$ is the optimal robust assortment.
In order words, in order to enable efficient learning of the optimal robust assortment, it suffices to ensure that \emph{each single item $i$ in the optimal robust assortment} is observed for enough times across the whole dataset.
It does not require the observation of the whole optimal robust assortment $S^{\star}$ nor other items outside the optimal robust assortment.
This corresponds to the \emph{item-wise coverage condition} \citep{han2025learning} in the non-robust learning setup, and our results generalizes this condition to the robust learning problem, which we call the \emph{robust item-wise coverage condition}.
The reason that the PR$^2$B-C algorithm requires only the robust item-wise data coverage is that (i) we use the careful design of the pessimistic robust assortment optimization step guided by the spirit of double pessimism in the face of two sources of uncertainty \citep{blanchet2023double}; and (ii) we use the rank breaking based estimation of nominal choice models which decouples the estimation of each $v_i$ from other $v_j$ with $j\neq i$ from observational data.

\emph{Secondly,} as is shown by the different bounds in Theorem~\ref{thm:_algorithm_design_1}, the suboptimality of Algorithm~\ref{alg: jin} receives an $\cO(\sqrt{K})$ improvement when the problem instance is of uniform revenue.
This phenomenon coincidences with the observations in the non-robust data-driven assortment problems \citep{han2025learning} and generalizes their conclusion to the robust assortment learning problem.

\emph{Thirdly,} we note that the suboptimality upper bound of Algorithm~\ref{alg: jin} does not explode when the constant robust set size $\rho$ is approaching $0$.
This improves the dependence on the robust set size upon existing works on data-driven robust decision policy learning based upon KL robust sets, e.g., \cite{blanchet2023double, shi2024distributionally}. 
When $\rho=0$, the suboptimality gaps reduce to the performance metric for standard assortment learning problem and also recover the same bounds in \cite{han2025learning}.
Further, we remark that when $\rho$ gets relatively large, the general result Theorem~\ref{thm:_algorithm_design_1_general} shows that the suboptimality gap decays with a $\rho^{-1}$ factor. 
Since we consider constant level $\rho=\cO(1)$, we suppress such a dependence in the upper bound presented in Theorem~\ref{thm:_algorithm_design_1}.
Interested readers can refer to Theorem~\ref{thm:_algorithm_design_1_general} and its proofs.

\emph{Lastly,} we remark that there is no dependence on factors in the form of $\min_{i\in S^{\star}_+}\QQ_{S^{\star}_+}(i)^{-1}$ or $\min_{i\in S^{\star}} v_i$ in the suboptimality gap, This contrasts with that in the suboptimality analysis for robust RL in KL-RMDPs  \citep{shi2024distributionally}, which depends on the minimal inverse transition probability induced by optimal robust policy.
This is actually a bless of using the MNL structure as the nominal choice model. 
Meanwhile, we do point out that the burn-in complexity \eqref{eq: condition_n_i_main} does depend on $v_i^{-1}$, but the rate of the overall sample complexity does not depend on such quantities.

\subsubsection{Minimax Lower Bounds}
To demonstrate the optimality of Algorithm~\ref{alg: jin} for learning robust assortments with constant robust set size, we further establish the minimax lower bound for the suboptimality gap of this problem.

\begin{theorem}[Suboptimality minimax lower bound (Example~\ref{exp: jin})]\label{thm:_lower_bound_1}
    Consider the setup of robust assortment optimization problem of constant robust set size (Example~\ref{exp: jin}) with  $N$ items, $K$-constraint satisfying $N\geq 5K$, and robust set size $\rho\in [\underline{c}_{\rho}\cdot\log2,(1-\overline{c}_{\rho})\cdot \log 2]$ for any absolute constant $\underline{c}_{\rho},\overline{c}_{\rho}>0$, there exists an offline dataset of $K$-sized assortments $\{S_k\}_{k=1}^n$ and a class of problem instances of Example~\ref{exp: jin} (denoted by $\cV$) s.t.,
    \begin{align}
        \inf_{\pi(\cdot):(2^{[N]}\times[N])^n\mapsto2^{[N]}}\sup_{(\bv,\br)\in\cV}\mathbb{E}_{\mathbb{D}\sim \otimes_{k=1}^n\mathbb{P}_{\bv}(\cdot|S_k)}\big[\mathrm{SupOpt}_{\rho}(\pi(\mathbb{D});\bv,\br) \big]\geq \Omega\left(r_{\max}\cdot K\cdot \sqrt{\frac{1}{\min_{i\in S_{\bv}^{\star}}n_i}}\right),
    \end{align}
    where $\Omega(\cdot)$ only hides universal absolute constants.
    Moreover, when we further restrict the problem into the uniform revenue case, under the same ranges of problem parameters as before, there exists another class of problem instances of Example~\ref{exp: jin} (denoted by $\cV_{\mathrm{uni}}$) such that 
    \begin{align}
        \inf_{\pi(\cdot):(2^{[N]}\times[N])^n\mapsto2^{[N]}}\sup_{(\bv,\br)\in\cV_{\mathrm{uni}}}\mathbb{E}_{\mathbb{D}\sim \otimes_{k=1}^n\mathbb{P}_{\bv}(\cdot|S_k)}\big[\mathrm{SupOpt}_{\rho}(\pi(\mathbb{D});\bv,\br) \big]\geq \Omega\left(r_{\max}\cdot \sqrt{K}\cdot \sqrt{\frac{1}{\min_{i\in S_{\bv}^{\star}}n_i}}\right).
    \end{align}
\end{theorem}

\begin{proof}[Proof of Theorem~\ref{thm:_lower_bound_1}]
    See Appendix~\ref{subsec:_proof_lower_bound_1} for a detailed proof.
\end{proof}

As suggested by Theorem~\ref{thm:_lower_bound_1} and Theorem~\ref{thm:_algorithm_design_1}, the suboptimality of PR$^2$B-C (Algorithm~\ref{alg: jin}) is actually nearly-minimax optimal with respect to the key problem parameters $K$, $\min_{i\in S_{\bv}^{\star}}n_i$, and $r_{\max}$ up to logarithm factors, demonstrating the optimality of our algorithm design for the constant robust set size case.

Especially, the lower bounds show that the minimal requirement for efficient robust assortment learning is that \emph{each single item in the optimal robust assortment $S^{\star}$ should be covered enough time by the assortments in the observational dataset,}
regardless of the robust set size $\rho$.
In the end, we highlight that the lower bounds together with the upper bounds indicate that the \emph{statistical gap} of order $\cO(\sqrt{K})$ still exists for the robust assortment optimization problem with constant robust set size, again regardless of the robust set size $\rho$.

\subsection{Theoretical Results for Algorithm~\ref{alg: new}}\label{subsec: theory new}

\subsubsection{Suboptimality Upper Bounds}

Parallel to the constant robust set size case, we first establish the suboptimality upper bounds of Algorithm~\ref{alg: new} for both the general non-uniform revenue case and the special case of uniform revenue.

\begin{theorem}[Suboptimality of PR$^2$B-V (Algorithm~\ref{alg: new})]\label{thm:_algorithm_design_2}
    Suppose that Assumption~\ref{ass: known total attraction} holds. 
    Conditioning on the observed assortments $\{S_k\}_{k=1}^n$, suppose that 
    \begin{align}
        n_i:= \sum_{k=1}^n\mathbf{1}\{i \in S_k\}\geq 138\cdot\max\left\{1,\frac{256}{138v_i}\right\}\cdot (1+v_{\max})(1+Kv_{\max})\log(3N/\delta),\quad \forall i\in S^{\star},
    \end{align}
    where $S^{\star}$ is the optimal robust assortment \eqref{eq: optimal robust assortment}.
    Then with probability at least $1-2\delta$, the suboptimality of the robust expected revenue of the assortment $\widehat S$ output by Algorithm~\ref{alg: new} satisfies
    \begin{align}
        \mathrm{SubOpt}_{\rho(\cdot;\cdot)}(\widehat{S};\bv)\leq  \frac{r_{\max} \cdot (1+v_{\max})\cdot K}{\sqrt{1+v(S^{\star})/2 - (1-e^{-\rho_0})\cdot(1+v_{\mathrm{tot}})}}\cdot  32  \sqrt{\frac{2\log(N/\delta)}{\min_{i\in S^{\star}} n_i}} + \mathrm{\textnormal{higher order term}},
    \end{align}
    where for notational simplicity we define $v(S):=\sum_{j\in S}v_j$.
    If the problem instance further satisfies $r_j=r_{\max}$ for each $j\in[N]$, then the following tighter bound holds,
    \begin{align}
        \mathrm{SubOpt}_{\rho(\cdot;\cdot)}(\widehat{S};\bv)\leq  \frac{ \sqrt{1+ v(S^{\star})} \cdot r_{\max} \cdot (1+v_{\max})\cdot\sqrt{K}}{\sqrt{1 +  v(S^{\star})/2 - (1-e^{-\rho_0})\cdot(1+v_{\mathrm{tot}})}}\cdot  32  \sqrt{\frac{2\log(3N/\delta)}{\min_{i\in S^{\star}} n_i}} + \mathrm{\textnormal{higher order term}}.
    \end{align}
\end{theorem}

\begin{proof}[Proof of Theorem~\ref{thm:_algorithm_design_2}]
    Both of the two bounds are corollaries of a more general result in Theorem~\ref{thm:_algorithm_design_general_2} (see Appendix~\ref{subsec:_general_upper_bound_new}), based on which we can obtain Theorem~\ref{thm:_algorithm_design_2} through additional proofs in Appendix~\ref{subsec:_proof_algorithm_design_2}.
    We sketch the proof and highlight the technique novelties in Section~\ref{sec: sketch}.
\end{proof}

\subsubsection{Minimax Lower Bounds}\label{subsubsec: lower bound new}
To show the optimality of Algorithm~\ref{alg: new} for learning robust assortments with varying robust set size, we further establish the minimax lower bound for the suboptimality gap of this problem.

\begin{theorem}[Suboptimality minimax lower bound (Example~\ref{exp: new})]\label{thm:_lower_bound_new_1}
    Consider the setup of robust assortment optimization problem of varying robust set size (Example~\ref{exp: new}) with  $N$ items, $K$-constraint satisfying $N\geq 5K$, there exists an offline dataset of $K$-sized assortments $\{S_k\}_{k=1}^n$ and a class of problem instances of Example~\ref{exp: new} (denoted by $\cV$) such that for any robust set size parameter $\rho_0$ satisfying
    \begin{align}
         \log\left(\frac{1+v_{\mathrm{tot}}(\bv)}{1 + v_{\mathrm{tot}}(\bv) - \underline{c}_{\rho} }\right) \leq \rho_0 \leq \log\left(\frac{1+v_{\mathrm{tot}}(\bv)}{\overline{c}_{\rho} + v_{\mathrm{tot}} (\bv)}\right),\qquad \forall (\bv,\br)\in\cV,
    \end{align}
    where $\underline{c}_{\rho},\overline{c}_{\rho}>0$ are any absolute constants, and $v_{\mathrm{tot}}(\bv)$ denotes the total attraction parameter associated with instance $\bv$\footnote{In our construction of hard instance class the total attraction is constant across different instances, see Appendix~\ref{subsec:_proof_lower_bound_new_1}.}, it holds that
    \begin{align}
        &\inf_{\pi(\cdot):(2^{[N]}\times[N])^n\mapsto2^{[N]}}\sup_{(\bv,\br)\in\cV}\mathbb{E}_{\mathbb{D}\sim \otimes_{k=1}^n\mathbb{P}_{\bv}(\cdot|S_k)}\big[\mathrm{SupOpt}_{\rho(\cdot;\cdot)}(\pi(\mathbb{D});\bv,\br) \big]\\
        &\qquad \geq \Omega\left(\frac{r_{\max}\cdot K}{\sqrt{1+v(S^{\star}_{\bv})/2 - (1-e^{-\rho_0})\cdot(1+v_{\mathrm{tot}}(\bv))}}\cdot    \sqrt{\frac{1}{\min_{i\in S^{\star}_{\bv}} n_i}}\right),
    \end{align}
    where $\Omega(\cdot)$ only hides universal absolute constants.
    Moreover, when we further restrict the problem into the uniform revenue case, under the same ranges of problem parameters as before, there exists another class of problem instances of Example~\ref{exp: new} (denoted by $\cV_{\mathrm{uni}}$) such that for any robust set size parameter $\rho_0$ satisfying
    \begin{align}
         \log\left(\frac{1+v_{\mathrm{tot}}(\bv)}{1 + v_{\mathrm{tot}}(\bv) - \underline{c}_{\rho} }\right) \leq \rho_0 \leq \log\left(\frac{1+v_{\mathrm{tot}}(\bv)}{\overline{c}_{\rho} + v_{\mathrm{tot}} (\bv)}\right),\qquad \forall (\bv,\br)\in\cV_{\mathrm{uni}},
    \end{align}
    where $\underline{c}_{\rho},\overline{c}_{\rho}>0$ are any absolute constants, it holds that
    \begin{align}
        &\inf_{\pi(\cdot):(2^{[N]}\times[N])^n\mapsto2^{[N]}}\sup_{(\bv,\br)\in\cV}\mathbb{E}_{\mathbb{D}\sim \otimes_{k=1}^n\mathbb{P}_{\bv}(\cdot|S_k)}\big[\mathrm{SupOpt}_{\rho(\cdot;\cdot)}(\pi(\mathbb{D});\bv,\br) \big]\\
        &\qquad \geq \Omega\left(\frac{r_{\max}\cdot \sqrt{K}}{\sqrt{1+v(S^{\star}_{\bv})/2 - (1-e^{-\rho_0})\cdot(1+v_{\mathrm{tot}}(\bv))}}\cdot    \sqrt{\frac{1}{\min_{i\in S^{\star}_{\bv}} n_i}}\right).
    \end{align}
\end{theorem}

\begin{proof}[Proof of Theorem~\ref{thm:_lower_bound_new_1}]
    See Appendix~\ref{subsec:_proof_lower_bound_new_1} for a detailed proof.
\end{proof}

The result shows that the PR$^2$B-V algorithm is also nearly-minimax optimal on the dependence of the key parameters $K$, $r_{\max}$, $\min_{i\in S^{\star}} n_i$, and the additional factor $\sqrt{1+v(S^{\star}_{\bv})/2 - (1-e^{-\rho_0})\cdot(1+v_{\mathrm{tot}}(\bv))}$ that does not appear in the case of constant robust set size (Example~\ref{exp: jin}), unveiling the different difficulty facing the learning problem between Examples~\ref{exp: jin} and \ref{exp: new}.
Parallel to Example~\ref{exp: jin}, we also demonstrate the $\sqrt{K}$ statistical gap between the uniform revenue case and the general non-uniform revenue case for Example~\ref{exp: new}.
Finally, we note that when $\rho_0\rightarrow 0$, Example~\ref{exp: new} reduces to the standard data driven assortment optimization problem, and the corresponding upper and lower bounds also reduces to that of \cite{han2025learning}.

\subsection{Overview of Techniques}\label{sec: sketch}

In this section, we highlight some key techniques in establishing the suboptimality upper bounds for PR$^2$B(-C/V) algorithms.
We refer to Appendices for detailed proofs of the results introduced in this section.

\paragraph{Monotonicity argument to prove the effectiveness of the pessimistic robust assortment optimization step.} 
As we discuss in Section~\ref{subsec: algorithm jin}, a key component of Algorithm~\ref{alg: jin} is to use the robust expected revenue associated with a nominal choice  model with pessimistically estimated attraction parameters as a surrogate of pessimistic estimation of the robust expected revenue associated with the true nominal choice model.
This hugely releases the computation burden of the original doubly pessimistic optimization procedure \eqref{eq: double pessimism}. 
The observation behind this step is the following key lemma. 

\begin{lemma}[Monotonicity under optimal robust assortment (Example~\ref{exp: jin})]\label{lem:_monotonicity_1_main}
    Consider Example~\ref{exp: jin} and two sets of parameters $\bv$ and $\bv'$ such that $v_j\leq v_j'$ for all $j\in[N]$. 
    Then for the optimal robust assortment $S^{\star}_{\bv}$ associated with $\bv$, it holds that $ R_{\rho}(S^{\star}_{\bv};\bv)\leq R_{\rho}(S^{\star}_{\bv};\bv')$.
\end{lemma}

The lemma shows that, despite the highly non-linearity between the robust expected revenue $R_{\rho}(S;\bv)$ and the nominal attraction parameters $\bv$, there still exists certain type of monotonicity of $R_{\rho}(S;\bv)$ respect to $\bv$. 
Specifically, it shows that if we solve the optimal robust assortment associated with the pessimistically estimated nominal choice parameters, the resulting robust expected revenue itself would also be pessimistic. 

This result is key to establish the sample-efficiency and the minimal data coverage condition, namely, the robust item-wise coverage condition for Algorithm~\ref{alg: jin}. 
Similarly, the same argument also holds for Example~\ref{alg: new} (Lemma~\ref{lem:_monotonicity_2}), which then allows us to establish the theory for Algorithm~\ref{alg: new}. 
We remark that results of similar spirit were proposed for non-robust assortment optimization problems based on MNL model for both online setup and offline setup \citep{agrawal2019mnl, han2025learning}.
However, because of a different mathematical structure of the robust assortment optimization problem, we need new techniques to prove Lemma~\ref{lem:_monotonicity_1_main}. 
Please see Appendices~\ref{lem: technical} for a detailed proof of Lemma~\ref{lem:_monotonicity_1_main} and Lemma~\ref{lem:_monotonicity_2}.

\paragraph{Handling the varying robust set size learned from data.}
Another critical challenge is to handle and analyze the varying robust set size in Example~\ref{exp: new} in order to guarantee a reasonable and tight performance guarantee for Algorithm~\ref{alg: new}.
Recall that to adaptively assign the robust set sizes based on the assortment as in Example~\ref{exp: new}, the robust set size itself involves unknown parameters to learn from the data.
However, the robust set size directly controls how the robust expected revenue is represented in its dual representation form, namely, the range of the dual variables, which is critical in deriving the performance guarantees of the algorithm.
A naive approach would result in amplification of the error in estimating the nominal parameters so that the resulting suboptimality upper bound explodes when the robust set size diminishes (i.e., $\rho_0\rightarrow 0$).
In the proof of Theorem~\ref{thm:_algorithm_design_2} (Appendix~\ref{subsec:_proof_algorithm_design_2}), we consider a clever decomposition of the suboptimality gap and carefully handles the error coming from an estimated robust set size. 
The resulting suboptimality bound does not explode as $\rho_0\rightarrow 0$ and is nearly-minimax optimal on the dependence of key parameters (Section~\ref{subsubsec: lower bound new}).

\section{Numerical Experiments}\label{sec: experiments}

\paragraph{The goals of experiments.}
We demonstrate the sample-efficiency of the proposed two algorithms and the robustness of the learned assortment through simulated numerical experiments.
To this end, we aim to study three key perspectives of this problem, namely: 
\begin{itemize}
    \item[(i)]  The sample efficiency of the proposed algorithms v.s. naive baselines (introduced later). See Section~\ref{subsec: exp1}.
    \item[(ii)] The robustness of the learned assortments and the comparison of the robustness properties under the two different formulations, i.e., Example~\ref{exp: jin} and Example~\ref{exp: new}. See Section~\ref{subsec: exp2}.
    \item[(iii)] The influence of the cardinality constraint $K$ on the sample complexity. See Section~\ref{subsec: exp3}. 
\end{itemize}

We conduct three lines of experiments to investigate these perspectives respectively in the following.

\subsection{Experiment 1: Sample Efficiency of Pessimistic Robust Rank Breaking.}\label{subsec: exp1}
In the first line of experiments, we demonstrate the sample efficiency of P$\mathrm{R}^2$B(-C/V) compared with naive baselines in terms of minimizing the suboptimality of the robust expected revenue of the learned assortments.
Here the baseline is implemented by the \emph{single-pessimism} counterpart of P$\mathrm{R}^2$B(-C/V). 
More concretely, for Example~\ref{exp: jin}, we consider the following learned assortment,
\begin{align}
    \widehat{S}_{\mathrm{vanilla}}:=\argmax_{S\subset[N],|S|\leq K} R_{\rho}(S;\widehat{\bv}),\quad \widehat{v}_j:=\frac{\widehat{p}_j}{1-\widehat{p}_j},\quad \text{$\widehat{p}_j$ defined in \eqref{eq: tau}}. \label{eq: vanilla 1}
\end{align}
and for Example~\ref{exp: new}, we consider 
\begin{align}
    \widehat{S}_{\mathrm{vanilla}}:=\argmax_{S\subset[N],|S|\leq K} \sup_{\lambda\geq 0}\,\widetilde{H}(S,\lambda;\widehat{\bv}),\quad \widehat{v}_j:=\frac{\widehat{p}_j}{1-\widehat{p}_j},\quad \text{$\widehat{p}_j$ defined in \eqref{eq: tau} and $\widetilde{H}$ defined in \eqref{eq: h tilde}}.\label{eq: vanilla 2}
\end{align}
They represent the methods non-pessimistic to the data uncertainty which simply plug in the non-pessimistic rank-breaking estimate $\widehat{\bv}$ of the attraction parameters to the robust expected revenue.

\begin{figure}[!t]   
       \begin{minipage}[b]{0.49\textwidth}
        \centering
        \includegraphics[width=0.93\textwidth]{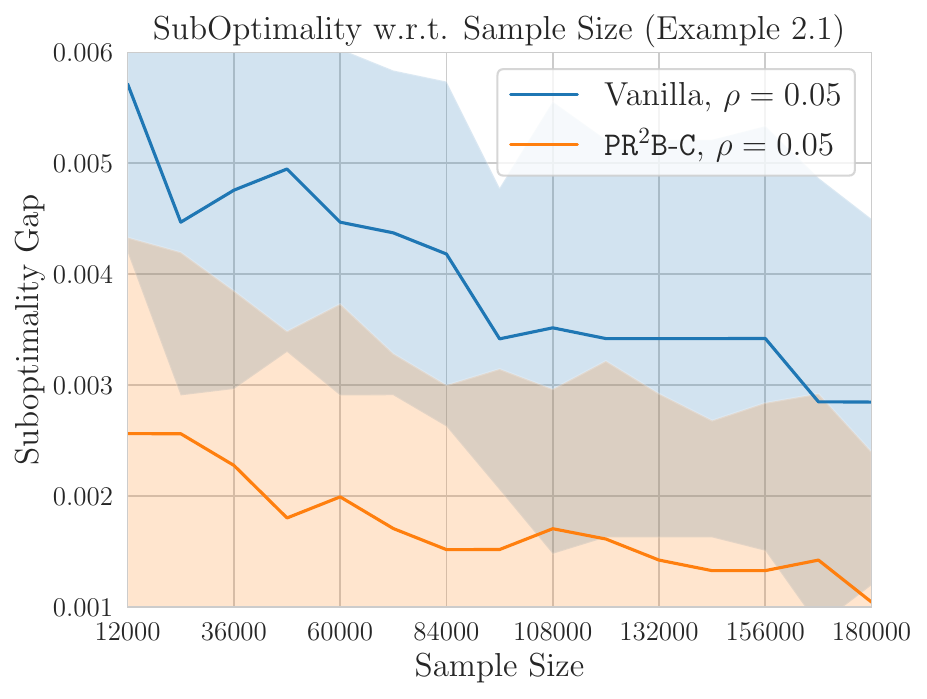}
    \end{minipage}
    \hspace{3mm}
    \begin{minipage}[b]{0.49\textwidth}
        \centering
        \includegraphics[width=0.98\textwidth]{pics/constant_size_suboptimality_vs_sample_size_varying_rho.pdf}   
    \end{minipage}
    \caption{Suboptimality gap of P$\mathrm{R}^2$B-C (Algorithm~\ref{alg: jin}) compared to the single-pessimistic counterpart \eqref{eq: vanilla 1}. All the parameter configurations are averaged over $25$ independent runs.}\label{fig: jin experiment}
\end{figure}

\begin{figure}[!t]
   
       \begin{minipage}[b]{0.49\textwidth}
        \centering
        \includegraphics[width=0.93\textwidth]{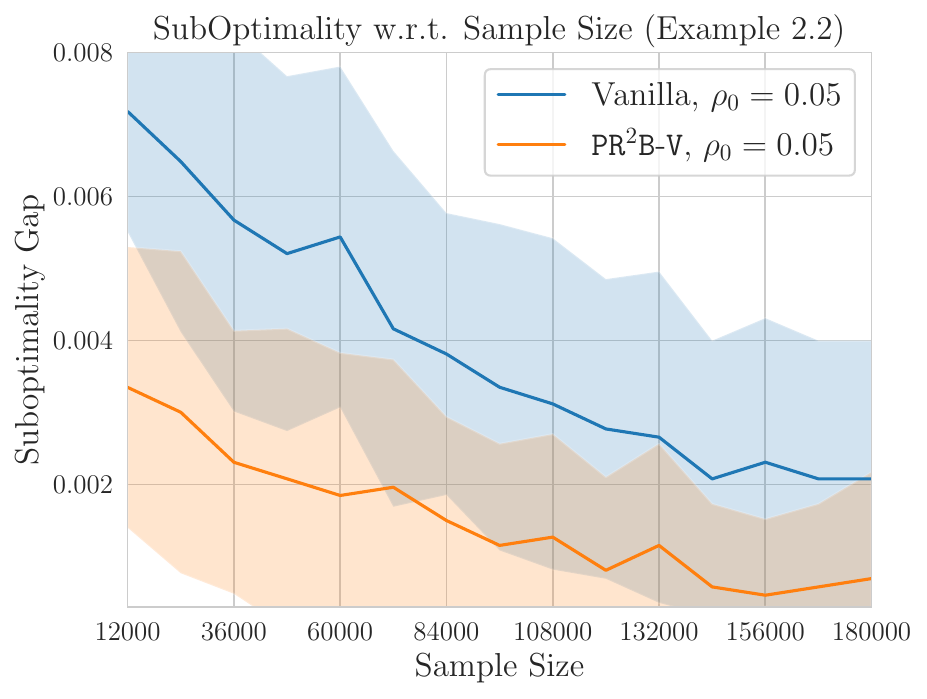}
    \end{minipage}
    \hspace{3mm}
    \begin{minipage}[b]{0.49\textwidth}
        \centering
        \includegraphics[width=0.98\textwidth]{pics/varying_size_suboptimality_vs_sample_size_varying_rho.pdf}   
    \end{minipage}
    \caption{Suboptimality gap of P$\mathrm{R}^2$B-V (Algorithm~\ref{alg: new}) compared to single-pessimism counterpart \eqref{eq: vanilla 2}. All the parameter configurations are averaged over $25$ independent runs.}\label{fig: new experiment}
\end{figure}

\paragraph{Construction of the instances and data.} 
We construct the following instance of nominal choice model (as motivated by the hard instances in the proof of statistical lower bounds, see Appendices~\ref{subsec:_proof_lower_bound_1} and \ref{subsec:_proof_lower_bound_new_1}).
We consider $N=15$ and $K=3$.
We define the nominal choice model by 
\begin{align}
    v_j=\frac{1}{K}+\epsilon,\quad \forall 1\leq j\leq K,\quad \text{and}\quad v_i = \frac{1}{K},\quad \forall K+1\leq i\leq N. 
\end{align}
with $\epsilon=0.01$ and define $r_i = 1$ for any $1\leq i\leq N$.
To show case the efficiency of our proposed algorithms compared to the vanilla algorithms we just introduced, we construct a partial coverage style offline dataset satisfying the ``robust item-wise coverage'' condition.
Specifically, we solve the optimal robust assortment (which is unique under the above setup) and randomly substitute an element in it by an item not inside the optimal robust assortment. 
Such a dataset has two key features: (i) the optimal robust assortment as a whole is not observed across the entire dataset; (ii) the suboptimal items are observed far less frequent than the items in the optimal robust assortment, causing essentially larger uncertainty for those items compared with the items in the optimal robust assortment, making the dataset a partially covered one.

\paragraph{Experiment results.} 
We vary two key quantities of the problem and observe the suboptimality gap of the learned assortment for both Example~\ref{exp: jin} and Example~\ref{exp: new}. 
For both examples, we range the dataset size from $12000$ to $180000$. 
For Example~\ref{exp: jin}, we vary $\rho\in\{0.05, 0.1, 0.15, 0.2, 0.25, 0.3, 0.35, 0.4, 0.45, 0.5\}$, and for Example~\ref{exp: new}, we vary $\rho_0\in\{0.05, 0.075, 0.1, 0.125, 0.15, 0.175\}$.
All the experiment parameters are run for $25$ times and we take the average.
See Figures~\ref{fig: jin experiment} and \ref{fig: new experiment}.
We observe that our Pessimistic Rank-breaking algorithms consistently outperform the vanilla algorithms in terms of the sample efficiency across different examples, sample sizes, and robust set parameters.
Especially, given the same sample size, our algorithms achieve much smaller suboptimality gap than the vanilla baseline, demonstrating their efficiency.

Moreover, as we can see from the experiments varying the robust set size parameters (right figures), the suboptimality gap shrinks when the robust set size parameter grows, matching the theoretical predictions (see Theorem~\ref{thm:_algorithm_design_1_general} and Theorem~\ref{thm:_algorithm_design_general_2}, which predict a decaying suboptimality if increasing the robust set size).

\subsection{Experiment 2: Robustness of the Learned Assortment.}\label{subsec: exp2}

In the second line of experiments, we aim to demonstrate the robustness of the learned assortment by the robust assortment optimization framework we consider.
Our approach is by showing how well our robust data-driven algorithms can bring revenue improvement against state-of-the-art non-robust counterpart \citep{han2025learning} when the customer preference distribution shifts from data generating patterns. 

\paragraph{Experiment pipeline.}
We first describe the pipeline of this experiment. 
We fix a nominal MNL model as the data generation environment, denoted by $\boldsymbol{v}_0$.
We also fix a general non-uniform revenue.
From the nominal environment $\PP_{\bv_0}$, we generate choice data $i_k\sim \PP_{\bv_0}(\cdot|S_k)$ based on randomly sampled assortments $S_k$, forming a dataset of size $n=20000$. 
Then we conduct P$\mathrm{R}^2$B-C (Algorithm~\ref{alg: jin}), P$\mathrm{R}^2$B-V (Algorithm~\ref{alg: new}), and PRB \citep{han2025learning}. 
In specific, we range the robust set parameters for P$\mathrm{R}^2$B(-C/V) to obtain a series of robust assortments. 
We note that the non-robust algorithm PRB can be obtained from P$\mathrm{R}^2$B by setting the robust set parameter to be zero. 

Then to test the robustness properties of the learned assortments, we deploy the series of assortments in perturbed choice distributions. 
We consider the following way of perturbing the choice distribution.
We treat $\bp_0:=\bv_0/\sum_{j\in [N]_+}v_{0,j}$ as a prior distribution over the whole items and we perturb such a prior distribution to obtain different choice models with different prior distributions.
In our experiments, we randomly generate the perturbed prior distributions $\widetilde{\bp}_0$, denoted as the set $\mathcal{P}_{\mathrm{perturb}}$. 
For each of the $\widetilde{\bp}_0\in\mathcal{P}_{\mathrm{perturb}}$, we calculate the suboptimality gap under the MNL model $\widetilde{\bv}_0$ corresponding to $\widetilde{\bp}_0$ between the learned assortments and the counterfactual optimal assortment $S^{\star}_{\text{non-robust}}(\widetilde{\bv}_0)$.
After ranging over the randomly generated perturbation set $\cP_{\mathrm{perturb}}$, we look into the statistics (defined later) of how the assortments given by our proposed robust assortment learning algorithms could improve the expected revenue over the non-robust assortment under shifted choice distributions.

We take $N=50$, $K=50$.
For P$\mathrm{R}^2$B-C, the robust set size is taken over $\rho\in\mathscr{R}_\mathrm{c}:=\{0, 0.1, 0.2, \cdots, 0.9, 1.0\}$, and for P$\mathrm{R}^2$B-V, the robust set parameter $\rho_0$ is taken over $\rho_0\in \mathscr{R}_{\mathrm{v}}:=\{0, 0.02, 0.04, \cdots,0.18, 0.2\}$.
Finally, to investigate the influence of the degree of the perturbation on the robustness properties, we divide the set $\cP_{\mathrm{perturb}}$ into $\cP_{\mathrm{perturb}}^{[0,1)}$ and $\cP_{\mathrm{perturb}}^{[1,\infty)}$. 
Here $\cP_{\mathrm{perturb}}^{[0,1)}$ means that any $\widetilde{\bp}_0\in\cP_{\mathrm{perturb}}^{[0,1)}$ satisfies $D_{\mathrm{KL}}(\widetilde{\bp}_0\|\bp_0)<1$, and similarly for the other set.
We let each of the two sets contain $10000$ randomly perturbed prior distributions.

\begin{figure}[!t]
       \begin{minipage}[b]{0.33\textwidth}
        \centering
        \includegraphics[width=1.1\textwidth]{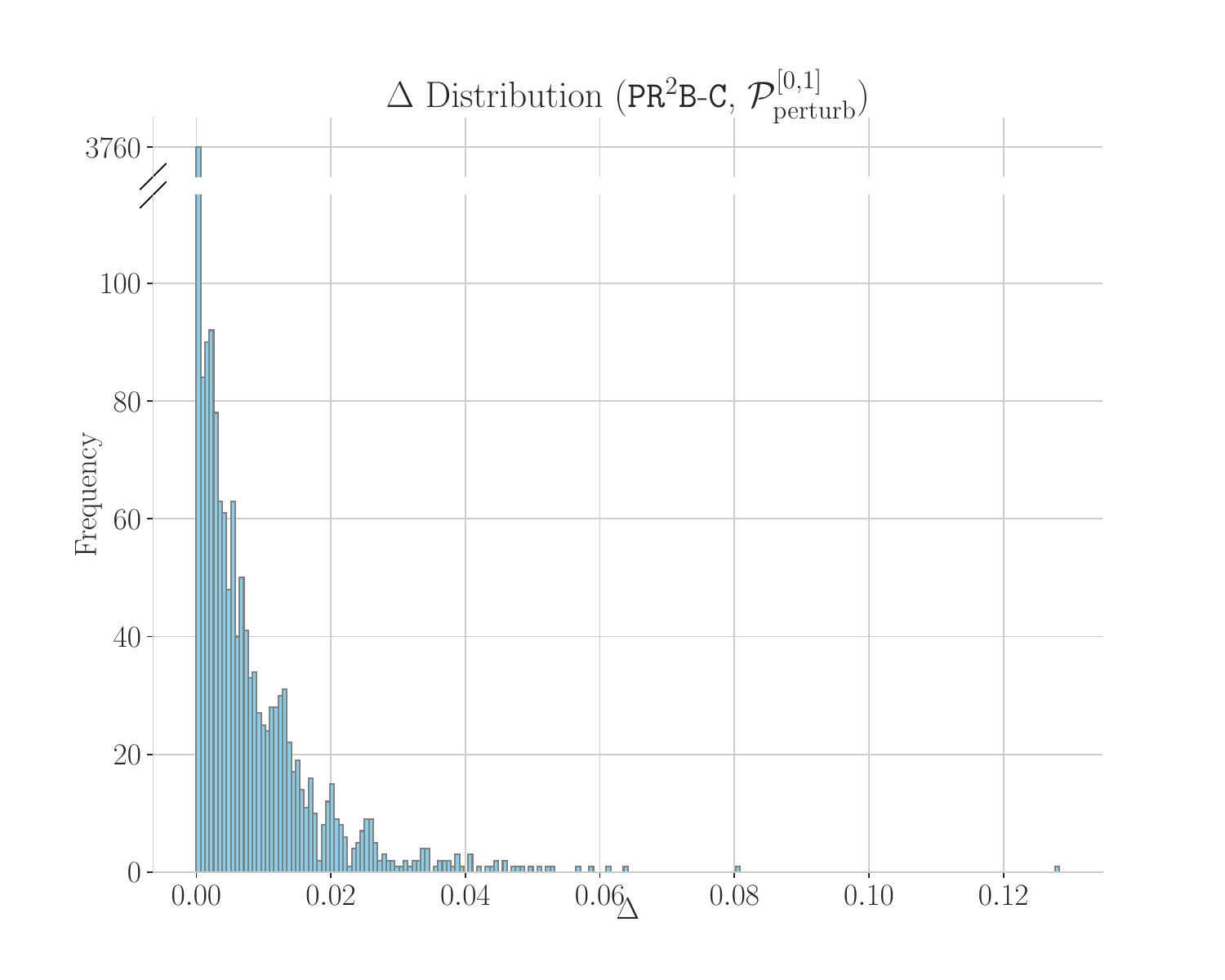}
    \end{minipage}
    \hspace{-5mm}
    \begin{minipage}[b]{0.33\textwidth}
        \centering
        \includegraphics[width=1.1\textwidth]{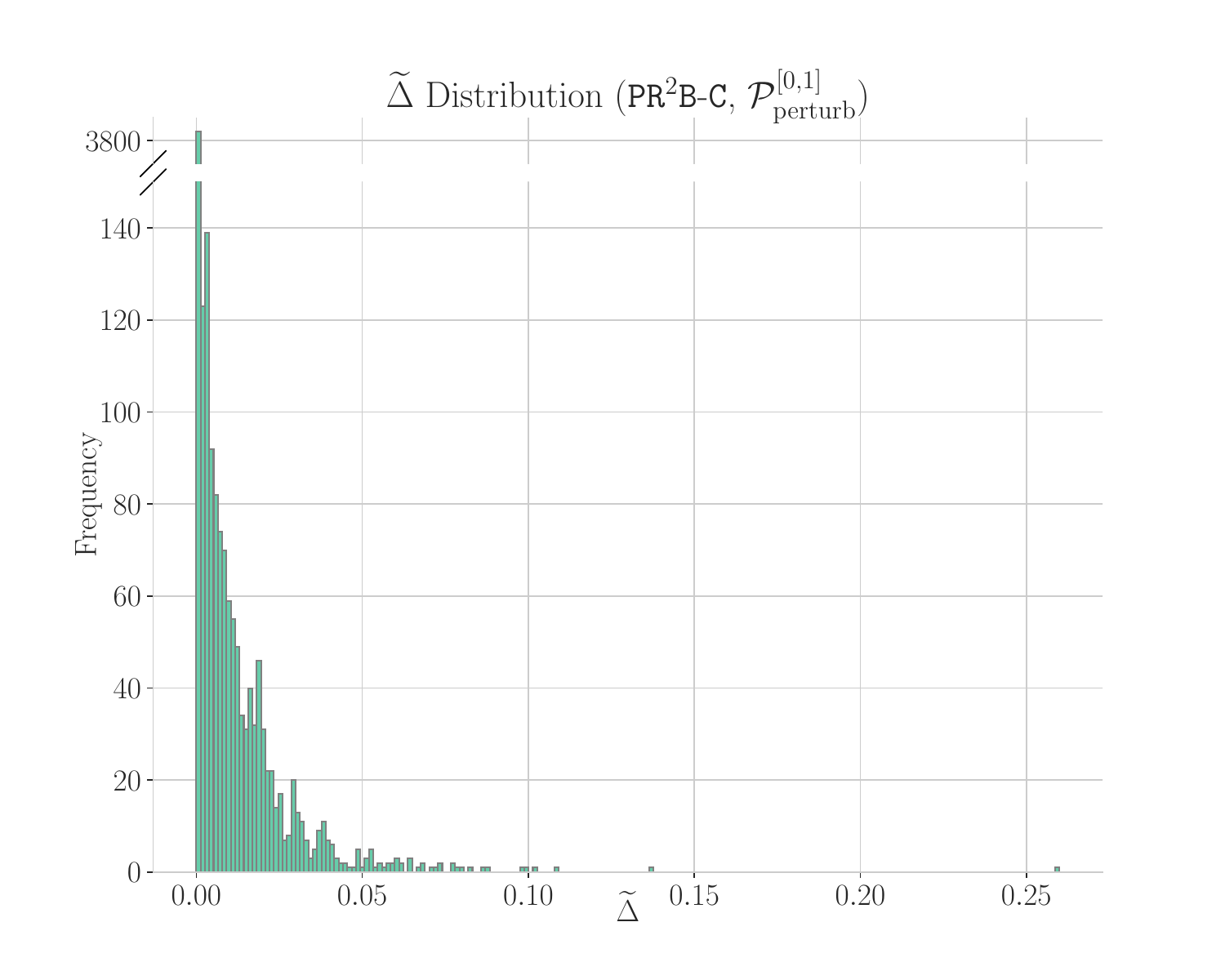}   
    \end{minipage}
    \hspace{-5mm}
    \begin{minipage}[b]{0.33\textwidth}
        \centering
        \includegraphics[width=1.15\textwidth]{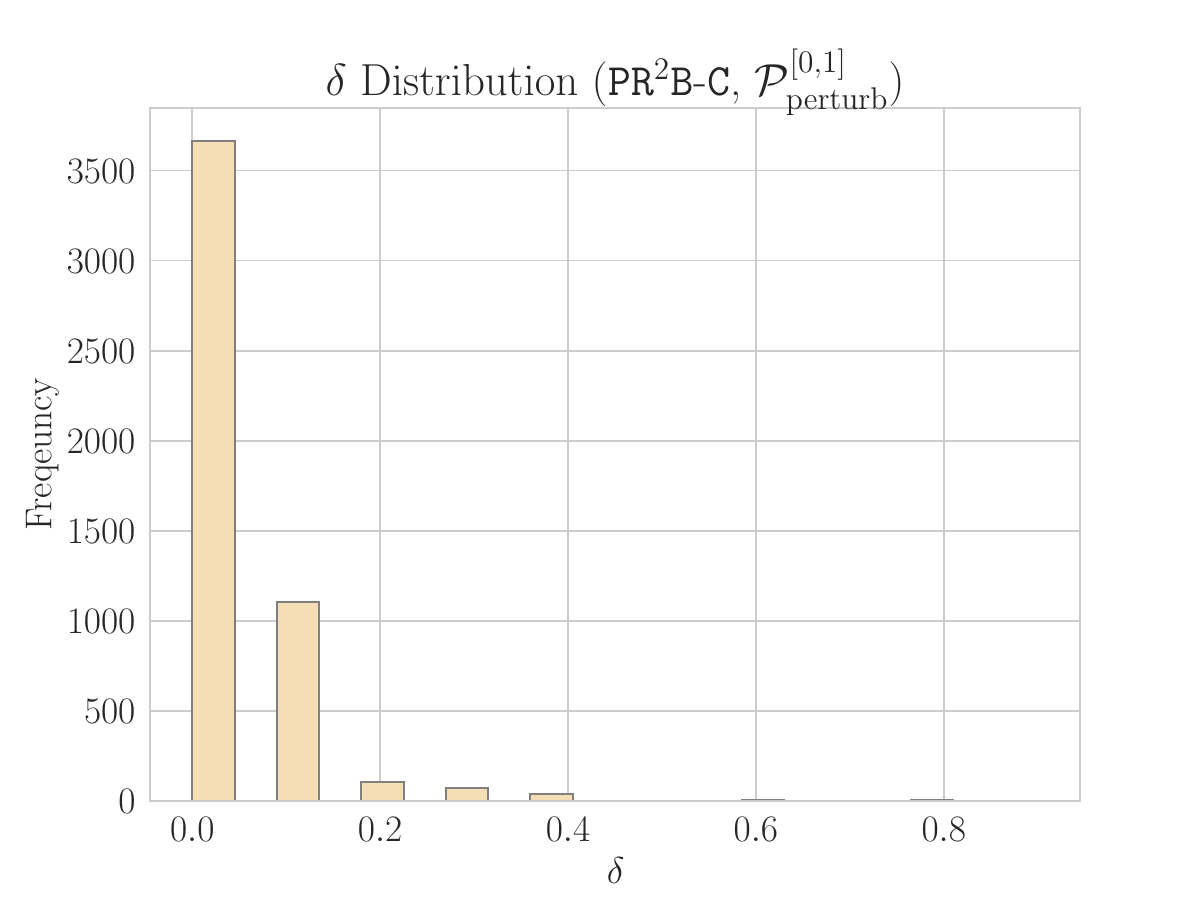}   
    \end{minipage}

           \begin{minipage}[b]{0.33\textwidth}
        \centering
        \includegraphics[width=1.1\textwidth]{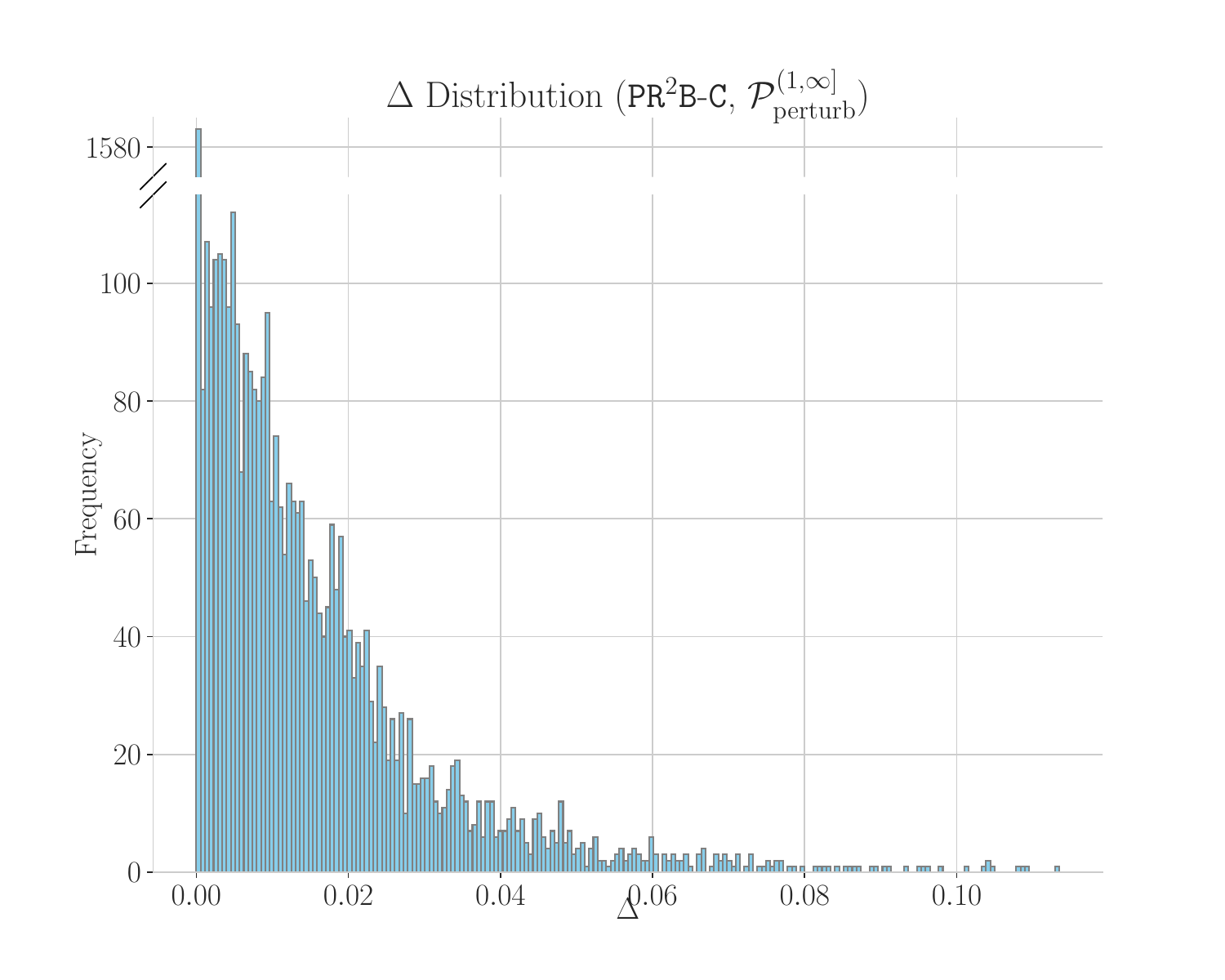}
    \end{minipage}
    \hspace{-5mm}
    \begin{minipage}[b]{0.33\textwidth}
        \centering
        \includegraphics[width=1.1\textwidth]{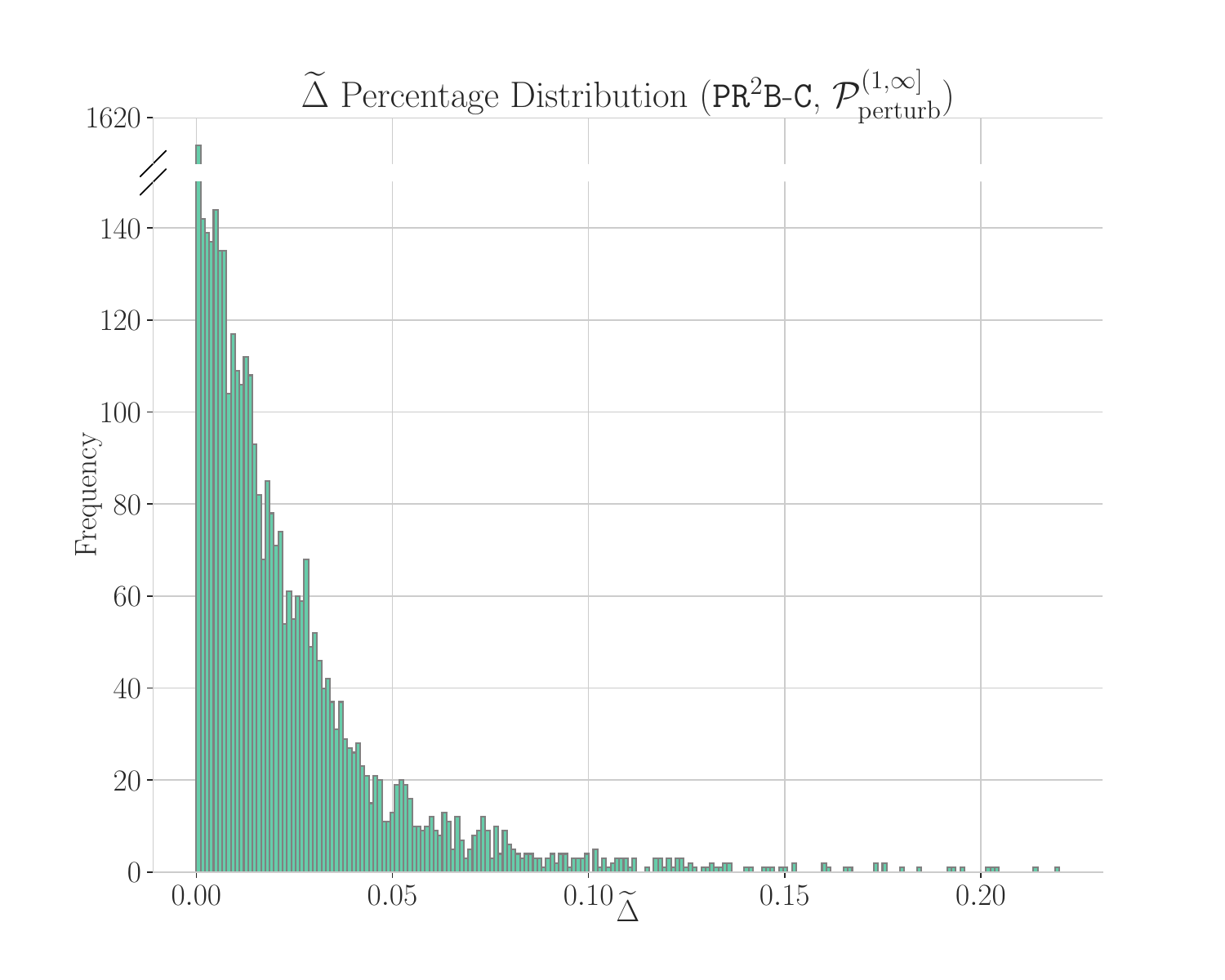}   
    \end{minipage}
    \hspace{-5mm}
    \begin{minipage}[b]{0.33\textwidth}
        \centering
        \includegraphics[width=1.15\textwidth]{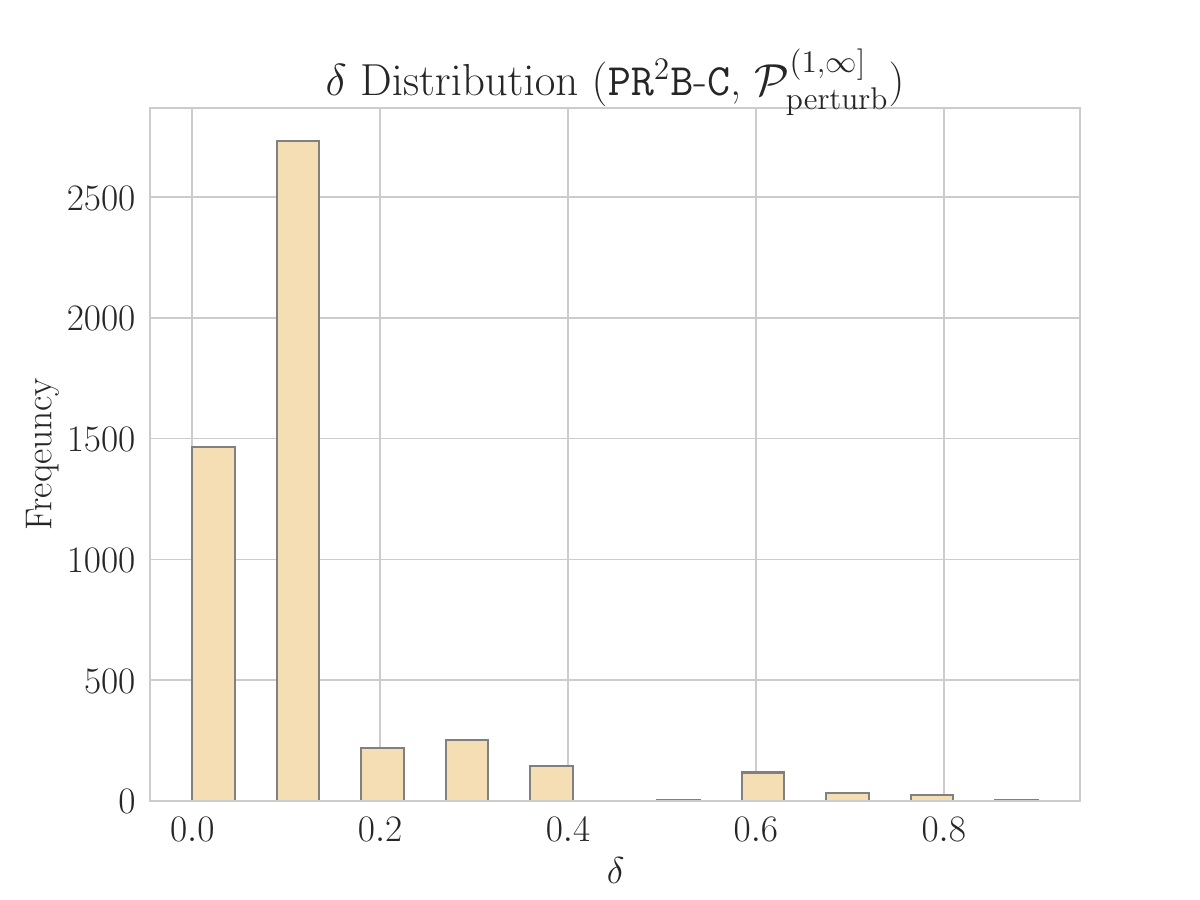}   
    \end{minipage}
    
           \begin{minipage}[b]{0.33\textwidth}
        \centering
        \includegraphics[width=1.1\textwidth]{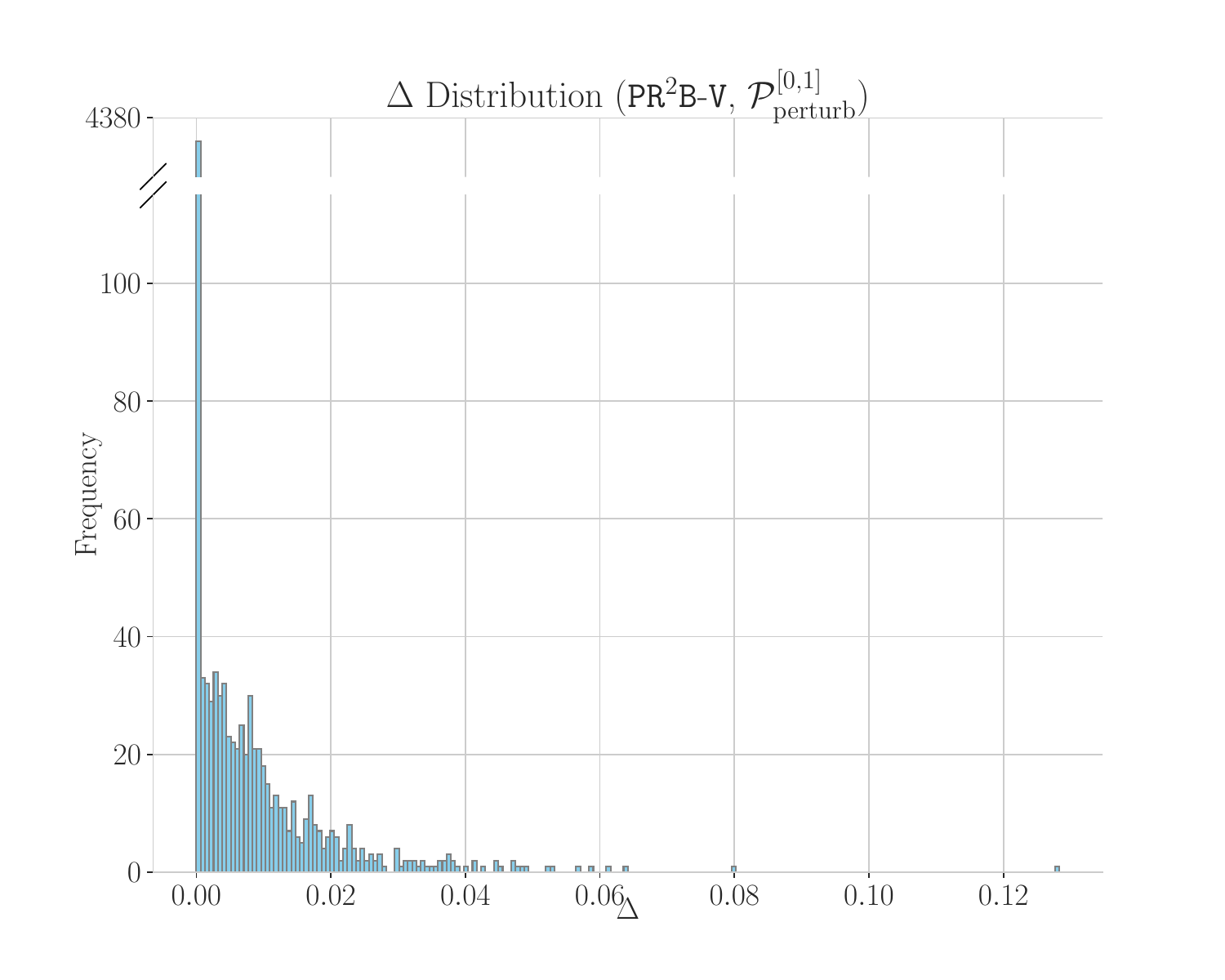}
    \end{minipage}
    \hspace{-5mm}
    \begin{minipage}[b]{0.33\textwidth}
        \centering
        \includegraphics[width=1.1\textwidth]{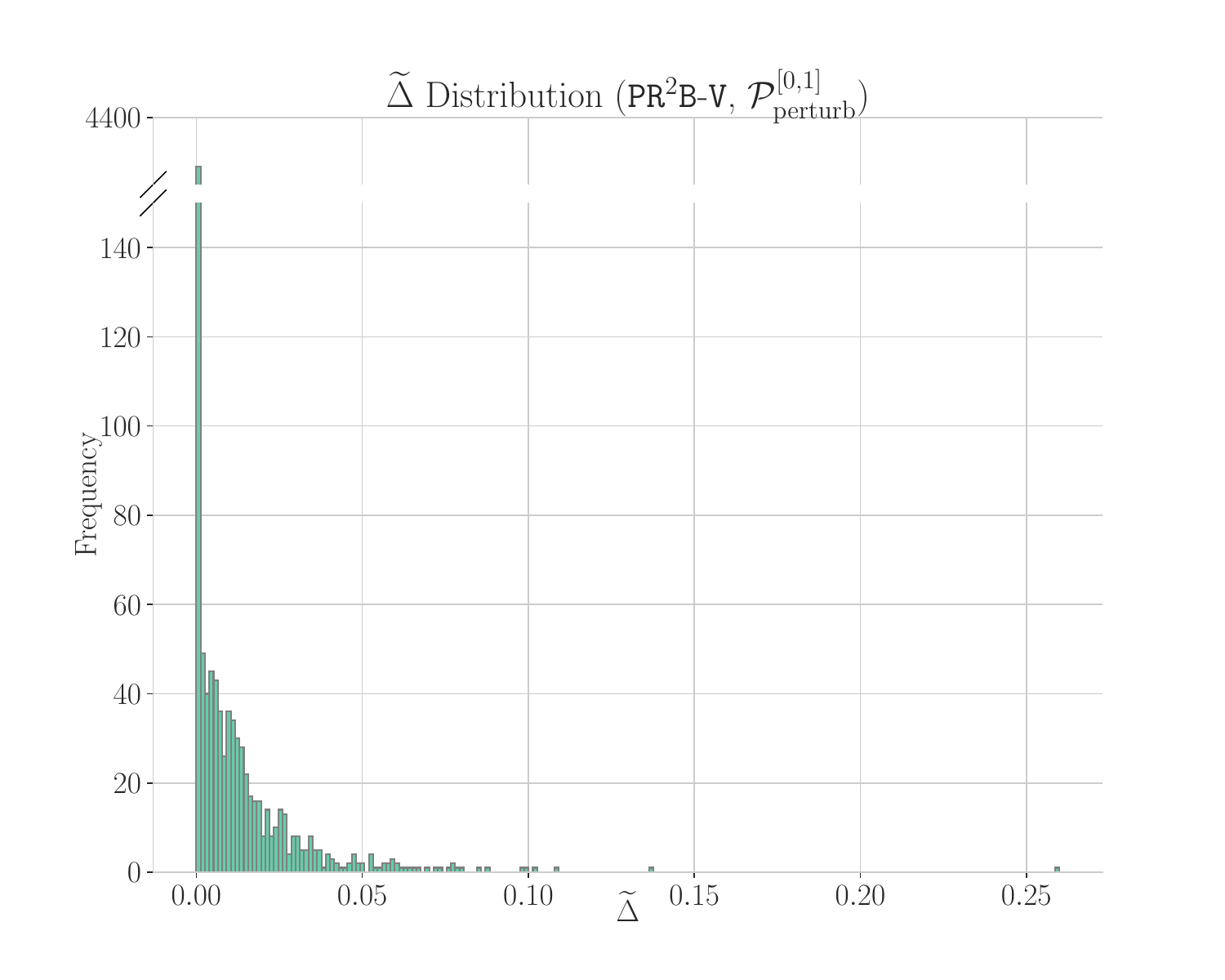}   
    \end{minipage}
    \hspace{-5mm}
    \begin{minipage}[b]{0.33\textwidth}
        \centering
        \includegraphics[width=1.15\textwidth]{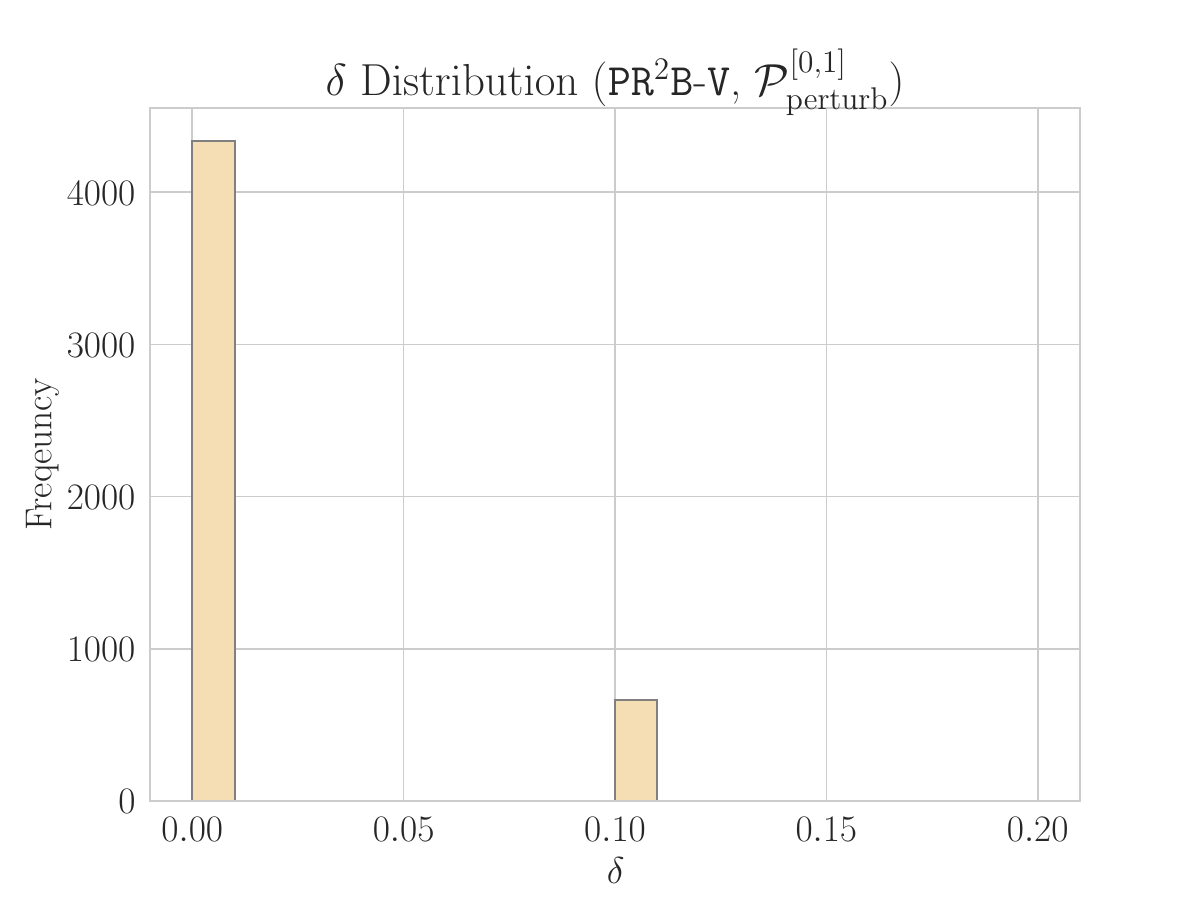}   
    \end{minipage}

           \begin{minipage}[b]{0.33\textwidth}
        \centering
        \includegraphics[width=1.1\textwidth]{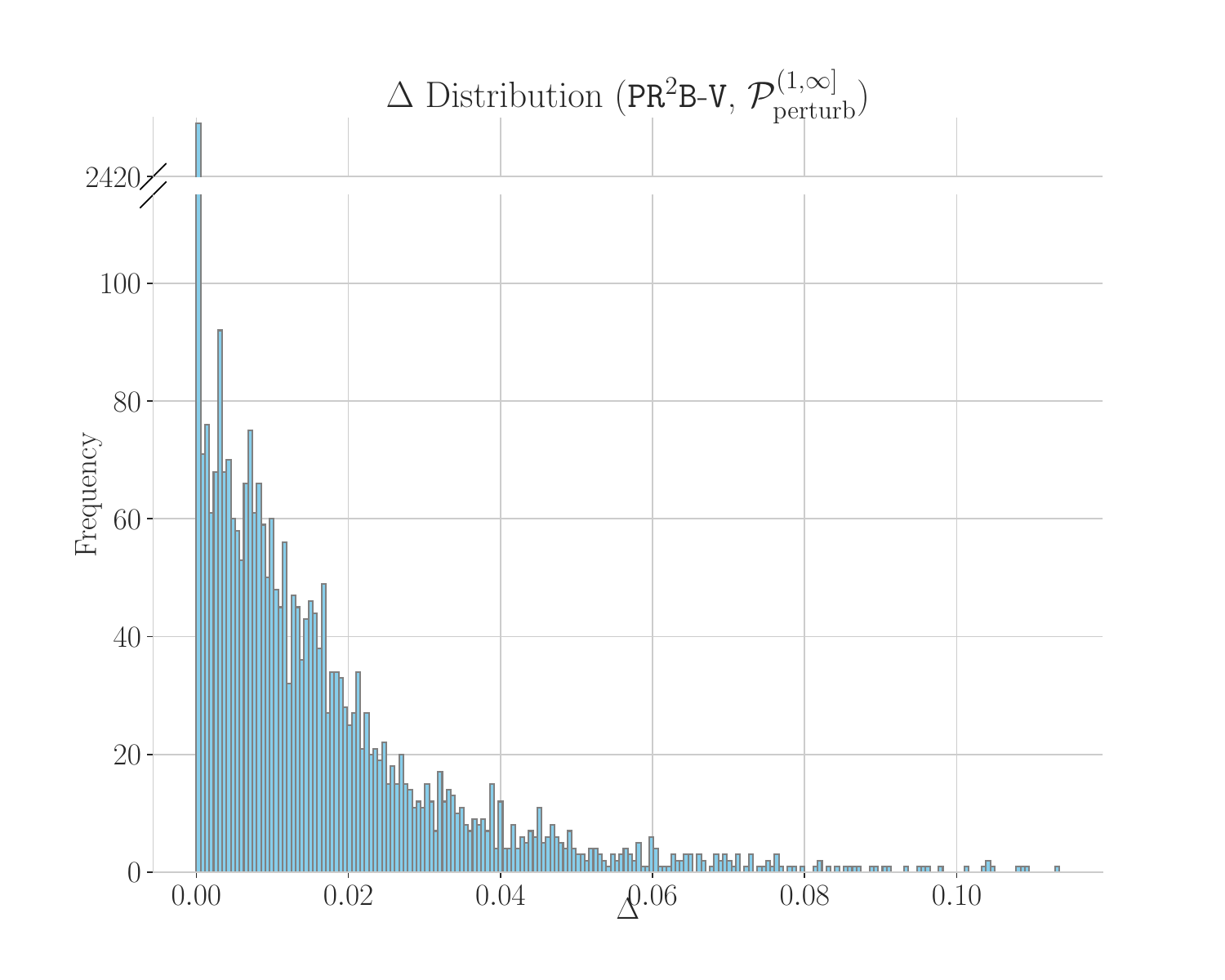}
    \end{minipage}
    \hspace{-5mm}
    \begin{minipage}[b]{0.33\textwidth}
        \centering
        \includegraphics[width=1.1\textwidth]{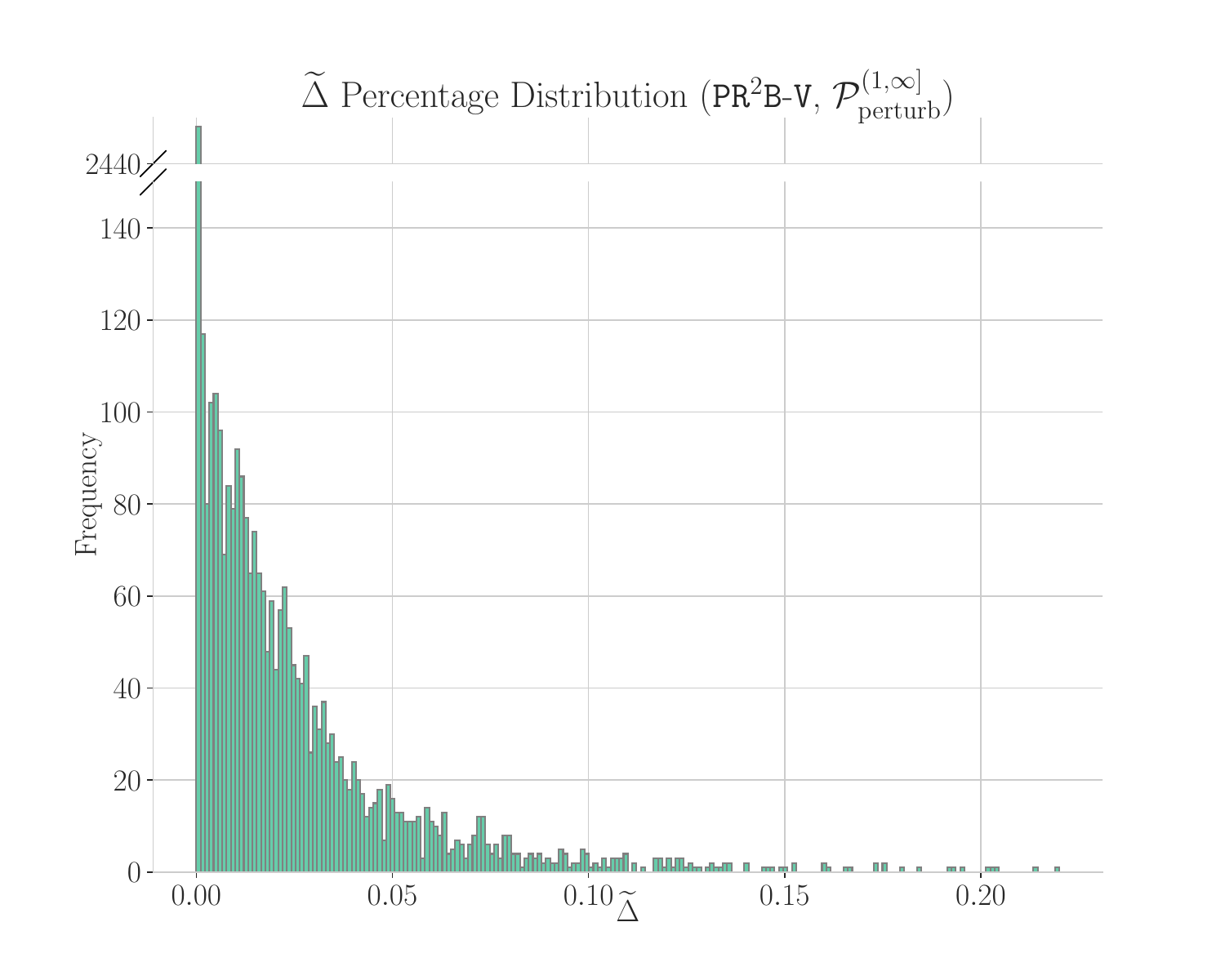}   
    \end{minipage}
    \hspace{-5mm}
    \begin{minipage}[b]{0.33\textwidth}
        \centering
        \includegraphics[width=1.1\textwidth]{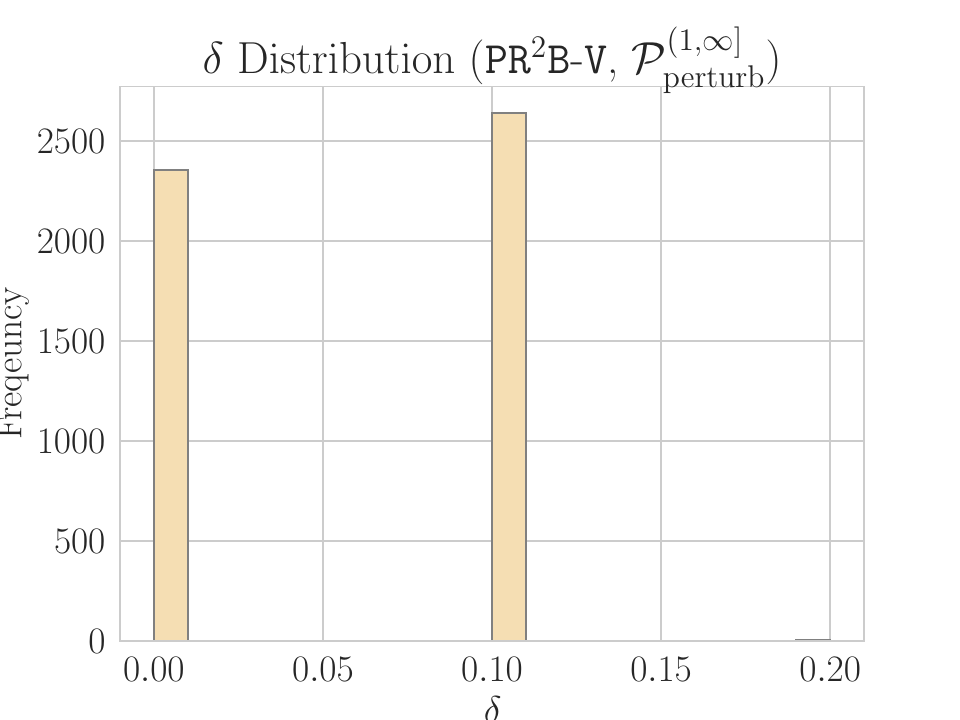}   
    \end{minipage}
    \caption{Results for experiments on robustness of the learned assortment. The first column is the distribution for the absolute improvement $\Delta$. The second column is the distribution of the relative improvement $\widetilde{\Delta}$, and the last column is the distribution of the optimal robust parameter. The first two rows are for P$\mathrm{R}^2$B-C, and the last two rows are for P$\mathrm{R}^2$B-V.}\label{fig: robustness exp}
\end{figure}

\paragraph{Evaluation metrics and results.}
Properly evaluating the robustness of a learned model is a non-trivial task. 
To test the robustness of the assortments learned by our algorithms, we look into three statistics. 

The first statistics is the distribution of the improvement brought by using robust assortment learning approaches.
More specifically, we define the random variable $\Delta$ with  instantiation on $\widetilde{\bp}_0$ as
\begin{align}
    \Delta(\widetilde{\bp}_0):= \sup_{\rho\in \mathscr{R}}R(\widehat{S}_{\rho};\widetilde{\bv}_0) - R(\widehat{S}_{0};\widetilde{\bv}_0),\quad \forall \widetilde{\bp}_0\in\cP_{\mathrm{perturb}},
\end{align}
and we plot the histogram of the distribution of $\Delta$.
This visualizes the distribution of largest possible revenue the robust approach could bring compared to non-robust method. 
For completeness, we consider the second statistics to be the distribution of another variable $\widetilde{\Delta}$, defined as the relative improvement,
\begin{align}
    \widetilde{\Delta}(\widetilde{\bp}_0):= \frac{\Delta(\widetilde{\bp}_0)}{R(\widehat{S}_{0};\widetilde{\bv}_0)}.
\end{align}
The third statistics is defined as another random variable $\varrho$ with instantiation on $\widetilde{\bp}_0$ as
\begin{align}
    \varrho(\widetilde{\bp}_0):= \argsup_{\rho\in \mathscr{R}}R(\widehat{S}_{\rho};\widetilde{\bv}_0) - R(\widehat{S}_{0};\widetilde{\bv}_0),\quad \forall \widetilde{\bp}_0\in\cP_{\mathrm{perturb}}.
\end{align}
This informs us under random choice pattern shifts what is the distribution of an optimal robust parameter. 
By looking into the distribution of $\varrho$ under P$\mathrm{R}^2$B-C and P$\mathrm{R}^2$B-V, we also gain insights into the difference between the formulation of Example~\ref{exp: jin} and Example~\ref{exp: new}.

The results are shown in Figure~\ref{fig: robustness exp}. In general, the robust assortments is able to bring significant revenue improvement compared to non-robust assortment. 
Under extreme circumstances, the robust assortment can bring up to 25\% revenue improvement.
For larger choice distribution shifts, the advantage of the robust assortment over the non-robust assortment is much more significant.
This can be observed by comparing the first (third) row to the second (fourth) row. 
In the meanwhile, for larger choice distribution shifts, one can observe that the optimal choice of the robust set size parameter also becomes larger, meaning that the larger distributional shift necessitates a larger robust set size.

\subsection{Experiment 3: Influence of Cardinality Constraints.}\label{subsec: exp3}

In the last line of experiments, we aim to demonstrate the theoretical predictions regarding the influence of the cardinality constraints on the suboptimality of the learned assortments. 
As predicted by the theory, given the same effective sample size $\min_{i\in S^{\star}}n_i$, the learned robust assortment suffer larger suboptimality if the robust assortment optimization problem involves a larger cardinality constraint $K$, see Theorem~\ref{thm:_algorithm_design_1} and \ref{thm:_algorithm_design_2}.
Also, for both Examples~\ref{exp: jin} and \ref{exp: new}, the uniform revenue case suffers a lower order dependence of $K$, i.e., $\widetilde{\cO}(K^{1/2})$, compared to the general non-uniform revenue case  of $\cO(K)$.
We aim to verify this by experiments.

\paragraph{Construction of the instances and data.} 
We construct the following instances of the nominal choice model and the observational data following  the same protocol as in the proof of statistical lower bounds, see Appendices~\ref{subsec:_proof_lower_bound_1} and \ref{subsec:_proof_lower_bound_new_1}.
We consider $N=50$, $K\in\{1,2,\cdots,9,10\}$, and $n_{\mathrm{effect}}:=\min_{i\in S^{\star}}n_i=1000$.
For the uniform revenue case and each $K\in\{1,2,\cdots,9,10\}$, we define the nominal choice model by 
\begin{align}
    v_j=\frac{1}{K}+\frac{1}{\sqrt{n_{\mathrm{effect}}\cdot K}},\quad \forall 1\leq j\leq K,\quad \text{and}\quad v_i = \frac{1}{K},\quad \forall K+1\leq i\leq N,
\end{align}
and in this case the revenues are uniform, i.e., $r_i = 1$ for any $1\leq i\leq N$. 
For the general non-uniform revenue case and each $K\in\{1,2,\cdots,9,10\}$, we define the nominal choice model by 
\begin{align}
    v_j=\frac{1}{K}+\frac{1}{\sqrt{n_{\mathrm{effect}}}},\quad \forall 1\leq j\leq K;\quad  v_i = \frac{1}{K},\quad \forall K+1\leq i\leq 4K;\quad v_k=1,\quad \forall 4K+1\leq k\leq N.
\end{align}
Then the non-uniform revenues are defined as 
\begin{align}
    r_j = 1,\quad \forall 1\leq j\leq 4K;\quad r_i = 0, \quad\forall 4K+1\leq i\leq N.
\end{align}
Finally, the offline dataset is defined as following.
We let $n_{\mathrm{total}} = 4Kn_{\mathrm{effect}}$ and for each $1\leq k\leq n_{\mathrm{total}}$,
    \begin{align}
        S_k :=\left\{\left\lceil\frac{k}{n_{\mathrm{effect}}}\right\rceil,4K+2,\cdots,5K\right\}.\label{eq: offline dataset main}
    \end{align}

\paragraph{Experiment results.}
We plot the suboptimality gap of the robust expected revenue of P$\mathrm{R}^2$B(-C/V) with respect to cardinality sizes $K$ separately for uniform and non-uniform revenue cases.
See Figures~\ref{fig: cardinality constant size} and \ref{fig: cardinality varying size}.
As we can see from the relation between the logarithm of suboptimality gap and the logarithm of size constraints, it matches our theoretical predictions of $\widetilde{\cO}(K^{1/2})$ dependency for the special case of uniform revenue and $\widetilde{\cO}(K)$ dependency for the general case of non-uniform revenue.
The only slight exception is for the varying robust set size case (Example~\ref{exp: new}, Figure~\ref{fig: cardinality varying size}) where the dependency deviates a bit from prediction when the size constraint $K=1$.
Beyond $K=1$, the result well matches our theory.

Finally, we also vary the robust set parameters and plot the suboptimality gap to cardinality size $K$ for different parameters.
Specifically, for Example~\ref{exp: jin}, we consider $\rho \in \{0.1, 0.15, 0.2, 0.25, 0.3, 0.35, 0.4, 0.45, 0.5\}$, and for Example~\ref{exp: new}, we consider $\rho_0\in\{0.005, 0.01, 0.015, 0.02, 0.025\}$. 
The results are provided in Figures~\ref{fig: cardinality constant size varying rho} and \ref{fig: cardinality varying size varying rho}. 
It further confirms our theory that for both two examples the uniform revenue setup enjoys a lower order of suboptimality gap than the non-uniform revenue setup for different robust set parameters.

\begin{figure}[!t]
   
       \begin{minipage}[b]{0.49\textwidth}
        \centering
        \includegraphics[width=0.95\textwidth]{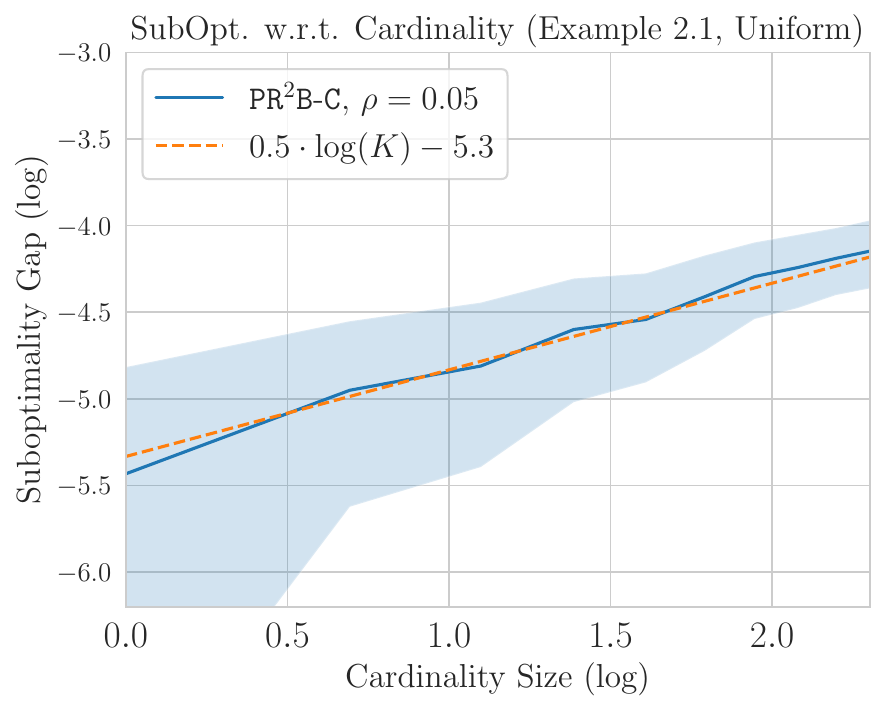}
    \end{minipage}
    \hspace{3mm}
    \begin{minipage}[b]{0.49\textwidth}
        \centering
        \includegraphics[width=0.95\textwidth]{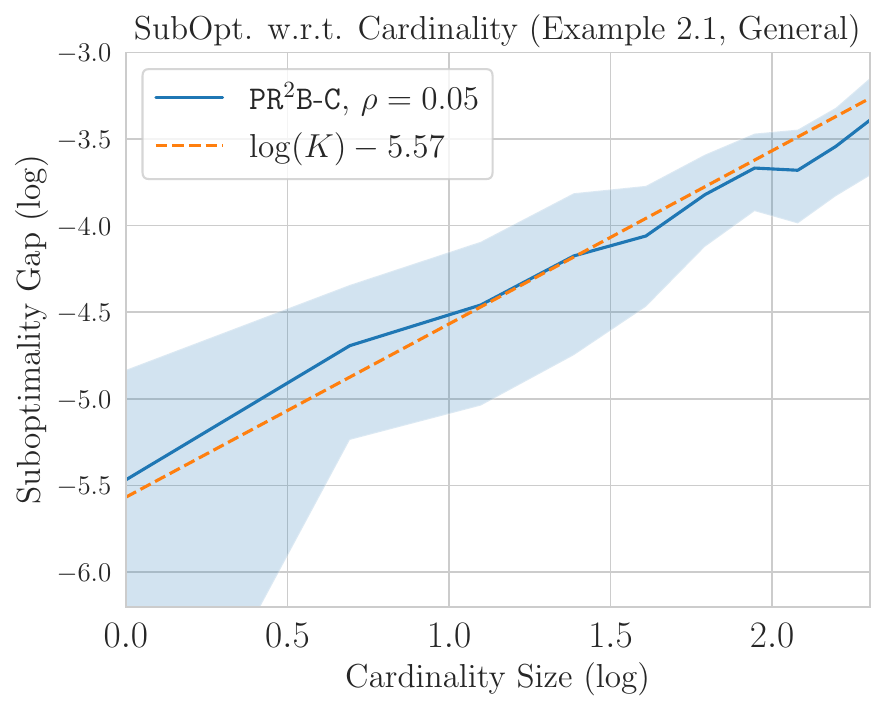}
    \end{minipage}
    \caption{Suboptimality gap of P$\mathrm{R}^2$B-C (Algorithm~\ref{alg: jin}) with respect to the cardinality constraint. All the parameter configurations are averaged over $25$ independent runs.}\label{fig: cardinality constant size}
\end{figure}

\begin{figure}[!t]
   
       \begin{minipage}[b]{0.49\textwidth}
        \centering
        \includegraphics[width=0.95\textwidth]{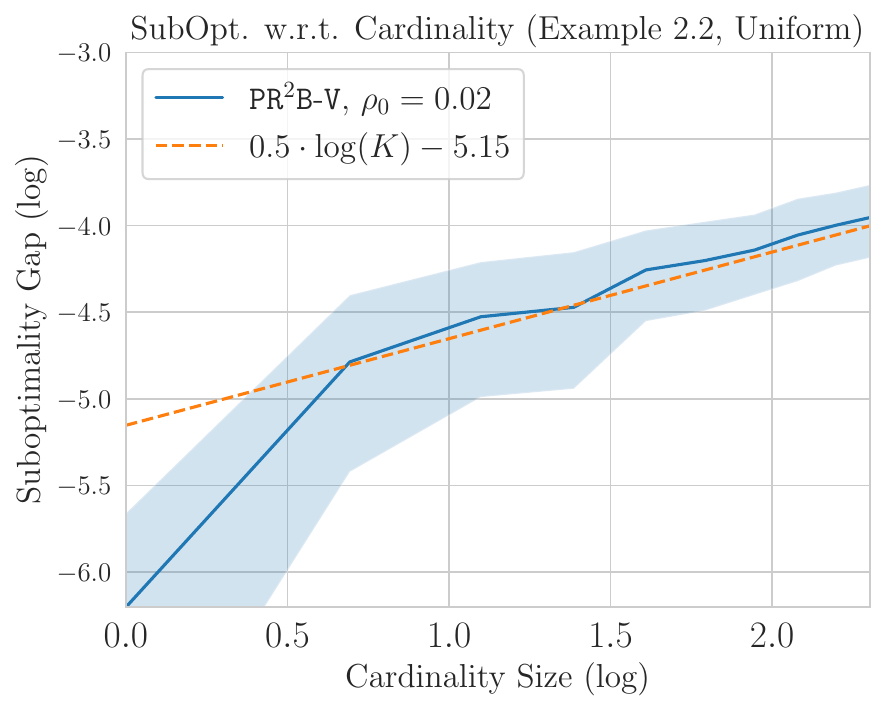}
    \end{minipage}
    \hspace{3mm}
    \begin{minipage}[b]{0.49\textwidth}
        \centering
        \includegraphics[width=0.95\textwidth]{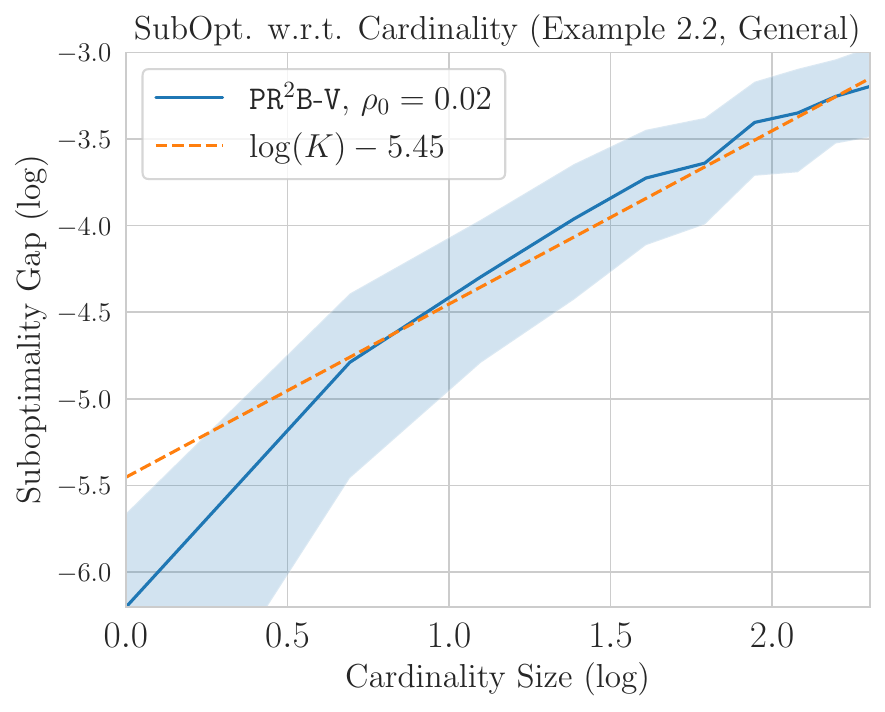}
    \end{minipage}
    \caption{Suboptimality gap of P$\mathrm{R}^2$B-V (Algorithm~\ref{alg: new}) with respect to the cardinality constraint. All the parameter configurations are averaged over $25$ independent runs.}\label{fig: cardinality varying size}
\end{figure}

\begin{figure}[!t]
   
       \begin{minipage}[b]{0.49\textwidth}
        \centering
        \includegraphics[width=1\textwidth]{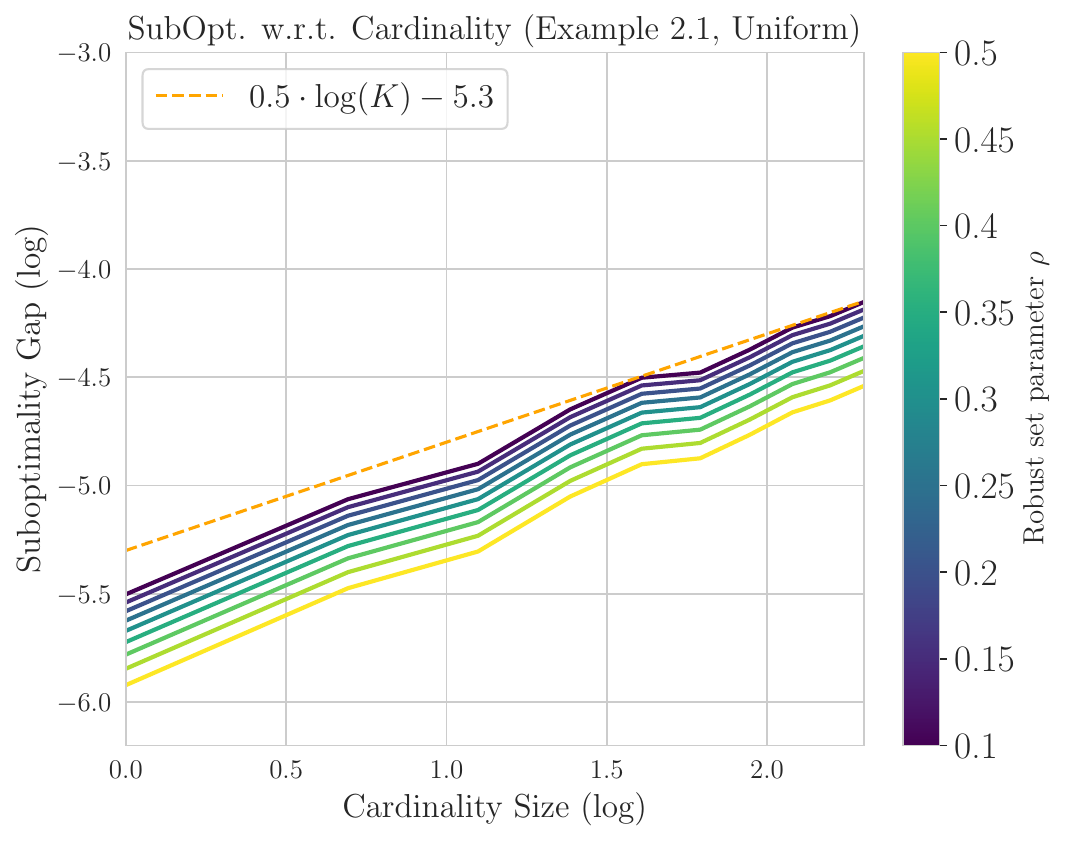}
    \end{minipage}
    \hspace{3mm}
    \begin{minipage}[b]{0.49\textwidth}
        \centering
        \includegraphics[width=1\textwidth]{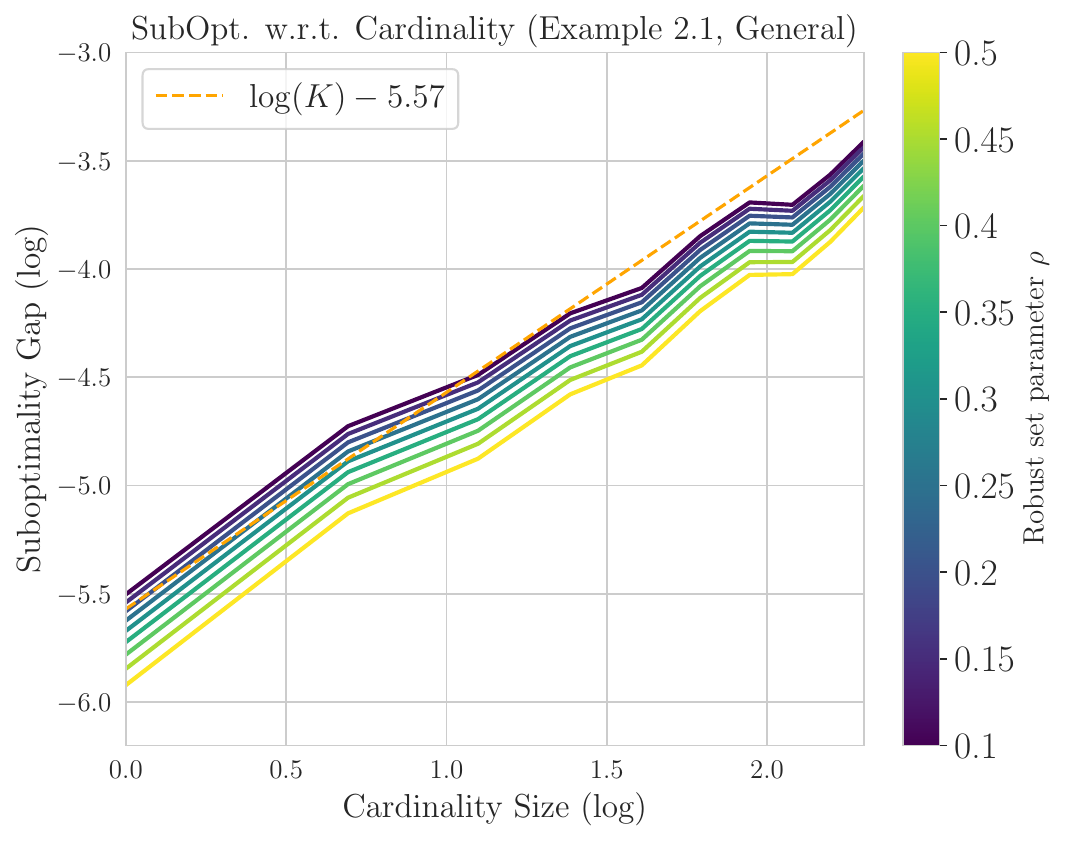}
    \end{minipage}
    \caption{Suboptimality gap of P$\mathrm{R}^2$B-C (Algorithm~\ref{alg: jin}) with respect to the cardinality constraint for different robust set size $\rho$. All the parameter configurations are averaged over $25$ independent runs.}\label{fig: cardinality constant size varying rho}
\end{figure}

\begin{figure}[!t]
   
       \begin{minipage}[b]{0.49\textwidth}
        \centering
        \includegraphics[width=1\textwidth]{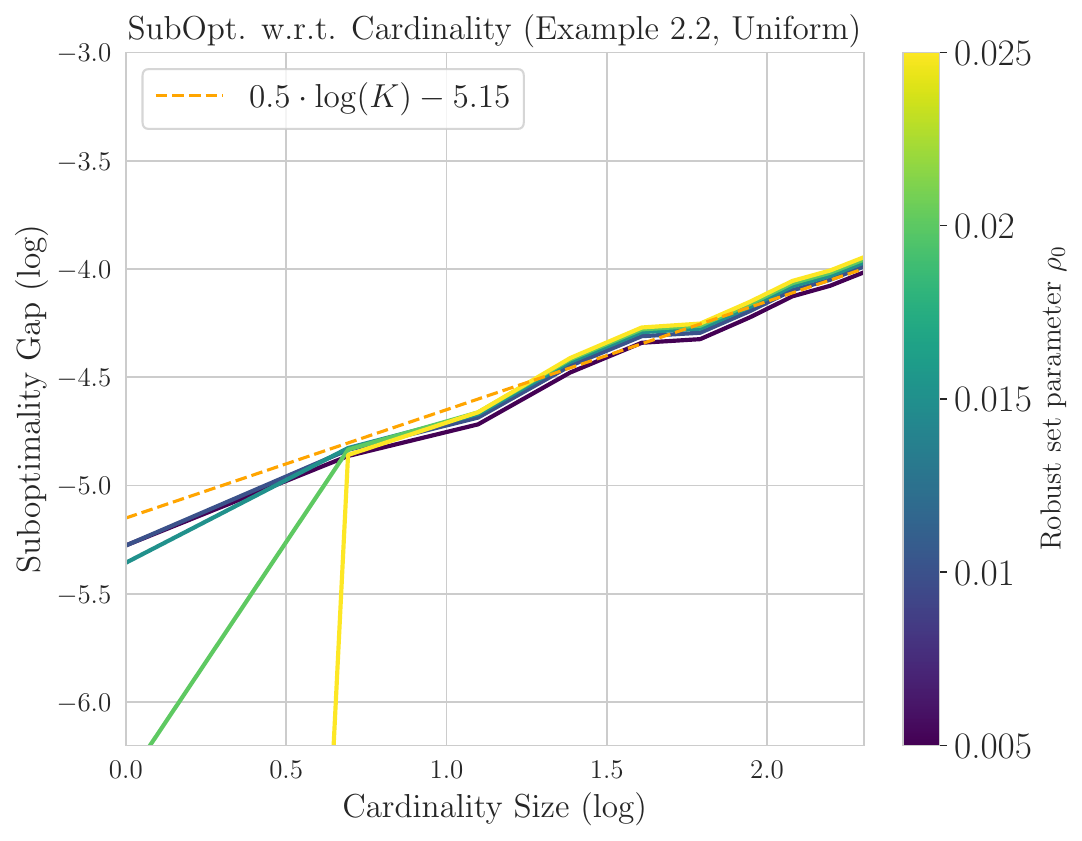}
    \end{minipage}
    \hspace{3mm}
    \begin{minipage}[b]{0.49\textwidth}
        \centering
        \includegraphics[width=1\textwidth]{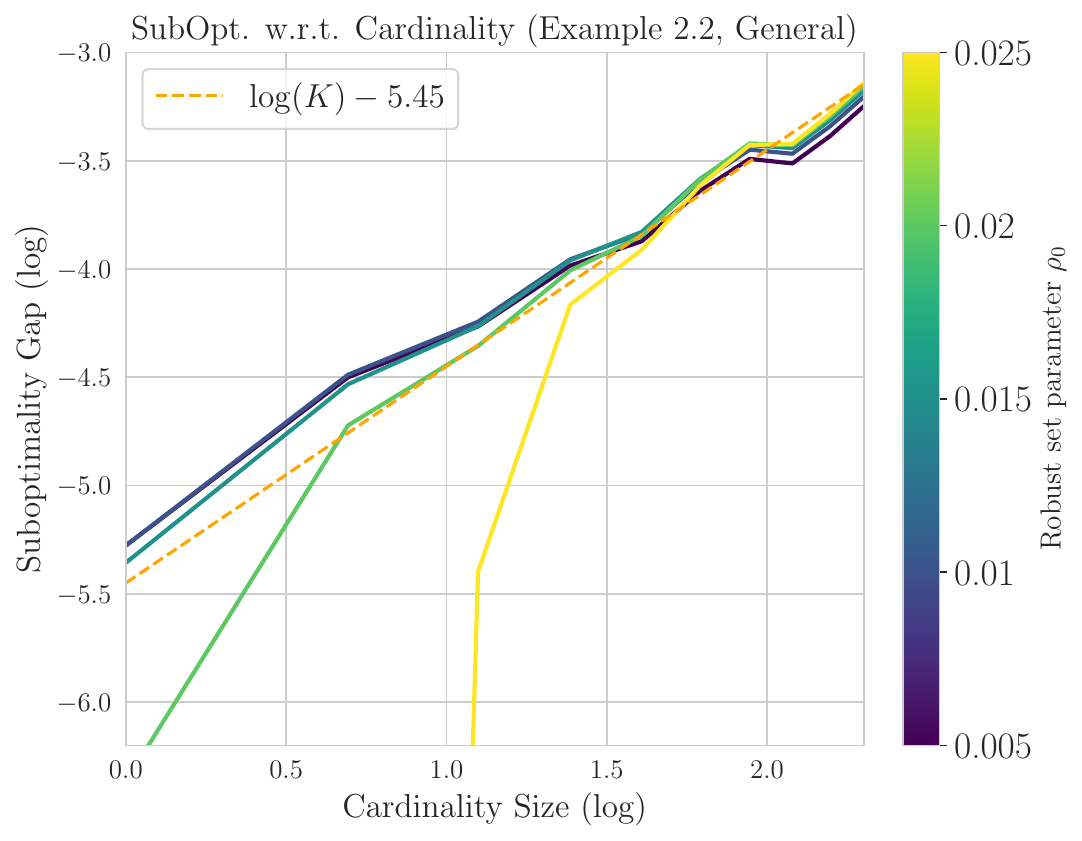}
    \end{minipage}
    \caption{Suboptimality gap of P$\mathrm{R}^2$B-V (Algorithm~\ref{alg: new}) with respect to the cardinality constraint for different robust set parameter $\rho_0$. All the parameter configurations are averaged over $25$ independent runs.}\label{fig: cardinality varying size varying rho}
\end{figure}

\section{Conclusions}

In this work, we addressed the limitations of traditional data-driven assortment optimization methods, which often falter under shifting customer preferences and model misspecification. To mitigate these challenges, we introduced a robust framework that accounts for distributional uncertainty in customer choice behavior, aiming to safeguard against worst-case revenue outcomes. We demonstrated that robust assortment planning remains computationally tractable even when incorporating uncertainty, and we extended this framework to the data-driven setting by developing statistically optimal algorithms that balance robustness with sample efficiency. Our theoretical contributions include tight bounds on sample complexity and the introduction of ``robust item-wise coverage'' as a key concept for enabling efficient learning. These results underscore the importance of designing algorithms that generalize reliably under uncertainty, offering both practical tools and theoretical insights for data-driven robust assortment optimization.

\section*{Acknowledgment}
Jose Blanchet gratefully acknowledges supports from ONR under award N00014-24-1-2655, and the National Science Foundation (NSF) under grants 2312204 and 2403007. 
Zhengyuan Zhou gratefully acknowledges supports from ONR under award N00014-24-1-2672 and NSF 2312205.

\bibliography{reference.bib}

\bibliographystyle{ims}

\newpage

\appendix

\section{Algorithms and Proofs for Solving the Optimal Robust Assortment with Known Nominal Model}

\subsection{Proof of Proposition~\ref{prop:_optimal_set_unconstrained}}\label{subsec:_proof_optimal_set_unconstrained}

\begin{proof}[Proof of Proposition~\ref{prop:_optimal_set_unconstrained}]
    Proposition~\ref{prop:_optimal_set_unconstrained} hold for Example~\ref{exp: jin} due to Theorem 2 of \cite{jin2022distributionally}. 
    We now prove this conclusion for Example~\ref{exp: new}.
    By the dual representation (Proposition~\ref{prop:_dual}), any assortment $S\subseteq[N]$ is optimal if and only if there exists $\lambda \geq 0$ such that 
    \begin{align}
        &-\lambda\cdot\log\left(\frac{\sum_{i\in S_+}v_i\cdot\exp(-r_i/\lambda)}{\sum_{i\in S_+}v_i - (1-e^{-\rho_0})\cdot \sum_{i\in[N]_+}v_i }\right) \\
        &\qquad \geq \sup_{S'\subseteq[N],\lambda'\geq 0}\left\{-\lambda'\cdot\log\left(\frac{\sum_{i\in S'_+}v_i\cdot\exp(-r_i/\lambda')}{\sum_{i\in S'_+}v_i - (1-e^{-\rho_0})\cdot \sum_{i\in[N]_+}v_i }\right)\right\}:=R^{\star},
    \end{align}
    which is further equivalent to 
    \begin{align}
        \sum_{i\in S}v_i\cdot \Big(\exp(-r_i/\lambda) - \exp(-R^{\star}/\lambda)\Big) \leq \exp(-R^{\star}/\lambda) - 1  - \exp(-R^{\star}/\lambda)\cdot (1-e^{-\rho_0})\cdot \sum_{i\in[N]_+}v_i.\label{eq: unconstrained equiva}
    \end{align}
    Consider the complete set assortment $S^{\star}:=\{1,\cdots,i^{\star}\}$ with $r_i\geq R^{\star}$ for all $i\in S^{\star}$. 
    We can find that such $S^{\star}$ is a minimizer of the left hand side of \eqref{eq: unconstrained equiva} regardless of specific $\lambda\geq 0$.
    Therefore, such assortment $S^{\star}$ must be an optimal robust assortment, because otherwise, for any assortment $S$ and any dual $\lambda\geq 0$, 
     \begin{align}
        \sum_{i\in S}v_i\cdot \Big(\exp(-r_i/\lambda) - \exp(-R^{\star}/\lambda)\Big) & \geq \sum_{i\in S^{\star}}v_i\cdot \Big(\exp(-r_i/\lambda) - \exp(-R^{\star}/\lambda)\Big) \\
&> \exp(-R^{\star}/\lambda) - 1  - \exp(-R^{\star}/\lambda)\cdot (1-e^{-\rho_0})\cdot \sum_{i\in[N]_+}v_i,\label{eq: unconstrained equiva contra}
    \end{align}
    which contradicts the feasibility of the  optimality condition. 
    This proves Proposition~\ref{prop:_dual}.
\end{proof}

\subsection{Proof of Proposition~\ref{prop:_optimal_set_uniform}}\label{subsec:_proof_optimal_set_uniform}

\begin{proof}[Proof of Proposition~\ref{prop:_optimal_set_uniform}]
    We first prove the conclusion for Example~\ref{exp: jin}. 
    By the dual representation (Proposition~\ref{prop:_dual}), for any assortment $S\subseteq[N]$ with $|S|\leq K$,
    \begin{align}
        R_{\rho}(S;\bv)&=\sup_{\lambda\geq 0} \left\{-\lambda\cdot\log\left(\frac{1+ \exp(-r_{\max}/\lambda)\cdot\sum_{i\in S}v_i}{\sum_{i\in S_+}v_i}\right) - \lambda\cdot \rho\right\} \\
        &=\sup_{\lambda\geq 0} \left\{-\lambda\cdot\log\left(\exp(-r_{\max}/\lambda) + \frac{1-\exp(-r_{\max}/\lambda)}{\sum_{i\in S_+}v_i}\right) - \lambda\cdot \rho\right\}.
    \end{align}
    Thus to maximize the robust expected revenue, it suffices to find $S$ such that $\sum_{i\in S} v_i$ is maximized, which is equivalent to find the assortment with top $K$ attraction parameter $v_i$. This proves for Example~\ref{exp: jin}. 

    Now we prove the conclusion for Example~\ref{exp: new}. 
    Similarly, by the dual representation, for any assortment $S\subseteq[N]$ with $|S|\leq K$, we have that 
    \begin{align}
        R_{\rho(\cdot;\cdot)}(S;\bv)&=\sup_{\lambda\geq 0}\left\{-\lambda\cdot\log\left(\frac{1 + \exp(-r_{\max}/\lambda)\cdot \sum_{i\in S}v_i}{\sum_{i\in S_+}v_i - (1-e^{-\rho_0})\cdot \sum_{i\in[N]_+}v_i }\right)\right\}\\
        &=\sup_{\lambda\geq 0}\left\{-\lambda\cdot\log\left(\exp(-r_{\max}/\lambda) + \frac{1-\exp(-r_{\max}/\lambda) \cdot(1-(1-e^{-\rho_0})\cdot \sum_{i\in[N]_+}v_i)}{\sum_{i\in S_+}v_i - (1-e^{-\rho_0})\cdot \sum_{i\in[N]_+}v_i }\right)\right\}.
    \end{align}
    under the condition that $\rho_0\leq \log(1+1/v_{\mathrm{tot}})$, we know that  $1-(1-e^{-\rho_0})\cdot \sum_{i\in[N]_+}v_i)>0$. 
    Thus the robust expected revenue is maximized when $\sum_{i\in S}$ is maximized, which is equivalent to find the assortment with top $K$ attraction parameter $v_i$. This proves for Example~\ref{exp: jin} and finishing the proof of Proposition~\ref{prop:_optimal_set_uniform}. 
\end{proof}

\subsection{Proof of Proposition~\ref{prop:_optimal_set_general}}\label{subsec:_proof_optimal_set_general}

\begin{proof}[Proof of Proposition~\ref{prop:_optimal_set_general}]
    The algorithm and analysis for the general case under Example~\ref{exp: jin} is presented by \cite{jin2022distributionally} in Section 3.3. 
    We refer the readers to their work for the details. 
    In the coming subsections, we give the algorithm design and analysis for Example~\ref{exp: new}. 
    These combined prove Proposition~\ref{prop:_optimal_set_general}.
\end{proof}

\subsubsection{Algorithm Design for Varying Robust Set Size (Example~\ref{exp: new})}\label{subsubsec: algorithm design}

Our algorithm design is based upon the geometric perspective for assortment optimization introduced in \cite{rusmevichientong2010dynamic}. 
We first introduce the basic observations for the problem and present the intuitions of the algorithm design, after which we present the full algorithms and the theoretical guarantees.

\paragraph{Intuitions.}
Recall that here the goal is to (approximately) solve the following optimization problem,
\begin{align}
    S^{\star}:=\argmax_{S\subseteq[N],|S|\leq K} R_{\rho(\cdot;\cdot)}(S;\bv),\quad R^{\star}:=\max_{S\subseteq[N],|S|\leq K} R_{\rho(\cdot;\cdot)}(S;\bv)=R_{\rho(\cdot;\cdot)}(S^{\star};\bv)\label{eq: computation new 1}
\end{align}
with the robust set size function $\rho(\cdot;\cdot)$ given in Example~\ref{exp: new}, the nominal parameter $\bv$ known, and a given constraint size $K$.
By the dual representation (Proposition~\ref{prop:_dual}), it suffices to equivalently solve the following, 
\begin{align}
    (S^{\star},\lambda^{\star}):=\argmax_{\substack{S\subseteq[N],|S|\leq K\\
    0\leq \lambda\leq B(S,\bv)}} H_{\rho(\cdot;\cdot)}(S,\lambda;\bv),
\end{align}
where the function $H$ is defined as 
\begin{align}
    H_{\rho(\cdot;\cdot)}(S,\lambda;\bv)&:=-\lambda\cdot\log\left(\frac{\sum_{i\in S_+}v_i\cdot\exp(-r_i/\lambda)}{\sum_{i\in S_+}v_i}\right) - \lambda\cdot \rho_0 +\lambda\cdot  \log\left(e^{\rho_0} - \frac{(e^{\rho_0} - 1)\cdot \sum_{i\in[N]_+}v_i}{\sum_{i\in S_+}v_i}\right)\\
    &=-\lambda\cdot\log\left(\frac{\sum_{i\in S_+}v_i\cdot\exp(-r_i/\lambda)}{\sum_{i\in S_+}v_i - (1-e^{-\rho_0})\cdot \sum_{i\in[N]_+}v_i }\right),
\end{align}
and the upper bound of the dual variable $\lambda$ is defined as 
\begin{align}
    B(S,\bv):=\frac{r_{\max}}{\log\left(\frac{\sum_{j\in S_+}v_j}{\sum_{j\in S_+}v_j - (1-e^{-\rho_0})\cdot \sum_{j\in [N]_+}v_j}\right)}.
\end{align}
To design an efficient computational algorithm for the above optimization problem, it is useful to introduce an auxiliary variable $t$ as following, 
\begin{align}
    R^{\star}&=\max_{S\subseteq[N],|S|\leq K, 0\leq \lambda\leq B(S,\bv)}\left\{-\lambda\cdot\log\left(\frac{\sum_{i\in S_+}v_i\cdot\exp(-r_i/\lambda)}{\sum_{i\in S_+}v_i - (1-e^{-\rho_0})\cdot \sum_{i\in[N]_+}v_i }\right)\right\} \\
    &=\max_{0\leq t\leq r_{\max}}\left\{\exists S\subseteq[N],|S|\leq K, \lambda\in [0, B(S,\bv)],-\lambda\cdot\log\left(\frac{\sum_{i\in S_+}v_i\cdot\exp(-r_i/\lambda)}{\sum_{i\in S_+}v_i - (1-e^{-\rho_0})\cdot \sum_{i\in[N]_+}v_i }\right)\geq t\right\} \\
    &=\max_{0\leq t\leq r_{\max}}\left\{\exists S\subseteq[N],|S|\leq K,  \lambda\in [0, B(S,\bv)],\sum_{i\in S_+}v_i \left(\exp\big((t-r_i)/\lambda\big) - 1\right) \leq -(1-e^{-\rho_0})\!\!\sum_{i\in[N]_+}v_i \right\} \\
    &=\max_{0\leq t\leq r_{\max}}\left\{\min_{S\subseteq[N],|S|\leq K, 0\leq \lambda\leq B(S,\bv)}\Bigg\{\sum_{i\in S_+}v_i \left(\exp\big((t-r_i)/\lambda\big) - 1\right)\Bigg\}\leq -(1-e^{-\rho_0}) \sum_{i\in[N]_+}v_i\right\}.
\end{align}
For notational simplicity, we define the following function,
\begin{align}
    G(t):=\min_{\substack{S\subseteq[N],|S|\leq K\\ 0\leq \lambda\leq B(S,\bv)}}\sum_{i\in S_+} g_i(t,\lambda),\quad g_i(t,\lambda):=v_i\cdot \left(\exp\big((t-r_i)/\lambda\big) - 1\right),
\end{align}
under which we can reformulate the original optimization problem \eqref{eq: computation new 1} via 
\begin{align}
    R^{\star}=\sup_{0\leq t\leq r_{\max}}\left\{G(t)\leq -(1-e^{-\rho_0})\cdot \sum_{i\in[N]_+}v_i \right\}.\label{eq: computation new 2} 
\end{align}

\begin{remark}
Notice that if we find the $\epsilon$-optimal solution $t_{\epsilon}$ of \eqref{eq: computation new 2} and  assortment $S(t_{\epsilon})$ associated with $t_{\epsilon}$ in the sense of the following,
\begin{align}
    \hspace{-5mm}\min_{0\leq \lambda \leq B(S(t_{\epsilon}),\bv)}\sum_{i\in S(t_{\epsilon})_+ }g_i(t_{\epsilon},\lambda)\leq -(1-e^{-\rho_0})\cdot \sum_{i\in[N]_+}v_i,\quad \lambda^{\star}(t_{\epsilon}):= \argmin_{0\leq \lambda \leq B(S(t_{\epsilon}),\bv)}\sum_{i\in S(t_{\epsilon})_+ }g_i(t_{\epsilon},\lambda), \label{eq: s epsilon}
\end{align}
then $S(t_{\epsilon})$ is an $\epsilon$-optimal robust assortment. 
This is because according to \eqref{eq: s epsilon}, it holds that 
\begin{align}
    \lambda^{\star}(t_{\epsilon})\cdot\log\left(\frac{\sum_{i\in S(t_{\epsilon})_+}v_i\cdot\exp(-r_i/\lambda^{\star}(t_{\epsilon}))}{\sum_{i\in S(t_{\epsilon})_+}v_i - (1-e^{-\rho_0})\cdot \sum_{i\in[N]_+}v_i }\right)\geq t_{\epsilon} \geq R^{\star}-\epsilon,
\end{align}
where the first inequality is by \eqref{eq: s epsilon}, and the second inequality is by the $\epsilon$-optimality of $t_{\epsilon}$.
Consequently, 
\begin{align}
    R_{\rho(\cdot;\cdot)}(S(t_{\epsilon});\bv)&=\max_{\lambda\geq 0}\left\{-\lambda\cdot\log\left(\frac{\sum_{i\in S(t_{\epsilon})_+}v_i\cdot\exp(-r_i/\lambda)}{\sum_{i\in S(t_{\epsilon})_+}v_i - (1-e^{-\rho_0})\cdot \sum_{i\in[N]_+}v_i }\right)\right\}\\
    &\geq \lambda^{\star}(t_{\epsilon})\cdot\log\left(\frac{\sum_{i\in S(t_{\epsilon})_+}v_i\cdot\exp(-r_i/\lambda^{\star}(t_{\epsilon}))}{\sum_{i\in S(t_{\epsilon})_+}v_i - (1-e^{-\rho_0})\cdot \sum_{i\in[N]_+}v_i }\right)\\
    &\geq R^{\star}-\epsilon.
\end{align}
Therefore, $S(t_{\epsilon})$ is an $\epsilon$-optimal robust assortment.
\end{remark}

In the following, we focus on how to solve the reformulated problem \eqref{eq: computation new 2}.
Obviously we can see that $G(t)$ is monotonically increasing with respect to $t$, and thus \eqref{eq: computation new 2} can be efficiently solved by a bisection search if we can evaluate $G(t)$ efficiently for any fixed $t$.

Now notice that $G$ can be equivalently formulated as 
\begin{align}
    G(t):=\min_{\substack{S\subseteq \cN(t),|S|\leq K\\0\leq \lambda\leq B(S,\bv)}}\sum_{i\in S_+} g_i(t,\lambda),\quad \cN(t):=\big\{i\in[N]:r_i\geq t\big\}.
\end{align}
where we restrict the assortment inside the set $\cN(t)$ whose items have revenues larger than $t$.
Otherwise including an item with revenue smaller than $t$ would not help minimize the summation.
If $|\cN(t)|\leq K$, then $S(t)=\cN(t)$. 
Else, if $|\cN(t)|>K$, we then need to decide the $K$ items $i$ with the smallest values $g_i(t,\lambda)$ given any $\lambda$. 
That is, we need to efficiently determine the order of the $|\cN(t)|$ curves $\{g_i(t,\lambda)\}_{\lambda\geq 0, i\in \cN(t)}$.

To this end, we have the following lemma, which shows that the intersection points between these curves are well controlled. 
This is beneficial because between the intersection points the order of the curves are fixed and thus evaluating $G(t)$ reduces to optimizing over $\lambda$ between these intersection points.
Moreover, it shows that the summation of $g_i(t,\lambda)$ forms a quasi-convex function over $\lambda$, and thus the optimization over $\lambda$ can be performed efficiently via bisection search.

\begin{lemma}\label{lem: intersection}
    Fix $t\geq 0$, the number of intersection points among curves $\{g_i(t,\lambda)\}_{\lambda \geq 0, i\in [N]}$ is less than $N(N-1)/2$. Furthermore, for any assortment set $S\subset[N]$, it holds that $\sum_{i\in S_+}g_i(t,\lambda)$ is a quasi-convex function in $\lambda$.
\end{lemma}

\begin{proof}[Proof of Lemma~\ref{lem: intersection}]
We begin with proving the first claim on the number of intersection points.
It suffices to prove that any two curves $g_i(t,\lambda)$ and $g_j(t,\lambda)$ intersect at no more than one point.
    Consider any different $i,j\in [N]$ with $(v_i,r_i)\neq (v_j,r_j)$. 
    By symmetry, without loss of generality we assume that $v_i\geq v_j$. 
    We prove the claim by discussing three different cases.
    \begin{itemize}
        \item \emph{Case 1:} $v_i>v_j$ and $r_i \geq  r_j$. In this case, we have that for any $\lambda \geq 0$,
        \begin{align}
            g_i(t,\lambda) = v_i\cdot \left(\exp\big((t-r_i)/\lambda\big) - 1\right) < v_j\cdot \left(\exp\big((t-r_j)/\lambda\big) - 1\right) = g_j(t,\lambda).
        \end{align}
        Therefore, we directly conclude that there is no intersection points between these two curves. 
        \item \emph{Case 2:} $v_i>v_j$ and $r_i < r_j$. In this case, consider that the existence of intersection points is equivalent to the existence of $\lambda$ such that 
        \begin{align}
            \phi_{i,j}(\lambda):=v_i\cdot \exp\big((t-r_i)/\lambda\big)  - v_j\cdot \exp\big((t-r_j)/\lambda\big)  = v_i - v_j > 0.
        \end{align}
        Notice that $\phi_{i,j}(0) = 0$ and that the derivative of $\phi_{i,j}$ is given by 
        \begin{align}
            \partial_{\lambda}\phi_{i,j}(\lambda)=\frac{1}{\lambda^2}\cdot \Big(v_i(r_i-t)\cdot \exp\big((t-r_i)/\lambda\big) - v_j(r_j-t)\cdot \exp\big((t-r_j)/\lambda\big)\Big).
        \end{align}
        We can figure out the behavior of $\phi_{i,j}$ by considering the following sub-cases:
        \begin{itemize}
            \item \emph{Case 2.1:} $v_i(r_i-t)<v_j(r_j-t)$. 
            In this case, the derivative $\partial_{\lambda}\phi_{i,j}(\lambda)$ has a unique zero point $\lambda_{0} = (r_j-r_i)/(\log(v_j(r_j-t)) - \log(v_i(r_i-t)))$, and $\partial_{\lambda}\phi_{i,j}(\lambda)\geq 0$ on $\lambda\in[0,\lambda_0]$ and $\partial_{\lambda}\phi_{i,j}(\lambda)\leq 0$ on $[\lambda_0,\infty)$. 
            Thus $\phi_{i,j}(\lambda)$ is increasing on $[0,\lambda_0]$ and is decreasing on $[\lambda_0,\infty)$. 
            Noticing that the limit $\lim_{\lambda\rightarrow+\infty}\phi_{i,j}(\lambda) = v_i-v_j$ and $\phi_{i,j}(0)=0$,  we can conclude that there is a single $\lambda_{i,j}\in[0,\lambda_0]$ such that $\phi_{i,j}(\lambda_{i,j}) = v_i-v_j$.
            Therefore, there is only one intersection point. 
            \item \emph{Case 2.2:} $v_i(r_i-t)\geq v_j(r_j-t)$. 
            In this case, the derivative $\partial_{\lambda}\phi_{i,j}(\lambda)$ is always non-negative and thus $\phi_{i,j}(\lambda)$ is increasing. 
            Due to the limit $\lim_{\lambda\rightarrow+\infty}\phi_{i,j}(\lambda) = v_i-v_j$, we can see that there is no $\lambda$ such that $\phi_{i,j}(\lambda)=v_i-v_j$. 
            Consequently, there is no intersection point for this case. 
        \end{itemize}
        \item \emph{Case 3:} $v_i = v_j$. In this case, no matter $r_i>r_j$ or $r_i<r_j$, it is obvious to find that $g_i(t,\lambda)$ and $g_j(t,\lambda)$ can only intersect at $\lambda=0$, which means that there is only one intersection point.
    \end{itemize}
    Therefore, we have proved the first claim of Lemma~\ref{lem: intersection}. 

    Now we prove the second claim of Lemma~\ref{lem: intersection}. For any $S\subset [N]$, let $g_S(t,\lambda):=\sum_{i\in S_+}g_i(t,\lambda)$. 
    We have 
    \begin{align}
        \partial_{\lambda} g_S(t,\lambda)=\frac{1}{\lambda^2}\cdot \left(-v_0t\cdot\exp\big(t/\lambda\big) + \sum_{i\in S} v_i(r_i-t)\cdot \exp\big((t-r_i)/\lambda\big)\right):=\frac{1}{\lambda^2}\cdot m(t,\lambda)
    \end{align}
    To investigate the sign of $\partial_{\lambda} g_S(t,\lambda)$, we consider the derivative of $m(t,\lambda)$,
    \begin{align}
        \partial_{\lambda} m(t,\lambda)=\frac{1}{\lambda^2}\cdot v_0t^2\cdot\exp\big(t/\lambda\big) + \frac{1}{\lambda^2}\cdot \sum_{i\in S} v_i \cdot (r_i-t)^2\cdot  \exp\big((t-r_i)/\lambda\big) >0,
    \end{align}
    which means that $m(t,\lambda)$ is increasing over $\lambda$. 
    Now that $\lim_{\lambda\rightarrow 0}m(t,\lambda) = -\infty$, we conclude that $\partial_{\lambda} g_S(t,\lambda)$ is either always negative or is negative on $[0,\lambda_1]$ for some $\lambda_1>0$ and then positive on $[\lambda_1,\infty)$.
    This means that $g_S(t,\lambda)$ is either monotonically decreasing on $\lambda$ or decreasing on $\lambda\in[0,\lambda_1]$ and increasing on $\lambda\in[\lambda_1,\infty)$. 
    In both cases $g_S(t, \lambda)$ is a quasi-convex function. 
    This completes the proof of Lemma~\ref{lem: intersection}.
\end{proof}

Now with Lemma~\ref{lem: intersection} and all the above discussions, we are ready to give the computational algorithm to solve the original optimization problem \eqref{eq: computation new 1} and \eqref{eq: computation new 2}.

\paragraph{Algorithm to solve the optimal robust assortment for Example~\ref{exp: new}.}
The algorithm is presented in the following Algorithms~\ref{alg: computation new 1} and Algorithm~\ref{alg: computation new 2}.

\begin{algorithm}[h]
\caption{Solving the optimal robust assortment (Example~\ref{exp: new})} \label{alg: computation new 1}
\begin{algorithmic}[1]
\STATE \textbf{Inputs:} Error level $\epsilon,\epsilon_G\geq 0$, target $G$-value $a^{\dagger}(\rho_0,\bv)$.
\STATE Initialize $t_{\mathrm{left}}= 0$ and $t_{\mathrm{right}}=r_{\max}$.
 \STATE \textcolor{blue!55}{\texttt{// Use  bisection search to solve the largest possible $t$.}}
\FOR{iteration $m =1,\cdots,M_{\epsilon}:=\lceil\log(4r_{\max}/\epsilon)/\log 2\rceil$}
    \STATE Set $t_{\mathrm{mid}}=(t_{\mathrm{left}} + t_{\mathrm{right}})/2$ and call Algorithm~\ref{alg: computation new 2} to approximately solve $G(t_{\mathrm{mid}})$ up to error $\epsilon_G$, obtain the returned value $\widehat{G}$ and the assortment $S$.
    \IF{$\widehat{G}>a^{\dagger}(\rho_0,\bv)+\epsilon_G$}
    \STATE Set the right point $t_{\mathrm{right}} = t_{\mathrm{mid}}$.
    \ELSIF{$\widehat{G}<a^{\dagger}(\rho_0,\bv)$}
    \STATE Set the left point $t_{\mathrm{left}} = t_{\mathrm{mid}}$.
    \ELSE
    \STATE \textbf{break}
    \ENDIF
\ENDFOR
\STATE Set $t' = t_{\mathrm{mid}} - \epsilon/2$ and call Algorithm~\ref{alg: computation new 2} to approximately solve $G(t')$ up to error $\epsilon_G$, and obtain the corresponding assortment $S$.
\STATE \textbf{Output:} value $t'$ and assortment $S$
\end{algorithmic}
\end{algorithm}

\begin{algorithm}[!h]
\caption{Approximately evaluate the function $G$} \label{alg: computation new 2}
\begin{algorithmic}[1]
\STATE \textbf{Inputs:} Error level $\epsilon_G\geq 0$, real number $t$, functions $\{g_i(t,\lambda)\}_{i\in \cN(t),\lambda \geq 0}$, the intersection points among the functions  $\{g_i(t,\lambda)\}_{i\in \cN(t),\lambda \geq 0}$, denoted in an ascent order as $0=\lambda_0\leq \lambda_1\leq \cdots\leq \lambda_{I(t)} < \lambda_{I(t)+1}:=\infty$.

 \STATE \textcolor{blue!55}{\texttt{// Case 1:} $|\cN(t)|\leq K$\texttt{, directly set the assortment $S=\cN(t)$.}}\label{line: case 1}
\IF{$|\cN(t)|\leq K$}
\STATE Set the assortment $S = \cN(t)$
\STATE Approximately solve the following minimization problem up to error $\epsilon_G$ via bisection search: 
\begin{align}
    \min_{0\leq \lambda \leq B(S,\bv)}\sum_{i\in S_+}g_i(t,\lambda),\label{eq: case 1}
\end{align}
and denote the $\epsilon_G$-optimal value as $\widehat{G}$.
\STATE \textbf{Output:} assortment $S$ and approximate value $\widehat{G}$.
\ENDIF
\STATE \textcolor{blue!55}{\texttt{// Case 2:} $|\cN(t)|> K$\texttt{, need to enumerate over possible optimal assortments.}}
\label{line: case 2}
\FOR{interval $\ell=1,\cdots,I(t),I(t)+1$}
    \STATE Obtain $S_{\ell}$ containing the $K$ functions in $\{g_i(t,\lambda)\}_{i\in \cN(t),\lambda_{\ell-1}\leq \lambda\leq\lambda_{\ell}}$ with the $K$ smallest values. \label{line: sort}
    \IF{$B(S_{\ell},\bv)< \lambda_{\ell-1}$} 
    \STATE \textbf{continue}
    \ENDIF 
    \STATE Approximately solve the following minimization problem up to error $\epsilon_G$ via bisection search:
    \begin{align}
    \min_{\lambda_{\ell-1}\leq \lambda \leq B(S_{\ell},\bv)\wedge \lambda_{\ell}}\sum_{i\in (S_{\ell})_+}g_i(t,\lambda),\label{eq: case 2}
\end{align}
and denote the $\epsilon_G$-optimal value as $\widehat{G}_{\ell}$.
\ENDFOR
\STATE Set $\widehat{\ell}=\argmin_{\ell\in[I(t)+1]}\widehat{G}_{\ell}$, and set $\widehat{G} = \widehat{G}_{\widehat{\ell}}$ and $S=S_{\widehat{\ell}}$.
\STATE \textbf{Output:} assortment $S$ and approximate value $\widehat{G}$.
\end{algorithmic}
\end{algorithm}

\begin{remark}
    About solving the intersection points, it is a simple one-dimensional numerical equation, for which assume that we have solved all of them accurately before running Algorithms~\ref{alg: computation new 1} and \ref{alg: computation new 2}. Actually, for any two $i,j\in[N]$, the cases such that the intersection point exists reduce to either a trivial case $\lambda_{\mathrm{intersect}}=0$ or to solving the zero point of a monotone function (see Case 3 and Case 2.1 in Lemma~\ref{lem: intersection} respectively).
\end{remark}

\begin{remark}\label{rmk: computation}
    In the data driven algorithm Algorithm~\ref{alg: new}, we need to solve a slightly different objective than \eqref{eq: computation new 1}. 
    More concretely, we assume the knowledge of the true value of $\sum_{i\in[N]}v_i$, denoted as $v_{\mathrm{tot}}$, and need to solve the new duality objective \eqref{eq: object new}. However, the resulting new objective shares the same form as the objective $\eqref{alg: computation new 2}$ here except that the target $G$-value becomes $-(1-e^{-\rho_0})\cdot (1+v_{\mathrm{tot}})$. 
    Therefore, it suffices to define the quantity $a^{\dagger}(\rho_0,\bv):=-(1-e^{-\rho_0})\cdot (1+v_{\mathrm{tot}})$ in the input of Algorithm~\ref{alg: computation new 1}. 
\end{remark}

\subsubsection{Analysis of Algorithm Design for Varying Robust Set Size (Example~\ref{exp: new})}

\begin{theorem}\label{thm: computation new}
    Given an error level $\epsilon<R^{\star}/2$, by setting the inner loop error level $\epsilon_G$ as 
    \begin{align}
        \epsilon_G=\frac{v_0\epsilon}{4B(\emptyset,\bv)},
    \end{align}
    then the following two conclusions hold:
    \begin{enumerate}
        \item Algorithm~\ref{alg: computation new 1}  returns a value $R^{\star}\geq t'\geq R^{\star} - \epsilon$ and an assortment $S$ satisfying condition \eqref{eq: s epsilon}, i.e., 
        \begin{align}
            \min_{0\leq \lambda \leq B(S,\bv)}\sum_{i\in S_+ }g_i(t',\lambda)\leq a^{\dagger}(\rho_0,\bv).
        \end{align}
        \item The iteration complexity of Algorithm~\ref{alg: computation new 1} is of order $\cO(\log(r_{\max}/\epsilon))$, and in each iteration of Algorithm~\ref{alg: computation new 1}, the sub-algorithm Algorithm~\ref{alg: computation new 2} needs to evaluate functions in the form of $\sum_{i\in S}g_i(t,\cdot)$ and its first-order derivative for no more than $\cO(N^2\cdot \log(B(S,\bv)/\epsilon_G'))$ times, where $\epsilon_G'$ is defined as the largest positive real number such that
        \begin{align}
            \epsilon_G'\cdot \sup_{\lambda\in[0,B(S,\bv)]\cap [\lambda_t^{\star}(S,\bv)-\epsilon_G',\lambda_t^{\star}(S,\bv)+\epsilon_G']}\left|\sum_{i\in S_+}\partial_{\lambda}g_i(t,\lambda)\right| \leq \epsilon_G,\label{eq: epsilon prime}
    \end{align}
    where $\lambda_t^{\star}(S,\bv):=\argmin_{0\leq \lambda\leq B(S,\bv)}\sum_{i\in S_+}g_i(t,\lambda)$ for any tuple $(t,S,\bv)$.
    Overall, the running complexity of Algorithm~\ref{alg: computation new 2} is of order $\widetilde{\cO}(N^2)$.
    \end{enumerate}
\end{theorem}

\begin{proof}[Proof of Theorem~\ref{thm: computation new}]
    We first prove the second result regarding the running complexity of Algorithms~\ref{alg: computation new 1} and \ref{alg: computation new 2}.
    The first conclusion regarding the iteration complexity of Algorithm~\ref{alg: computation new 1} is trivial given that Algorithm~\ref{alg: computation new 1} iterates for no more than $M_{\epsilon}=\lceil \log(4r_{\max}/\epsilon)/\log 2\rceil$ times. 
    Regarding the second conclusion for the running complexity of each call of Algorithm~\ref{alg: computation new 2}, we need to invoke Lemma~\ref{lem: intersection} that proves the intersection between the curves $\{g_i(t,\lambda)\}_{i\in\cN(t),\lambda\geq 0}$ for any $t$ is no more than $\cO(N^2)$.
    More specifically, in Algorithm~\ref{alg: computation new 2}, if \textcolor{blue!55}{Case 1} holds (see Line~\ref{line: case 1}), then we only need evaluate the function $\sum_{i\in \cN(t)_+} g_i(t,\cdot)$ for $\cO(\log(B(\cN(t),\bv)/\epsilon_G'))$ times to use bisection search to solve the minimization problem \eqref{eq: case 1} up to error level of $\epsilon_{G}'$ in terms of $\lambda$.
    This is because within these many calls of the quasi-convex function $\sum_{i\in \cN(t)_+}g_i(t,\cdot)$ and its derivative, bisection search returns $\widehat{\lambda}_t(\cN(t),\bv)$ satisfying $|\widehat{\lambda}_t(\cN(t),\bv) - \lambda^{\star}_t(\cN(t),\bv)|\leq \epsilon_G'$, and thus 
    \begin{align}
        &\sum_{i\in \cN(t)_+}g_i(t,\widehat{\lambda}_t(\cN(t),\bv)) - \min_{0\leq \lambda \leq B(\cN(t),\bv)}\sum_{i\in \cN(t)_+}g_i(t,\lambda)\\
        &\qquad = \left|\sum_{i\in \cN(t)_+}g_i(t,\widehat{\lambda}_t(\cN(t),\bv)) - \sum_{i\in \cN(t)_+}g_i(t,\lambda^{\star}_t(\cN(t),\bv))\right| \\
        &\qquad \leq \int_{\lambda^{\star}_t(\cN(t),\bv)\wedge \widehat{\lambda}_t(\cN(t),\bv)}^{\lambda^{\star}_t(\cN(t),\bv)\vee \widehat{\lambda}_t(\cN(t),\bv)}\left|\sum_{i\in \cN(t)_+}\partial_{\lambda}g_i(t,\lambda)\right|\mathrm{d}\lambda\\
        &\qquad \leq  \epsilon_G'\cdot \sup_{\lambda\in[0,B(\cN(t),\bv)]\cap [\lambda_t^{\star}(S,\bv)-\epsilon_G',\lambda_t^{\star}(S,\bv)+\epsilon_G']}\left|\sum_{i\in S_+}\partial_{\lambda}g_i(t,\lambda)\right|\\
        &\qquad \leq \epsilon_G,
    \end{align}
    where the last inequality uses the condition on the parameter $\epsilon_G'$ in \eqref{eq: epsilon prime}.
    Thus $\widehat{\lambda}_t(\cN(t),\bv)$ is an $\epsilon_G$-optimal solution to \eqref{eq: computation new 1}.
    Then, if \textcolor{blue!55}{Case 2} holds (see Line~\ref{line: case 2}), by Lemma~\ref{lem: intersection}, we know that the number of the intervals $|I(t)|=\cO(N^2)$, and thus we only need to solve the optimization problem in the form of \eqref{eq: case 2} for less than $\cO(N^2)$ times. 
    Using arguments similar to those for \textcolor{blue!55}{Case 1}, we prove that for each instance of the optimization problem \eqref{eq: case 2}, we need to evaluate the function $\sum_{i\in S_\ell} g_i(t,\cdot)$ and its derivative for no more than $\cO(\log(B(S_{\ell},\bv)/\epsilon_G')$ times to obtain an $\epsilon_G$-optimal solution.
    Finally, we note that the process to find the $K$ functions with the $K$ smallest values in the $\ell$-th interval (Line~\ref{line: sort}) only requires $\cO(N)$ times of evaluation of these functions due to the fact that there are no intersection points among these functions inside each of the interval.
    This proves the second result of Theorem~\ref{thm: computation new}. 

    Now we prove the first result of Theorem~\ref{thm: computation new} regarding the correctness of Algorithm~\ref{alg: computation new 1}. 
    Firstly, by the above proofs for the second result, we see that given value $t$ queried by Algorithm~\ref{alg: computation new 1}, Algorithm~\ref{alg: computation new 2} returns an approximated value $\widehat{G}$ of $G(t)$ within error level $\widehat{G} \leq  G(t) + \epsilon_G$, together with an assortment $S$ satisfying $G(t)\leq \min_{0\leq \lambda\leq B(S,\bv)}\sum_{i\in S_+}g_i(t,\lambda)\leq \widehat{G}$. 
    With such a fact, the bisection process of Algorithm~\ref{alg: computation new 1} can always maintain the optimal $t$ value, i.e., $t^{\star}:=R^{\star}$, inside the remaining interval $[t_{\mathrm{left}},t_{\mathrm{right}}]$.
    This is because:
    \begin{itemize}
        \item If $\widehat{G}>a^{\dagger}(\rho_0,\bv)+\epsilon_G$, then $G(t^{\star})< \widehat{G} - \epsilon_G \leq G(t_{\mathrm{mid}})$.
        Thus $t^{\star}\leq  t_{\mathrm{mid}}$ due to monotonicity of $G$.
        \item Else if $\widehat{G}<a^{\dagger}(\rho_0,\bv)$, then $G(t_{\mathrm{mid}})\leq \widehat{G}  <G(t^{\star})$. Thus $t_{\mathrm{mid}}\leq t^{\star}$ due to monotonicity of $G$.
        \item Else if $a^{\dagger}(\rho_0,\bv)\leq  \widehat{G} \leq a^{\dagger} + \epsilon_G$, we are going to show that it must hold that $t^{\star}-\epsilon/4\leq t_{\mathrm{mid}} \leq t^{\star} + \epsilon/4$.
        It suffices to prove $G(t^{\star} - \epsilon/4) \leq G(t_{\mathrm{mid}}) \leq G(t^{\star}+\epsilon/4)$.
        We first show the left side by proving the following inequality: $G(t^{\star}-\epsilon/4) + \epsilon_G \leq  G(t^{\star})$, with which we can obtain that $G(t^{\star}-\epsilon/4)\leq G(t^{\star}) - \epsilon_G \leq \widehat{G} - \epsilon_G \leq G(t_{\mathrm{mid}})$.
        To prove that $G(t^{\star}-\epsilon/4) + \epsilon_G \leq  G(t^{\star})$, consider the following,
        \begin{align}
            G(t^{\star}) - G(t^{\star}-\epsilon/4)&= \min_{\substack{S\subseteq \cN(t),|S|\leq K\\0\leq \lambda\leq B(S,\bv)}}\sum_{i\in S_+} g_i(t^{\star},\lambda) - \min_{\substack{S\subseteq \cN(t),|S|\leq K\\0\leq \lambda\leq B(S,\bv)}}\sum_{i\in S_+} g_i(t^{\star} - \epsilon/4,\lambda) \\
            &\geq \min_{\substack{S\subseteq \cN(t),|S|\leq K\\0\leq \lambda\leq B(S,\bv)}}\left\{\sum_{i\in S_+} g_i(t^{\star},\lambda) - g_i(t^{\star} - \epsilon/4,\lambda)\right\} \\
            &= \min_{\substack{S\subseteq \cN(t),|S|\leq K\\0\leq \lambda\leq B(S,\bv)}}\left\{\sum_{i\in S_+} v_i\cdot\Big(\exp\big((t^{\star}-r_i)/\lambda\big) - \exp\big((t^{\star}-r_i-\epsilon/4)/\lambda\big) \Big)\right\} \\
            &\geq \min_{0\leq \lambda\leq B(\emptyset,\bv)}v_0\cdot\Big(\exp\big(t^{\star}/\lambda\big) - \exp\big((t^{\star}-\epsilon/4)/\lambda\big) \Big) \\
            &=\min_{0\leq \lambda\leq B(\emptyset,\bv)}v_0\cdot \frac{\epsilon}{4\lambda}\cdot \exp\big(\zeta(t^{\star},t^{\star}-\epsilon/4)/\lambda\big) \\
            &\geq  \frac{v_0\epsilon}{4B(\emptyset,\bv)}\\
            &=\epsilon_G,
        \end{align}
        where the first inequality is due to $\min_{x\in\cX} f(x) - \min_{x\in\cX} g(x)\geq \min_{x\in\cX}\{f(x) - g(x)\}$, the second inequality is by the fact that each summand is positive and the objective is minimized when $S=\emptyset$, the second to last equality is by mean-value theorem with $\zeta(t^{\star},t^{\star}-\epsilon/4)\in[ t^{\star}-\epsilon/4,t^{\star}]$, 
        and the last equality is by the choice of $\epsilon_G$.
        This proves that $G(t^{\star}-\epsilon/4) + \epsilon_G \leq  G(t^{\star})$ and therefore $G(t^{\star} - \epsilon/4) \leq G(t_{\mathrm{mid}}) $. 
        By the same argument, we can also prove that $G(t^{\star}) + \epsilon_G \leq  G(t^{\star}+\epsilon/4)$, for which we can further obtain that $G(t_{\mathrm{mid}}) \leq \widehat{G}\leq G(t^{\star})+\epsilon_G \leq G(t^{\star} + \epsilon/4)$. 
        This proves that if $a^{\dagger}(\rho_0,\bv)\leq  \widehat{G} \leq a^{\dagger} + \epsilon_G$, it must hold that $t^{\star}-\epsilon/4\leq t_{\mathrm{mid}} \leq t^{\star} + \epsilon/4$, and we can terminate the iteration now.
    \end{itemize}
    Therefore, when the iterations of Algorithm~\ref{alg: computation new 1} end, we must have that $t^{\star} - \epsilon/4 \leq t_{\mathrm{mid}} \leq t^{\star} + \epsilon/4$. 
    Then we have that $t' = t_{\mathrm{mid}} - \epsilon/2$ satisfying $t^{\star} - \epsilon < t^{\star} - 3\epsilon/4 \leq t'\leq t^{\star} - \epsilon/4<t^{\star}$. 
    Furthermore, the corresponding assortment $S$ satisfies 
    \begin{align}
        \min_{0\leq \lambda\leq B(S,\bv)}\sum_{i\in S_+}g_i(t',\lambda) \leq G(t') + \epsilon_G\leq G(t^{\star}-\epsilon/4) + \epsilon_G\leq G(t^{\star})=a^{\dagger}(\rho_0,\bv).
    \end{align}
    This completes the proof of Theorem~\ref{thm: computation new}.
\end{proof} 

\newpage

\section{Proofs for Constant Robust Set Size Case}\label{sec:_proof_main_theorems}

\subsection{General Upper Bound for Suboptimality Gap}\label{subsec:_general_upper_bound}

\begin{theorem}[Suboptimality of Algorithm~\ref{alg: jin}]\label{thm:_algorithm_design_1_general}
    Conditioning on the observed assortments $\{S_k\}_{k=1}^n$, suppose that 
    \begin{align}
        n_i:= \sum_{k=1}^n\mathbf{1}\{i \in S_k\}\geq 138\cdot\max\left\{1,\frac{256}{138v_i}\right\}\cdot (1+v_{\max})(1+Kv_{\max})\log(3N/\delta),\quad \forall i\in S^{\star},\label{eq: condition_n_i_main_general}
    \end{align}
    then with probability at least $1-2\delta$, the robust expected revenue of the assortment $\widehat S$ satisfies
    \begin{align}
        &R_{\rho}(S^{\star};\bv) - R_{\rho}(\widehat S;\bv)\\
        &\qquad \leq  \frac{\mathsf{C}(\rho)\cdot r_{\max}}{1+v(S^{\star})}\cdot \sum_{i\in S^{\star}}\left(16 \sqrt{\frac{v_i(1+v_i)\log(3N/\delta)}{\sum_{k=1}^n\mathbf{1}\{i_k\in S_k\}\cdot (1+v(S_k))^{-1}}} + \frac{88(1+v_i)\log(3N/\delta)}{\sum_{k=1}^n\mathbf{1}\{i_k\in S_k\}\cdot (1+v(S_k))^{-1}}\right).
    \end{align}
    where $v(S):= \sum_{j\in S}v_j$ and the function $\mathsf{C}(\rho)$ is defined as 
    \begin{align}
        \mathsf{C}(\rho) = \min\left\{ \inf_{c\in[0,1/\rho]}\Big\{c\cdot \max\big\{2, \exp(1/c)-1\big\}\Big\}, \frac{1}{\rho}\right\}.
    \end{align}
\end{theorem}

\begin{proof}[Proof of Theorem~\ref{thm:_algorithm_design_1_general}]
    Please refer to Appendix~\ref{subsec:_proof_algorithm_design_general_1} for a detailed proof of Theorem~\ref{thm:_algorithm_design_1_general}.
\end{proof}

\subsection{Proof of Theorem~\ref{thm:_algorithm_design_1_general}}\label{subsec:_proof_algorithm_design_general_1}

\begin{proof}[Proof of Theorem~\ref{thm:_algorithm_design_1_general}]
    Consider the following upper bound of the suboptimality of the learned assortment, 
    \begin{align}
        R_{\rho}(S^{\star};\bv) - R_{\rho}(\widehat S;\bv) \leq R_{\rho}(S^{\star};\bv) - R_{\rho}(\widehat S;\bv^{\mathrm{LCB}}) \leq  R_{\rho}(S^{\star};\bv) - R_{\rho}(S^{\star};\bv^{\mathrm{LCB}}),\label{eq:_proof_algorithm_design_1_0}
    \end{align}
    where the first inequality uses Lemma~\ref{lem:_monotonicity_1} which shows that the robust expected revenue is non-decreasing with respect to the parameter given the optimal robust assortment, and the second inequality is due to the optimality of $\widehat S$.
    Now we can further bound the right-hand side of the above inequality by the following,
    \begin{align}
        R_{\rho}(S^{\star};\bv) - R_{\rho}(S^{\star};\bv^{\mathrm{LCB}}) &= \sup_{0\leq \lambda\leq r_{\max}/\rho} \left\{-\lambda\cdot\log\left(\frac{\sum_{i\in S^{\star}_+}v_i\cdot\exp(-r_i/\lambda)}{\sum_{i\in S^{\star}_+}v_i}\right)-\lambda\cdot\delta\right\}\label{eq:_proof_algorithm_design_1_0+}\\
        &\qquad  - \sup_{0\leq \lambda\leq r_{\max}/\rho} \left\{-\lambda\cdot\log\left(\frac{\sum_{i\in S^{\star}_+}v_i^{\mathrm{LCB}}\cdot\exp(-r_i/\lambda)}{\sum_{i\in S^{\star}_+}v_i^{\mathrm{LCB}}}\right)-\lambda\cdot\delta\right\},\\
        &\leq \sup_{0\leq \lambda\leq r_{\max}/\rho} \left\{\lambda\cdot\log\left(\frac{\sum_{i\in S^{\star}_+}v_i^{\mathrm{LCB}}\cdot\exp(-r_i/\lambda)}{\sum_{i\in S^{\star}_+}v_i^{\mathrm{LCB}}}\cdot \frac{\sum_{i\in S^{\star}_+}v_i}{\sum_{i\in S^{\star}_+}v_i\cdot\exp(-r_i/\lambda)}\right)\right\}.
    \end{align}
    Now using the inequality that $\log(a/b) = \log(1+(a-b)/b)\leq (a-b)/b$ for all $a,b>0$, we can further bound the above expression by
    \begin{align}
        R_{\rho}(S^{\star};\bv) - R_{\rho}(S^{\star};\bv^{\mathrm{LCB}}) &\leq  \sup_{0\leq \lambda\leq r_{\max}/\rho} \left\{\lambda\cdot\left(\frac{\sum_{i\in S^{\star}_+}v_i^{\mathrm{LCB}}\cdot\exp(-r_i/\lambda)}{\sum_{i\in S^{\star}_+}v_i^{\mathrm{LCB}}}\cdot \frac{\sum_{i\in S^{\star}_+}v_i}{\sum_{i\in S^{\star}_+}v_i\cdot\exp(-r_i/\lambda)}-1\right)\right\}.
    \end{align}
    To derive a tight upper bound on the right hand side of the above inequality, we consider two ranges of the dual variable $\lambda$ separately. 
    More specifically, we consider 
    \allowdisplaybreaks
    \begin{align}
        &R_{\rho}(S^{\star};\bv) - R_{\rho}(S^{\star};\bv^{\mathrm{LCB}}) \\
        &\qquad \leq\max\left\{\underbrace{ \sup_{0\leq \lambda\leq c\cdot r_{\max}} \left\{\lambda\cdot\left(\frac{\sum_{i\in S^{\star}_+}v_i^{\mathrm{LCB}}\cdot\exp(-r_i/\lambda)}{\sum_{i\in S^{\star}_+}v_i^{\mathrm{LCB}}}\cdot \frac{\sum_{i\in S^{\star}_+}v_i}{\sum_{i\in S^{\star}_+}v_i\cdot\exp(-r_i/\lambda)}-1\right)\right\}}_{\displaystyle{\text{Term (i)}}},\right.\\
        &\qquad\qquad \left.\underbrace{ \sup_{c\cdot r_{\max}\leq \lambda\leq r_{\max} / \rho} \left\{\lambda\cdot\left(\frac{\sum_{i\in S^{\star}_+}v_i^{\mathrm{LCB}}\cdot\exp(-r_i/\lambda)}{\sum_{i\in S^{\star}_+}v_i^{\mathrm{LCB}}}\cdot \frac{\sum_{i\in S^{\star}_+}v_i}{\sum_{i\in S^{\star}_+}v_i\cdot\exp(-r_i/\lambda)}-1\right)\right\}}_{\displaystyle{\text{Term (ii)}}}\right\},\label{eq:_proof_algorithm_design_1_1-}
    \end{align}
    where $c>0$ is a parameter to be tuned later.
    We now upper bound \text{Term (i)} and \text{Term (ii)} separately.

    \paragraph{Upper bounding \texttt{Term (i).}}
    To upper bound \text{Term (i)}, we utilize the inequality that 
    \begin{align}
        \sum_ia_i = \sum_i\frac{a_i}{b_i}\cdot b_i\leq \max_i\left\{\frac{a_i}{b_i}\right\}\cdot\sum_ib_i,\quad \forall a_i\in\RR,\,\,b_i>0,\label{eq:_proof_algorithm_design_1_1}
    \end{align}
    which gives that for any fixed dual variable $0\leq \lambda\leq c\cdot r_{\max}$, we have the following inequalities,
    \allowdisplaybreaks
    \begin{align}
        &\frac{\sum_{i\in S^{\star}_+}v_i^{\mathrm{LCB}}\cdot\exp(-r_i/\lambda)}{\sum_{i\in S^{\star}_+}v_i^{\mathrm{LCB}}}\cdot \frac{\sum_{i\in S^{\star}_+}v_i}{\sum_{i\in S^{\star}_+}v_i\cdot\exp(-r_i/\lambda)}-1 \\
        &\qquad =  \frac{\sum_{i\in S^{\star}_+}v_i}{\sum_{i\in S^{\star}_+}v_i\cdot\exp(-r_i/\lambda)}\cdot \left(\sum_{i\in S^{\star}_+}\left(\frac{v_i^{\mathrm{LCB}}}{\sum_{i\in S^{\star}_+}v_i^{\mathrm{LCB}}} - \frac{v_i}{\sum_{i\in S^{\star}_+}v_i}\right)\cdot\exp(-r_i/\lambda)\right)\\ 
        &\qquad \leq \frac{\sum_{i\in S^{\star}_+}v_i}{\sum_{i\in S^{\star}_+}v_i\cdot\exp(-r_i/\lambda)}\cdot \max_{i\in S^{\star}_+}\left\{\left(\frac{v_i^{\mathrm{LCB}}}{\sum_{i\in S^{\star}_+}v_i^{\mathrm{LCB}}} - \frac{v_i}{\sum_{i\in S^{\star}_+}v_i}\right)\cdot \frac{\sum_{i\in S^{\star}_+}v_i}{v_i} \right\}\cdot \frac{\sum_{i\in S^{\star}_+}v_i\cdot\exp(-r_i/\lambda)}{\sum_{i\in S^{\star}_+}v_i}\\
        &\qquad = \max_{i\in S^{\star}_+}\left\{\left(\frac{v_i^{\mathrm{LCB}}}{\sum_{i\in S^{\star}_+}v_i^{\mathrm{LCB}}} - \frac{v_i}{\sum_{i\in S^{\star}_+}v_i}\right)\cdot \frac{\sum_{i\in S^{\star}_+}v_i}{v_i} \right\},\label{eq:_proof_algorithm_design_1_2}
    \end{align}
    where the inequality uses \eqref{eq:_proof_algorithm_design_1_1}.
    To further upper bound the right hand side of \eqref{eq:_proof_algorithm_design_1_2}, we utilize the pessimistic property of $\bv^{\mathrm{LCB}}$ in Lemma~\ref{lem:_concentration} to derive that for any item $i\in S^{\star}_+$,
    \begin{align}
        \left(\frac{v_i^{\mathrm{LCB}}}{\sum_{i\in S^{\star}_+}v_i^{\mathrm{LCB}}} - \frac{v_i}{\sum_{i\in S^{\star}_+}v_i}\right)\cdot \frac{\sum_{i\in S^{\star}_+}v_i}{v_i} &\leq  \left(\frac{1}{\sum_{i\in S^{\star}_+}v_i^{\mathrm{LCB}}} - \frac{1}{\sum_{i\in S^{\star}_+}v_i}\right)\cdot \sum_{i\in S^{\star}_+}v_i = \frac{\sum_{i\in S_+^{\star}}v_i - v_i^{\mathrm{LCB}}}{\sum_{i\in S^{\star}_+}v_i^{\mathrm{LCB}}}.
    \end{align}
    Therefore,  we can upper bound \text{Term (i)} via the following,
    \begin{align}
        \text{Term (i)} \leq c\cdot r_{\max} \cdot \frac{\sum_{i\in S_+^{\star}}v_i - v_i^{\mathrm{LCB}}}{\sum_{i\in S^{\star}_+}v_i^{\mathrm{LCB}}}. \label{eq:_proof_algorithm_design_1_2+}
    \end{align}

    \paragraph{Upper bounding \texttt{Term (ii).}}
    To upper bound \text{Term (ii)}, we consider the following approach, 
    \begin{align}
        &\frac{\sum_{i\in S^{\star}_+}v_i^{\mathrm{LCB}}\cdot\exp(-r_i/\lambda)}{\sum_{i\in S^{\star}_+}v_i^{\mathrm{LCB}}}\cdot \frac{\sum_{i\in S^{\star}_+}v_i}{\sum_{i\in S^{\star}_+}v_i\cdot\exp(-r_i/\lambda)}-1 \\
        &\qquad \leq   \frac{\sum_{i\in S^{\star}_+}v_i}{\sum_{i\in S^{\star}_+}v_i\cdot\exp(-r_i/\lambda)}\cdot \left|\sum_{i\in S^{\star}_+}\left(\frac{v_i^{\mathrm{LCB}}}{\sum_{i\in S^{\star}_+}v_i^{\mathrm{LCB}}} - \frac{v_i}{\sum_{i\in S^{\star}_+}v_i}\right)\cdot\exp(-r_i/\lambda)\right|\\  
        &\qquad =  \frac{\sum_{i\in S^{\star}_+}v_i}{\sum_{i\in S^{\star}_+}v_i\cdot\exp(-r_i/\lambda)}\cdot \left|\sum_{i\in S^{\star}_+}\left(\frac{v_i^{\mathrm{LCB}}}{\sum_{i\in S^{\star}_+}v_i^{\mathrm{LCB}}} - \frac{v_i}{\sum_{i\in S^{\star}_+}v_i}\right)\cdot\big(1 - \exp(-r_i/\lambda) \big)\right|\\  
        & \qquad \leq \exp(r_{\max}/\lambda)\cdot \big(1 - \exp(-r_{\max}/\lambda) \big)\cdot \sum_{i\in S_+^{\star}}\left|\frac{v_i^{\mathrm{LCB}}}{\sum_{i\in S^{\star}_+}v_i^{\mathrm{LCB}}} - \frac{v_i}{\sum_{i\in S^{\star}_+}v_i} \right|,\label{eq:_proof_algorithm_design_1_3}
    \end{align}
    where the first inequality just uses $x\leq |x|$, while the second inequality uses $|\sum_{i}a_i b_i|\leq \max_i|b_i|\cdot\sum_i|a_i|$.
    We further upper bound the right handside of \eqref{eq:_proof_algorithm_design_1_3} as following. 
    On the one hand, using the fact that $\bv^{\mathrm{LCB}}$ is a pessimistic estimation of $\bv$ (Lemma~\ref{lem:_concentration}), we can derive that 
    \begin{align}
        \frac{v_i^{\mathrm{LCB}}}{\sum_{i\in S^{\star}_+}v_i^{\mathrm{LCB}}} - \frac{v_i}{\sum_{i\in S^{\star}_+}v_i}  \leq \frac{v_i}{\sum_{i\in S^{\star}_+}v_i^{\mathrm{LCB}}} - \frac{v_i}{\sum_{i\in S^{\star}_+}v_i} = \frac{v_i\cdot \left(\sum_{i\in S_+^{\star}}v_i - v_i^{\mathrm{LCB}}\right)}{\left(\sum_{i\in S_+^{\star}}v_i\right)\cdot \left(\sum_{i\in S_+^{\star}}v_i^{\mathrm{LCB}}\right)},
    \end{align}
    and that 
    \begin{align}
        \frac{v_i}{\sum_{i\in S^{\star}_+}v_i} - \frac{v_i^{\mathrm{LCB}}}{\sum_{i\in S^{\star}_+}v_i^{\mathrm{LCB}}}   \leq \frac{v_i}{\sum_{i\in S^{\star}_+}v_i^{\mathrm{LCB}}} - \frac{v_i^{\mathrm{LCB}}}{\sum_{i\in S^{\star}_+}v_i^{\mathrm{LCB}}} = \frac{v_i - v_i^{\mathrm{LCB}}}{\sum_{i\in S_+^{\star}}v_i^{\mathrm{LCB}}},
    \end{align}
    which together gives that 
    \begin{align}
        \sum_{i\in S_+^{\star}}\left|\frac{v_i^{\mathrm{LCB}}}{\sum_{i\in S^{\star}_+}v_i^{\mathrm{LCB}}} - \frac{v_i}{\sum_{i\in S^{\star}_+}v_i}\right| & \leq \sum_{i\in S_+^{\star}}\max\left\{\frac{v_i\cdot \left(\sum_{i\in S_+^{\star}}v_i - v_i^{\mathrm{LCB}}\right)}{\left(\sum_{i\in S_+^{\star}}v_i\right)\cdot \left(\sum_{i\in S_+^{\star}}v_i^{\mathrm{LCB}}\right)}, \frac{v_i - v_i^{\mathrm{LCB}}}{\sum_{i\in S_+^{\star}}v_i^{\mathrm{LCB}}}\right\} \\
        & \leq \sum_{i\in S_+^{\star}} \frac{v_i\cdot \left(\sum_{i\in S_+^{\star}}v_i - v_i^{\mathrm{LCB}}\right)}{\left(\sum_{i\in S_+^{\star}}v_i\right)\cdot \left(\sum_{i\in S_+^{\star}}v_i^{\mathrm{LCB}}\right)} + \frac{v_i - v_i^{\mathrm{LCB}}}{\sum_{i\in S_+^{\star}}v_i^{\mathrm{LCB}}} \\
        &\leq 2\cdot  \frac{\sum_{i\in S_+^{\star}} v_i - v_i^{\mathrm{LCB}}}{\sum_{i\in S_+^{\star}}v_i^{\mathrm{LCB}}}. \label{eq:_proof_algorithm_design_1_4}
    \end{align}
    Here the last inequality uses the fact that $\max\{a,b\}\leq a+b$ for any $a\geq 0$ and $b\geq 0$.
    On the other hand, by Lemma~\ref{lem:_function_upper_bound}, we have that for any $\lambda\in[c\cdot r_{\max}, r_{\max}/\rho]$, it holds that 
    \begin{align}
        \lambda\cdot\exp(r_{\max}/\lambda)\cdot \big(1 - \exp(-r_{\max}/\lambda) \big)\leq c\cdot r_{\max} \cdot\big(\exp(1/c)-1\big),\label{eq:_proof_algorithm_design_1_5}
    \end{align}
    Consequently, combining \eqref{eq:_proof_algorithm_design_1_4} and \eqref{eq:_proof_algorithm_design_1_5}, we can bound \text{Term (ii)} by 
    \begin{align}
        \text{Term (ii)} \leq 2c\cdot r_{\max} \cdot\big(\exp(1/c)-1\big) \cdot\frac{\sum_{i\in S^{\star}_+ }v_i - v_i^{\mathrm{LCB}}}{\sum_{i\in S_+^{\star}}v_i^{\mathrm{LCB}}}.\label{eq:_proof_algorithm_design_1_6}
    \end{align}
    
    \paragraph{Combining the bounds.} 
    Now combining \eqref{eq:_proof_algorithm_design_1_2+} and \eqref{eq:_proof_algorithm_design_1_6}, we can derive that for any $c\in[0,1/\rho]$, 
    \begin{align}
        R_{\rho}(S^{\star};\bv) - R_{\rho}(S^{\star};\bv^{\mathrm{LCB}}) &\leq \max\big\{\text{Term (i)}, \text{Term (ii)}\big\} \\
        &= c\cdot r_{\max}\cdot \max\big\{2, \exp(1/c)-1\big\}\cdot \frac{\sum_{i\in S^{\star}_+ }v_i - v_i^{\mathrm{LCB}}}{\sum_{i\in S_+^{\star}}v_i^{\mathrm{LCB}}}.\label{eq:_proof_algorithm_design_1_7}
    \end{align}
    In the meanwhile, by using the same argument as in bounding \text{Term (i)}, i.e., \eqref{eq:_proof_algorithm_design_1_2+}, except that the dual variable $\lambda$ is in the full range $[0, r_{\max}/\rho]$, we can obtain that 
    \begin{align}
        R_{\rho}(S^{\star};\bv) - R_{\rho}(S^{\star};\bv^{\mathrm{LCB}}) \leq \frac{r_{\max}}{\rho}\cdot  \frac{\sum_{i\in S_+^{\star}}v_i - v_i^{\mathrm{LCB}}}{\sum_{i\in S^{\star}_+}v_i^{\mathrm{LCB}}}.\label{eq:_proof_algorithm_design_1_8}
    \end{align}
    Consequently, by combining \eqref{eq:_proof_algorithm_design_1_0}, \eqref{eq:_proof_algorithm_design_1_7} and \eqref{eq:_proof_algorithm_design_1_8}, we can obtain that 
    \begin{align}
        R_{\rho}(S^{\star};\bv) - R_{\rho}(\widehat S;\bv) & \leq R_{\rho}(S^{\star};\bv) - R_{\rho}(S^{\star};\bv^{\mathrm{LCB}})\\
        & \leq r_{\max}\cdot \min\left\{ \inf_{c\in[0,1/\rho]}\Big\{c\cdot \max\big\{2, \exp(1/c)-1\big\}\Big\}, \frac{1}{\rho}\right\}\cdot \frac{\sum_{i\in S_+^{\star}}v_i - v_i^{\mathrm{LCB}}}{\sum_{i\in S^{\star}_+}v_i^{\mathrm{LCB}}}.
    \end{align}
    Now applying Lemma~\ref{lem:_concentration_2}, we can obtain the desired result as follows, with probability at least $1-2\delta$,
    \begin{align}
        &R_{\rho}(S^{\star};\bv) - R_{\rho}(\widehat S;\bv)\\
        &\qquad \leq  \frac{\mathsf{C}(\rho)\cdot r_{\max}}{1+v(S^{\star})}\cdot \sum_{i\in S^{\star}}\left(16 \sqrt{\frac{v_i(1+v_i)\log(3N/\delta)}{\sum_{k=1}^n\mathbf{1}\{i_k\in S_k\}\cdot (1+v(S_k))^{-1}}} + \frac{88(1+v_i)\log(3N/\delta)}{\sum_{k=1}^n\mathbf{1}\{i_k\in S_k\}\cdot (1+v(S_k))^{-1}}\right).
    \end{align}
    where the function $\mathsf{C}(\rho)$ is defined as $\mathsf{C}(\rho) := \min\{ \inf_{c\in[0,1/\rho]}\{c\cdot \max\{2, \exp(1/c)-1\}\}, \rho^{-1}\}$, and $v(S)$ is defined as $v(S):=\sum_{i\in S}v_i$
    for any assortment $S\subset[N]$.
\end{proof}

\subsection{Proof of Theorem~\ref{thm:_algorithm_design_1}}\label{subsec:_proof_algorithm_design_1}

With the general results in Theorem~\ref{thm:_algorithm_design_1_general}, we are ready to specify it to Theorem~\ref{thm:_algorithm_design_1}.

\proof[Proof of Theorem~\ref{thm:_algorithm_design_1}] 
We prove the non-uniform revenue case and the uniform revenue case respectively as following. 
Since $\mathsf{C}(\rho)=\cO(1)$, it suffices to calculate the upper bounds for   
\begin{align}
    &\frac{1}{1+v(S^{\star})}\cdot \sum_{i\in S^{\star}} \sqrt{\frac{v_i(1+v_i)\log(3N/\delta)}{\sum_{k=1}^n\mathbf{1}\{i_k\in S_k\}\cdot (1+v(S_k))^{-1}}},\\
    & \frac{1}{1+v(S^{\star})}\cdot\sum_{i\in S^{\star}} \frac{(1+v_i)\log(3N/\delta)}{\sum_{k=1}^n\mathbf{1}\{i_k\in S_k\}\cdot (1+v(S_k))^{-1}},\label{eq: term to bound 1}
\end{align}
respectively, which results in different upper bounds under different setups.

\paragraph{Non-uniform revenue case.}
For the first term in Theorem~\eqref{eq: term to bound 1}, we have the following upper bound,
\begin{align*}
&\frac{1}{1+\sum_{j\in S^\star} v_j}\cdot \sum_{i\in S^\star} \sqrt{\frac{v_i(1+v_i)\log (3N/\delta) }{\sum_{k = 1}^n \bm{1}\{i\in S_k\}(1+\sum_{j\in S_k}v_j)^{-1}}} \\
& \qquad\leq \frac{1}{1+\sum_{j\in S^\star} v_j}\cdot\sqrt{(1+v_{\max})(1+Kv_{\max})\log(3N/\delta)} \cdot \sum_{i\in S^\star} \sqrt{\frac{v_i}{n_i}}\\
&\qquad \leq \sqrt{\frac{(1+v_{\max})(1+Kv_{\max})\log(3N/\delta)}{1+\sum_{j\in S^\star} v_j}}  \cdot\sqrt{\sum_{i\in S^\star}\frac{1}{n_i}}\\
&\qquad \leq (1+v_{\max})K\cdot\sqrt{\frac{\log (3N/\delta)}{\min_{i\in S^\star}n_i}},
\end{align*}
where the first inequality uses that $v_i\leq v_{\max}$ and the fact that 
\begin{align}
    1+\sum_{j\in S_k}v_j \leq 1+Kv_{\max},\quad \forall k\in[n],
\end{align}
the second inequality uses Cauchy-Schwartz inequality that 
\begin{align}
    \sum_{i\in S^{\star}}\sqrt{\frac{v_i}{n_i}}\leq \sqrt{\sum_{i\in S^{\star}}v_i}\cdot\sqrt{\sum_{i\in S^{\star}}\frac{1}{n_i}}\leq \sqrt{1+\sum_{i\in S^{\star}}v_i}\cdot\sqrt{\sum_{i\in S^{\star}}\frac{1}{n_i}},
\end{align}
and the last inequality  uses the fact that $1+Kv_{\max}\leq K(1+v_{\max})$ and the inequalities that 
\begin{align}
    \sqrt{\sum_{i\in S^\star}\frac{1}{n_i}} \leq \sqrt{\frac{K}{\min_{i\in S^{\star}}n_i}},\quad 1\leq 1+\sum_{j\in S^{\star}}v_j.
\end{align}
For the second term in \eqref{eq: term to bound 1}, we similarly have the following  upper bound,
\begin{align*}
\frac{1}{1+\sum_{j\in S^\star} v_j}\cdot\sum_{i\in S^\star}   \frac{(1+v_i)\log (3N/\delta) }{\sum_{k = 1}^n \bm{1}\{i\in S_k\}(1+\sum_{j\in S_k}v_j)^{-1}} & \leq  \frac{(1+v_{\max})^2K^2\log(3N/\delta)}{\min_{i\in S^\star} n_i}.
\end{align*}
This finishes the proof of Theorem~\ref{thm:_algorithm_design_1} for the non-uniform revenue case.

\paragraph{Uniform revenue case.}
In the uniform revenue setting, due to Proposition~\ref{prop:_optimal_set_uniform}, we always have that
\begin{align}
    \sum_{i\in S^\star} v_i \geq \sum_{i\in S} v_i,\quad \text{for any assortment $S$ with $|S|\leq K$}.\label{eq: uniform condition}
\end{align}
Consequently, the first term in \eqref{eq: term to bound 1} can be upper bounded as
\allowdisplaybreaks
\begin{align*}
  &\frac{1}{1+\sum_{j\in S^\star} v_j}\cdot\sum_{i\in S^\star}  \sqrt{\frac{v_i(1+v_i)\log (3N/\delta) }{\sum_{k = 1}^n \bm{1}\{i\in S_k\}(1+\sum_{j\in S_k}v_j)^{-1}}}\\
  &\qquad \leq  \sqrt{\frac{(1+v_{\max})\log(3N/\delta)}{1+\sum_{j\in S^\star} v_j}} \cdot \sum_{i\in S^\star} \sqrt{\frac{v_i }{n_i}} \\
&\qquad \leq \sqrt{\frac{(1+v_{\max})K\log(3N/\delta)}{\min_{i\in S^\star} n_i}},
\end{align*}
where the first inequality uses \eqref{eq: uniform condition} and $v_i\leq v_{\max}$, the second inequality uses Cauchy-Schwartz inequality that 
\begin{align}
    \sum_{i\in S^{\star}}\sqrt{\frac{v_i}{n_i}}\leq \sqrt{\sum_{i\in S^{\star}}v_i}\cdot\sqrt{\sum_{i\in S^{\star}}\frac{1}{n_i}}\leq \sqrt{1+\sum_{i\in S^{\star}}v_i}\cdot\sqrt{\sum_{i\in S^{\star}}\frac{1}{n_i}},
\end{align}
together with the fact that
\begin{align}
    \sqrt{\sum_{i\in S^\star}\frac{1}{n_i}} \leq \sqrt{\frac{K}{\min_{i\in S^{\star}}n_i}}.
\end{align}
Similarly, for the second term in \eqref{eq: term to bound 1}, we have the following upper bound that
\begin{align*}
    \frac{1}{1+\sum_{i\in S^\star} v_i}\cdot \sum_{i\in S^\star} \frac{(1+v_i)\log(3N/\delta) }{\sum_{k = 1}^n \bm{1}\{i\in S_k\}(1+\sum_{j\in S_k}v_j)^{-1}} &\leq \frac{(1+v_{\max})K\log(3N/\delta)}{\min_{i\in S^\star} n_i},
\end{align*}
as desired, finishing the proof of Theorem~\ref{thm:_algorithm_design_1} for the uniform revenue case.
\endproof

\subsection{Proof of Theorem~\ref{thm:_lower_bound_1}}\label{subsec:_proof_lower_bound_1}

In this section, we prove the suboptimality lower bounds for Example~\ref{exp: jin}. 
We first prove the lower bound for the general non-uniform revenue case, after which we prove the lower bound for the uniform revenue case.

\begin{proof}[\textcolor{blue}{Proof of Theorem~\ref{thm:_lower_bound_1}: non-uniform revenue case}]
    We utilize the construction of the hard instances proposed by \cite{han2025learning}.
    Consider that we are given the number of products $N$, the cardinality constraint $K$,  the effective sample size $n_{\min}$, and the robust set size $\rho$ satisfying
    \begin{align}
        K\geq C_0,\quad N\geq 5K,\quad n_{\min}\geq K^2,\quad \underline{c}_{\rho}\cdot\log 2\leq \rho\leq (1-\overline{c}_{\rho})\cdot\log 2.
    \end{align}
    Then we can parameterize each instance of the robust assortment optimization problem by the attraction parameters of the nominal choice model and the revenues, denoted by $(\bv,\br)\in\RR^{N}_+\times\RR^N_+$.
    The optimal robust assortment with size constraint $K$ associated with $(\bv,\br)$ is denoted by $S_{\bv,\br}^{\star}$.
    The data generation distribution conditioning on an assortment $S$, i.e., the nominal choice probability given $S$, is denoted by $\PP_{\bv}(\cdot|S)$.

    To help present the whole proof of Theorem~\ref{thm:_lower_bound_1}, we divide the proof into the following steps.

    \paragraph{Step 1: construction of hard instance class and offline dataset.} 
    We denote the hard instance class we are going to construct as $\cV\subset\RR_+^N\times\RR_+^N$. 
    Consider the divisions of the whole item set $[N]$ in the form of 
    \begin{align}
        [N]= \cN^{\star}\cup \cN^{\mathrm{c}}\cup \cN^0,\quad \text{with}\quad |\cN^{\star}|=K, \,\,\cN^{\star}\cup\cN_\mathrm{c} = [4K],\,\,\cN^0=[N]\setminus [4K].\label{eq: division rule}
    \end{align}
    By such a rule, the division is fully decided by the choice of the $K$-sized subset $\cN^{\star}$ of $[4K]$.
    The elements of $\cV$ correspond to different ways of dividing $[N]$ following such a rule.
    More specifically, given any partition $(\cN^{\star}, \cN^{\mathrm{c}}, \cN^0)$ satisfying the rule \eqref{eq: division rule}, an instance $(\bv,\br)$ is defined by 
    \begin{align}
        \big(\bv_{\cN^{\star}} ,\br_{\cN^{\star}}\big) := \left(\frac{1}{K}+\epsilon,r_{\max}\right),\quad \big(\bv_{\cN^{\mathrm{c}}} ,\br_{
        \cN^{\mathrm{c}}}\big) = \left(\frac{1}{K},r_{\max}\right),\quad \big(\bv_{\cN^0} ,\br_{\cN^0}\big) = (1,0),\label{eq:_proof_lower_bound_0}
    \end{align}
    where the notation $\bv_{\cN}=v$ is an abbreviate for ``$v_{i}=v$ for any $i\in\cN$''.
    We choose the parameter $\epsilon$ as 
    \begin{align}
        \epsilon:=\sqrt{\frac{C_1}{80 n_{\min}}},\label{eq: epsilon jin}
    \end{align}
    where $C_1>0$ is an absolute constant specified later in Lemma~\ref{lem-packing-number}.

    To specify the collection of divisions $(\cN^{\star}, \cN^{\mathrm{c}}, \cN^0)$ that determines $\cV$, we invoke the following lemma.

    \begin{lemma}[Packing number]\label{lem-packing-number}
        For a sufficiently large $K$, there exists a collection $\mathcal{F}$ of $K$-sized subsets of $[4K]$ so that $\Delta(S,S') \geq \lfloor K/4\rfloor$ for any $S, S' \in \mathcal{F}$ and that $\log \lvert \mathcal{F} \rvert \geq C_1\cdot K$ for some absolute constant $C_1$.
    \end{lemma}

    \begin{proof}[Proof of Lemma~\ref{lem-packing-number}]
        Please refer to Lemma 5.1 in \cite{han2025learning} for a detailed proof.
    \end{proof}

    We construct the hard instance class $\cV$ induced by the packing set of $K$-sized subset of $[4K]$ in Lemma~\ref{lem-packing-number}.
    \begin{align}
        \cV := \Big\{\text{$(\bv,\br)$ induced by $(\cN^{\star}, \cN^{\mathrm{c}}, \cN^0)$ via \eqref{eq:_proof_lower_bound_0}}\,\Big|\,\text{$\cN^{\star}\in \cF$ with $\cF$ given by Lemma~\ref{lem-packing-number}}\Big\}.\label{eq: hard instance class}
    \end{align}
    The following shows that the optimal robust assortment induced by a division $(\cN^{\star}, \cN^{\mathrm{c}}, \cN^0)$ is exactly $\cN^{\star}$.

    \begin{lemma}[Optimal robust assortment for class \eqref{eq: hard instance class}]\label{lem: optimal robust assortment}
        For any problem instance $(\bv,\br)\in\cV$  associated with the division $(\cN^{\star}, \cN^{\mathrm{c}}, \cN^0)$, its optimal robust assortment $S^{\star}_{\bv,\br}$ in the sense of Example~\ref{exp: jin} is $S^{\star}_{\bv,\br}=\cN^{\star}$.
    \end{lemma}

    \begin{proof}[Proof of Lemma~\ref{lem: optimal robust assortment}]
        By definition, for any instance $(\bv,\br)$ with partitions given by $(\cN^{\star},\cN^\mathrm{c}, \cN^0)$, the robust expected revenue of any assortment $S$ with $|S|\leq K$ is given by 
        \begin{align}
            R_{\rho}(S;\bv, \br)&=\sup_{\lambda\geq 0}\left\{-\lambda\cdot\log\left(\frac{1+|S\cap \cN^0|+e^{-\frac{r_{\max}}{\lambda}}\cdot(1/K\cdot|S\cap \cN^{\mathrm{c}}| + (1/K+\epsilon)\cdot |S\cap \cN^{\star}|)}{1+|S\cap\cN^0|+1/K\cdot|S\cap \cN^{\mathrm{c}}|+(1/K+\epsilon)\cdot|S\cap \cN^{\star}|)}\right)-\lambda\rho\right\} \\
            &=\sup_{\lambda\geq 0}\left\{-\lambda\cdot\log\left(e^{-\frac{r_{\max}}{\lambda}}+\frac{(1-e^{-\frac{r_{\max}}{\lambda}})\cdot(1+|S\cap \cN^0|)}{1+|S\cap\cN^0|+1/K\cdot|S\cap \cN^{\mathrm{c}}|+(1/K+\epsilon)\cdot|S\cap \cN^{\star}|)}\right)-\lambda\rho\right\}\\
            &=\sup_{\lambda\geq 0}\left\{-\lambda\cdot\log\left(e^{-\frac{r_{\max}}{\lambda}}+\frac{1-e^{-\frac{r_{\max}}{\lambda}}}{1+\frac{1}{K}\cdot\frac{|S\cap \cN^{\mathrm{c}}|}{1+|S\cap \cN^0|} + \left(\frac{1}{K}+\epsilon\right)\cdot\frac{|S\cap \cN^{\star}|}{1+|S\cap \cN^0|}}\right)-\lambda\rho\right\}\\
            &\leq \sup_{\lambda\geq 0}\left\{-\lambda\cdot\log\left(e^{-\frac{r_{\max}}{\lambda}}+\frac{1-e^{-\frac{r_{\max}}{\lambda}}}{1+\left(\frac{1}{K}+\epsilon\right)\cdot|\cN^{\star}|}\right)-\lambda\rho\right\},
        \end{align}
        where the last inequality becomes equality only when the assortment $S = \cN^{\star}$. This proves Lemma~\ref{lem: optimal robust assortment}.
    \end{proof}
    
    Finally, we construct the observational data $\mathbb{D}=\{(S_k,i_k)\}_{k=1}^n$.
    This boils down to specify the assortments $\{S_k\}_{k=1}^n$. 
    We let $n = 4Kn_{\min}$ and define each $S_k$ as
    \begin{align}
        S_k :=\left\{\left\lceil\frac{k}{n_{\min}}\right\rceil,4K+2,\cdots,5K\right\}.\label{eq: offline dataset proof}
    \end{align}
    Intuitively, each $S_k$ contains $K-1$ items in $\cN^0$ and just one item in $\cN^{\star}\cup\cN^{\mathrm{c}}$.
    Each item in $\cN^{\star}\cup\cN^{\mathrm{c}}$ appears in $n_{\mathrm{min}}$ of the $n$ assortments.
    Therefore, for any instance $(\bv,\br)\in\cV$, the effective sample size $\min_{i\in S^{\star}_{\bv,\br}} n_i$ is exactly $n_{\min}$.  
    The choices $\{i_k\}_{k=1}^n$ are independently generated obeying $\PP_{\bv}(\cdot|S_k)$ for $k\in[n]$.

    \paragraph{Step 2: establish the minimax lower bound.}
    For any $(\bv,\br)\in\cV$ and any assortment $S\subseteq[N]$ satisfying $|S|\leq K$, we need to lower bound the suboptimality of $S$ under instance $(\bv,\br)$.
    Actually, it suffices to consider assortments $S\subseteq \cN^{\star}\cup\cN^{\mathrm{c}}=[4K]$ with $|S|= K$. 
    This is because using a similar argument as in the proof of Lemma~\ref{lem: optimal condition jin}, we can show that for any $S\subseteq [N]$ with $|S|\leq K$, under any instance $(\bv,\br)$,
    \begin{align}
            R_{\rho}(S;\bv, \br)
            &=\sup_{\lambda\geq 0}\left\{-\lambda\cdot\log\left(\exp(-r_{\max}/\lambda)+\frac{1-\exp(-r_{\max}/\lambda)}{1+\frac{1}{K}\cdot\frac{|S\cap \cN^{\mathrm{c}}|}{1+|S\cap \cN^0|} + \left(\frac{1}{K}+\epsilon\right)\cdot\frac{|S\cap \cN^{\star}|}{1+|S\cap \cN^0|}}\right)-\lambda\rho\right\}\\
            &\leq \sup_{\lambda\geq 0}\left\{-\lambda\cdot\log\left(\exp(-r_{\max}/\lambda)+\frac{1-\exp(-r_{\max}/\lambda)}{1+\frac{1}{K}\cdot(K-|S\cap \cN^{\star}|) + \left(\frac{1}{K}+\epsilon\right)\cdot|S\cap \cN^{\star}|}\right)-\lambda\rho\right\}\\
            &= R_{\rho}(\widetilde{S}(S);\bv,\br),\quad \text{where} \quad \widetilde{S}(S):= (S\cap \cN^{\star}) \cup \cN^{\mathrm{c}}_{K-|S\cap \cN^{\star}|}\subseteq \cN^{\star}\cup \cN^{\mathrm{c}},\label{eq: simplification jin}
        \end{align}
    where $(\cN^{\star},\cN^{\mathrm{c}},\cN^0)$ is the division associated with the instance $(\bv,\br)$, and $\cN^{\mathrm{c}}_{K-|S\cap \cN^{\star}|}$ is any $(K-|S\cap \cN^{\star}|)$-sized subset of $\cN^{\mathrm{c}}$. 
    That being said, it suffices to consider all assortments $S\subseteq \cN^{\star}\cup \cN^{\mathrm{c}}=[4K]$ with $|S|=K$. 
    To this end, for any such assortment, we have the following (with abbreviation $S^{\star}_{\bv,\br}$ as $S^{\star}$ if no confusion),
    \allowdisplaybreaks
    \begin{align}
        &\mathrm{SubOpt}_{\rho}(S;\bv,\br)=R_{\rho}(S^{\star}_{\bv,\br};\bv,\br) - R_{\rho}(S;\bv,\br) \\
        &\qquad=\sup_{\lambda\geq 0}\left\{-\lambda\cdot \log\left(\frac{1+\exp(-r_{\max}/\lambda)\cdot(1+K\epsilon)}{2+K\epsilon}\right)-\lambda\rho\right\}\\
        &\qquad\qquad  - \sup_{\lambda\geq 0}\left\{-\lambda\cdot \log\left(\frac{1+\exp(-r_{\max}/\lambda)\cdot(1+|S\cap S^{\star}|\cdot \epsilon)}{2+|S\cap S^{\star}|\cdot \epsilon}\right)-\lambda\rho\right\}\\
        &\qquad \geq -\lambda^{\star}_{\bv,\br}(S)\cdot \log\left(\frac{1+\exp(-r_{\max}/\lambda^{\star}_{\bv,\br}(S))\cdot(1+K\epsilon)}{2+K\epsilon}\cdot \frac{2+|S\cap S^{\star}|\cdot \epsilon}{1+\exp(-r_{\max}/\lambda^{\star}_{\bv,\br}(S))\cdot(1+|S\cap S^{\star}|\cdot \epsilon)}\right)\\
        &\qquad\geq -\lambda^{\star}_{\bv,\br}(S)\cdot \left(\frac{1+\exp(-r_{\max}/\lambda^{\star}_{\bv,\br}(S))\cdot(1+K\epsilon)}{2+K\epsilon}\cdot \frac{2+|S\cap S^{\star}|\cdot \epsilon}{1+\exp(-r_{\max}/\lambda^{\star}_{\bv,\br}(S))\cdot(1+|S\cap S^{\star}|\cdot \epsilon)} - 1\right)\\
        &\qquad = \lambda^{\star}_{\bv,\br}(S)\cdot \frac{(1-\exp(-r_{\max}/\lambda_{\bv,\br}^{\star}(S)))\cdot(2+|S\cap S^{\star}|\cdot \epsilon)}{1+\exp(-r_{\max}/\lambda^{\star}_{\bv,\br}(S))\cdot(1+|S\cap S^{\star}|\cdot \epsilon)}\cdot\left(\frac{1}{2+|S\cap S^{\star}|\cdot\epsilon} - \frac{1}{2+K\epsilon}\right)\\
        &\qquad  \geq \frac{\epsilon}{18}\cdot \lambda^{\star}_{\bv,\br}(S)\cdot \frac{(1-\exp(-r_{\max}/\lambda^{\star}_{\bv,\br}(S)))\cdot(2+|S\cap S^{\star}|\cdot \epsilon)}{1+\exp(-r_{\max}/\lambda^{\star}_{\bv,\br}(S))\cdot(1+|S\cap S^{\star}|\cdot \epsilon)}\cdot \Delta(S^{\star}_{\bv,\br},S)\\
        &\qquad \geq \frac{\epsilon}{18}\cdot \lambda^{\star}_{\bv,\br}(S)\cdot \big(1-\exp(-r_{\max}/\lambda^{\star}_{\bv,\br}(S))\big)\cdot \Delta(S^{\star}_{\bv,\br},S),\label{eq:_proof_lower_bound_1}
    \end{align}
    where the first equality uses the dual representation of the robust expected revenue (Proposition~\ref{prop:_dual}), the first inequality utilizes the notion of $\lambda^{\star}_{\bv,\br}(S)$ defined as 
    \begin{align}
        \lambda^{\star}_{\bv,\br}(S):=\argsup_{\lambda\geq 0}\left\{-\lambda\cdot \log\left(\frac{1+\exp(-r_{\max}/\lambda)\cdot(1+|S\cap S^{\star}|\cdot \epsilon)}{2+|S\cap S^{\star}|\cdot \epsilon}\right)-\lambda\rho\right\},\label{eq: dual proof}
    \end{align}
    the second inequality uses $\log(1+x)\leq x$ for $x>-1$, and the third inequality uses the fact that that 
    \begin{align}
        \frac{1}{2+|S\cap S^{\star}|\cdot\epsilon} - \frac{1}{2+K\epsilon} \geq \frac{(K-|S\cap S^{\star}|)\cdot \epsilon}{(2+K\epsilon)^2}\geq \frac{\epsilon}{18}\cdot \underbrace{\Delta(S^{\star}_{\bv,\br},S)}_{\displaystyle{:=\Delta_{\bv,\br}(S)}}.
    \end{align}

    \paragraph{Step 2.1: upper and lower bound the dual variable.} Now we need to upper and lower bound the dual variable $\lambda^{\star}_{\bv,\br}(S)$ defined in \eqref{eq: dual proof} in order to further lower bound the right hand side of \eqref{eq:_proof_lower_bound_1}. 
    Thanks to \eqref{eq: simplification jin}, we only need consider $\lambda^{\star}_{\bv,\br}(S)$ for $(\bv,\br)\in\cV$ and most importantly $S\subseteq[4K]$ with size $|S|=K$.
    
    On the one hand, by Proposition~\ref{prop:_dual}, it follows directly that 
    \begin{align}
        \lambda^{\star}_{\bv,\br}(S) \leq B(S,\bv):= \frac{r_{\max}}{\rho} \leq \frac{r_{\max}}{\underline{c}_{\rho}\cdot\log 2},
    \end{align}
    where we use the condition on the robust set size $\rho$ that 
    \begin{align}
        \rho \geq \underline{c}_{\rho}\cdot\log2.
    \end{align}
    On the other hand, we now lower bound the dual variable as follows. 
    For simplicity, let's define the dual representation function as following, given $(\bv,\br;S)\in \cV\times 2^{[4K]}$ with $|S|=K$,
    \begin{align}
        h_{\bv,\br,S}(\lambda):= -\lambda\cdot \log\left(\frac{1+\exp(-r_{\max}/\lambda)\cdot(1+|S\cap S^{\star}|\cdot \epsilon)}{2+|S\cap S^{\star}|\cdot \epsilon}\right)-\lambda\rho.
    \end{align}
    Noting that $h_{\bv,\br,S}$ is a smooth and concave function on $(0,\infty)$, in order to prove $\lambda^{\star}_{\bv,\br}(S)\geq \underline{\lambda}$ for some $\underline{\lambda}>0$, it suffices to prove the following inequality,
    \begin{align}
        \partial_{\lambda} h_{\bv,\br,S}(\underline{\lambda}) \geq 0.
    \end{align}
    To this end, we first calculate the derivative of $h_{\bv,\br,S}$ with respect to $\lambda$ as following, 
    \allowdisplaybreaks
    \begin{align}
        \partial_{\lambda} h_{\bv,\br,S}(\lambda)&=-\rho - \log\left(\frac{1+\exp(-r_{\max}/\lambda)\cdot(1+|S\cap S^{\star}|\cdot \epsilon)}{2+|S\cap S^{\star}|\cdot \epsilon}\right) \\
        &\qquad - \frac{r_{\max}}{\lambda}\cdot\frac{\exp(-r_{\max}/\lambda)\cdot(1+|S\cap S^{\star}|\cdot \epsilon)}{1+\exp(-r_{\max}/\lambda)\cdot(1+|S\cap S^{\star}|\cdot \epsilon)}\\
        &= \log\left(2+|S\cap S^{\star}|\cdot \epsilon\right) -\rho - \log\Big(1+\exp(-r_{\max}/\lambda)\cdot(1+|S\cap S^{\star}|\cdot \epsilon)\Big)\\
        &\qquad - \frac{r_{\max}}{\lambda}\cdot\frac{\exp(-r_{\max}/\lambda)\cdot(1+|S\cap S^{\star}|\cdot \epsilon)}{1+\exp(-r_{\max}/\lambda)\cdot(1+|S\cap S^{\star}|\cdot \epsilon)}\\
        & \geq \overline{c}_{\rho}\cdot \log\left(2+|S\cap S^{\star}|\cdot \epsilon\right) - \log\Big(1+\exp(-r_{\max}/\lambda)\cdot(1+|S\cap S^{\star}|\cdot \epsilon)\Big)\\
        &\qquad - \frac{r_{\max}}{\lambda}\cdot\frac{\exp(-r_{\max}/\lambda)\cdot(1+|S\cap S^{\star}|\cdot \epsilon)}{1+\exp(-r_{\max}/\lambda)\cdot(1+|S\cap S^{\star}|\cdot \epsilon)},\label{eq: lower bound proof 1}
    \end{align}
    where the last inequality uses the condition on the robust set size $\rho$ that 
    \begin{align}
        \rho \leq (1-\overline{c}_\rho)\cdot\log2.\label{eq: lambda star upper bound}
    \end{align}
    Now we need to figure out the largest possible $\underline{\lambda}>0$ such that $\partial  h_{\bv,\br,S}(\underline{\lambda})\geq 0$ holds. 
    To this end, we consider the following reparametrization to simplify the notations and calculation, 
    \begin{align}
        x:=1+|S\cap S^{\star}|\cdot \epsilon\in[1,2],\quad t:= \exp(-r_{\max}/\lambda)\in[0,1],\quad u=t\cdot x\in[0,2],\label{eq: reparametrization jin}
    \end{align}
    where the upper bound on $x$ originates from the choice of $\epsilon$ in \eqref{eq: epsilon jin}.  
    With \eqref{eq: reparametrization jin}, we can further handle the right hand side of \eqref{eq: lower bound proof 1} as
    \allowdisplaybreaks
    \begin{align}
        \partial  h_{\bv,\br,S}(\underline{\lambda})&\geq \overline{c}_{\rho}\cdot \log(1+x) - \log(1+u) + \log\left(\frac{u}{x}\right)\cdot \frac{u}{1+u}\\
        &\geq \overline{c}_{\rho}\cdot \log(1+x) - \log(1+u) + \log\left(\frac{u}{2}\right)\cdot u,\label{eq: lower bound proof 1+}
    \end{align}
    Now consider taking $u=(c^{\dagger}_{\rho}\wedge  \overline{c}_{\rho})/2\cdot \log(1+x)\leq 2$, where $c^{\dagger}_{\rho}$ is given by solving the equation that 
    \begin{align}
        c^{\dagger}_{\rho}:=\argmax_{c\geq 0}\left\{c\cdot \log\left(\frac{\log2}{4}\cdot c\wedge\overline{c}_{\rho}\right) \geq -\overline{c}_{\rho}\right\}.\label{eq: c dagger jin}
    \end{align}
     It holds that $c^{\dagger}_{\rho}>0$ is an absolute constant which only dependents on the constant $\overline{c}_{\rho}$. 
     With such a choice of $u$, we can lower bound \eqref{eq: lower bound proof 1+} via 
     \begin{align}
         \partial  h_{\bv,\br,S}(\underline{\lambda}) &\geq \overline{c}_{\rho} \cdot \log(1+x) - \log\left(1+\frac{\overline{c}_{\rho}}{2}\cdot \log(1+x)\right) +\frac{c^{\dagger}_{\rho}}{2}\cdot\log\left(\frac{c^{\dagger}_{\rho}\wedge\overline{c}_{\rho}}{4}\cdot\log(1+x)\right)\cdot \log(1+x)\\
         & \geq \overline{c}_{\rho} \cdot \log(1+x)-\frac{\overline{c}_{\rho}}{2}\cdot \log(1+x) + \frac{c^{\dagger}_{\rho}}{2}\cdot \log\left(\frac{\log 2}{2}\cdot c^{\dagger}_{\rho}\wedge \overline{c}_{\rho}\right) \cdot\log(1+x) \\
         &\geq \overline{c}_{\rho} \cdot \log(1+x)-\frac{\overline{c}_{\rho}}{2}\cdot \log(1+x) -\frac{\overline{c}_{\rho}}{2}\cdot \log(1+x) =0,
     \end{align}
     where the first inequality plugs in the choice of $u$, the second inequality uses the inequality that $\log(1+x)\leq x$ for $x>-1$, and the fact that $x\geq 1$, the last inequality uses the property \eqref{eq: c dagger jin} that the constant $c_{\rho}^{\dagger}$ satisfies. 
     Therefore, by taking $u=(c^{\dagger}_{\rho}\wedge  \overline{c}_{\rho})/2\cdot \log(1+x)$, the associated $\lambda$ satisfies $\partial  h_{\bv,\br,S}(\lambda)\geq 0$. 
     Therefore, using the reparametrization formula \eqref{eq: reparametrization jin}, we get back to the desired $\underline{\lambda}$ as 
     \begin{align}
         \underline{\lambda}:= \frac{r_{\max}}{\log\left(\frac{x}{u}\right)}=\frac{r_{\max}}{\log\left(\frac{x}{\frac{1}{2}\cdot c^{\dagger}_{\rho}\wedge \overline{c}_{\rho}\cdot\log(1+x)}\right)}.
     \end{align}
     Consequently, we can conclude that the optimal dual variable is lower bounded by 
     \begin{align}
         \lambda^{\star}_{\bv,\br}(S)\geq \underline{\lambda}\geq \frac{r_{\max}}{\log\left(\frac{x}{\frac{1}{2}\cdot c^{\dagger}_{\rho}\wedge \overline{c}_{\rho}\cdot\log(1+x)}\right)} \geq \frac{r_{\max}}{\log\left(\frac{2}{\frac{1}{2}\cdot c^{\dagger}_{\rho}\wedge \overline{c}_{\rho}\cdot\log2}\right)},\label{eq: lambda star lower bound}
     \end{align}
     where the third inequality uses the fact that $x\in[1,2]$.
     Finally, combining \eqref{eq: lambda star upper bound} and \eqref{eq: lambda star lower bound}, we can conclude that for any $(\bv,\br;S)\in\cV\times 2^{[4K]}$ with $|S|=K$,
     \begin{align}
         \frac{r_{\max}}{\log\left(\frac{2}{\frac{1}{2}\cdot c^{\dagger}_{\rho}\wedge \overline{c}_{\rho}\cdot\log2}\right)}\leq \lambda_{\bv,\br}^{\star}(S)\leq \frac{r_{\max}}{\underline{c}_{\rho}\cdot \log 2}.\label{eq: lambda star lower and upper bound}
     \end{align}
     For a quick sanity check of the above result, we remind the reader that the above range of $\lambda^{\star}_{\bv,\br}(S)$ is a valid range since 
     \begin{align}
         \log\left(\frac{2}{\frac{1}{2}\cdot c^{\dagger}_{\rho}\wedge \overline{c}_{\rho}\cdot\log2}\right) \geq \underline{c}_{\rho}\cdot \log\left(\frac{2}{\frac{1}{2}\cdot c^{\dagger}_{\rho}\wedge \overline{c}_{\rho}\cdot\log2}\right)  \geq \underline{c}_{\rho}\cdot\log 2.
     \end{align}

    \paragraph{Step 2.2: lower bound the minimax rate of the quantity $\Delta(S^{\star}_{\bv,\br},\pi(\mathbb{D}))$.}

    To lower bound the minimax rate of $\Delta(S^{\star}_{\bv,\br},\pi(\mathbb{D}))$, it suffices to apply Fano’s Lemma and carefully upper bound the KL divergence between the product choice distributions induced by different instances $(\bv,\br),(\bv',\br')\in\cV$.  
    We conclude the results in the following lemma.

    \begin{lemma}[Minimax lower bound 1 of $\Delta(S^{\star}_{\bv,\br},\pi(\mathbb{D}))$]\label{lem: minimax rate 1}
        Under the settings inside the proof of Theorem~\ref{thm:_lower_bound_1} for general non-uniform revenue setting, by taking the parameter $\epsilon=\sqrt{C_1/80n_{\min}}$ where constant $C_1>0$ is defined in Lemma~\ref{lem-packing-number}, then there is another absolute constant $C_2>0$ such that for $K\geq C_2$ it holds that
        \begin{align}
            \inf_{\pi(\cdot):(2^{[N]}\times[N])^n\mapsto2^{[N]}}\sup_{(\bv,\br)\in\cV}\mathbb{E}_{\mathbb{D}\sim \otimes_{k=1}^n\mathbb{P}_{\bv}(\cdot|S_k)}\big[\Delta(S^{\star}_{\bv,\br},\pi(\mathbb{D}))\big] \geq  \frac{1}{128}\cdot \sqrt{\frac{C_1}{5}}\cdot\frac{K}{\epsilon\cdot\sqrt{n_{\min}}},
        \end{align}
        where the expectation is taken w.r.t. the randomness in the item choices $\{i_k\}_{k=1}^n$ in the observational data.
    \end{lemma}

    \begin{proof}[Proof of Lemma~\ref{lem: minimax rate 1}]
        To prove the minimax rate for the difference $\Delta(S^{\star}_{\bv,\br},\pi(\mathbb{D}))$, we invoke the Fano's lemma (Lemma~\ref{lem-fano}).
        Specifically, in Lemma~\ref{lem-fano}, if we choose $\Theta := 2^{[N]}$, $\Gamma := \{S^{\star}_{\bv,\br}: (\bm r, \bm v) \in \mathcal{V}\}$, $\rho(\cdot,\cdot) := \Delta(\cdot,\cdot)$, then by the definition of $\cV$ and Lemma~\ref{lem-packing-number}, $\Gamma$ is $2\delta_0$-separated with $\delta_0 = K/8$. 
        We define the distribution associated with each element in $\Gamma$ as
        $$\mathbb{P}_{\bv,\br}(\cdot): = \bigotimes_{k=1}^n\PP_{\bv}(\cdot|S_k)=\bigotimes_{\ell=1}^{4K}\bigotimes_{j = 1}^{n_{\min}} \mathbb{P}_{\bm v}(\cdot|S_{j}^{(\ell)}), \quad\text{where}\quad S_{j}^{(\ell)} = S_{(\ell-1)\cdot n_{\min}+j}.$$ 
        Since we have that $\log M := \log|\Gamma| = \log \lvert \mathcal{V}\rvert \geq C_1K$ for some absolute constant $C_1$ (Lemma~\ref{lem-packing-number}), Lemma~\ref{lem-fano} implies that when $K\geq C_2:=\max\{C_0, C_1^{-1}2\log 2\} $, it holds that \begin{align*}
                \inf_{\pi(\cdot):(2^{[N]}\times[N])^n\mapsto2^{[N]}}\sup_{(\bv,\br)\in\cV}\mathbb{E}_{\mathbb{D}\sim \PP_{\bv,\br}(\cdot)}\big[\Delta(S^{\star}_{\bv,\br},\pi(\mathbb{D}))\big]\geq \frac{K}{8}\cdot\left(\frac{1}{2} - \frac{\sum_{(\bm r, \bm v)\in \mathcal{V},(\bm r', \bm v')\in \mathcal{V}}D_{\mathrm{KL}}(\mathbb{P}_{\bm r, \bm v} \lVert \mathbb{P}_{\bm r', \bm v'}) }{C_1 K\lvert \mathcal V \rvert^2} \right).
        \end{align*}
        It then remains to establish an upper bound for the summation of KL divergence terms.
        To this end, consider that for any pair $(\bm r, \bm v),(\bm r', \bm v') \in \mathcal{V}$, we have that
        \begin{align}
            D_{\mathrm{KL}}(\mathbb{P}_{\bm r, \bm v} \lVert \mathbb{P}_{\bm r', \bm v'}) &= D_{\mathrm{KL}}\left( \bigotimes_{\ell=1}^{4K}\bigotimes_{j = 1}^{n_{\min}} \mathbb{P}_{\bm v}(\cdot|S_{j}^{(\ell)})\middle\| \bigotimes_{\ell=1}^{4K}\bigotimes_{j = 1}^{n_{\min}} \mathbb{P}_{\bm v'}(\cdot|S_{j}^{(\ell)}) \right)\\
            &= n_{\min}\sum_{\ell = 1}^{4K} D_{\mathrm{KL}}\left(\mathbb{P}_{\bm v}(\cdot|S_{j}^{(\ell)})\middle\|\mathbb{P}_{\bm v'}(\cdot|S_{j}^{(\ell)})\right).\label{eq: kl decomposition}
        \end{align}
        We then have the following result to upper bound each single KL divergence term in \eqref{eq: kl decomposition}.

        \begin{lemma}[KL divergence bound 1]\label{lem-KL-new-upper-bound-non-uniform}
            For every $\ell\in[4K]$ and any pair $(\bm r, \bm v),(\bm r', \bm v') \in \mathcal{V}$, it holds that
            \begin{align*}
                D_{\mathrm{KL}}\left(\mathbb{P}_{\bm v}(\cdot|S_{j}^{(\ell)})\middle\|\mathbb{P}_{\bm v'}(\cdot|S_{j}^{(\ell)})\right) \leq 
                \begin{cases}
                    5\epsilon^2,  & \text{if }  \ell \in (S^{\star}_{\bv,\br}(\bm r, \bm v) \cup S^{\star}_{\bv,\br}(\bm r', \bm v'))\setminus(S^{\star}_{\bv,\br}(\bm r, \bm v)\cap S^{\star}_{\bv,\br}(\bm r', \bm v')), \\
                    0,& \text{otherwise.}
                \end{cases} 
            \end{align*}
        \end{lemma}

        \begin{proof}[Proof of Lemma~\ref{lem-KL-new-upper-bound-non-uniform}]
            Please refer to Lemma~5.4 in \cite{han2025learning} for a detailed proof.
        \end{proof}
        Therefore, with \eqref{eq: kl decomposition} and Lemma~\ref{lem-KL-new-upper-bound-non-uniform}, we can finally obtain that 
        \begin{align}
            \inf_{\pi(\cdot):(2^{[N]}\times[N])^n\mapsto2^{[N]}}\sup_{(\bv,\br)\in\cV}\mathbb{E}_{\mathbb{D}\sim \PP_{\bv,\br}(\cdot)}\big[\Delta(S^{\star}_{\bv,\br},\pi(\mathbb{D}))\big]&\geq \frac{K}{8}\cdot\left(\frac{1}{2} - \frac{\max_{(\bm r, \bm v)\in \mathcal{V},(\bm r', \bm v')\in \mathcal{V}}D_{\mathrm{KL}}(\mathbb{P}_{\bm r, \bm v} \lVert \mathbb{P}_{\bm r', \bm v'}) }{C_1 K} \right)\\
            &\geq \frac{K}{8}\cdot\left(\frac{1}{2} - \frac{n_{\min}\cdot 4K\cdot 5\epsilon^2}{C_1K}\right) \\
            &\geq \frac{K}{32} = \frac{1}{128}\cdot \sqrt{\frac{C_1}{5}}\cdot\frac{K}{\epsilon\cdot\sqrt{n_{\min}}},
        \end{align}
        where the second inequality uses Lemma~\ref{lem-KL-new-upper-bound-non-uniform} and \eqref{eq: kl decomposition}, while the last inequality and the last equality both use the choice of $\epsilon$ that $
        \epsilon = \sqrt{C_1/80n_{\min}}$. 
        This completes the proof of Lemma~\ref{lem: minimax rate 1}.
    \end{proof}

    \paragraph{Step 2.3: finishing the proof.}
    With all the preparations above, we can lower bound the minimax lower bound of the expected suboptimality gap by 
    \begin{align}
        &\inf_{\pi(\cdot):(2^{[N]}\times[N])^n\mapsto2^{[N]}}\sup_{(\bv,\br)\in\cV}\mathbb{E}_{\mathbb{D}\sim \otimes_{k=1}^n\mathbb{P}_{\bv}(\cdot|S_k)}\big[\mathrm{SupOpt}_{\rho}(\pi(\mathbb{D});\bv,\br) \big]\\
        &\qquad  \geq \inf_{\pi(\cdot):(2^{[N]}\times[N])^n\mapsto2^{[N]}}\sup_{(\bv,\br)\in\cV}\mathbb{E}_{\mathbb{D}\sim \otimes_{k=1}^n\mathbb{P}_{\bv}(\cdot|S_k)}\big[\mathrm{SupOpt}_{\rho}(\widetilde{S}(\pi(\mathbb{D}));\bv,\br) \big]\\
        &\qquad \geq \inf_{\pi(\cdot):(2^{[N]}\times[N])^n\mapsto2^{[N]}}\sup_{(\bv,\br)\in\cV}\bigg\{\frac{\epsilon}{18}\cdot\mathbb{E}_{\mathbb{D}\sim \otimes_{k=1}^n\mathbb{P}_{\bv}(\cdot|S_k)}\Big[\lambda^{\star}_{\bv,\br}(\widetilde{S}(\pi(\mathbb{D})))\\
        &\qquad\qquad \cdot \Big(1-\exp\big(-r_{\max}/\lambda^{\star}_{\bv,\br}(\widetilde{S}(\pi(\mathbb{D})))\big)\Big)\cdot \Delta(S^{\star}_{\bv,\br},\widetilde{S}(\pi(\mathbb{D}))) \Big]\bigg\}\\
        &\qquad \geq \frac{\epsilon}{18}\cdot \inf_{S\subseteq [4K],|S|=K}\sup_{(\bv,\br)\in\cV}\lambda^{\star}_{\bv,\br}(S)\cdot \big(1-\exp(-r_{\max}/\lambda^{\star}_{\bv,\br}(S))\big)\\
        &\qquad\qquad\cdot  \inf_{\pi(\cdot):(2^{[N]}\times[N])^n\mapsto2^{[N]}}\sup_{(\bv,\br)\in\cV}\mathbb{E}_{\mathbb{D}\sim \otimes_{k=1}^n\mathbb{P}_{\bv}(\cdot|S_k)}\big[\Delta(S^{\star}_{\bv,\br},\pi(\mathbb{D}))\big] \\
        &\qquad \geq \frac{1}{2304}\cdot\sqrt{\frac{C_1}{5}}\cdot \big(1-2^{-\underline{c}_{\rho}}\big)\cdot \frac{r_{\max}}{\log\left(\frac{2}{\frac{1}{2}\cdot c^{\dagger}_{\rho}\wedge \overline{c}_{\rho}\cdot\log2}\right)} \cdot \frac{K}{\sqrt{n_{\min}}}.
    \end{align}
    where the first inequality uses \eqref{eq: simplification jin} to reduce any potential output $\pi(\mathbb{D})$ into an assortment $\widetilde{S}(\pi(\mathbb{D}))\subseteq[4K]$ with size $K$, the second inequality uses inequality \eqref{eq:_proof_lower_bound_1}, 
    the third inequality uses the property of $\widetilde{S}(\pi(\mathbb{D}))$ and the fact that 
    \begin{align}
        &\inf_{\pi(\cdot):(2^{[N]}\times[N])^n\mapsto2^{[N]}}\sup_{(\bv,\br)\in\cV}\mathbb{E}_{\mathbb{D}\sim \otimes_{k=1}^n\mathbb{P}_{\bv}(\cdot|S_k)}\big[\Delta(S^{\star}_{\bv,\br},\widetilde{S}(\pi(\mathbb{D})))\big]\\
        &\qquad \geq \inf_{\pi(\cdot):(2^{[N]}\times[N])^n\mapsto2^{[N]}}\sup_{(\bv,\br)\in\cV}\mathbb{E}_{\mathbb{D}\sim \otimes_{k=1}^n\mathbb{P}_{\bv}(\cdot|S_k)}\big[\Delta(S^{\star}_{\bv,\br},\pi(\mathbb{D}))\big],\label{eq: removing simplification}
    \end{align}
    the last inequality uses the upper and lower bounds of the optimal dual variable $\lambda^{\star}_{\bv,\br}(S)$ for any assortment $S\subseteq[4K]$ with size $|S|=K$ and any instance $(\bv,\br)\in\cV$ (see \eqref{eq: lambda star lower and upper bound} in \textbf{Step 2.1}), together with Lemma~\ref{lem: minimax rate 1}.
    Now using the fact that $n_{\min} = \min_{i\in S^{\star}_{\bv,\br}}n_i$ according to our choice of the dataset $\mathbb{D}$, we can
    complete the proof of Theorem~\ref{thm:_lower_bound_1} for the general non-uniform revenue case.
\end{proof}

\begin{proof}[\textcolor{blue}{Proof of Theorem~\ref{thm:_lower_bound_1}: uniform revenue case}]
    The proofs for the uniform revenue case can follow a similar approach to the general non-uniform revenue case, but it requires a slightly different construction of the hard instance class to ensure the condition of uniform revenue. 
    More specifically, we still rely on the division rule \eqref{eq: division rule}. 
    Then given any division $(\cN^{\star},\cN^{\mathrm{c}},\cN^0)$ satisfying \eqref{eq: division rule}, we define an instance $(\bv,\br)$ by 
    \begin{align}
        \big(\bv_{\cN^{\star}} ,\br_{\cN^{\star}}\big) := \left(\frac{1}{K}+\epsilon',r_{\max}\right),\quad \big(\bv_{\cN^{\mathrm{c}}} ,\br_{
        \cN^{\mathrm{c}}}\big) = \left(\frac{1}{K},r_{\max}\right),\quad \big(\bv_{\cN^0} ,\br_{\cN^0}\big) = \left(\frac{1}{K},r_{\max}\right),\label{eq:_proof_lower_bound_2}
    \end{align}
    where we choose the parameter $\epsilon'$ by the following,
    \begin{align}
        \epsilon':=\sqrt{\frac{C_1}{80K n_{\min}}}.\label{eq: epsilon jin uniform}
    \end{align}
    with the constant $C_1$ specified in Lemma~\ref{lem-packing-number}.
    Now we construct the hard instance class $\cV_{\mathrm{uni}}$ as following, 
    \begin{align}
        \cV_{\mathrm{uni}} := \Big\{\text{$(\bv,\br)$ induced by $(\cN^{\star}, \cN^{\mathrm{c}}, \cN^0)$ via \eqref{eq:_proof_lower_bound_2}}\,\Big|\,\text{$\cN^{\star}\in \cF$ with $\cF$ given by Lemma~\ref{lem-packing-number}}\Big\}.\label{eq: hard instance class uniform}
    \end{align}
    We have the following result for the optimal robust assortment associated with instances in the class $\cV_{\mathrm{uni}}$.

    \begin{lemma}[Optimal robust assortment for class \eqref{eq: hard instance class uniform}]\label{lem: optimal robust assortment uniform}
        For any problem instance $(\bv,\br)\in\cV_{\mathrm{uni}}$  associated with the division $(\cN^{\star}, \cN^{\mathrm{c}}, \cN^0)$, its optimal robust assortment $S^{\star}_{\bv,\br}$ is given by $S^{\star}_{\bv,\br}=\cN^{\star}$.
    \end{lemma}

    \begin{proof}[Proof of Lemma~\ref{lem: optimal robust assortment uniform}]
        By definition, for any instance $(\bv,\br)$ with partitions given by $(\cN^{\star},\cN^\mathrm{c}, \cN^0)$, the robust expected revenue of any assortment $S$ is given by 
        \begin{align}
            R_{\rho}(S;\bv, \br)&=\sup_{\lambda\geq 0}\left\{-\lambda\cdot\log\left(\frac{1+e^{-\frac{r_{\max}}{\lambda}}\cdot(1/K\cdot (|S\cap \cN^{\mathrm{c}}| + |S\cap \cN^0|)+ (1/K+\epsilon')\cdot |S\cap \cN^{\star}|)}{1+1/K\cdot(|S\cap \cN^{\mathrm{c}}|+|S\cap\cN^0|)+(1/K+\epsilon')\cdot|S\cap \cN^{\star}|)}\right)-\lambda\rho\right\} \\
            &=\sup_{\lambda\geq 0}\left\{-\lambda\cdot\log\left(e^{-\frac{r_{\max}}{\lambda}}+\frac{1-e^{-\frac{r_{\max}}{\lambda}}}{1+1/K\cdot(|S\cap \cN^{\mathrm{c}}|+|S\cap\cN^0|)+(1/K+\epsilon')\cdot|S\cap \cN^{\star}|)}\right)-\lambda\rho\right\}\\
            &\leq \sup_{\lambda\geq 0}\left\{-\lambda\cdot\log\left(e^{-\frac{r_{\max}}{\lambda}}+\frac{1-e^{-\frac{r_{\max}}{\lambda}}}{1+\left(\frac{1}{K}+\epsilon'\right)\cdot|\cN^{\star}|}\right)-\lambda\rho\right\},
        \end{align}
        where the last inequality becomes equality only when the assortment $S = \cN^{\star}$. This proves Lemma~\ref{lem: optimal robust assortment uniform}.
    \end{proof}
    Therefore, Lemma~\ref{lem: optimal robust assortment uniform} tells us that under the new hard instance class $\cV_{\mathrm{uni}}$, the optimal robust assortment of an instance associated with division $(\cN^{\star},\cN^\mathrm{c}, \cN^0)$ is still $S_{\bv,\br}^{\star}=\cN^{\star}$.
    In view of this, the rest of the proofs for the uniform case align well with the proofs for the general non-uniform case, with exceptional difference we discuss in detail in the sequel.  
    
    First of all, we construct the observational dataset $\mathbb{D}$ in the same way as in \eqref{eq: offline dataset proof}. 
    Then we can lower bound the suboptimality gap of any assortment $S\subseteq[N]$ with $|S|\leq K$ as following. 
    Noting that for any such assortment $S$, we can upper bound its robust expected revenue via 
    \begin{align}
        R_{\rho}(S;\bv, \br)
            &=\sup_{\lambda\geq 0}\left\{-\lambda\cdot\log\left(e^{-\frac{r_{\max}}{\lambda}}+\frac{1-e^{-\frac{r_{\max}}{\lambda}}}{1+1/K\cdot(|S\cap \cN^{\mathrm{c}}|+|S\cap\cN^0|)+(1/K+\epsilon')\cdot|S\cap \cN^{\star}|)}\right)-\lambda\rho\right\}\\
            &\leq \sup_{\lambda\geq 0}\left\{-\lambda\cdot\log\left(e^{-\frac{r_{\max}}{\lambda}}+\frac{1-e^{-\frac{r_{\max}}{\lambda}}}{1+1/K\cdot(K - |S\cap \cN^{\star}|)+(1/K+\epsilon')\cdot|S\cap \cN^{\star}|)}\right)-\lambda\rho\right\}\\
            &= R_{\rho}(\widetilde{S}(S);\bv,\br),\quad \text{where} \quad \widetilde{S}(S):= (S\cap \cN^{\star}) \cup \cN^{\mathrm{c}}_{K-|S\cap \cN^{\star}|}\subseteq \cN^{\star}\cup \cN^{\mathrm{c}},\label{eq: simplification jin uniform}
    \end{align}
     where $(\cN^{\star},\cN^{\mathrm{c}},\cN^0)$ is the division associated with the instance $(\bv,\br)$, and $\cN^{\mathrm{c}}_{K-|S\cap \cN^{\star}|}$ is any $(K-|S\cap \cN^{\star}|)$-sized subset of $\cN^{\mathrm{c}}$. 
    That being said, it suffices to consider all assortments $S\subseteq \cN^{\star}\cup \cN^{\mathrm{c}}=[4K]$ with $|S|=K$.
    Now for any assortment of this kind and for any instance $(\bv,\br)\in\cV_{\mathrm{uni}}$, we have that 
    \allowdisplaybreaks
    \begin{align}
        &\mathrm{SubOpt}_{\rho}(S;\bv,\br)=R_{\rho}(S^{\star}_{\bv,\br};\bv,\br) - R_{\rho}(S;\bv,\br) \\
        &\qquad=\sup_{\lambda\geq 0}\left\{-\lambda\cdot \log\left(\frac{1+\exp(-r_{\max}/\lambda)\cdot(1+K\epsilon')}{2+K\epsilon'}\right)-\lambda\rho\right\}\\
        &\qquad\qquad  - \sup_{\lambda\geq 0}\left\{-\lambda\cdot \log\left(\frac{1+\exp(-r_{\max}/\lambda)\cdot(1+|S\cap S^{\star}|\cdot \epsilon')}{2+|S\cap S^{\star}|\cdot \epsilon'}\right)-\lambda\rho\right\}\\
        &\qquad \geq \frac{\epsilon'}{18}\cdot \lambda^{\star}_{\bv,\br}(S)\cdot \big(1-\exp(-r_{\max}/\lambda^{\star}_{\bv,\br}(S))\big)\cdot \Delta(S^{\star}_{\bv,\br},S),\label{eq:_proof_lower_bound_3}
    \end{align}
    where the first inequality is from the dual representation of the robust expected revenue, and the inequality follows the exactly the same argument as \eqref{eq:_proof_lower_bound_1}, since the left hand side of the inequality shares the same form with the left hand side of \eqref{eq:_proof_lower_bound_1}!
    Here the optimal dual $\lambda^{\star}_{\bv,\br}(S)$ is defined in the same way as \eqref{eq: dual proof}, which means that it also enjoys the upper and lower bounds we derived in \eqref{eq: lambda star lower and upper bound}\footnote{We note that the difference in $\epsilon$ and $\epsilon'$ does not change the conclusion of the proofs for \eqref{eq:_proof_lower_bound_1} and  the bounds of $\lambda^{\star}_{\bv,\br}(S)$. Besides $\epsilon'$, the difference of the classes $\cV$ and $\cV_{\mathrm{uni}}$ only lies in the items in $\cN^{0}$, and it does not make a difference in proving the bounds of $\lambda_{\bv,\br}^{\star}(S)$ for assortments $S\subseteq[4K]=[N]\setminus \cN^0$.}.
    Consequently, it remains to derive the minimax lower bound for the difference $\Delta(S^{\star}_{\bv,\br},\pi(\mathbb{D}))$ for the hard instance class $\cV_{\mathrm{uni}}$.

        \begin{lemma}[Minimax lower bound 2 of $\Delta(S^{\star}_{\bv,\br},\pi(\mathbb{D}))$]\label{lem: minimax rate 2}
        Under the settings inside the proof of Theorem~\ref{thm:_lower_bound_1} for the uniform revenue setting, by taking the parameter $\epsilon=\sqrt{C_1/80Kn_{\min}}$ where constant $C_1>0$ is defined in Lemma~\ref{lem-packing-number}, then there is another absolute constant $C_2>0$ such that for $K\geq C_2$ it holds that
        \begin{align}
            \inf_{\pi(\cdot):(2^{[N]}\times[N])^n\mapsto2^{[N]}}\sup_{(\bv,\br)\in\cV_{\mathrm{uni}}}\mathbb{E}_{\mathbb{D}\sim \otimes_{k=1}^n\mathbb{P}_{\bv}(\cdot|S_k)}\big[\Delta(S^{\star}_{\bv,\br},\pi(\mathbb{D}))\big] \geq  \frac{1}{128}\cdot \sqrt{\frac{C_1}{5}}\cdot\frac{\sqrt{K}}{\epsilon'\cdot\sqrt{n_{\min}}},
        \end{align}
        where the expectation is taken w.r.t. the randomness in the item choices $\{i_k\}_{k=1}^n$ in the observational data.
    \end{lemma}

    \begin{proof}[Proof of Lemma~\ref{lem: minimax rate 2}]
        It suffices to repeat the proofs for Lemma~\ref{lem: minimax rate 1}, except by using a different upper bound on KL divergence terms in \eqref{eq: kl decomposition} in the following lemma. 
    \end{proof}
        \begin{lemma}[KL divergence bound 2]\label{lem-KL-new-upper-bound-uniform}
             For every $\ell\in[4K]$ and any pair $(\bm r, \bm v),(\bm r', \bm v') \in \mathcal{V}_{\mathrm{uni}}$, it holds that
            \begin{align*}
                D_{\mathrm{KL}}\left(\mathbb{P}_{\bm v}(\cdot|S_{j}^{(\ell)})\middle\|\mathbb{P}_{\bm v'}(\cdot|S_{j}^{(\ell)})\right) \leq 
                \begin{cases}
                    5K(\epsilon')^2,  & \text{if }  \ell \in (S^{\star}_{\bv,\br}(\bm r, \bm v) \cup S^{\star}_{\bv,\br}(\bm r', \bm v'))\setminus(S^{\star}_{\bv,\br}(\bm r, \bm v)\cap S^{\star}_{\bv,\br}(\bm r', \bm v')), \\
                    0,& \text{otherwise.}
                \end{cases} 
            \end{align*}
        \end{lemma}

        \begin{proof}[Proof of Lemma~\ref{lem-KL-new-upper-bound-uniform}]
            Please refer to Appendix C.2 in \cite{han2025learning} for a detailed proof.
        \end{proof}

        With Lemma~\ref{lem-KL-new-upper-bound-uniform}, we can repeat the proof of Lemma~\ref{lem: minimax rate 1} to obtain that, 
        \begin{align}
            \inf_{\pi(\cdot):(2^{[N]}\times[N])^n\mapsto2^{[N]}}\sup_{(\bv,\br)\in\cV_{\mathrm{uni}}}\mathbb{E}_{\mathbb{D}\sim \PP_{\bv,\br}(\cdot)}\big[\Delta(S^{\star}_{\bv,\br},\pi(\mathbb{D}))\big]&\geq \frac{K}{8}\cdot\left(\frac{1}{2} - \frac{\max_{(\bm r, \bm v)\in \mathcal{V},(\bm r', \bm v')\in \mathcal{V}}D_{\mathrm{KL}}(\mathbb{P}_{\bm r, \bm v} \lVert \mathbb{P}_{\bm r', \bm v'}) }{C_1 K} \right)\\
            &\geq \frac{K}{8}\cdot\left(\frac{1}{2} - \frac{n_{\min}\cdot 4K^2\cdot 5(\epsilon')^2}{C_1K}\right) \\
            &\geq \frac{K}{32} = \frac{1}{128}\cdot \sqrt{\frac{C_1}{5}}\cdot\frac{\sqrt{K}}{\epsilon'\cdot\sqrt{n_{\min}}},
        \end{align}
         where the second inequality uses Lemma~\ref{lem-KL-new-upper-bound-uniform} and \eqref{eq: kl decomposition}, while the last inequality and the last equality both use the choice of $\epsilon'$ that $
        \epsilon' = \sqrt{C_1/80Kn_{\min}}$. 
        This completes the proof of Lemma~\ref{lem: minimax rate 2}.

        Finally, we can lower bound the minimax lower bound of the expected suboptimality gap by 
        \begin{align}
        &\inf_{\pi(\cdot):(2^{[N]}\times[N])^n\mapsto2^{[N]}}\sup_{(\bv,\br)\in\cV_{\mathrm{uni}}}\mathbb{E}_{\mathbb{D}\sim \otimes_{k=1}^n\mathbb{P}_{\bv}(\cdot|S_k)}\big[\mathrm{SupOpt}_{\rho}(\pi(\mathbb{D});\bv,\br) \big]\\
        &\qquad  \geq \inf_{\pi(\cdot):(2^{[N]}\times[N])^n\mapsto2^{[N]}}\sup_{(\bv,\br)\in\cV_{\mathrm{uni}}}\mathbb{E}_{\mathbb{D}\sim \otimes_{k=1}^n\mathbb{P}_{\bv}(\cdot|S_k)}\big[\mathrm{SupOpt}_{\rho}(\widetilde{S}(\pi(\mathbb{D}));\bv,\br) \big]\\
        &\qquad \geq \inf_{\pi(\cdot):(2^{[N]}\times[N])^n\mapsto2^{[N]}}\sup_{(\bv,\br)\in\cV_{\mathrm{uni}}}\bigg\{\frac{\epsilon'}{18}\cdot\lambda^{\star}_{\bv,\br}(\widetilde{S}(\pi(\mathbb{D})))\cdot \Big(1-\exp\big(-r_{\max}/\lambda^{\star}_{\bv,\br}(\widetilde{S}(\pi(\mathbb{D})))\big)\Big)\\
        &\qquad\qquad \cdot \mathbb{E}_{\mathbb{D}\sim \otimes_{k=1}^n\mathbb{P}_{\bv}(\cdot|S_k)}\big[\Delta(S^{\star}_{\bv,\br},\widetilde{S}(\pi(\mathbb{D}))) \big]\bigg\}\\
        &\qquad \geq \frac{\epsilon'}{18}\cdot \inf_{(\bv,\br)\in\cV_{\mathrm{uni}}}\inf_{S\in [4K],|S|=K}\lambda^{\star}_{\bv,\br}(S)\cdot \big(1-\exp(-r_{\max}/\lambda^{\star}_{\bv,\br}(S))\big)\\
        &\qquad\qquad\cdot  \inf_{\pi(\cdot):(2^{[N]}\times[N])^n\mapsto2^{[N]}}\sup_{(\bv,\br)\in\cV_{\mathrm{uni}}}\mathbb{E}_{\mathbb{D}\sim \otimes_{k=1}^n\mathbb{P}_{\bv}(\cdot|S_k)}\big[\Delta(S^{\star}_{\bv,\br},\pi(\mathbb{D}))\big] \\
        &\qquad \geq \frac{1}{2304}\cdot\sqrt{\frac{C_1}{5}}\cdot \big(1-2^{-\underline{c}_{\rho}}\big)\cdot \frac{r_{\max}}{\log\left(\frac{2}{\frac{1}{2}\cdot c^{\dagger}_{\rho}\wedge \overline{c}_{\rho}\cdot\log2}\right)} \cdot \frac{\sqrt{K}}{\sqrt{n_{\min}}}.
    \end{align}
        where the first inequality uses \eqref{eq: simplification jin uniform}, the second inequality uses \eqref{eq:_proof_lower_bound_3}, the third inequality uses $\widetilde{S}(\pi(\mathbb{D}))\subseteq[4K]$ with size $K$ and the fact of \eqref{eq: removing simplification}, and the last inequality uses the upper and lower bounds \eqref{eq: lambda star lower and upper bound} of the optimal dual variable $\lambda^{\star}_{\bv,\br}(S)$ for any assortment $S\subseteq[4K]$ with size $|S|=K$ and any instance $(\bv,\br)\in\cV_{\mathrm{uni}}$, together with Lemma~\ref{lem: minimax rate 2}.
        This completes the proof of Theorem~\ref{thm:_lower_bound_1} for the uniform revenue case.
\end{proof}

\newpage

\section{Proofs for Varying Robust Set Size Case}\label{sec:_proof_main_theorems_new}

\subsection{General Upper Bound for Suboptimality Gap}\label{subsec:_general_upper_bound_new}

\begin{theorem}[Suboptimality of Algorithm~\ref{alg: new}]\label{thm:_algorithm_design_general_2}
Conditioning on the observed assortments $\{S_k\}_{k=1}^n$, suppose that 
    \begin{align}
        n_i:= \sum_{k=1}^n\mathbf{1}\{i \in S_k\}\geq 138\cdot\max\left\{1,\frac{256}{138v_i}\right\}\cdot (1+v_{\max})(1+Kv_{\max})\log(3N/\delta),\quad \forall i\in S^{\star},\label{eq: condition_n_i_main_new_appendix}
    \end{align}
    then with probability at least $1-2\delta$, the robust expected revenue of the assortment $\widehat S$ satisfies
    \begin{align}
        &R_{\rho(\cdot,\cdot)}(S^{\star};\bv) - R_{\rho(\cdot,\cdot)}(\widehat S;\bv)  \\
        &\qquad \leq  \frac{3r_{\max}}{1+v(S^{\star})/2 - (1-e^{-\rho_0})\cdot(1+v_{\mathrm{tot}})}\\
        &\qquad\qquad \cdot \sum_{i\in S^{\star}}\left(16 \sqrt{\frac{v_i(1+v_i)\log(3N/\delta)}{\sum_{k=1}^n\mathbf{1}\{i_k\in S_k\}\cdot (1+v(S_k))^{-1}}} + \frac{88(1+v_i)\log(3N/\delta)}{\sum_{k=1}^n\mathbf{1}\{i_k\in S_k\}\cdot (1+v(S_k))^{-1}}\right).
    \end{align}
\end{theorem}

\begin{proof}[Proof of Theorem~\ref{thm:_algorithm_design_general_2}]
    See Appendix~\ref{subsec:_proof_algorithm_design_general_2} for a detailed proof.
\end{proof}

\subsection{Proof of Theorem~\ref{thm:_algorithm_design_general_2}}\label{subsec:_proof_algorithm_design_general_2}

\begin{proof}[Proof of Theorem~\ref{thm:_algorithm_design_2}]   
    Consider the following upper bounds of the suboptimality, in terms of the robust expected revenue defined in Example~\ref{exp: new}, of the learned assortment $\widehat{S}$,
    \begin{align}
        R_{\rho(\cdot,\cdot)}(S^{\star};\bv) - R_{\rho(\cdot,\cdot)}(\widehat S;\bv) & =\sup_{\lambda \geq 0} \left\{\widetilde{H}(S^{\star},\lambda;\bv)\right\} - \sup_{\lambda \geq 0} \left\{\widetilde{H}(\widehat{S},\lambda;\bv)\right\}\\
        &\leq \sup_{\lambda \geq 0} \left\{\widetilde{H}(S^{\star},\lambda;\bv)\right\} - \sup_{\lambda \geq 0} \left\{\widetilde{H}(\widehat{S},\lambda;\bv^{\mathrm{LCB}})\right\} \\
        &\leq \sup_{\lambda \geq 0} \left\{\widetilde{H}(S^{\star},\lambda;\bv)\right\} - \sup_{\lambda \geq 0} \left\{\widetilde{H}(S^{\star},\lambda;\bv^{\mathrm{LCB}})\right\},\label{eq:_proof_algorithm_design_2_0}
    \end{align}
    where the first inequality is due to the dual representation \eqref{eq: dual new}, the first inequality is by Lemma~\ref{lem:_monotonicity_2} which shows the monotonicity property associated with $\widetilde{H}$, and the last inequality is from the optimality of $\widehat{S}$ under the nominal parameter $\bv^{\mathrm{LCB}}$.
    Now we further upper bound the suboptimality gap via \eqref{eq:_proof_algorithm_design_2_0} as
    \begin{align}
        &R_{\rho(\cdot,\cdot)}(S^{\star};\bv) - R_{\rho(\cdot,\cdot)}(\widehat S;\bv)  \label{eq:_proof_algorithm_design_2_1} \\
        &\qquad \leq \sup_{0\leq \lambda\leq B(S^{\star};\bv)} \left\{-\lambda\cdot\log\left(\frac{\sum_{i\in S_+^{\star}}v_i\cdot\exp(-r_i/\lambda)}{\sum_{i\in S_+^{\star}}v_i}\right) - \lambda\cdot \rho_0 +\lambda\cdot  \log\left(e^{\rho_0} - \frac{a(\rho_0)}{\sum_{i\in S_+^{\star}}v_i}\right)\right\}\\
        &\qquad \qquad - \sup_{0\leq \lambda\leq B(S^{\star};\bv^{\mathrm{LCB}})} \left\{-\lambda\!\cdot\!\log\!\left(\frac{\sum_{i\in S_+^{\star}}v_i^{\mathrm{LCB}}\cdot\exp(-r_i/\lambda)}{\sum_{i\in S_+^{\star}}v_i^{\mathrm{LCB}}}\right) - \lambda\!\cdot\! \rho_0 +\lambda\!\cdot\!  \log\left(e^{\rho_0} - \frac{a(\rho_0)}{\sum_{i\in S_+^{\star}}v_i^{\mathrm{LCB}}}\right)\right\}\\
        &\qquad \leq \underbrace{\sup_{\lambda\geq 0}\left\{\lambda\cdot\left(\log\left(\frac{\sum_{i\in S^{\star}_+}v_i^{\mathrm{LCB}}\cdot\exp(-r_i/\lambda)}{\sum_{i\in S^{\star}_+}v_i^{\mathrm{LCB}}}\cdot \right) - \log\left(\frac{\sum_{i\in S^{\star}_+}v_i\cdot\exp(-r_i/\lambda)}{\sum_{i\in S^{\star}_+}v_i}\right)\right)\right\} }_{\displaystyle{\text{Term (i)}}}\\
        &\qquad \qquad + \underbrace{\sup_{0\leq \lambda\leq B(S^{\star};\bv^{\mathrm{LCB}})}\left\{\lambda\cdot \left(\log\left(\frac{e^{\rho_0}\cdot \sum_{i\in S^{\star}_{+}}v_i^{\mathrm{LCB}}}{e^{\rho_0}\cdot\sum_{i\in S^{\star}_{+}}v_i^{\mathrm{LCB}} - a(\rho_0)}\right)- \log\left(\frac{e^{\rho_0}\cdot \sum_{i\in S^{\star}_{+}}v_i}{e^{\rho_0}\cdot\sum_{i\in S^{\star}_{+}}v_i - a(\rho_0)}\right)\right)\right\}}_{\displaystyle{\text{Term (ii)}}}.
    \end{align}
    where the first inequality applies \eqref{eq:_proof_algorithm_design_2_0} and the definition of function $\widetilde{H}$ in \eqref{eq: h tilde}, together with the definition that $a(\rho_0) := (e^{\rho_0} - 1) (1+v_{\mathrm{tot}})$ for the seek of simplicity.
    The second inequality uses the inequality that $\sup_x f(x) - \sup_x g(x)\leq \sup_x \{f(x)-g(x)\}$.
    Now we upper bound the two terms in \eqref{eq:_proof_algorithm_design_2_1} respectively.

    For the first term in \eqref{eq:_proof_algorithm_design_2_1}, it actually coincides with the right hand side of \eqref{eq:_proof_algorithm_design_1_0+}. 
    Therefore, by the same argument as proving Theorem~\ref{thm:_algorithm_design_1}, we can obtain that with probability at least $1-2\delta$, it holds that 
    \begin{align}
        \text{Term (i)} \leq \frac{2\cdot r_{\max}}{\log 3} \cdot \frac{\sum_{i\in S_+^{\star}}v_i - v_i^{\mathrm{LCB}}}{\sum_{i\in S^{\star}_+}v_i^{\mathrm{LCB}}}.\label{eq:_proof_algorithm_design_2_2}
    \end{align}

    For the second term in \eqref{eq:_proof_algorithm_design_2_1}, firstly  consider that 
    \allowdisplaybreaks
    \begin{align}
        &\log\left(\frac{e^{\rho_0}\cdot \sum_{i\in S^{\star}_{+}}v_i^{\mathrm{LCB}}}{e^{\rho_0}\cdot\sum_{i\in S^{\star}_{+}}v_i^{\mathrm{LCB}} - a(\rho_0)}\right)- \log\left(\frac{e^{\rho_0}\cdot \sum_{i\in S^{\star}_{+}}v_i}{e^{\rho_0}\cdot\sum_{i\in S^{\star}_{+}}v_i - a(\rho_0)}\right)\\
        &\qquad  = \log\left(1 + \frac{e^{\rho_0}\cdot\sum_{i\in S^{\star}_{+}}v_i - a(\rho_0)}{e^{\rho_0}\cdot \sum_{i\in S^{\star}_{+}}v_i}\cdot\left(\frac{e^{\rho_0}\cdot \sum_{i\in S^{\star}_{+}}v_i^{\mathrm{LCB}}}{e^{\rho_0}\cdot\sum_{i\in S^{\star}_{+}}v_i^{\mathrm{LCB}} - a(\rho_0)} - \frac{e^{\rho_0}\cdot \sum_{i\in S^{\star}_{+}}v_i}{e^{\rho_0}\cdot\sum_{i\in S^{\star}_{+}}v_i - a(\rho_0)}\right)\right)\\
        &\qquad =  \log\left(1 + a(\rho_0)\cdot \frac{e^{\rho_0}\cdot\sum_{i\in S^{\star}_{+}}v_i - a(\rho_0)}{e^{\rho_0}\cdot \sum_{i\in S^{\star}_{+}}v_i}\cdot\left(\frac{1}{e^{\rho_0}\cdot\sum_{i\in S^{\star}_{+}}v_i^{\mathrm{LCB}} - a(\rho_0)} - \frac{1}{e^{\rho_0}\cdot\sum_{i\in S^{\star}_{+}}v_i - a(\rho_0)}\right)\right)\\
        &\qquad =  \log\left(1 +\frac{ a(\rho_0)}{ \left(\sum_{i\in S^{\star}_{+}}v_i\right)\cdot\left(e^{\rho_0}\cdot\sum_{i\in S^{\star}_{+}}v_i^{\mathrm{LCB}} - a(\rho_0)\right)}\cdot\left(\sum_{i\in S^{\star}_{+}}v_i - \sum_{i\in S^{\star}_{+}}v_i^{\mathrm{LCB}}\right)\right)\\
        &\qquad \leq \frac{ a(\rho_0)}{ \left(\sum_{i\in S^{\star}_{+}}v_i\right)\cdot\left(e^{\rho_0}\cdot\sum_{i\in S^{\star}_{+}}v_i^{\mathrm{LCB}} - a(\rho_0)\right)}\cdot\left(\sum_{i\in S^{\star}_{+}}v_i - \sum_{i\in S^{\star}_{+}}v_i^{\mathrm{LCB}}\right),
    \end{align}
    where the last inequality uses that $\log(1+x)\leq x$ for $x> -1$.
    Meanwhile, the dual variable $\lambda$ satisfies that 
    \begin{align}
        \lambda &\leq \frac{r_{\max}}{\log\left(\frac{e^{\rho_0}\cdot \sum_{i\in S^{\star}_{+}}v_i^{\mathrm{LCB}}}{e^{\rho_0}\cdot\sum_{i\in S^{\star}_{+}}v_i^{\mathrm{LCB}} - a(\rho_0)}\right)}= \frac{r_{\max}}{\log\left(1+\frac{a(\rho_0)}{e^{\rho_0}\cdot\sum_{i\in S^{\star}_{+}}v_i^{\mathrm{LCB}} - a(\rho_0)}\right)}\leq r_{\max}\cdot \frac{e^{\rho_0}\cdot\sum_{i\in S^{\star}_{+}}v_i}{a(\rho_0)},
    \end{align}
    where the last inequality uses that $\log(1+x)\geq x/(x+1)$ for $x>-1$ and Lemma~\ref{lem:_concentration}.
    Consequently, we can upper bound the second term in \eqref{eq:_proof_algorithm_design_2_1} as 
    \begin{align}
        \text{Term (ii)} \leq r_{\max}\cdot \frac{\sum_{i\in S^{\star}_{+}}v_i - \sum_{i\in S^{\star}_{+}}v_i^{\mathrm{LCB}}}{\sum_{i\in S^{\star}_{+}}v_i^{\mathrm{LCB}} - e^{-\rho_0}\cdot a(\rho_0)}.\label{eq:_proof_algorithm_design_2_3}
    \end{align}
    Therefore, by combining \eqref{eq:_proof_algorithm_design_2_2} and \eqref{eq:_proof_algorithm_design_2_3}, we can obtain that with probability at least $1-2\delta$, 
    \begin{align}
        R_{\rho(\cdot,\cdot)}(S^{\star};\bv) - R_{\rho(\cdot,\cdot)}(\widehat S;\bv)  \leq 3r_{\max}\cdot \frac{\sum_{i\in S^{\star}_{+}}v_i - \sum_{i\in S^{\star}_{+}}v_i^{\mathrm{LCB}}}{\sum_{i\in S^{\star}_{+}}v_i^{\mathrm{LCB}} - (1-e^{-\rho_0})\cdot(1+v_{\mathrm{tot}}) }.
    \end{align}
    Finally, by using Lemma~\ref{lem:_concentration_2} to control the error of the estimation $\bv^{\mathrm{LCB}}$, we can obtain that 
    \begin{align}
        &R_{\rho(\cdot,\cdot)}(S^{\star};\bv) - R_{\rho(\cdot,\cdot)}(\widehat S;\bv)  \\
        &\qquad \leq  \frac{3r_{\max}}{1+v(S^{\star})/2 - (1-e^{-\rho_0})\cdot(1+v_{\mathrm{tot}})}\\
        &\qquad\qquad \cdot \sum_{i\in S^{\star}}\left(16 \sqrt{\frac{v_i(1+v_i)\log(3N/\delta)}{\sum_{k=1}^n\mathbf{1}\{i_k\in S_k\}\cdot (1+v(S_k))^{-1}}} + \frac{88(1+v_i)\log(3N/\delta)}{\sum_{k=1}^n\mathbf{1}\{i_k\in S_k\}\cdot (1+v(S_k))^{-1}}\right).
    \end{align}
    This completes the proof of Theorem~\ref{thm:_algorithm_design_2}.
\end{proof}

\subsection{Proof of Theorem~\ref{thm:_algorithm_design_2}}\label{subsec:_proof_algorithm_design_2}

With the general results in Theorem~\ref{thm:_algorithm_design_general_2}, we are ready to specify it to Theorem~\ref{thm:_algorithm_design_2}.

\begin{proof}[Proof of Theorem~\ref{thm:_algorithm_design_2}]
We prove the non-uniform revenue case and the uniform revenue case respectively as following. 
It suffices to calculate the upper bounds for   
\begin{align}
    &\frac{1}{1+v(S^{\star})/2 - (1-e^{-\rho_0})\cdot(1+v_{\mathrm{tot}})}\cdot \sum_{i\in S^{\star}} \sqrt{\frac{v_i(1+v_i)\log(3N/\delta)}{\sum_{k=1}^n\mathbf{1}\{i_k\in S_k\}\cdot (1+v(S_k))^{-1}}},\\
    & \frac{1}{1+v(S^{\star})/2 - (1-e^{-\rho_0})\cdot(1+v_{\mathrm{tot}})}\cdot\sum_{i\in S^{\star}} \frac{(1+v_i)\log(3N/\delta)}{\sum_{k=1}^n\mathbf{1}\{i_k\in S_k\}\cdot (1+v(S_k))^{-1}},\label{eq: term to bound 2}
\end{align}
respectively, which results in different upper bounds under different setups.

    \paragraph{Non-uniform revenue case.}
For the first term in Theorem~\eqref{eq: term to bound 2}, we have the following upper bound,
\begin{align*}
&\frac{1}{1+v(S^{\star})/2 - (1-e^{-\rho_0})\cdot(1+v_{\mathrm{tot}})}\cdot \sum_{i\in S^\star} \sqrt{\frac{v_i(1+v_i)\log (3N/\delta) }{\sum_{k = 1}^n \bm{1}\{i\in S_k\}(1+\sum_{j\in S_k}v_j)^{-1}}} \\
& \qquad\leq \frac{1}{1+v(S^{\star})/2 - (1-e^{-\rho_0})\cdot(1+v_{\mathrm{tot}})}\cdot\sqrt{(1+v_{\max})(1+Kv_{\max})\log(3N/\delta)} \cdot \sum_{i\in S^\star} \sqrt{\frac{v_i}{n_i}}\\
&\qquad \leq \sqrt{\frac{2(1+v_{\max})(1+Kv_{\max})\log(3N/\delta)}{1+v(S^{\star})/2 - (1-e^{-\rho_0})\cdot(1+v_{\mathrm{tot}})}}  \cdot\sqrt{\sum_{i\in S^\star}\frac{1}{n_i}}\\
&\qquad \leq \frac{(1+v_{\max})K}{\sqrt{1+v(S^{\star})/2 - (1-e^{-\rho_0})\cdot(1+v_{\mathrm{tot}})}}\cdot\sqrt{\frac{2\log (3N/\delta)}{\min_{i\in S^\star}n_i}},
\end{align*}
where the first inequality uses that $v_i\leq v_{\max}$ and the fact that 
\begin{align}
    1+\sum_{j\in S_k}v_j \leq 1+Kv_{\max},\quad \forall k\in[n],
\end{align}
the second inequality is based on Cauchy-Schwartz inequality that 
\begin{align}
    \sum_{i\in S^{\star}}\sqrt{\frac{v_i}{n_i}}\leq \sqrt{\sum_{i\in S^{\star}}v_i}\cdot\sqrt{\sum_{i\in S^{\star}}\frac{1}{n_i}}\leq \sqrt{2}\cdot\sqrt{1+\sum_{i\in S^{\star}}v_i - (1-e^{-\rho_0})\cdot(1+v_{\mathrm{tot}})}\cdot\sqrt{\sum_{i\in S^{\star}}\frac{1}{n_i}},
\end{align}
and the last inequality uses the fact that $1+Kv_{\max}\leq K(1+v_{\max})$ and the inequality that 
\begin{align}
    \sqrt{\sum_{i\in S^\star}\frac{1}{n_i}} \leq \sqrt{\frac{K}{\min_{i\in S^{\star}}n_i}}.
\end{align}
For the second term in \eqref{eq: term to bound 1}, we similarly have the following  upper bound,
\begin{align*}
&\frac{1}{1+v(S^{\star})/2 - (1-e^{-\rho_0})\cdot(1+v_{\mathrm{tot}})}\cdot\sum_{i\in S^\star}   \frac{(1+v_i)\log (3N/\delta) }{\sum_{k = 1}^n \bm{1}\{i\in S_k\}(1+\sum_{j\in S_k}v_j)^{-1}} \\
&\qquad\leq  \frac{1}{1+v(S^{\star})/2 - (1-e^{-\rho_0})\cdot(1+v_{\mathrm{tot}})}\cdot\frac{(1+v_{\max})^2K^2\log(3N/\delta)}{\min_{i\in S^\star} n_i}.
\end{align*}
This finishes the proof of Theorem~\ref{thm:_algorithm_design_2} for the non-uniform revenue case.

\paragraph{Uniform revenue case.}
In the uniform revenue setting, due to Proposition~\ref{prop:_optimal_set_uniform}, we always have that
\begin{align}
    \sum_{i\in S^\star} v_i \geq \sum_{i\in S} v_i,\quad \text{for any assortment $S$ with $|S|\leq K$}.\label{eq: uniform condition 2}
\end{align}
Consequently, the first term in \eqref{eq: term to bound 2} can be upper bounded as
\allowdisplaybreaks
\begin{align*}
  &\frac{1}{1+v(S^{\star})/2 - (1-e^{-\rho_0})\cdot(1+v_{\mathrm{tot}})}\cdot\sum_{i\in S^\star}  \sqrt{\frac{v_i(1+v_i)\log (3N/\delta) }{\sum_{k = 1}^n \bm{1}\{i\in S_k\}(1+\sum_{j\in S_k}v_j)^{-1}}}\\
  &\qquad \leq   \frac{ \sqrt{1+\sum_{j\in S^{\star}}v_j}\cdot \sqrt{(1+v_{\max})\log(3N/\delta)}}{1+\sum_{j\in S^{\star}}v_j/2 - (1-e^{-\rho_0})\cdot(1+v_{\mathrm{tot}})}\cdot \sum_{i\in S^\star} \sqrt{\frac{v_i }{n_i}} \\
&\qquad \leq \frac{ \sqrt{1+\sum_{j\in S^{\star}}v_j}}{\sqrt{1+\sum_{j\in S^{\star}}v_j/2 - (1-e^{-\rho_0})\cdot(1+v_{\mathrm{tot}})}}\cdot \sqrt{\frac{2(1+v_{\max})K\log(3N/\delta)}{\min_{i\in S^\star} n_i}},
\end{align*}
where the first inequality uses \eqref{eq: uniform condition 2} and $v_i\leq v_{\max}$, the second inequality uses Cauchy-Schwartz inequality that 
\begin{align}
    \sum_{i\in S^{\star}}\sqrt{\frac{v_i}{n_i}}\leq \sqrt{\sum_{i\in S^{\star}}v_i}\cdot\sqrt{\sum_{i\in S^{\star}}\frac{1}{n_i}}\leq \sqrt{2}\cdot\sqrt{1+\sum_{i\in S^{\star}}v_i - (1-e^{-\rho_0})\cdot(1+v_{\mathrm{tot}})}\cdot\sqrt{\sum_{i\in S^{\star}}\frac{1}{n_i}},
\end{align}
together with the fact that
\begin{align}
    \sqrt{\sum_{i\in S^\star}\frac{1}{n_i}} \leq \sqrt{\frac{K}{\min_{i\in S^{\star}}n_i}}.
\end{align}
Similarly, for the second term in \eqref{eq: term to bound 1}, we have the following upper bound that
\begin{align*}
    &\frac{1}{1+\sum_{j\in S^{\star}}v_j/2 - (1-e^{-\rho_0})\cdot(1+v_{\mathrm{tot}})}\cdot \sum_{i\in S^\star} \frac{(1+v_i)\log(3N/\delta) }{\sum_{k = 1}^n \bm{1}\{i\in S_k\}(1+\sum_{j\in S_k}v_j)^{-1}} \\
    &\qquad \leq \frac{1+\sum_{j\in S^{\star}}v_j}{1+\sum_{j\in S^{\star}}v_j/2 - (1-e^{-\rho_0})\cdot(1+v_{\mathrm{tot}})}\cdot\frac{(1+v_{\max})K\log(3N/\delta)}{\min_{i\in S^\star} n_i},
\end{align*}
as desired, finishing the proof of Theorem~\ref{thm:_algorithm_design_2} for the uniform revenue case.
\end{proof}

\subsection{Proof of Theorem~\ref{thm:_lower_bound_new_1}}\label{subsec:_proof_lower_bound_new_1}

In this section, we prove the suboptimality lower bounds for Example~\ref{exp: new}. We first prove the lower bound for
the general non-uniform revenue case, after which we prove the lower bound for the uniform revenue case.

\begin{proof}[\textcolor{blue}{Proof of Theorem~\ref{thm:_lower_bound_new_1}: non-uniform revenue case}]
The overall structure of the proof of Theorem~\ref{thm:_lower_bound_new_1} resembles that of Theorem~\ref{thm:_lower_bound_1}, but involves some key technical differences in the details.

\paragraph{Step 1: construction of hard instance class and offline dataset.} 
    The construction of both the hard instances and the offline data for the non-uniform revenue case coincide with those of Theorem~\ref{thm:_lower_bound_1} for the non-uniform revenue case.
    We refer readers to  Section~\ref{subsec:_proof_lower_bound_1} for the detailed introduction and definitions. 
    For clarity, we let the hard instance class $\cV$ be defined in the same way as in \eqref{eq: hard instance class}, that is, 
        \begin{align}
        \cV := \Big\{\text{$(\bv,\br)$ induced by $(\cN^{\star}, \cN^{\mathrm{c}}, \cN^0)$ via \eqref{eq:_proof_lower_bound_0}}\,\Big|\,\text{$\cN^{\star}\in \cF$ with $\cF$ given by Lemma~\ref{lem-packing-number}}\Big\}.\label{eq: hard instance class repeated}
    \end{align}
    The following shows that the optimal robust assortment induced by a division $(\cN^{\star}, \cN^{\mathrm{c}}, \cN^0)$ is exactly $\cN^{\star}$.

    \begin{lemma}[Optimal robust assortment for class \eqref{eq: hard instance class repeated}]\label{lem: optimal robust assortment new}
        For any problem instance $(\bv,\br)\in\cV$  associated with the division $(\cN^{\star}, \cN^{\mathrm{c}}, \cN^0)$, its optimal robust assortment $S^{\star}_{\bv,\br}$ in the sense of Example~\ref{exp: new} is $S^{\star}_{\bv,\br}=\cN^{\star}$.
    \end{lemma}

\begin{proof}[Proof of Lemma~\ref{lem: optimal robust assortment new}]
    By definition, for any instance $(\bv,\br)$ with partitions given by $(\cN^{\star},\cN^\mathrm{c}, \cN^0)$, the robust expected revenue of any assortment $S$ with $|S|\leq K$ is given by 
    \begin{align}
            R_{\rho(\cdot;\cdot)}(S;\bv, \br) &=\sup_{\lambda\geq 0}\left\{-\lambda\cdot\log\left(\frac{1+|S\cap \cN^0|+e^{-\frac{r_{\max}}{\lambda}}\cdot(1/K\cdot|S\cap \cN^{\mathrm{c}}| + (1/K+\epsilon)\cdot |S\cap \cN^{\star}|)}{1+|S\cap \cN^{0}| + 1/K\cdot|S\cap \cN^{\mathrm{c}}| + (1/K+\epsilon)\cdot |S\cap \cN^{\star}| - a^{\dagger}(\rho_0)}\right)\right\} \\
            & = \sup_{\lambda\geq 0}\left\{-\lambda\cdot\log\left(e^{-\frac{r_{\max}}{\lambda}} + \frac{(1+|S\cap \cN^0|)\cdot (1-e^{-\frac{r_{\max}}{\lambda}}) + a^{\dagger}(\rho_0)\cdot e^{-\frac{r_{\max}}{\lambda}}}{1+|S\cap \cN^{0}| + 1/K\cdot|S\cap \cN^{\mathrm{c}}| + (1/K+\epsilon)\cdot |S\cap \cN^{\star}| - a^{\dagger}(\rho_0)}\right)\right\}.
        \end{align}
        Obviously, the right hand side of the above equality is optimized (over $S$) when $S\cap \cN^{\mathrm{c}}=\emptyset$ and $|S\cap \cN^{\star}| = K - |S\cap \cN^0|$. 
        Therefore, it suffices to consider how to maximize
        \begin{align}
            \!\!\!\!\phi(|S\cap \cN^0|):=\sup_{\lambda\geq 0}\left\{-\lambda\cdot\log\left(e^{-\frac{r_{\max}}{\lambda}} + \frac{(1+|S\cap \cN^0|)\cdot (1-e^{-\frac{r_{\max}}{\lambda}}) + a^{\dagger}(\rho_0)\cdot e^{-\frac{r_{\max}}{\lambda}}}{1+|S\cap \cN^{0}| +   (1/K+\epsilon)\cdot (K-|S\cap \cN^{0}|) - a^{\dagger}(\rho_0)}\right)\right\},\label{eq: optimal robust assortment new proof 1}
        \end{align}
        over $|S\cap \cN^0|\leq K$. 
        To this end, consider that the term inside the logarithm  in \eqref{eq: optimal robust assortment new proof 1} can be analyzed by
        \begin{align}
            &\frac{(1+|S\cap \cN^0|)\cdot (1-e^{-\frac{r_{\max}}{\lambda}}) + a^{\dagger}(\rho_0)\cdot e^{-\frac{r_{\max}}{\lambda}}}{1+|S\cap \cN^{0}| +   (1/K+\epsilon)\cdot (K-|S\cap \cN^{0}|) - a^{\dagger}(\rho_0)} \\
            &\qquad =\frac{1-e^{-\frac{r_{\max}}{\lambda}}}{1-1/K-\epsilon} + \frac{1-(1-a)\cdot e^{-\frac{r_{\max}}{\lambda}} - (2+K\epsilon-a^{\dagger}(\rho_0))\cdot \frac{1-e^{-\frac{r_{\max}}{\lambda}}}{1-1/K-\epsilon}}{(1-1/K-\epsilon)\cdot |S\cap \cN^0| + 2+K\epsilon-a^{\dagger}(\rho_0)}.\label{eq: optimal robust assortment new proof 2}
        \end{align}
        We find that the sign of the numerator of the second term on the right hand side of \eqref{eq: optimal robust assortment new proof 2} is given by 
        \begin{align}
            1-(1-a)\cdot e^{-\frac{r_{\max}}{\lambda}} - \big(2+K\epsilon-a^{\dagger}(\rho_0)\big)\cdot \frac{1-e^{-\frac{r_{\max}}{\lambda}}}{1-1/K-\epsilon}\leq 0\,\,\, \Leftrightarrow \,\,\, \lambda \leq \lambda_0,\label{eq: sign}
        \end{align}
        where the threshold $\lambda_0$ is given by 
        $ \lambda_0:=r_{\max}/\log(((1+1/K-a^{\dagger}(\rho_0))(1+K\epsilon))/((1+1/K)(1+K\epsilon)-a^{\dagger}(\rho_0)))$.
        Thus we can further calculate the optimal value of \eqref{eq: optimal robust assortment new proof 1} via 
        \begin{align}
            \sup_{|S\cap \cN^0|\leq K}\phi(|S\cap \cN^0|)  &= \sup_{|S\cap \cN^0|\leq K}\max\Bigg\{\sup_{0\leq \lambda\leq \lambda_0}\left\{\psi(|S\cap \cN^0|;\lambda)\right\},  \sup_{\lambda>  \lambda_0}\left\{\psi(|S\cap \cN^0|;\lambda)\right\}\Bigg\}\\
            &= \max\Bigg\{\sup_{|S\cap \cN^0|\leq K}\sup_{0\leq \lambda\leq \lambda_0}\left\{\psi(|S\cap \cN^0|;\lambda)\right\},\sup_{|S\cap \cN^0|\leq K}\sup_{\lambda> \lambda_0}\left\{\psi(|S\cap \cN^0|;\lambda)\right\}\Bigg\},
        \end{align}
        where for notational simplicity we denote
        \begin{align}
            \psi(|S\cap \cN^0|;\lambda):=-\lambda\cdot\log\left(e^{-\frac{r_{\max}}{\lambda}} + \frac{(1+|S\cap \cN^0|)\cdot (1-e^{-\frac{r_{\max}}{\lambda}}) + a^{\dagger}(\rho_0)\cdot e^{-\frac{r_{\max}}{\lambda}}}{1+|S\cap \cN^{0}| +   (1/K+\epsilon)\cdot (K-|S\cap \cN^{0}|) - a^{\dagger}(\rho_0)}\right).
        \end{align}
        Notice that for $\lambda>\lambda_0$, according to \eqref{eq: sign}, we have that 
        \begin{align}
            \sup_{|S\cap \cN^0|\leq K}\sup_{\lambda>\lambda_0}\psi(|S\cap \cN^0|;\lambda) = \sup_{\lambda>\lambda_0}\psi(K;\lambda) \leq  \sup_{\lambda\geq 0}\psi(K;\lambda)=R_{\rho(\cdot;\cdot)}(\cN^0_K;\bv,\br)=0,
        \end{align}
        where the second equality uses the dual representation of the robust expected revenue of an assortment $\cN^0_K$ which denotes any assortment being a subset of $\cN^0$ and of size $K$, and the last equality uses the fact the any assortment containing items solely form $\cN^0$ has robust expected revenue $0$. 
        Meanwhile, for $\lambda\leq \lambda_0$, according to \eqref{eq: sign}, we have the following, 
        \begin{align}
            \sup_{|S\cap \cN^0|\leq K}\sup_{0\leq \lambda\leq \lambda_0}\psi(|S\cap \cN^0|;\lambda)= \sup_{0\leq \lambda\leq \lambda_0}\psi(0;\lambda).
        \end{align}
        Moreover, for $0\leq \lambda <\lambda_+:=r_{\max}/\log((1+K\epsilon)/(1+K\epsilon-a^{\dagger}(\rho_0))$, it holds that $\psi(0;\lambda)>0$. 
        Therefore, we can conclude that 
        \begin{align}
            \sup_{S\subseteq [N],|S|\leq K}R_{\rho(\cdot;\cdot)}(S;\bv,\br)=\sup_{|S\cap \cN^0|\leq K}\phi(|S\cap \cN^0|) = \sup_{0\leq \lambda<\lambda_+\wedge \lambda_0}\psi(0;\lambda),
        \end{align}
        and the optimal value is achieved when the assortment $S$ satisfies $S\cap \cN^0=\emptyset$ and $S\cap \cN^{\mathrm{c}}=\emptyset$, i.e., $S=\cN^{\star}$. 
        This completes the proof of Lemma~\ref{lem: optimal condition new}.
    \end{proof}

    Finally, we take the same construction of the observational data $\mathbb{D}$ as in \eqref{eq: offline dataset proof}, that is, $\mathbb{D}=\{(S_k,i_k)\}_{k=1}^n$.
    and each assortment  $S_k$ is given by
    \begin{align}
        S_k :=\left\{\left\lceil\frac{k}{n_{\min}}\right\rceil,4K+2,\cdots,5K\right\}.\label{eq: offline dataset proof repeated}
    \end{align}

    \paragraph{Step 2: establish the minimax lower bound.}
    For any $(\bv,\br)\in\cV$ and 
    any assortment $S\subseteq[N]$ with $|S|\leq K$, we need to lower bound the suboptimality of $S$ under the instance $(\bv,\br)\in\cV$.
    In the following, we show that it suffices to consider assortments $S\subseteq\cN^{\star}\cup \cN^{\mathrm{c}}=[4K]$ with $|S|=K$.
    This approach has been used when we prove Theorem~\ref{thm:_lower_bound_1}, but here we need more delicate proofs to showcase the correctness of the argument, which we conclude in the following lemma.
    \begin{lemma}\label{lem: aux 1}
        Under the same setups in this section, for any assortment $S\subseteq[N]$ with $|S|\leq K$, and for any instance $(\bv,\br)\in\cV$, it holds that 
        \begin{align}
            R_{\rho(\cdot;
            \cdot)}(S;\bv,\br)\leq R_{\rho(\cdot;\cdot)}(\widetilde{S}(S);\bv,\br),\quad\text{where}\quad \widetilde{S}(S):=(S\cap \cN^{\star})\cup \cN^{\mathrm{c}}_{K-|S\cap \cN^{\star}|}\subseteq\cN^{\star}\cup\cN^{\mathrm{c}}=[4K],
        \end{align}
        where $\cN^{\mathrm{c}}_{K-|S\cap \cN^{\star}|}$ is any $|K-|S\cap \cN^{\star}|$-sized subset of $\cN^{\mathrm{c}}$.
    \end{lemma}

    \begin{proof}[Proof of Lemma~\ref{lem: aux 1}]
        We first prove that 
        \begin{align}
            R_{\rho(\cdot;
            \cdot)}(S;\bv,\br)\leq R_{\rho(\cdot;\cdot)}(S'(S);\bv,\br),\quad\text{where}\quad S'(S):=(S\cap \cN^{\star})\cup (S\cap \cN^0)\cap\cN^{\mathrm{c}}_{K-|S\cap \cN^{\star}|-|S\cap \cN^0|}.
        \end{align}
        This is because, by the dual representation, 
        \begin{align}
            &R_{\rho(\cdot;\cdot)}(S;\bv, \br) \\&\qquad =\sup_{\lambda\geq 0}\left\{-\lambda\cdot\log\left(e^{-\frac{r_{\max}}{\lambda}} + \frac{(1+|S\cap \cN^0|)\cdot (1-e^{-\frac{r_{\max}}{\lambda}}) + a^{\dagger}(\rho_0)\cdot e^{-\frac{r_{\max}}{\lambda}}}{1+|S\cap \cN^{0}| + 1/K\cdot|S\cap \cN^{\mathrm{c}}| + (1/K+\epsilon)\cdot |S\cap \cN^{\star}| - a^{\dagger}(\rho_0)}\right)\right\},
        \end{align}
        and by the fact that $|S\cap \cN^{\star}|+|S\cap \cN^{\mathrm{c}}|+|S\cap \cN^0|\leq K$, 
        \begin{align}
            &\frac{(1+|S\cap \cN^0|)\cdot (1-e^{-\frac{r_{\max}}{\lambda}}) + a^{\dagger}(\rho_0)\cdot e^{-\frac{r_{\max}}{\lambda}}}{1+|S\cap \cN^{0}| + 1/K\cdot|S\cap \cN^{\mathrm{c}}| + (1/K+\epsilon)\cdot |S\cap \cN^{\star}| - a^{\dagger}(\rho_0)} \\
            &\qquad \geq \frac{(1+|S\cap \cN^0|)\cdot (1-e^{-\frac{r_{\max}}{\lambda}}) + a^{\dagger}(\rho_0)\cdot e^{-\frac{r_{\max}}{\lambda}}}{1+|S\cap \cN^{0}| + 1/K\cdot(K-|S\cap\cN^{\star}| - |S\cap \cN^{0}|) + (1/K+\epsilon)\cdot |S\cap \cN^{\star}| - a^{\dagger}(\rho_0)}.
        \end{align}
        Thus $R_{\rho(\cdot;
            \cdot)}(S;\bv,\br)\leq R_{\rho(\cdot;\cdot)}(S'(S);\bv,\br)$.
        Now we further prove that $R_{\rho(\cdot;
            \cdot)}(S'(S);\bv,\br)\leq R_{\rho(\cdot;\cdot)}(\widetilde{S}(S);\bv,\br)$.
        To simplify the notations, let's now directly assume that $|S|=K$ and prove $R_{\rho(\cdot;
            \cdot)}(S;\bv,\br)\leq R_{\rho(\cdot;\cdot)}(\widetilde{S}(S);\bv,\br)$.
        To this end, by calculation, we have that
        \begin{align}
            R_{\rho(\cdot;
            \cdot)}(S;\bv,\br)=\sup_{\lambda\geq 0}\left\{-\lambda\cdot\log\left(\frac{1+e^{-\frac{r_{\max}}{\lambda}}\cdot(1+\epsilon\cdot|S\cap \cN^{\star}|)+(1-1/K\cdot e^{-\frac{r_{\max}}{\lambda}})\cdot|S\cap \cN^{0}|}{2+\epsilon\cdot|S\cap \cN^{\star}|+(1-1/K)\cdot|S\cap \cN^{0}|-a^{\dagger}(\rho_0)}\right)\right\},
        \end{align}
        where we have applied $|S\cap \cN^{\star}|+|S\cap \cN^{\mathrm{c}}|+|S\cap \cN^0|= K$.
        Meanwhile, 
        \begin{align}
            R_{\rho(\cdot;
            \cdot)}(\widetilde{S}(S);\bv,\br)=\sup_{\lambda\geq 0}\left\{-\lambda\cdot\log\left(\frac{1+e^{-\frac{r_{\max}}{\lambda}}\cdot(1+\epsilon\cdot|S\cap \cN^{\star}|)}{2+\epsilon\cdot|S\cap \cN^{\star}|-a^{\dagger}(\rho_0)}\right)\right\}.
        \end{align}
        To prove that $R_{\rho(\cdot;
            \cdot)}(S;\bv,\br)\leq R_{\rho(\cdot;\cdot)}(\widetilde{S}(S);\bv,\br)$, consider the following,
        \begin{align}
            R_{\rho(\cdot;
            \cdot)}(S;\bv,\br)&=\sup_{\lambda\geq 0}\left\{-\lambda\cdot\log\left(\frac{1+e^{-\frac{r_{\max}}{\lambda}}\cdot(1+\epsilon\cdot|S\cap \cN^{\star}|)+(1-1/K\cdot e^{-\frac{r_{\max}}{\lambda}})\cdot|S\cap \cN^{0}|}{2+\epsilon\cdot|S\cap \cN^{\star}|+(1-1/K)\cdot|S\cap \cN^{0}|-a^{\dagger}(\rho_0)}\right)\right\} \\
            & = -\lambda^{\star}_S\cdot\log\left(\frac{1+e^{-\frac{r_{\max}}{\lambda^{\star}}}\cdot(1+\epsilon\cdot|S\cap \cN^{\star}|)+(1-1/K\cdot e^{-\frac{r_{\max}}{\lambda^{\star}_S}})\cdot|S\cap \cN^{0}|}{2+\epsilon\cdot|S\cap \cN^{\star}|-a^{\dagger}(\rho_0)+(1-1/K)\cdot|S\cap \cN^{0}|}\right) \\
            &\leq -\lambda^{\star}_S\cdot\log\left(\frac{1+e^{-\frac{r_{\max}}{\lambda^{\star}}}\cdot(1+\epsilon\cdot|S\cap \cN^{\star}|)}{2+\epsilon\cdot|S\cap \cN^{\star}|-a^{\dagger}(\rho_0)}\right)\\
            & \leq \sup_{\lambda\geq 0}\left\{-\lambda\cdot\log\left(\frac{1+e^{-\frac{r_{\max}}{\lambda}}\cdot(1+\epsilon\cdot|S\cap \cN^{\star}|)}{2+\epsilon\cdot|S\cap \cN^{\star}|-a^{\dagger}(\rho_0)}\right)\right\}=   R_{\rho(\cdot;\cdot)}(\widetilde{S}(S);\bv,\br),
        \end{align}
        where the first inequality uses the inequality that 
        \begin{align}
            \frac{a+x}{b+y}\geq \frac{a}{b}\quad \mathrm{if}\quad \text{(i)}\,x\geq y\geq 0\,\,\mathrm{and}\,\,\text{(ii)}\,b\geq a\geq 0.
        \end{align}
        To satisfy the conditions of the above inequality, notice that (i) is trivial since 
        \begin{align}
            x=(1-1/K\cdot e^{-\frac{r_{\max}}{\lambda^{\star}_S}})\cdot|S\cap \cN^{0}| \geq (1-1/K)\cdot|S\cap \cN^{0}|=y.
        \end{align}
        To show the condition (ii), notice that by the optimality of $\lambda^{\star}_S$, we know that 
        \begin{align}
            \frac{a+x}{b+y} = \frac{1+e^{-\frac{r_{\max}}{\lambda^{\star}}}\cdot(1+\epsilon\cdot|S\cap \cN^{\star}|)+(1-1/K\cdot e^{-\frac{r_{\max}}{\lambda^{\star}_S}})\cdot|S\cap \cN^{0}|}{2+\epsilon\cdot|S\cap \cN^{\star}|-a^{\dagger}(\rho_0)+(1-1/K)\cdot|S\cap \cN^{0}|} \leq 1,
        \end{align}
        because we always have that $R_{\rho(\cdot;\cdot)}(S;\bv,\br)\geq 0$.
        This in turn gives that 
        \begin{align}
            a = a+x-x \leq b+y-x\leq b,
        \end{align}
        where the second inequality uses $y\leq x$.
        Thus both two conditions are satisfied, completing the proof.
    \end{proof}
    
    Now for any assortment $S\subseteq \cN^{\star}\cup\cN^{\mathrm{c}}=[4K]$ with $|S|= K$ (with abbreviation $S^{\star}_{\bv,\br}$ as $S^{\star}$ if no confusion), 
    \allowdisplaybreaks
    \begin{align}
        &\mathrm{SubOpt}_{\rho(\cdot;\cdot)}(S;\bv,\br)=R_{\rho(\cdot;\cdot)}(S^{\star};\bv,\br) - R_{\rho(\cdot;\cdot)}(S;\bv,\br) \\
        &\qquad=\sup_{\lambda\geq 0}\left\{-\lambda\cdot \log\left(\frac{1+\exp(-r_{\max}/\lambda)\cdot(1+K\epsilon)}{2+K\epsilon - a^{\dagger}(\rho_0)}\right)\right\}\\
        &\qquad\qquad  - \sup_{\lambda\geq 0}\left\{-\lambda\cdot \log\left(\frac{1+\exp(-r_{\max}/\lambda)\cdot(1+|S\cap S^{\star}|\cdot \epsilon)}{2+|S\cap S^{\star}|\cdot \epsilon -  a^{\dagger}(\rho_0)}\right)\right\}\\
        &\qquad \geq -\lambda^{\star}_{\bv,\br}(S)\cdot \log\left(\frac{1+\exp(-r_{\max}/\lambda^{\star}_{\bv,\br}(S))\cdot(1+K\epsilon)}{2+K\epsilon - a^{\dagger}(\rho_0)}\cdot \frac{2+|S\cap S^{\star}|\cdot \epsilon- a^{\dagger}(\rho_0)}{1+\exp(-r_{\max}/\lambda^{\star}_{\bv,\br}(S))\cdot(1+|S\cap S^{\star}|\cdot \epsilon)}\right)\\
        &\qquad\geq -\lambda^{\star}_{\bv,\br}(S)\cdot \left(\frac{1+\exp(-r_{\max}/\lambda^{\star}_{\bv,\br}(S))\cdot(1+K\epsilon)}{2+K\epsilon- a^{\dagger}(\rho_0)}\cdot \frac{2+|S\cap S^{\star}|\cdot \epsilon- a^{\dagger}(\rho_0)}{1+\exp(-r_{\max}/\lambda^{\star}_{\bv,\br}(S))\cdot(1+|S\cap S^{\star}|\cdot \epsilon)} - 1\right)\\
        &\qquad = \lambda^{\star}_{\bv,\br}(S)\cdot \Big(1-(1-a^{\dagger}(\rho_0))\cdot \exp(-r_{\max}/\lambda_{\bv,\br}^{\star}(S))\Big)\\
        &\qquad\qquad\cdot \frac{2+|S\cap S^{\star}|\cdot \epsilon-a^{\dagger}(\rho_0)}{1+\exp(-r_{\max}/\lambda^{\star}_{\bv,\br}(S))\cdot(1+|S\cap S^{\star}|\cdot \epsilon)}\cdot\left(\frac{1}{2+|S\cap S^{\star}|\cdot\epsilon - a^{\dagger}(\rho_0)} - \frac{1}{2+K\epsilon-a^{\dagger}(\rho_0)}\right)\\
        &\qquad  \geq \frac{\epsilon}{4\cdot (2+K\epsilon-a^{\dagger}(\rho_0))^2}\cdot \lambda^{\star}_{\bv,\br}(S)\cdot \Big(1-(1-a^{\dagger}(\rho_0))\cdot \exp(-r_{\max}/\lambda_{\bv,\br}^{\star}(S))\Big)\cdot \Delta_{\bv, \br}(S),\label{eq:_proof_lower_bound_new_2}
    \end{align}
 where the first equality uses the dual representation of the robust expected revenue (Proposition~\ref{prop:_dual}), the first inequality utilizes the notion of $\lambda^{\star}_{\bv,\br}(S)$ defined as 
    \begin{align}
        \lambda^{\star}_{\bv,\br}(S):=\argmax_{\lambda\geq 0}\left\{-\lambda\cdot \log\left(\frac{1+\exp(-r_{\max}/\lambda)\cdot(1+|S\cap S^{\star}|\cdot \epsilon)}{2+|S\cap S^{\star}|\cdot \epsilon - a^{\dagger}(\rho_0)}\right)\right\},\label{eq: dual proof new}
    \end{align}
    the second inequality $\log(1+x)\leq x$ for $x>-1$, and the last inequality uses the fact that 
    \begin{align}
        \frac{1}{2+|S\cap S^{\star}|\cdot\epsilon-a^{\dagger}(\rho_0)} - \frac{1}{2+K\epsilon-a^{\dagger}(\rho_0)}\geq \frac{\epsilon}{2\cdot (2+K\epsilon-a^{\dagger}(\rho_0))^2}\cdot \underbrace{\Delta(S^{\star}_{\bv,\br},S)}_{\displaystyle{:=\Delta_{\bv,\br}(S)}},
    \end{align}
    and the fact that 
    \begin{align}
        \frac{2+|S\cap S^{\star}|\cdot \epsilon-a^{\dagger}(\rho_0)}{1+\exp(-r_{\max}/\lambda^{\star}_{\bv,\br}(S))\cdot(1+|S\cap S^{\star}|\cdot \epsilon)} \geq \frac{1}{2}.
    \end{align}

\paragraph{Step 2.1: upper and lower bound the dual variable.} Now we need to upper and lower bound the dual variable $\lambda^{\star}_{\bv,\br}(S)$ defined in \eqref{eq: dual proof new}. Thanks to Lemma~\ref{lem: aux 1}, we only need to consider $\lambda^{\star}_{\bv,\br}(S)$ for $(\bv,\br)\in\cV$ and $S\subseteq[4K]$ with $|S|=K$.
    On the one hand, by Proposition~\ref{prop:_dual}, it follows directly that 
    \begin{align}
        \lambda^{\star}_{\bv,\br}(S) \leq \frac{r_{\max}}{\rho(\bv;S)} = \frac{r_{\max}}{\log\left(\frac{2+|S\cap S^{\star}|\cdot\epsilon}{2+|S\cap S^{\star}|\cdot\epsilon - a^{\dagger}(\rho_0)}\right)} \leq \frac{r_{\max}}{\log\left(\frac{2+K\cdot\epsilon}{2+K\cdot\epsilon - a^{\dagger}(\rho_0)}\right)}.\label{eq: dual upper bound new}
    \end{align}
    On the other hand, we now lower bound the dual variable as follows. 
    For simplicity, let's define the dual representation function as following, given $(\bv,\br;S)\in \cV\times 2^{[4K]}$ with $|S|=K$,
    \begin{align}
        h_{\bv,\br,S}(\lambda):= -\lambda\cdot \log\left(\frac{1+\exp(-r_{\max}/\lambda)\cdot(1+|S\cap S^{\star}|\cdot \epsilon)}{2+|S\cap S^{\star}|\cdot \epsilon - a^{\dagger}(\rho_0)}\right).
    \end{align}
    Similar to the proofs in Appendix~\ref{subsec:_proof_lower_bound_1} (\textbf{Step 2.1}), to prove $\lambda^{\star}_{\bv,\br}(S)\geq \underline{\lambda}$ for some $\underline{\lambda}>0$, it suffices to prove that $\partial_{\lambda} h_{\bv,\br,S}(\underline{\lambda}) \geq 0$.
    To this end, we first lower bound the derivative of $h_{\bv,\br,S}$ with respect to $\lambda$ as following, 
    \allowdisplaybreaks
    \begin{align}
        \partial_{\lambda} h_{\bv,\br,S}(\lambda)& =- \log\left(\frac{1+\exp(-r_{\max}/\lambda)\cdot(1+|S\cap S^{\star}|\cdot \epsilon) }{2+|S\cap S^{\star}|\cdot \epsilon -  a^{\dagger}(\rho_0)} \right) \\
        &\qquad - \frac{r_{\max}}{\lambda}\cdot\frac{\exp(-r_{\max}/\lambda)\cdot(1+|S\cap S^{\star}|\cdot \epsilon)}{1+\exp(-r_{\max}/\lambda)\cdot(1+|S\cap S^{\star}|\cdot \epsilon)}\\
        &= \log\left(2+|S\cap S^{\star}|\cdot \epsilon-  a^{\dagger}(\rho_0)\right)  - \log\Big(1+\exp(-r_{\max}/\lambda)\cdot(1+|S\cap S^{\star}|\cdot \epsilon)\Big)\\
        &\qquad - \frac{r_{\max}}{\lambda}\cdot\frac{\exp(-r_{\max}/\lambda)\cdot(1+|S\cap S^{\star}|\cdot \epsilon)}{1+\exp(-r_{\max}/\lambda)\cdot(1+|S\cap S^{\star}|\cdot \epsilon)}.\label{eq:_proof_lower_bound_new_3}
    \end{align}
    Now consider the same reparametrization as \eqref{eq: reparametrization jin}, we can further handle the right hand side of \eqref{eq:_proof_lower_bound_new_3} as 
    \begin{align}
        \partial_{\lambda} h_{\bv,\br,S}(\lambda)& = \log\left(1-a^{\dagger}(\rho_0) + x\right)-\log(1+u) + \log\left(\frac{u}{x}\right)\cdot\frac{u}{1+u} \\
        &\geq \log\left(1-a^{\dagger}(\rho_0) + x\right)-\log(1+u) + \log\left(\frac{u}{2}\right)\cdot u.\label{eq:_proof_lower_bound_new_4}
    \end{align}
    Taking $u = (c^{\clubsuit}_{\rho}\wedge 1/2)\cdot \log(1-a^{\dagger}(\rho_0)+x)\leq 2$, where $c^{\clubsuit}_{\rho}$ is given by solving the equation that 
    \begin{align}
        c^{\clubsuit}_{\rho}:=\argmax_{c\geq 0}\left\{c\cdot\log\left(\frac{(c\wedge 1/2)\cdot\log(1+\overline{c}_{\rho})}{2}\right) \geq -\frac{1}{2}\right\},\label{eq: c dagger new}
    \end{align}
    It holds that $c^{\clubsuit}_{\rho}$ is an absolute positive constant which only dependents on the constant $\overline{c}_{\rho}$. 
    We then lower bound the right hand side of \eqref{eq:_proof_lower_bound_new_4} by 
    \begin{align}
        \partial_{\lambda} h_{\bv,\br,S}(\lambda)&\geq \log\left(1-a^{\dagger}(\rho_0) + x\right)-\log\left(1+\frac{1}{2}\cdot \log(1-a^{\dagger}(\rho_0)+x)\right) \\
        &\qquad + c^{\clubsuit}_{\rho}\cdot\log\left(\frac{(c^{\clubsuit}_{\rho}\wedge 1/2)\cdot \log(1-a^{\dagger}(\rho_0)+x)}{2}\right)\cdot \log\left(1-a^{\dagger}(\rho_0)+x\right)\\
        &\geq \log\left(1-a^{\dagger}(\rho_0) + x\right)-\frac{1}{2}\cdot \log\left(1-a^{\dagger}(\rho_0)+x\right)\\
        &\qquad + c^{\clubsuit}_{\rho}\cdot\log\left(\frac{(c^{\clubsuit}_{\rho}\wedge 1/2)\cdot \log(1+\overline{c}_{\rho})}{2}\right)\cdot \log\left(1-a^{\dagger}(\rho_0)+x\right)\\
        &\geq \log\left(1-a^{\dagger}(\rho_0) + x\right)-\frac{1}{2}\cdot \log\left(1-a^{\dagger}(\rho_0)+x\right)-\frac{1}{2}\cdot \log\left(1-a^{\dagger}(\rho_0)+x\right)=0,
    \end{align}
    where the first inequality uses the choice of $u$, the second inequality uses the inequality that $\log(1+x)\geq x$ for $x>-1$, and the fact that $x\geq 1$, and that $1-a^{\dagger}(\rho_0)\geq \overline{c}_{\rho}$ under the condition on $\rho_0$ that
    \begin{align}
        \rho_0 \leq \log\left(\frac{1+v_{\mathrm{tot}}(\bv)}{\overline{c}_{\rho} + v_{\mathrm{tot}} (\bv)}\right)=\log\left(\frac{5+K\epsilon + N-4K}{\overline{c}_{\rho}+ 4+ K\epsilon + N-4K}\right).
    \end{align}
    The last inequality uses property \eqref{eq: c dagger new} that $c^{\clubsuit}_{\rho}$ satisfies.
    Therefore, with the above choice of $u$, the associated $\lambda$ satisfies $\partial h_{\bv,\br,S}(\lambda)\geq 0$. 
    Thus, using the reparametrization formula \eqref{eq: reparametrization jin}, we get back to the desired $\underline{\lambda}$ as  
    \begin{align}
        \underline{\lambda}:= \frac{r_{\max}}{\log\left(\frac{x}{u}\right)}=\frac{r_{\max}}{\log\left(\frac{x}{(c^{\clubsuit}_{\rho}\wedge 1/2)\cdot \log(1-a^{\dagger}(\rho_0)+x)}\right)}.
    \end{align}
    Consequently, we can conclude that the optimal dual variable is lower bounded by 
    \begin{align}
        \lambda^{\star}_{\bv,\br}(S)\geq \underline{\lambda}=\frac{r_{\max}}{\log\left(\frac{x}{(c^{\clubsuit}_{\rho}\wedge 1/2)\cdot \log(1-a^{\dagger}(\rho_0)+x)}\right)} \geq \frac{r_{\max}}{\log\left(\frac{2}{(c^{\clubsuit}_{\rho}\wedge 1/2)\cdot \log(1+\overline{c}_{\rho})}\right)}.\label{eq: lambda lower bound new}
    \end{align}

     \paragraph{Step 2.2: lower bound the minimax rate of the quantity $\Delta(S^{\star}_{\bv,\br},\pi(\mathbb{D}))$.} 
     This coincides with the \textbf{Step 2.2} in the proof for constant robust set size case. 
     See Lemma~\ref{lem: minimax rate 1} in Appendix~\ref{subsec:_proof_lower_bound_1}.

         \paragraph{Step 2.3: finishing the proof.}
    We can lower bound the expected suboptimality gap by 
    \begin{align}
        &\inf_{\pi(\cdot):(2^{[N]}\times[N])^n\mapsto2^{[N]}}\sup_{(\bv,\br)\in\cV}\mathbb{E}_{\mathbb{D}\sim \otimes_{k=1}^n\mathbb{P}_{\bv}(\cdot|S_k)}\big[\mathrm{SupOpt}_{\rho(\cdot;\cdot)}(\pi(\mathbb{D});\bv,\br) \big]\\
        &\qquad \geq \inf_{\pi(\cdot):(2^{[N]}\times[N])^n\mapsto2^{[N]}}\sup_{(\bv,\br)\in\cV}\mathbb{E}_{\mathbb{D}\sim \otimes_{k=1}^n\mathbb{P}_{\bv}(\cdot|S_k)}\big[\mathrm{SupOpt}_{\rho(\cdot;\cdot)}(\widetilde{S}(\pi(\mathbb{D}));\bv,\br) \big]\\
        &\qquad \geq \inf_{\pi(\cdot):(2^{[N]}\times[N])^n\mapsto2^{[N]}}\sup_{(\bv,\br)\in\cV}\bigg\{\frac{\epsilon}{4\cdot (2+K\epsilon-a^{\dagger}(\rho_0))^2}\cdot \mathbb{E}_{\mathbb{D}\sim \otimes_{k=1}^n\mathbb{P}_{\bv}(\cdot|S_k)}\bigg[\lambda^{\star}_{\bv,\br}(\widetilde{S}(\pi(\mathbb{D})))\\
        &\qquad\qquad \cdot \Big(1-(1-a^{\dagger}(\rho_0))\cdot \exp\big(-r_{\max}/\lambda_{\bv,\br}^{\star}(\widetilde{S}(\pi(\mathbb{D})))\big)\Big)\cdot \Delta(S^{\star}_{\bv,\br},\widetilde{S}(\pi(\mathbb{D})))\bigg]\bigg\}\\
        &\qquad \geq \inf_{S\subseteq[4k],|S|=K}\sup_{(\bv,\br)\in\cV}\bigg\{\frac{\epsilon}{4\cdot (2+K\epsilon-a^{\dagger}(\rho_0))^2}\cdot \lambda^{\star}_{\bv,\br}(S)\cdot \Big(1-(1-a^{\dagger}(\rho_0))\cdot \exp\big(-r_{\max}/\lambda_{\bv,\br}^{\star}(S)\big)\Big)\bigg\}\\
        &\qquad\qquad \cdot\inf_{\pi(\cdot):(2^{[N]}\times[N])^n\mapsto2^{[N]}}\sup_{(\bv,\br)\in\cV}\bigg\{ \mathbb{E}_{\mathbb{D}\sim \otimes_{k=1}^n\mathbb{P}_{\bv}(\cdot|S_k)}\big[\Delta(S^{\star}_{\bv,\br},\pi(\mathbb{D}))\big]\bigg\}\\
        &\qquad \geq \inf_{S\subseteq[4k],|S|=K}\sup_{(\bv,\br)\in\cV}\bigg\{\frac{\epsilon}{4\cdot (2+K\epsilon-a^{\dagger}(\rho_0))^2}\cdot \lambda^{\star}_{\bv,\br}(S)\cdot \Big(1-(1-a^{\dagger}(\rho_0))\cdot \exp\big(-r_{\max}/\lambda_{\bv,\br}^{\star}(S)\big)\Big)\bigg\}\\
        &\qquad\qquad \cdot \frac{1}{128}\cdot \sqrt{\frac{C_1}{5}}\cdot\frac{K}{\epsilon}\cdot\sqrt{\frac{1}{\min_{i\in S^{\star}_{\bv,\br}}n_i}},\label{eq: lower bound new conclude 1}
    \end{align}
    where the first inequality uses Lemma~\ref{lem: aux 1} to reduce any potential output $\pi(\mathbb{D})$ into an assortment $S(\pi(\mathbb{D}))\subseteq[4K]$, the second inequality uses \eqref{eq:_proof_lower_bound_new_2}, the third inequality uses that $\widetilde{S}(\cdot)\subseteq [4K]$ and the trivial fact that
    \begin{align}
        &\inf_{\pi(\cdot):(2^{[N]}\times[N])^n\mapsto2^{[N]}}\sup_{(\bv,\br)\in\cV}\bigg\{ \mathbb{E}_{\mathbb{D}\sim \otimes_{k=1}^n\mathbb{P}_{\bv}(\cdot|S_k)}\big[\Delta(S^{\star}_{\bv,\br},\widetilde{S}(\pi(\mathbb{D})))\big]\bigg\} \\
        &\qquad \geq \inf_{\pi(\cdot):(2^{[N]}\times[N])^n\mapsto2^{[N]}}\sup_{(\bv,\br)\in\cV}\bigg\{ \mathbb{E}_{\mathbb{D}\sim \otimes_{k=1}^n\mathbb{P}_{\bv}(\cdot|S_k)}\big[\Delta(S^{\star}_{\bv,\br},\pi(\mathbb{D}))\big]\bigg\},
    \end{align}
    and the last inequality uses Lemma~\ref{lem-KL-new-upper-bound-non-uniform} and the fact that $n_{\min}=\min_{i\in S^{\star}_{\bv,\br}}n_i$.
    To move on, using the upper and lower bounds on the dual variable, i.e., \eqref{eq: dual upper bound new} and \eqref{eq: lambda lower bound new}, we have that for any $S\subseteq[4K]$ with $|S|=K$,
    \begin{align}
        \frac{r_{\max}}{\log\left(\frac{2}{(c^{\clubsuit}_{\rho}\wedge 1/2)\cdot \log(1+\overline{c}_{\rho})}\right)} \leq \lambda^{\star}_{\bv,\br}(S)\leq  \frac{r_{\max}}{\log\left(\frac{2+K\cdot\epsilon}{2+K\cdot\epsilon - a^{\dagger}(\rho_0)}\right)},\label{eq: lower bound new conclude 2}
    \end{align}
    and consequently that
    \begin{align}
        1-(1-a^{\dagger}(\rho_0))\cdot \exp\big(-r_{\max}/\lambda_{\bv,\br}^{\star}(S)\big) &\geq 1 - (1-a^{\dagger}(\rho_0))\cdot\frac{2+K\epsilon - a^{\dagger}(\rho_0)}{2+K\epsilon}\\
        &\geq a^{\dagger}(\rho_0)\cdot\frac{2+K\epsilon-a^{\dagger}(\rho_0)}{2}\\
        &\geq \underline{c}_{\rho}\cdot\frac{2+K\epsilon-a^{\dagger}(\rho_0)}{2},\label{eq: lower bound new conclude 3}
    \end{align}
    where the last inequality follows from that $a^{\dagger}(\rho_0)\geq \underline{c}_{\rho}$ under the condition on the parameter $\rho_0$ that 
    \begin{align}
        \rho_0 \geq \log\left(\frac{1+v_{\mathrm{tot}}(\bv)}{1 + v_{\mathrm{tot}}(\bv) - \underline{c}_{\rho} }\right)=\log\left(\frac{5+K\epsilon + N-4K}{5+ K\epsilon + N-4K-\underline{c}_{\rho}}\right)
    \end{align}
    Therefore, we can conclude that the expected suboptimality is lower bounded by 
    \begin{align}
        &\inf_{\pi(\cdot):(2^{[N]}\times[N])^n\mapsto2^{[N]}}\sup_{(\bv,\br)\in\cV}\mathbb{E}_{\mathbb{D}\sim \otimes_{k=1}^n\mathbb{P}_{\bv}(\cdot|S_k)}\big[\mathrm{SupOpt}_{\rho(\cdot;\cdot)}(\pi(\mathbb{D});\bv,\br) \big]\\
        &\qquad\geq \frac{\underline{c}_{\rho}\cdot\sqrt{C_1}}{1024\sqrt{5}\cdot\log\left(\frac{2}{(c^{\clubsuit}_{\rho}\wedge 1/2)\cdot \log(1+\overline{c}_{\rho})}\right)}\cdot \frac{r_{\max}\cdot K}{2+K\epsilon - a^{\dagger}(\rho_0)}\cdot \sqrt{\frac{1}{\min_{i\in S^{\star}_{\bv,\br}}n_{i}}} \\
        &\qquad\geq \frac{\underline{c}_{\rho}\cdot\sqrt{C_1}}{2048\sqrt{5}\cdot\log\left(\frac{2}{(c^{\clubsuit}_{\rho}\wedge 1/2)\cdot \log(1+\overline{c}_{\rho})}\right)}\cdot \frac{r_{\max}\cdot K}{1+(1+K\epsilon)/2 - a^{\dagger}(\rho_0)}\cdot \sqrt{\frac{1}{\min_{i\in S^{\star}_{\bv,\br}}n_{i}}} \\
        &\qquad\geq \frac{\underline{c}_{\rho}\cdot\sqrt{C_1}}{2048\sqrt{10}\cdot\log\left(\frac{2}{(c^{\clubsuit}_{\rho}\wedge 1/2)\cdot \log(1+\overline{c}_{\rho})}\right)}\cdot \frac{r_{\max}\cdot K}{\sqrt{1+(1+K\epsilon)/2 - a^{\dagger}(\rho_0)}}\cdot \sqrt{\frac{1}{\min_{i\in S^{\star}_{\bv,\br}}n_{i}}} \\
        &\qquad \cong  \frac{ r_{\max}\cdot K}{\sqrt{1+\sum_{j\in S^{\star}_{\bv,\br}}v_j/2-(1-e^{-\rho_0})\cdot (1+v_{\mathrm{tot}}(\bv))}}\cdot \sqrt{\frac{1}{\min_{i\in S^{\star}_{\bv,\br}}n_{i}}},
    \end{align}
    where the first inequality combines \eqref{eq: lower bound new conclude 1}, \eqref{eq: lower bound new conclude 2}, and \eqref{eq: lower bound new conclude 3}, and the last inequality hides global constants and uses the fact that $\sum_{j\in S^{\star}_{\bv,\br}}v_j = 1+K\epsilon$.
    This completes the proof of Theorem~\ref{thm:_lower_bound_new_1} for general non-uniform revenue case.
\end{proof}

\begin{proof}[\textcolor{blue}{Proof of Theorem~\ref{thm:_lower_bound_new_1}: uniform revenue case}]
    The proof to the uniform revenue case resembles that for the general non-uniform revenue case, with the differences we highlight in the sequel.
    Firstly, we construct the hard instance class $\mathcal{V}_{\mathrm{uni}}$ and the offline dataset $\mathbb{D}$ following the same way as in the proof for Theorem~\ref{thm:_lower_bound_1}. We refer to Section~\ref{subsec:_proof_lower_bound_1} for details.
    For the ease of readers, we repeat the definition of $\cV_{\mathrm{uni}}$ here,
    \begin{align}
        \cV_{\mathrm{uni}} := \Big\{\text{$(\bv,\br)$ induced by $(\cN^{\star}, \cN^{\mathrm{c}}, \cN^0)$ via \eqref{eq:_proof_lower_bound_2}}\,\Big|\,\text{$\cN^{\star}\in \cF$ with $\cF$ given by Lemma~\ref{lem-packing-number}}\Big\}.\label{eq: hard instance class uniform repeated}
    \end{align}
    
    We have the following result for the optimal robust assortment associated with instances in the class $\cV_{\mathrm{uni}}$.
    \begin{lemma}[Optimal robust assortment for class \eqref{eq: hard instance class uniform repeated}]\label{lem: optimal robust new uniform}
    For any problem instance $(\bv,\br)\in\cV_{\mathrm{uni}}$ associated
with the division $(\cN^{\star},\cN^{\mathrm{c}},\cN^0)$, its optimal robust assortment $S^{\star}_{\bv,\br}$ in the sense of Example~\ref{exp: new} is $S^{\star}_{\bv,\br}=\cN^{\star}$.
        \end{lemma}
        \begin{proof}[Proof of Lemma~\ref{lem: optimal robust new uniform}]
         By definition, for any instance $(\bv,\br)\in\cV_{\mathrm{uni}}$ with partitions given by $(\cN^{\star},\cN^\mathrm{c}, \cN^0)$, the robust expected revenue of any assortment $S$ is given by 
    \begin{align}
            &R_{\rho(\cdot;\cdot)}(S;\bv, \br)\\
            &\qquad =\sup_{\lambda\geq 0}\left\{-\lambda\cdot\log\left(\frac{1+e^{-\frac{r_{\max}}{\lambda}}\cdot(1/K\cdot|S\cap \cN^0|+1/K\cdot|S\cap \cN^{\mathrm{c}}| + (1/K+\epsilon')\cdot |S\cap \cN^{\star}|)}{1+1/K\cdot|S\cap \cN^{0}| + 1/K\cdot|S\cap \cN^{\mathrm{c}}| + (1/K+\epsilon')\cdot |S\cap \cN^{\star}| - a^{\dagger}(\rho_0)}\right)\right\} \\
            &\qquad =\sup_{\lambda\geq 0}\left\{-\lambda\cdot\log\left(e^{-\frac{r_{\max}}{\lambda}}+\frac{1-e^{-\frac{r_{\max}}{\lambda}} + a^{\dagger}(\rho_0)\cdot e^{-\frac{r_{\max}}{\lambda}}}{1+1/K\cdot |S\cap\cN^0|+1/K\cdot|S\cap \cN^{\mathrm{c}}|+(1/K+\epsilon')\cdot|S\cap \cN^{\star}| - a^{\dagger}(\rho_0)}\right)\right\} \\
            &\qquad \leq \sup_{\lambda\geq 0}\left\{-\lambda\cdot\log\left(e^{-\frac{r_{\max}}{\lambda}}+\frac{1-e^{-\frac{r_{\max}}{\lambda}} + a^{\dagger}(\rho_0)\cdot e^{-\frac{r_{\max}}{\lambda}}}{1+(1/K+\epsilon')\cdot|\cN^{\star}| - a^{\dagger}(\rho_0)}\right)\right\}, 
        \end{align}
        and the inequality becomes equality only when $S=\cN^{\star}$, completing the proof of Lemma~\ref{lem: optimal robust assortment uniform}.
    \end{proof}
    The rest of the proof for the uniform revenue case aligns well with the proofs for the general non-uniform case, with exceptional difference we discuss in detail in the sequel. 

    With the same spirit of Lemma~\ref{lem: aux 1}, we first demonstrate that it suffices to consider assortments $S\subseteq[4K]$ satisfying $|S|=K$.
     This is relatively easy to prove for the uniform revenue instances. 
     To see this, consider that for any assortment $S\subseteq[N]$ with $|S|\leq K$, we have 
     \begin{align}
         &R_{\rho(\cdot;\cdot)}(S;\bv, \br)\\
            &\qquad =\sup_{\lambda\geq 0}\left\{-\lambda\cdot\log\left(e^{-\frac{r_{\max}}{\lambda}}+\frac{1-e^{-\frac{r_{\max}}{\lambda}} + a^{\dagger}(\rho_0)\cdot e^{-\frac{r_{\max}}{\lambda}}}{1+1/K\cdot |S\cap\cN^0|+1/K\cdot|S\cap \cN^{\mathrm{c}}|+(1/K+\epsilon')\cdot|S\cap \cN^{\star}| - a^{\dagger}(\rho_0)}\right)\right\} \\
            &\qquad \leq \sup_{\lambda\geq 0}\left\{-\lambda\cdot\log\left(e^{-\frac{r_{\max}}{\lambda}}+\frac{1-e^{-\frac{r_{\max}}{\lambda}} + a^{\dagger}(\rho_0)\cdot e^{-\frac{r_{\max}}{\lambda}}}{1+1/K\cdot(K-|S\cap \cN^{\star}|)+(1/K+\epsilon')\cdot|S\cap \cN^{\star}| - a^{\dagger}(\rho_0)}\right)\right\}\\
            &\qquad = R_{\rho(\cdot;\cdot)}(\widetilde{S};\bv, \br)\quad\text{where}\quad \widetilde{S}(S):=(S\cap \cN^{\star})\cup \cN^{\mathrm{c}}_{K-|S\cap \cN^{\star}|}\subseteq\cN^{\star}\cup\cN^{\mathrm{c}}=[4K].\label{eq:_proof_lower_bound_new_uniform_1}
     \end{align}
     Therefore, now we only need to lower bound the suboptimality of any assortment $S\subseteq[4K]$ with $|S|=K$.
     To this end, we have the following,
     \begin{align}
         &\mathrm{SubOpt}_{\rho(\cdot;\cdot)}(S;\bv,\br)=R_{\rho(\cdot;\cdot)}(S^{\star};\bv,\br) - R_{\rho(\cdot;\cdot)}(S;\bv,\br) \\
        &\qquad=\sup_{\lambda\geq 0}\left\{-\lambda\cdot \log\left(\frac{1+\exp(-r_{\max}/\lambda)\cdot(1+K\epsilon')}{2+K\epsilon - a^{\dagger}(\rho_0)}\right)\right\}\\
        &\qquad\qquad  - \sup_{\lambda\geq 0}\left\{-\lambda\cdot \log\left(\frac{1+\exp(-r_{\max}/\lambda)\cdot(1+|S\cap S^{\star}|\cdot \epsilon')}{2+|S\cap S^{\star}|\cdot \epsilon' -  a^{\dagger}(\rho_0)}\right)\right\}\\
        &\qquad  \geq \frac{\epsilon'}{4\cdot (2+K\epsilon'-a^{\dagger}(\rho_0))^2}\cdot \lambda^{\star}_{\bv,\br}(S)\cdot \Big(1-(1-a^{\dagger}(\rho_0))\cdot \exp(-r_{\max}/\lambda_{\bv,\br}^{\star}(S))\Big)\cdot \Delta_{\bv, \br}(S),\label{eq:_proof_lower_bound_new_uniform_2}
     \end{align}
    where the first inequality is by the dual representation of the robust expected revenue, the inequality follows the exactly the same argument as \eqref{eq:_proof_lower_bound_new_2}, since the left hand side of the inequality shares the same form with the left hand side of \eqref{eq:_proof_lower_bound_new_uniform_2}.
    Here the optimal dual $\lambda^{\star}_{\bv,\br}(S)$ is defined in the same way as \eqref{eq: dual proof new}, which means that it also enjoys a similar upper and lower bounds we derived in \eqref{eq: dual upper bound new} and \eqref{eq: lambda lower bound new}\footnote{Again, we remark that the difference in $\epsilon$ and $\epsilon'$ does not change the conclusion of the proofs for \eqref{eq:_proof_lower_bound_new_2} and  the bounds of $\lambda^{\star}_{\bv,\br}(S)$. Besides $\epsilon'$, the difference of the classes $\cV$ and $\cV_{\mathrm{uni}}$ only lies in the items in $\cN^{0}$, and it does not make a difference in proving the bounds of $\lambda_{\bv,\br}(S)$ for assortments $S\subseteq[4K]=[N]\setminus \cN^0$.}.
    Specifically, by the conditions on $\rho_0$ for the uniform revenue instances that 
    \begin{align}
       &\log\left(\frac{1+K\epsilon' + N/K}{1+ K\epsilon' + N/K-\underline{c}_{\rho}}\right) =\log\left(\frac{1+v_{\mathrm{tot}}(\bv)}{1 + v_{\mathrm{tot}}(\bv) - \underline{c}_{\rho} }\right)\\
       &\qquad \leq \rho_0 \leq \log\left(\frac{1+v_{\mathrm{tot}}(\bv)}{\overline{c}_{\rho} + v_{\mathrm{tot}} (\bv)}\right)=\log\left(\frac{1+K\epsilon' + N/K}{\overline{c}_{\rho}+ K\epsilon' + N/K}\right),
    \end{align}
    we have the that
    \begin{align}
        \frac{r_{\max}}{\log\left(\frac{2}{(c^{\clubsuit}_{\rho}\wedge 1/2)\cdot \log(1+\overline{c}_{\rho})}\right)} \leq \lambda^{\star}_{\bv,\br}(S) \leq \frac{r_{\max}}{\log\left(\frac{2+K\cdot\epsilon}{2+K\cdot\epsilon - a^{\dagger}(\rho_0)}\right)},\label{eq:_proof_lower_bound_new_uniform_3}
    \end{align}
    where $c^{\clubsuit}_{\rho}$ is a constant depending on $\overline{c}_{\rho}$, 
    and consequently that 
    \begin{align}
        1-(1-a^{\dagger}(\rho_0))\cdot \exp\big(-r_{\max}/\lambda_{\bv,\br}^{\star}(S)\big) \geq  \underline{c}_{\rho}\cdot\frac{2+K\epsilon'-a^{\dagger}(\rho_0)}{2},\label{eq:_proof_lower_bound_new_uniform_4}
    \end{align}
    following the same argument as \eqref{eq: lower bound new conclude 3}. 
    Therefore, we obtain from \eqref{eq:_proof_lower_bound_new_uniform_2}, \eqref{eq:_proof_lower_bound_new_uniform_3},  and \eqref{eq:_proof_lower_bound_new_uniform_4} that 
        \begin{align}
        &\inf_{\pi(\cdot):(2^{[N]}\times[N])^n\mapsto2^{[N]}}\sup_{(\bv,\br)\in\cV_{\mathrm{uni}}}\mathbb{E}_{\mathbb{D}\sim \otimes_{k=1}^n\mathbb{P}_{\bv}(\cdot|S_k)}\big[\mathrm{SupOpt}_{\rho(\cdot;\cdot)}(\pi(\mathbb{D});\bv,\br) \big]\\
        &\qquad \geq \inf_{\pi(\cdot):(2^{[N]}\times[N])^n\mapsto2^{[N]}}\sup_{(\bv,\br)\in\cV_{\mathrm{uni}}}\mathbb{E}_{\mathbb{D}\sim \otimes_{k=1}^n\mathbb{P}_{\bv}(\cdot|S_k)}\big[\mathrm{SupOpt}_{\rho(\cdot;\cdot)}(\widetilde{S}(\pi(\mathbb{D}));\bv,\br) \big]\\
        &\qquad \geq \inf_{\pi(\cdot):(2^{[N]}\times[N])^n\mapsto2^{[N]}}\sup_{(\bv,\br)\in\cV_{\mathrm{uni}}}\bigg\{\frac{\epsilon'}{4\cdot (2+K\epsilon'-a^{\dagger}(\rho_0))^2}\cdot \mathbb{E}_{\mathbb{D}\sim \otimes_{k=1}^n\mathbb{P}_{\bv}(\cdot|S_k)}\bigg[\lambda^{\star}_{\bv,\br}(\widetilde{S}(\pi(\mathbb{D})))\\
        &\qquad\qquad \cdot \Big(1-(1-a^{\dagger}(\rho_0))\cdot \exp\big(-r_{\max}/\lambda_{\bv,\br}^{\star}(\widetilde{S}(\pi(\mathbb{D})))\big)\Big)\cdot \Delta(S^{\star}_{\bv,\br},\widetilde{S}(\pi(\mathbb{D})))\bigg]\bigg\}\\
        &\qquad \geq \inf_{S\subseteq[4k],|S|=K}\sup_{(\bv,\br)\in\cV_{\mathrm{uni}}}\bigg\{\frac{\epsilon'}{4\cdot (2+K\epsilon'-a^{\dagger}(\rho_0))^2}\cdot \lambda^{\star}_{\bv,\br}(S)\cdot \Big(1-(1-a^{\dagger}(\rho_0))\cdot \exp\big(-r_{\max}/\lambda_{\bv,\br}^{\star}(S)\big)\Big)\bigg\}\\
        &\qquad\qquad \cdot\inf_{\pi(\cdot):(2^{[N]}\times[N])^n\mapsto2^{[N]}}\sup_{(\bv,\br)\in\cV_{\mathrm{uni}}}\bigg\{ \mathbb{E}_{\mathbb{D}\sim \otimes_{k=1}^n\mathbb{P}_{\bv}(\cdot|S_k)}\big[\Delta(S^{\star}_{\bv,\br},\pi(\mathbb{D}))\big]\bigg\}\\
        &\qquad \geq \frac{\underline{c}_{\rho}\cdot \epsilon'}{8\cdot \log\left(\frac{2}{(c^{\clubsuit}_{\rho}\wedge 1/2)\cdot \log(1+\overline{c}_{\rho})}\right)}\cdot \frac{r_{\max}}{2+K\epsilon'-a^{\dagger}(\rho_0)} \\
         &\qquad\qquad \cdot\inf_{\pi(\cdot):(2^{[N]}\times[N])^n\mapsto2^{[N]}}\sup_{(\bv,\br)\in\cV_{\mathrm{uni}}}\bigg\{ \mathbb{E}_{\mathbb{D}\sim \otimes_{k=1}^n\mathbb{P}_{\bv}(\cdot|S_k)}\big[\Delta(S^{\star}_{\bv,\br},\pi(\mathbb{D}))\big]\bigg\}.\label{eq: lower bound new uniform conclude 1}
    \end{align}
    Consequently, it remains to derive the minimax lower bound for the difference $\Delta(S^{\star}_{\bv,\br},\pi(\mathbb{D}))$ for the hard instance class $\cV_{\mathrm{uni}}$, for which we invoke Lemma~\ref{lem-KL-new-upper-bound-uniform}.
    This gives us 
    \begin{align}
        &\inf_{\pi(\cdot):(2^{[N]}\times[N])^n\mapsto2^{[N]}}\sup_{(\bv,\br)\in\cV_{\mathrm{uni}}}\mathbb{E}_{\mathbb{D}\sim \otimes_{k=1}^n\mathbb{P}_{\bv}(\cdot|S_k)}\big[\mathrm{SupOpt}_{\rho(\cdot;\cdot)}(\pi(\mathbb{D});\bv,\br) \big]\\
        &\qquad \geq \frac{\underline{c}_{\rho}\cdot \sqrt{C_1}\cdot \epsilon'}{1024\sqrt{5}\log\left(\frac{2}{(c^{\clubsuit}_{\rho}\wedge 1/2)\cdot \log(1+\overline{c}_{\rho})}\right)}\cdot \frac{r_{\max}}{2+K\epsilon'-a^{\dagger}(\rho_0)} \cdot\frac{\sqrt{K}}{\epsilon'\cdot\sqrt{n_{\min}}} \\
        &\qquad = \frac{\underline{c}_{\rho}\cdot \sqrt{C_1}}{1024\sqrt{5}\log\left(\frac{2}{(c^{\clubsuit}_{\rho}\wedge 1/2)\cdot \log(1+\overline{c}_{\rho})}\right)}\cdot \frac{r_{\max}\cdot \sqrt{K}}{2+K\epsilon'-a^{\dagger}(\rho_0)} \cdot\sqrt{\frac{1}{\min_{i\in S^{\star}_{\bv,\br}}n_i}} \\ 
        &\qquad \geq  \frac{\underline{c}_{\rho}\cdot \sqrt{C_1}}{2048\sqrt{10}\log\left(\frac{2}{(c^{\clubsuit}_{\rho}\wedge 1/2)\cdot \log(1+\overline{c}_{\rho})}\right)}\cdot \frac{r_{\max}\cdot \sqrt{K}}{\sqrt{1+(1+K\epsilon')/2-a^{\dagger}(\rho_0)}} \cdot\sqrt{\frac{1}{\min_{i\in S^{\star}_{\bv,\br}}n_i}}\\ 
        &\qquad \cong \frac{r_{\max}\cdot \sqrt{K}}{\sqrt{1+\sum_{j\in S^{\star}_{\bv,\br}}v_j/2-(1-e^{-\rho_0})\cdot(1+v_{\mathrm{tot}}(\bv))}} \cdot\sqrt{\frac{1}{\min_{i\in S^{\star}_{\bv,\br}}n_i}}, 
    \end{align}
    where the first inequality applies Lemma~\ref{lem-KL-new-upper-bound-uniform} to \eqref{eq: lower bound new uniform conclude 1}, and the last inequality hides the global constants and the fact that $\sum_{j\in S^{\star}_{\bv,\br}}v_j = 1+K\epsilon'$. This completes the proof of Theorem~\ref{thm:_lower_bound_new_1} for uniform revenue case.
\end{proof}

\newpage

\section{Technical Lemmas}\label{lem: technical}

\begin{lemma}[Monotonicity under optimal robust assortment (Example~\ref{exp: jin})]\label{lem:_monotonicity_1}
    Consider Example~\ref{exp: jin} and two sets of parameters $\bv$ and $\bv'$ such that $v_j\leq v_j'$ for all $j\in[N]$. 
    Then for the optimal robust assortment $S^{\star}_{\bv}$ associated with $\bv$, it holds that 
    \begin{align}
        R_{\rho}(S^{\star}_{\bv};\bv)\leq R_{\rho}(S^{\star}_{\bv};\bv').
    \end{align}
\end{lemma}

\begin{proof}[Proof of Lemma~\ref{lem:_monotonicity_1}]
    To facilitate presentation, for any parameter $\bv$, we define 
    \begin{align}
        (S^{\star}_{\bv},\lambda^{\star}_{\bv}) \in \argmax_{S\subseteq[N], |S|\leq K, \lambda\geq 0} H_{\rho}(S,\lambda;\bv),\quad R^{\star}_{\rho}(\bv):= H_{\rho}(S^{\star}_{\bv},\lambda^{\star}_{\bv};\bv),
    \end{align}
    where we recall that $H_{\rho}(S,\lambda;\bv)$ is defined as 
    \begin{align}
        H_{\rho}(S,\lambda;\bv) = - \lambda\cdot\log\left(\frac{\sum_{i\in S_+}v_i\cdot\exp(-r_i/\lambda)}{\sum_{i\in S_+}v_i}\right)-\lambda\cdot\rho.\label{eq:_proof_monotonicity_1_2}
    \end{align}
    Then by Proposition~\ref{prop:_dual}, proving Lemma~\ref{lem:_monotonicity_1} is equivalent to showing that 
    \begin{align}
        \sup_{\lambda\geq 0} H_{\rho}(S^{\star}_{\bv},\lambda;\bv)\leq \sup_{\lambda\geq 0} H_{\rho}(S^{\star}_{\bv},\lambda;\bv').
    \end{align}
    To this end, consider the following upper bound, 
    \begin{align}
        \sup_{\lambda\geq 0} H_{\rho}(S^{\star}_{\bv},\lambda;\bv)- \sup_{\lambda\geq 0} H_{\rho}(S^{\star}_{\bv},\lambda;\bv') \leq H_{\rho}(S^{\star}_{\bv},\lambda^{\star}_{\bv};\bv) - H_{\rho}(S^{\star}_{\bv},\lambda^{\star}_{\bv};\bv'), \label{eq:_proof_monotonicity_1_2+}
    \end{align}
    where we utilize the optimality of $\lambda^{\star}_{\bv}$ under $S^{\star}_{\bv}$ and the property of supremum operator. 
    Now given definition \eqref{eq:_proof_monotonicity_1_2}, to prove that the right hand side of \eqref{eq:_proof_monotonicity_1_2+} is non-positive, it suffices to prove that 
        \begin{align}
        \frac{\sum_{i\in S^{\star}_{\bv,+}}v_i'\cdot\exp(-r_i/\lambda^{\star}_{\bv}) }{\sum_{i\in S^{\star}_{\bv,+}}v_i' } \leq \frac{\sum_{i\in S^{\star}_{\bv,+}}v_i\cdot\exp(-r_i/\lambda^{\star}_{\bv})}{\sum_{i\in S^{\star}_{\bv,+}}v_i}. \label{eq:_proof_monotonicity_1_3}
    \end{align}
    Now we invoke Lemma~\ref{lem: optimal condition jin}, which shows that if an item $j$ belongs to the optimal assortment $S^{\star}_{\bv}$, then it holds
    \begin{align}
        r_j \geq R_{\rho}^{\star}(\bv) + \lambda^{\star}_{\bv}\cdot\rho,
    \end{align}
    which is equivalent to the following inequality,
    \begin{align}
        \exp(-r_j/\lambda^{\star}_{\bv})\leq  \frac{\sum_{i\in S^{\star}_{\bv,+}}v_i\cdot\exp(-r_i/\lambda^{\star}_{\bv})}{\sum_{i\in S^{\star}_{\bv,+}}v_i}. \label{eq:_proof_monotonicity_1_4}
    \end{align}
    Then with \eqref{eq:_proof_monotonicity_1_4}, considering the following arguments, 
    \begin{align}
        \frac{\sum_{i\in S^{\star}_{\bv,+}}v_i'\cdot\exp(-r_i/\lambda^{\star}_{\bv}) }{\sum_{i\in S^{\star}_{\bv,+}}v_i' } &= \frac{\sum_{i\in S^{\star}_{\bv,+}}v_i\cdot\exp(-r_i/\lambda^{\star}_{\bv})  + \sum_{i\in S^{\star}_{\bv,+}}(v_i'-v_i)\cdot \exp(-r_i/\lambda^{\star}_{\bv})}{\sum_{i\in S^{\star}_+}v_i + \sum_{i\in S_{\bv,+}^{\star}}v_i'-v_i } \\
        &\leq \frac{\sum_{i\in S^{\star}_{\bv,+}}v_i\cdot\exp(-r_i/\lambda^{\star}_{\bv})  + \sum_{i\in S^{\star}_{\bv,+}}(v_i'-v_i)\cdot \frac{\sum_{i\in S^{\star}_{\bv,+}}v_i\cdot\exp(-r_i/\lambda^{\star}_{\bv})}{\sum_{i\in S^{\star}_{\bv,+}}v_i}}{\sum_{i\in S^{\star}_{\bv,+}}v_i + \sum_{i\in S_{\bv,+}^{\star}}v_i'-v_i }\\
        &= \frac{\sum_{i\in S^{\star}_{\bv,+}}v_i\cdot\exp(-r_i/\lambda^{\star}_{\bv})}{\sum_{i\in S^{\star}_{\bv,+}}v_i},\label{eq:_proof_monotonicity_1_5}
    \end{align}
    where the inequality follows from \eqref{eq:_proof_monotonicity_1_4} and the last equality is from the equation that $(a + c\cdot (a/b))/(b+c) = a/b$.
    This proves \eqref{eq:_proof_monotonicity_1_3} and thus completes the proof of Lemma~\ref{lem:_monotonicity_1}.
\end{proof}

\begin{lemma}[Condition for items in optimal robust assortment (Example~\ref{exp: jin})]\label{lem: optimal condition jin}
    Consider Example~\ref{exp: jin} with fixed parameter $\bv$. If an item $j\in[N]$ belongs to the optimal robust assortment $S^{\star}$, then it holds that 
    \begin{align}
        r_j\geq R_{\rho}(S^{\star};\bv) + \lambda^{\star}\cdot\rho,\quad (S^{\star},\lambda^{\star})\in \argmax_{S\subseteq[N],|S|\leq K, \lambda\geq 0} H_{\rho}(S,\lambda;\bv)\,\, \text{with }j\in S^{\star}.
    \end{align}
\end{lemma}

\begin{proof}[Proof of Lemma~\ref{lem: optimal condition jin}]
    We prove the conclusion by contradiction. 
    If there exists an item $j$ such that $j\in S^{\star}$ for some $(S^{\star}, \lambda^{\star})$ satisfying
    \begin{align}
        (S^{\star}, \lambda^{\star})\in \argmax_{S\subseteq[N],|S|\leq K, \lambda\geq 0} H_{\rho}(S,\lambda;\bv),\label{eq: proof lem optimal condition jin 1}
    \end{align}
    but $r_j < R_{\rho}(S^{\star};\bv) + \lambda^{\star}\rho$.
    Then considering the new assortment $S^{\star}_{-j}:=S^{\star}\setminus\{j\}$ which satisfies the cardinality constraint $|S^{\star}_{-j}|\leq K$. 
    By calculation, we can obtain that $H_{\rho}(S^{\star}_{-j},\lambda^{\star};\bv) > H_{\rho}(S^{\star},\lambda^{\star};\bv)$, which contradicts \eqref{eq: proof lem optimal condition jin 1}.
    Thus we must have $r_j \geq R_{\rho}(S^{\star};\bv) + \lambda^{\star}\rho$. 
\end{proof}

\begin{lemma}[Monotonicity under optimal robust assortment (Example~\ref{exp: new})] \label{lem:_monotonicity_2} Consider Example~\ref{exp: new} and two sets of parameters $\bv$ and $\bv'$ such that $v_j\leq v_j'$ for all $j\in[N]$. 
Then for the assortment $S^{\star}_{\bv}$ defined as 
\begin{align}
    S^{\star}_{\bv}\in \argmax_{S\subseteq[N],|S|\leq K} \sup_{\lambda\geq 0} \left\{\widetilde{H}(S,\lambda;\bv)\right\},
\end{align}
the following inequality holds, 
\begin{align}
    \sup_{\lambda\geq 0} \left\{\widetilde{H}(S^{\star}_{\bv},\lambda;\bv)\right\} \leq \sup_{\lambda\geq 0} \left\{\widetilde{H}(S^{\star}_{\bv},\lambda;\bv')\right\}.
\end{align}
\end{lemma}

\begin{proof}[Proof of Lemma~\ref{lem:_monotonicity_2}]
        To facilitate presentation, for any parameter $\bv$, we define 
    \begin{align}
        (S^{\star}_{\bv},\lambda^{\star}_{\bv}) \in \argmax_{S\subseteq[N], |S|\leq K, \lambda\geq 0} \left\{\widetilde{H}(S,\lambda;\bv)\right\},
    \end{align}
    with which we can come to the following upper bound,
    \begin{align}
        \sup_{\lambda\geq 0} \left\{\widetilde{H}(S^{\star}_{\bv},\lambda;\bv)\right\} - \sup_{\lambda\geq 0} \left\{\widetilde{H}(S^{\star}_{\bv},\lambda;\bv')\right\} \leq \widetilde{H}(S^{\star}_{\bv},\lambda_{\bv}^{\star};\bv) - \widetilde{H}(S^{\star}_{\bv},\lambda_{\bv}^{\star};\bv'),\label{eq:_proof_monotonicity_2_1}
    \end{align}
    where we utilize the optimality of $\lambda^{\star}_{\bv}$ under $S^{\star}_{\bv}$ and the property of supremum operator. 
    Now given definition \eqref{eq: h tilde} of $\widetilde{H}$, to prove that the right hand side of \eqref{eq:_proof_monotonicity_2_1} is non-positive, it suffices to prove that 
    \begin{align}
        &-\lambda^{\star}_{\bv}\cdot\log\left(\frac{\sum_{i\in S^{\star}_{\bv,+}}v_i\cdot\exp(-r_i/\lambda^{\star}_{\bv})}{\sum_{i\in S^{\star}_{\bv,+}}v_i}\right) + \lambda_{\bv}^{\star}\cdot \log\left(e^{\rho_0} - \frac{(e^{\rho_0}-1)\cdot (1+v_{\mathrm{tot}})}{\sum_{i\in S^{\star}_{\bv,+}}v_i}\right)
        \\
        &\qquad \leq -\lambda^{\star}_{\bv}\cdot\log\left(\frac{\sum_{i\in S^{\star}_{\bv,+}}v_i'\cdot\exp(-r_i/\lambda^{\star}_{\bv})}{\sum_{i\in S^{\star}_{\bv,+}}v_i'}\right) + \lambda_{\bv}^{\star}\cdot \log\left(e^{\rho_0} - \frac{(e^{\rho_0}-1)\cdot (1+v_{\mathrm{tot}})}{\sum_{i\in S^{\star}_{\bv,+}}v_i'}\right).\label{eq:_proof_monotonicity_2_2}
    \end{align}
    This is further equivalent to proving the following, 
    \begin{align}
        \frac{\sum_{i\in S^{\star}_{\bv,+}}v_i'\cdot\exp(-r_i/\lambda^{\star}_{\bv})}{e^{\rho_0}\cdot\sum_{i\in S^{\star}_{\bv,+}}v_i' - (e^{\rho_0} - 1)\cdot (1+v_{\mathrm{tot}})}
        \leq \frac{\sum_{i\in S^{\star}_{\bv,+}}v_i\cdot\exp(-r_i/\lambda^{\star}_{\bv})}{e^{\rho_0}\cdot\sum_{i\in S^{\star}_{\bv,+}}v_i - (e^{\rho_0} - 1)\cdot (1+v_{\mathrm{tot}})}.
    \end{align}
    Now we invoke Lemma~\ref{lem: optimal condition new}, which shows that for any item $i\in S^{\star}_{\bv}$, it holds that 
    \begin{align}
        r_i \geq  -\lambda^{\star}_{\bv}\cdot  \log\left(\frac{e^{\rho_0}\cdot\sum_{i\in S_{\bv,+}^{\star}}v_i\exp(-r_i/\lambda^{\star}_{\bv})}{e^{\rho_0}\cdot\sum_{i\in S_{\bv,+}^{\star}}v_i - (e^{\rho_0}-1)\cdot (1+v^{\mathrm{tot}})}\right),
    \end{align}
    which is equivalent to the following inequality, 
    \begin{align}
        \exp(-r_i/\lambda^{\star}_{\bv}) \leq \frac{e^{\rho_0}\cdot\sum_{i\in S_{\bv,+}^{\star}}v_i\exp(-r_i/\lambda^{\star}_{\bv})}{e^{\rho_0}\cdot\sum_{i\in S_{\bv,+}^{\star}}v_i - (e^{\rho_0}-1)\cdot (1+v^{\mathrm{tot}})}.\label{eq:_proof_monotonicity_2_3}
    \end{align}
    Then with \eqref{eq:_proof_monotonicity_2_3}, we consider the following arguments, 
    \allowdisplaybreaks
    \begin{align}
        &\frac{\sum_{i\in S^{\star}_{\bv,+}}v_i'\cdot\exp(-r_i/\lambda^{\star}_{\bv})}{e^{\rho_0}\cdot\sum_{i\in S^{\star}_{\bv,+}}v_i' - (e^{\rho_0} - 1)\cdot (1+v_{\mathrm{tot}})} \\
        &\qquad = \frac{\sum_{i\in S^{\star}_{\bv,+}}v_i\cdot\exp(-r_i/\lambda^{\star}_{\bv}) + \sum_{i\in S^{\star}_{\bv,+}}(v_i'-v_i)\cdot\exp(-r_i/\lambda^{\star}_{\bv}) }{e^{\rho_0}\cdot\sum_{i\in S^{\star}_{\bv,+}}v_i + e^{\rho_0}\cdot\sum_{i\in S^{\star}_{\bv,+}}(v_i'-v_i) - (e^{\rho_0} - 1)\cdot (1+v_{\mathrm{tot}})}\\
        &\qquad \leq \frac{\sum_{i\in S^{\star}_{\bv,+}}v_i\cdot\exp(-r_i/\lambda^{\star}_{\bv}) + e^{\rho_0}\cdot \sum_{i\in S^{\star}_{\bv,+}}(v_i'-v_i)\cdot\frac{\sum_{i\in S_{\bv,+}^{\star}}v_i\exp(-r_i/\lambda^{\star}_{\bv})}{e^{\rho_0}\cdot\sum_{i\in S_{\bv,+}^{\star}}v_i - (e^{\rho_0}-1)\cdot (1+v^{\mathrm{tot}})}}{e^{\rho_0}\cdot\sum_{i\in S^{\star}_{\bv,+}}v_i + e^{\rho_0}\cdot\sum_{i\in S^{\star}_{\bv,+}}(v_i'-v_i) - (e^{\rho_0} - 1)\cdot (1+v_{\mathrm{tot}})} \\
        &\qquad = \frac{\sum_{i\in S_{\bv,+}^{\star}}v_i\exp(-r_i/\lambda^{\star}_{\bv})}{e^{\rho_0}\cdot\sum_{i\in S_{\bv,+}^{\star}}v_i - (e^{\rho_0}-1)\cdot (1+v^{\mathrm{tot}})},
    \end{align}
        where the inequality follows from \eqref{eq:_proof_monotonicity_2_3} and the last equality is from the equation that $(a + c\cdot (a/b))/(b+c) = a/b$.
    This proves \eqref{eq:_proof_monotonicity_2_2} and thus completes the proof of Lemma~\ref{lem:_monotonicity_2}.
\end{proof}

\begin{lemma}[Optimal condition for items in optimal robust assortment (Example~\ref{exp: new})]\label{lem: optimal condition new}
    Consider Example~\ref{exp: new} with fixed parameter $\bv$. If an item $j\in[N]$ belongs to the robust assortment $S^{\star}$ in the following sense,
    \begin{align}
        j\in S^{\star},\quad (S^{\star},\lambda^{\star})\in \argmax_{S\subseteq[N], |S|\leq K, \lambda\geq 0} \left\{\widetilde{H}(S,\lambda;\bv)\right\},
    \end{align}
    then the following inequality holds,
    \begin{align}
        r_j\geq  -\lambda^{\star}\cdot  \log\left(\frac{e^{\rho_0}\cdot\sum_{i\in S_{+}^{\star}}v_i\exp(-r_i/\lambda^{\star})}{e^{\rho_0}\cdot\sum_{i\in S_{+}^{\star}}v_i - (e^{\rho_0}-1)\cdot (1+v^{\mathrm{tot}})}\right).
    \end{align}
\end{lemma}

\begin{proof}[Proof of Example~\ref{lem: optimal condition new}]
    The proof follows from a similar strategy as the proof of Lemma~\ref{lem: optimal condition jin}.
     If there exists an item $j$ such that $j\in S^{\star}$ for some $(S^{\star}, \lambda^{\star})$ satisfying that
    \begin{align}
        (S^{\star}, \lambda^{\star})\in \argmax_{S\subseteq[N], |S|\leq K, \lambda\geq 0} \left\{\widetilde{H}(S,\lambda;\bv)\right\},\label{eq: proof lem optimal condition new 1}
    \end{align}
    but also satisfying that
    \begin{align}
        r_j < -\lambda^{\star}\cdot  \log\left(\frac{e^{\rho_0}\cdot\sum_{i\in S_{+}^{\star}}v_i\exp(-r_i/\lambda^{\star})}{e^{\rho_0}\cdot\sum_{i\in S_{+}^{\star}}v_i - (e^{\rho_0}-1)\cdot (1+v^{\mathrm{tot}})}\right).\label{eq: proof lem optimal condition new 1+}
    \end{align}
    Then considering the new assortment $S^{\star}_{-j}:=S^{\star}\setminus\{j\}$ which satisfies the cardinality constraint $|S^{\star}_{-j}|\leq K$. 
    The value of $\widetilde{H}(S^{\star}_{-j},\lambda^{\star};\bv)$ would be 
    \begin{align}
        \!\!\!\!\!\!\!\!\!\!\!\!\!\!\!\widetilde{H}(S^{\star}_{-j},\lambda^{\star};\bv)=-\lambda^{\star}\cdot\log\left(\frac{\sum_{i\in S^{\star}_{-j,+}}v_i\cdot\exp(-r_i/\lambda)}{\sum_{i\in S^{\star}_{-j,+}}v_i}\right) - \lambda^{\star} \rho_0 +\lambda^{\star}\cdot  \log\left(e^{\rho_0} - \frac{(e^{\rho_0} - 1)(1+v_{\mathrm{tot}})}{\sum_{i\in S_{-j,+}^{\star}}v_i}\right).\label{eq: proof lem optimal condition new 2}
    \end{align}
    In the sequel, we are going to show that the above $\widetilde{H}(S^{\star}_{-j},\lambda^{\star};\bv)$ is larger than 
    \begin{align}
        \!\!\!\!\!\!\!\!\!\widetilde{H}(S^{\star},\lambda^{\star};\bv)=-\lambda^{\star}\cdot\log\left(\frac{\sum_{i\in S^{\star}_{+}}v_i\cdot\exp(-r_i/\lambda)}{\sum_{i\in S^{\star}_{+}}v_i}\right) - \lambda^{\star} \rho_0 +\lambda^{\star}\cdot  \log\left(e^{\rho_0} - \frac{(e^{\rho_0} - 1) (1+v_{\mathrm{tot}})}{\sum_{i\in S_{+}^{\star}}v_i}\right),\label{eq: proof lem optimal condition new 3}
    \end{align}
    causing a contradiction.
    To this end, in view of \eqref{eq: proof lem optimal condition new 2} and \eqref{eq: proof lem optimal condition new 3}, it suffices to prove that 
    \begin{align}
        \frac{\sum_{i\in S^{\star}_{\bv,+}\setminus\{j\}}v_i\cdot\exp(-r_i/\lambda^{\star}_{\bv})}{e^{\rho_0}\cdot\sum_{i\in S^{\star}_{\bv,+}\setminus\{j\}}v_i - (e^{\rho_0} - 1)\cdot (1+v_{\mathrm{tot}})}
        \leq \frac{\sum_{i\in S^{\star}_{\bv,+}}v_i\cdot\exp(-r_i/\lambda^{\star}_{\bv})}{e^{\rho_0}\cdot\sum_{i\in S^{\star}_{\bv,+}}v_i - (e^{\rho_0} - 1)\cdot (1+v_{\mathrm{tot}})}.
    \end{align}
    By calculations, this is equivalent to the Condition~\eqref{eq: proof lem optimal condition new 1+}.
    Therefore, \eqref{eq: proof lem optimal condition new 1+} can not hold for any $j\in S^{\star}$ satisfying \eqref{eq: proof lem optimal condition new 1}. This completes the proof of Lemma~\ref{lem: optimal condition new}.
\end{proof}

\begin{lemma}[Concentration 1 \citep{saha2024stop}]\label{lem:_concentration}
    With probability at least $1-3Ne^{-\delta}$, it holds simultaneously for all $i\in[N]$ that $v^{\mathrm{LCB}}_i\leq v_i$ and that one of the following two inequalities is satisfied,
    \begin{align}
        \tau_{i,0} < 69\delta\cdot(1+v_i),
    \end{align}
    or 
    \begin{align}
        v_i - v_i^{\mathrm{LCB}}\leq  4(1+v_i)\cdot \sqrt{\frac{2v_i\delta}{\tau_{i,0}}} + \frac{22\delta\cdot(1+v_i)^2}{\tau_{i,0}}.
    \end{align}
\end{lemma}

\begin{proof}[Proof of Lemma~\ref{lem:_concentration}]
    See Lemma 1 in \cite{saha2024stop} for a detailed proof.
\end{proof}

\begin{lemma}[Concentration 2]\label{lem:_concentration_2}
    Conditioning on the assortment set $\{S_k\}_{k=1}^n$ and the event that 
    \begin{align}
        n_i:= \sum_{k=1}^n\mathbf{1}\{i \in S_k\}\geq 138\cdot\max\left\{1,\frac{256}{138v_i}\right\}\cdot (1+v_{\max})(1+Kv_{\max})\log(3N/\delta),\quad \forall i\in S^{\star},\label{eq:_condition_n_i}
    \end{align}
    with probability at least $1-2\delta$, it holds that for any item $i\in[N]$,
    \begin{align}
        v_i - v_i^{\mathrm{LCB}} \leq 8\cdot \sqrt{\frac{v_i(1+v_i)\log(3N/\delta)}{\sum_{k=1}^n\mathbf{1}\{i_k\in S_k\}\cdot (1+\sum_{j\in S_k} v_j)^{-1}}} + \frac{44(1+v_i)\log(3N/\delta)}{\sum_{k=1}^n\mathbf{1}\{i_k\in S_k\}\cdot (1+\sum_{j\in S_k} v_j)^{-1}}.
    \end{align}
    As a corollary, under also the Condition~\eqref{eq:_condition_n_i}, we have that for any item $i\in[N]$,
    \begin{align}
        v_i\leq 2v_i^{\mathrm{LCB}}.
    \end{align}
\end{lemma}

\begin{proof}[Proof of Lemma~\ref{lem:_concentration_2}]
    The proof of Lemma~\ref{lem:_concentration_2} is based upon Lemma~\ref{lem:_concentration}. 
    By definition, we have that 
    \begin{align}
        \tau_{i,0} = \sum_{k=1}^{n}\mathbf{1}\big\{i_k \in\{0,i\}\big\}\cdot \mathbf{1}\big\{i\in S_k\big\}.
    \end{align}
    Notice that conditioning on the assortments $\{S_k\}_{k=1}^n$, for any sample $k$ such that $i\in S_k$, it holds that 
    \begin{align}
        \mathbf{1}\big\{i_k \in\{0,i\}\big\} \sim \mathrm{Bernoulli}\left(\frac{1+v_i}{1+\sum_{j\in S_k}v_j}\right).
    \end{align}
    Therefore, by applying Chernoff's bound, we can derive that
    \begin{align}
        \PP\left(\tau_{i,0} \leq (1-\mu)\cdot \sum_{k=1}^n\mathbf{1}\big\{i\in S_k\big\}\cdot\frac{1+v_i}{1+\sum_{j\in S_k}v_j}\right) \leq \exp\left(-\frac{\mu^2}{2}\cdot \sum_{k=1}^n\mathbf{1}\big\{i\in S_k\big\}\cdot\frac{1+v_i}{1+\sum_{j\in S_k}v_j}\right).
    \end{align}
    Choosing $\mu=1/2$, applying the Condition~\eqref{eq:_condition_n_i}, we can derive that,
    \begin{align}
        \PP\Big(\tau_{i,0}\leq 69(1+v_i)\log(3N/\delta)\Big) \leq  \PP\left(\tau_{i,0} \leq (1-\mu)\cdot \sum_{k=1}^n\mathbf{1}\big\{i\in S_k\big\}\cdot\frac{1+v_i}{1+\sum_{j\in S_k}v_j}\right)  \leq \frac{\delta}{N}.
    \end{align}
    Therefore, in view of Lemma~\ref{lem:_concentration}, by applying a union bound argument, with probability at least $1-2\delta$, 
    \begin{align}
        v_i - v_i^{\mathrm{LCB}} \leq 8\cdot \sqrt{\frac{v_i(1+v_i)\log(3N/\delta)}{\sum_{k=1}^n\mathbf{1}\{i_k\in S_k\}\cdot (1+\sum_{j\in S_k} v_j)^{-1}}} + \frac{44(1+v_i)\log(3N/\delta)}{\sum_{k=1}^n\mathbf{1}\{i_k\in S_k\}\cdot (1+\sum_{j\in S_k} v_j)^{-1}}. 
    \end{align}
    Finally, to prove the corollary, we just need to verify that under Condition~\eqref{eq:_condition_n_i},
    \begin{align}
        v_i - v_i^{\mathrm{LCB}} \leq 8\cdot \sqrt{\frac{v_i(1+v_i)\log(3N/\delta)}{\sum_{k=1}^n\mathbf{1}\{i_k\in S_k\}\cdot (1+\sum_{j\in S_k} v_j)^{-1}}} + \frac{44(1+v_i)\log(3N/\delta)}{\sum_{k=1}^n\mathbf{1}\{i_k\in S_k\}\cdot (1+\sum_{j\in S_k} v_j)^{-1}} \leq \frac{v_i}{2}.
    \end{align}
    This completes the proof of Lemma~\ref{lem:_concentration_2}.
\end{proof} 

\begin{lemma}\label{lem:_function_upper_bound}
    For any $\lambda\in[c\cdot r_{\max}, r_{\max}/\rho]$, it holds that 
    \begin{align}
        \lambda\cdot\exp(r_{\max}/\lambda)\cdot \big(1 - \exp(-r_{\max}/\lambda) \big)\leq c\cdot r_{\max} \cdot\big(\exp(1/c)-1\big).
    \end{align}
\end{lemma}

\begin{proof}[Proof of Lemma~\ref{lem:_function_upper_bound}]
    We define the following function, $$\psi(\lambda) := \lambda\cdot\exp(r_{\max}/\lambda)\cdot (1 - \exp(-r_{\max}/\lambda) ).$$
    Then we can verify that the derivative of $\psi(\lambda)$ is non-positive, i.e., 
    \begin{align}
        \psi'(\lambda) = \exp(r_{\max} /\lambda)\cdot(1-r_{\max}/\lambda)-1\leq 0,\quad \forall \lambda> 0.
    \end{align}
    Therefore, we can derive the following result,
    \begin{align}
        \sup_{\lambda\in[c\cdot r_{\max},r_{\max}/\rho]}\psi(\lambda) = \psi(c\cdot r_{\max}) = c\cdot r_{\max}\cdot \big(\exp(1/c)-1\big).
    \end{align}
    This completes the proof of Lemma~\ref{lem:_function_upper_bound}.
\end{proof}

\begin{lemma}[Fano's lemma]\label{lem-fano} 
Let $ \Gamma:=  \{\theta^1,\dots, \theta^M\}\subset \Theta$ be a $2\delta$-separated set under some metric $\rho$ over $\Theta,$ each associated with a distribution $\mathbb{P}_{\theta}$ over some set $\mathcal{X}$. Then it holds that \begin{align*}
    \inf_{\pi:\cX\mapsto\Theta} \sup_{k\in [M]}\E_{D\sim P_{\theta_k}}[\rho(\pi(D),\theta_k) ] \geq \delta \cdot\left(1 - \frac{\frac{1}{M^2}\sum_{i,j = 1}^M D_{\mathrm{KL}}(\mathbb{P}_{\theta_i}\lVert \mathbb{P}_{\theta_j})  +\log 2}{\log M}\right).
\end{align*}
\end{lemma}

\begin{proof}[Proof of Lemma~\ref{lem-fano}]
    Please refer to Proposition 15.12 in \cite{wainwright2019high}.
\end{proof}

\end{document}